\documentclass[10pt]{article}
\usepackage{times}
\usepackage[
    letterpaper,
    left=1in,
    right=1in,
    top=1in,
    bottom=1in
]{geometry} 
\usepackage{hyperref}
\usepackage{amsmath,amsfonts,amssymb,bbm,bm}
\usepackage{dsfont} %
\usepackage{xcolor}
\usepackage{amsthm}
\usepackage{thmtools}
\usepackage{mathtools}

\definecolor{niceRed}{RGB}{190,38,38}
\definecolor{niceYellow}{HTML}{f5b400}
\definecolor{blueGrotto}{HTML}{059DC0}
\definecolor{royalBlue}{HTML}{057DCD}
\definecolor{navyBlue}{HTML}{0B579C}
\definecolor{limeGreen}{HTML}{81B622}
\definecolor{nicePurple}{HTML}{9c27b0}
\definecolor{lightRoyalBlue}{HTML}{def2ff}
\definecolor{gold}{HTML}{ffa300}
\definecolor{frameblue}{RGB}{0,123,255}     %
\definecolor{framebg}{RGB}{241,244,247}     %
\definecolor{shadecolor}{gray}{0.90}
\definecolor{plum}{HTML}{9c27b0}

\usepackage{enumitem}
\usepackage[breakable,skins]{tcolorbox}
\usepackage{amsthm}
\usepackage{thmtools}
\definecolor{shadecolor}{gray}{0.95}

\declaretheoremstyle[
headfont=\normalfont\bfseries,
notefont=\mdseries, notebraces={(}{)},
bodyfont=\normalfont,
postheadspace=0.5em,
spaceabove=0.5em,
spacebelow=0.5em,
]{shaded}

\declaretheoremstyle[
  bodyfont=\normalfont
]{nonitalic}
\newtcolorbox{myframe}{
    colback=framebg,       %
    colframe=framebg,      %
    boxrule=0pt,           %
    frame hidden,          %
    arc=5pt,               %
    top=4pt,               %
    bottom=4pt,            %
    left=4pt,              %
    right=4pt,             %
    before skip=8pt,       %
    after skip=8pt         %
}

\declaretheorem[within=section]{style}
\declaretheorem[style=nonitalic,sibling=style]{definition}
\declaretheorem[sibling=style]{main theorem}
\declaretheorem[style=shaded,sibling=style]{theorem}
\declaretheorem[style=nonitalic,sibling=style]{assumption}

\declaretheorem[style=shaded,sibling=style]{corollary}
\declaretheorem[style=nonitalic,sibling=style]{lemma}

\declaretheorem[sibling=style]{proposition}

\declaretheorem[style=nonitalic,sibling=style,]{remark}
\declaretheorem[style=nonitalic,sibling=style, numbered=no]{proof sketch}

\makeatletter
\@ifpackageloaded{lineno}%
  {\def\thmboxbreak{}}%
  {\def\thmboxbreak{breakable}}%
\makeatother
\tcolorboxenvironment{theorem}{
  colback=shadecolor, colframe=shadecolor, \thmboxbreak, enhanced,
  boxrule=0pt, left=0pt, right=0pt, top=0pt, bottom=0pt,
  before skip=6pt, after skip=6pt
}
\tcolorboxenvironment{corollary}{
  colback=shadecolor, colframe=shadecolor, \thmboxbreak, enhanced,
  boxrule=0pt, left=0pt, right=0pt, top=0pt, bottom=0pt,
  before skip=6pt, after skip=6pt
}

\usepackage[capitalize,nameinlink]{cleveref}
\crefname{assumption}{assumption}{assumptions}
\Crefname{assumption}{Assumption}{Assumptions}
\crefname{lemma}{lemma}{lemmas}
\Crefname{lemma}{Lemma}{Lemmas}
\crefname{theorem}{theorem}{theorems}
\Crefname{theorem}{Theorem}{Theorems}
\crefname{main theorem}{main theorem}{main theorems}
\Crefname{main theorem}{Main theorem}{Main theorems}
\crefname{definition}{definition}{definitions}
\Crefname{definition}{Definition}{Definitions}
\crefname{corollary}{corollary}{corollaries}
\Crefname{corollary}{Corollary}{Corollaries}
\crefname{proposition}{proposition}{propositions}
\Crefname{proposition}{Proposition}{Propositions}
\crefname{fact}{fact}{facts}
\Crefname{fact}{Fact}{Facts}
\crefname{identity}{identity}{identities}
\Crefname{identity}{Identity}{Identities}
\crefname{function}{function}{functions}
\Crefname{function}{Function}{Functions}

\renewcommand{\Pr}{\mathbb{P}}
\newcommand{\E}{\mathbb{E}}

\newcommand{\calA}{\mathcal{A}}
\newcommand{\calB}{\mathcal{B}}

\newcommand{\calE}{\mathcal{E}}
\newcommand{\calF}{\mathcal{F}}
\newcommand{\calG}{\mathcal{G}}

\newcommand{\calI}{\mathcal{I}}

\newcommand{\calL}{\mathcal{L}}

\newcommand{\calT}{\mathcal{T}}

\newcommand{\calX}{\mathcal{X}}

\newcommand{\calZ}{\mathcal{Z}}

\def\<{\langle}
\def\>{\rangle}

\newcommand{\bx}{\mathbf{x}} %
\newcommand{\by}{\mathbf{y}} %

\newcommand{\mathe}{\mathrm{e}}

\newcommand{\1}{\mathds{1}} %

\usepackage[numbers,sort&compress]{natbib}
\usepackage[utf8]{inputenc} 
\usepackage[T1]{fontenc}    
\usepackage{hyperref}       
\usepackage{url}            
\usepackage{booktabs}       
\usepackage{nicefrac}       
\usepackage{microtype}      
\usepackage{xcolor}         
\usepackage{graphicx}
\usepackage{multicol}
\usepackage{multirow}
\usepackage[english]{babel}
\usepackage{epstopdf}
\usepackage{caption}
\usepackage{makecell}   
\usepackage{colortbl}   
\usepackage{hhline}
\usepackage{longtable}
\usepackage{array}
\hypersetup{bookmarksdepth=2} 
\hypersetup{colorlinks=true,linkcolor=royalBlue,citecolor=navyBlue,urlcolor=navyBlue}
\usepackage[nohints]{minitoc}
\newcommand{\algofam}[1]{\noindent\textbf{#1}}
\newcommand{\algoname}[1]{\textcolor{black!70}{\textbf{#1}}}
\usepackage{titlesec}
\titleformat{\section}{\large\bfseries}{\thesection}{1em}{}
\titleformat{\subsection}{\normalsize\bfseries}{\thesubsection}{0.5em}{}
\titleformat{\subsubsection}{\normalsize\bfseries}{\thesubsubsection}{0.2em}{}
\titlespacing*{\section}      {0pt}{3.0ex plus 0.25ex minus 0.15ex}{1.25ex plus 0.12ex}
\titlespacing*{\subsection}   {0pt}{1.5ex plus 0.2ex minus 0.12ex}{0.92ex plus 0.1ex}
\titlespacing*{\subsubsection}{0pt}{0.55ex plus 0.15ex minus 0.1ex}{0.22ex plus 0.08ex}
\titlespacing*{\paragraph}    {0pt}{0.45ex plus 0.12ex minus 0.08ex}{0.45em}

\usepackage{algorithm}
\usepackage{algpseudocode}
\usepackage{enumitem}
\setlist{itemsep=2pt, topsep=4pt, parsep=0pt, partopsep=0pt, leftmargin=*}
\usepackage[nohints]{minitoc}
\doparttoc[n]

\newcommand{\blfootnote}[1]{%
  \begingroup
  \renewcommand{\thefootnote}{}%
  \footnote{#1}%
  \addtocounter{footnote}{-1}%
  \endgroup
}

\begin{document}
\sloppy 
\faketableofcontents

\begingroup
\setlength{\parskip}{0pt}
\begin{center}
    \rule{\textwidth}{3pt}\par
    \vspace{0.3cm}                                  
    {\LARGE \bfseries Unified High-Probability Analysis of Stochastic Variance-Reduced Estimation\par}
    \vspace{0.2cm}                                  
    \rule{\textwidth}{1.2pt}\par
    \vspace{1cm}                                  

    {\normalsize
      \textbf{Zhankun~Luo}\(^{\ddagger*}\) \quad
      \textbf{Antesh~Upadhyay}\(^{\ddagger*}\) \quad
      \textbf{M.~Berk~Sahin}\(^{*}\) \quad
      \textbf{Sang~Bin~Moon}\(^{*}\) \\[0.4em]
      \textbf{Anuran~Makur}\(^{\star*}\) \quad
      \textbf{Abolfazl~Hashemi}\(^{*}\) \\[0.6em]
      \(^{*}\)School of Electrical and Computer Engineering and\\
      \(^{\star}\)Department of Computer Science\\
      Purdue University\\
      West Lafayette, IN 47907\\
      \texttt{\{luo333, aantesh, sahinm, moon182, amakur, abolfazl\}@purdue.edu}
    }
    \vspace{0.3cm} 
\end{center}
\endgroup
\blfootnote{\(^\ddagger\)The authors marked with \(^\ddagger\) contributed equally.}

\doparttoc[n] 
\faketableofcontents 

\vspace{-2em}
\begin{abstract}
Stochastic estimators are fundamental to large-scale optimization, where population quantities must be inferred from noisy oracle observations. Although influential methods such as momentum, SPIDER, STORM, and PAGE have been highly successful, their analyses are largely estimator-specific and expectation-based, obscuring the structural tradeoffs that determine reliability. In this paper, we develop a unified framework for stochastic variance-reduced estimation based on a recursion with three components: memory retention, reset probability, and a correction term for iterate movement. This framework recovers several classical estimators, motivates new second-order variants, and yields a bias-variance decomposition of estimation error. Our main result is a unified high-probability bound proved using a new dimension-free vector-valued Freedman inequality, valid for smooth normed spaces involving random sums of vector martingales. The result applies in both Euclidean and non-Euclidean settings, including the analysis of mirror-descent-based methods in Banach spaces. As applications, we obtain high-probability oracle complexities for unconstrained optimization with mirror descent, establishing the logarithmic dependence on the confidence level. We also derive the first \(\tilde{\mathcal{O}}(\varepsilon^{-3})\) oracle-complexity bounds for stochastic optimization with expectation constraints, improving upon the existing \(\tilde{\mathcal{O}}(\varepsilon^{-4})\) complexity by leveraging variance-reduced estimation for the first time in this setting.

\end{abstract}

\section{Introduction}\label{sec:Intro}
\vspace{-2mm}
Stochastic optimization algorithms optimize a function over iterations by using noisy oracle information rather than exact population quantities. In the most familiar setting~\cite{robbins1951stochastic,polyak1967general,nemirovsky1983wiley,ghadimi2013stochastic}, the quantity of interest is the population gradient, and standard first-order algorithms must repeatedly construct accurate gradient estimates from stochastic samples to make progress. From this perspective, the effectiveness of a first-order method depends not only on how it updates the iterate, but also on how well it maintains an estimator that is accurate, stable, and inexpensive to compute as the iterate evolves. Many widely used methods can be viewed in this way. Classical momentum stabilizes SGD by averaging past gradients~\cite{liu2020improved}, whereas variance-reduced methods, such as SVRG, STORM, and PAGE, reuse information across iterates through stochastic differences and occasional reset steps~\cite{johnson2013accelerating,fang2018spider,cutkosky2019momentum,li2021page,tran2022better}.

Although these methods are often presented as distinct algorithmic ideas, they address a common underlying problem: how to recursively estimate an evolving population quantity from noisy samples while controlling both bias and variance. Despite the empirical successes of these methods, however, the theoretical picture remains fragmented. Existing analyses are typically developed separately for individual estimators, and their proofs are often tightly tied to the algebra of a particular recursion. As a result, it is difficult to isolate which design choices are actually responsible for performance: the reuse of past information, the frequency of resets, or the way past information is adjusted as the iterate moves. In addition, most existing guarantees are expectation-based. While such guarantees characterize average behavior, they can conceal rare but important estimation failures, which are especially relevant in large-scale learning and in constrained optimization.

In this paper, we develop a \emph{unified framework for high-probability analysis of stochastic variance-reduced estimation in optimization}. We consider a general population quantity \(g(w_t)\) observed through an unbiased stochastic oracle \(G(w_t,\xi_t)\), where the iterate \(w_t\) changes over time \(t\), and \(\xi_t\) denotes the oracle randomness. We introduce an estimator for \(g(w_t)\) governed by \textit{three simple design choices:} 1) how much past information must be retained, 2) how often resets using fresh estimates occur, and 3) how past information computed at \(w_{t-1}\) must be corrected so that it remains informative at the current iterate \(w_t\). This viewpoint places several seemingly different recursive estimators within a single template and makes it possible to compare them through a common error decomposition. Beyond gradient estimation, the framework also applies to other optimization-relevant quantities.

\noindent
\hspace{-1.8em}
\begin{minipage}[c]{0.42\linewidth}
\abovedisplayskip=0pt
\belowdisplayskip=0pt
\begin{equation}
\hspace{-9mm}
\min_{w \in \mathcal{W}} f(w) \coloneqq \mathbb{E}_{\xi}[F(w,\xi)]\vphantom{\bigg|}
\tag{\(\dagger\)}\label{eq:application1}
\end{equation}
\end{minipage}%
\hfill
\begin{minipage}[c]{0.55\linewidth}
\abovedisplayskip=0pt
\belowdisplayskip=0pt
\begin{equation}
\hspace{-3mm}
\min_{w \in \mathcal{W}} \mathbb{E}_{\xi}[F(w,\xi)]\;
\text{ s.t. }\;
h(w)\coloneqq \mathbb{E}_{\xi}[H(w,\xi)] \le 0.\vphantom{\bigg|}
\tag{\(\ddagger\)}\label{eq:application2}
\end{equation}
\end{minipage}

The two representative examples above illustrate the scope of our framework. In problem~\eqref{eq:application1}, the relevant quantity is the population gradient \(g(w)=\nabla f(w)\), and the goal is to maintain an accurate stochastic estimate of the gradient along the optimization trajectory. In problem \eqref{eq:application2}, however, the quantity of interest is the constraint value \(h(w)\), which must be estimated from noisy samples to assess feasibility~\cite{lan2020algorithms}. In this setting, poor estimation can lead either to unreliable feasibility detection or to overly conservative large-batch procedures. This paper, therefore, studies a broader estimation problem in stochastic optimization, with stochastic gradient estimation and expectation-constraint estimation as two central instances.

Another limitation of the existing literature is that high-probability guarantees are predominantly restricted to Euclidean geometries. In more general first-order methods, such as mirror descent~\cite{beck2003mirror,bauschke2003bregman,burachik2000proximal}, the estimator naturally lives in the dual space rather than the primal space. The literature currently lacks a unified, high-probability concentration framework that accommodates arbitrary smooth Banach spaces without relying on estimator-specific algebra. Our analysis in this paper is developed at a level of generality that covers both standard Euclidean updates and generalized first-order methods such as mirror descent, while focusing on high-probability guarantees.
Please see Appendix \ref{apdx:related} for a detailed discussion of related work.

\noindent\textbf{Contributions. }Our main contributions are as follows:
\begin{itemize}[leftmargin=*]
    \item \textbf{Unified estimator framework.} We propose a single recursive stochastic estimator~\eqref{eq:unified_est_general} that subsumes classical momentum, STORM, PAGE, and SPIDER as special cases, while also suggesting new second-order variants; see Table~\ref{tab:estimator_configs}.

    \item \textbf{Unified high-probability analysis.} We prove a general high-probability bound for the estimation error of every estimator in the framework, together with an explicit bias-variance decomposition; see Theorem~\ref{thm:unified}. The proof relies on a new vector-valued Freedman inequality; see Theorem~\ref{thm:freedman}.

    \item \textbf{Unconstrained optimization.} We establish high-probability guarantees for finding first-order \(\varepsilon\)-stationary points in stochastic optimization with mirror descent, and we derive the resulting oracle complexities for different estimator choices; see Table~\ref{tab:complexity_summary} and Theorem~\ref{thm:mirror_stationarity} in Appendix~\ref{apdx:application1}

    \item \textbf{Constrained optimization.} We provide the first variance-reduced oracle complexity guarantees for expectation-constrained stochastic optimization, improving the rate from \(\tilde{\mathcal{O}}(\varepsilon^{-4})\) in \cite{lan2020algorithms} to \(\tilde{\mathcal{O}}(\varepsilon^{-3})\); see Table~\ref{tab:unified_complexity}.
\end{itemize}

\section{Preliminaries}\label{sec:prelim}
\vspace{-2mm}
We briefly introduce the notations and setup used throughout the paper. Let \((\mathcal{X},\|\cdot\|)\) be a reflexive Banach space, and let \((\mathcal{X}^*,\|\cdot\|_*)\) denote its dual space of all continuous bounded linear functionals. For \(u\in \mathcal{X}^*\) and \(x\in \mathcal{X}\), we write \(\langle u,x\rangle\) for the canonical duality pairing between \(\mathcal{X}^*\) and \(\mathcal{X}\) such that \(\langle \cdot, \cdot \rangle:\mathcal{X}^*\times \mathcal{X} \to \mathbb{R}\). The dual norm is defined by \(\|u\|_* \coloneqq \sup_{\|x\|\le 1, x\in \mathcal{X}} \langle u,x\rangle\) for any \(u\in \mathcal{X}^*\). The optimization iterates lie in a nonempty closed, convex set \(\mathcal W \subseteq \mathcal{X}\). 
\begin{definition}[Mirror Map]\label{def:mirror_map}
The mirror map \(\Phi: \mathcal{D} \to \mathbb{R} \cup \{+\infty\}\) with \(\mathcal{D} \subseteq \mathcal{X}\) is proper, weakly lower semi-continuous (w.l.s.c), and \(1\)-strongly convex on a nonempty closed convex set \(\mathcal{W}\subseteq \text{int}(\mathcal{D})\) with respect to \(\|\cdot\|\). Furthermore, \(\Phi\) is G\^{a}teaux differentiable on \(\mathcal{W}\).
\end{definition}

\begin{definition}[Proximal Gradient Mapping]\label{def:proximal_gradient_mapping}
    For \(w \in \mathcal{W}\), an input vector \(u\in\mathcal{X}^*\), and a step size \(\eta > 0\), the \textbf{proximal gradient mapping} is 
    \(P(w, u, \eta) = (w - w^+)/\eta\), 
    where the updated point \(w^+\) is the unique minimizer of the proximal subproblem:
    \(w^+= \arg\min_{w'\in\mathcal{W}} \langle \eta u, w'\rangle + D_\Phi(w', w)\)
    and \(D_\Phi(w', w)\) denotes the Bregman divergence induced by \(\Phi\).
\end{definition}

\vspace{-1mm}
The existence and uniqueness are guaranteed by variational principles. The strong convexity of \(\Phi\) renders the objective strictly convex and coercive, while its closed convexity ensures it is w.l.s.c. \cite[Cor.~2.2, p.~11]{ekeland1999convex}. By the  Mazur's Lemma and Eberlein-\v{S}mulian theorem \cite[Thm.~3.7, 3.18]{brezis2011functional}, coercivity within the reflexive space \(\mathcal{X}\) forces bounded sequences to converge weakly to a feasible limit inside the closed set \(\mathcal{W}\). Consequently, the objective reliably attains a unique global minimum \cite[Prop.~1.2, p.~35]{ekeland1999convex}, \cite[Thm.~2.11, p.~72]{barbu2012convexity}.
Based on variational inequality as supported by \cite{ekeland1999convex}, we derive properties of \(P(w,u,\eta)\) (see Lemma~\ref{lemma:properties_proximal_gradient_mapping} and \cite{ghadimi2016mini,bauschke2003bregman}), including \(\|P(w, u, \eta)\| \leq \|u\|_*\).
For a more detailed discussion on the existence, uniqueness and properties of the proximal gradient mapping \(P(w, u, \eta)\), please refer to Remark~\ref{remark:existence_uniqueness_proximal_mapping} and Lemma~\ref{lemma:properties_proximal_gradient_mapping} in Appendix~\ref{apdx:prelim}.

We consider a stochastic oracle \(G:\mathcal W \times \Omega \to \calG\) and its population counterpart \(g(w) \coloneqq \mathbb E[G(w,\xi)],\) where \(\xi\) denotes the oracle randomness, \(\calG\) is a \(\kappa\)-smooth (\(\kappa\geq 1\)) normed space with its norm \(\|\cdot\|_\calG\). %
For high-probability analysis, we introduce \(C(\delta, \kappa) \coloneqq \sqrt{\kappa} + \sqrt{3\ln(1/\delta)}\) for \(\delta \in (0,1)\), arising from our proposed vector-valued Freedman inequality (see Theorem~\ref{thm:freedman}).
Our goal is to estimate the time-varying target \(g(w_t)\) along a trajectory \(\{w_t\}_{t\ge 0}\). Given a recursive estimator \(v_t\), we define the estimation error \(e_t \coloneqq v_t - g(w_t).\) Our analysis is not restricted to Euclidean SGD and applies more broadly to generalized first-order methods, including mirror descent~\cite{beck2003mirror,bauschke2003bregman,burachik2000proximal}. 

\textbf{Mirror Descent Update.} At each iteration \(t\), we form a unified estimator \(v_t\) and use its corresponding update vector \(U_t\), to perform a mirror descent update with step size \(\eta_t>0\):
\begin{equation}\label{eq:optimize_mirror_descent}
    w_{t+1} = \arg\min_{w\in\mathcal{W}} \left\{ \eta_t \langle U_t, w\rangle + D_\Phi(w,w_t) \right\} = w_t - \eta_t P(w_t, U_t, \eta_t).
\end{equation}

\vspace{-2mm}
\begin{assumption}[Boundedness of Generalized Updates]\label{ass:update_bound}
Let \(U_t\) be the update vector computed at iteration \(t\). We assume the algorithm employs a normalization or clipping mechanism such that the update vector is bounded almost surely, satisfying \(\|U_t\|_\ast \le G\) for all \(t \geq 0\).
\end{assumption}
We note that this boundedness assumption, while necessary, is not as strong since in many problems it can be ensured by applying normalization or clipping operations (e.g., Application 1 in Subsection \ref{sec:app1}), or holds due to certain Lipschitzness conditions (e.g., Application 2 in Subsection \ref{sec:app2}). Under the above assumption, we show that \(\|w_{t+1}-w_t\|=\eta_t\|P(w_t, U_t, \eta_t)\|\leq \eta_t \|U_t\|_*\leq \eta_t G\).

\section{A Unified Estimator}\label{sec:estimators}
In stochastic optimization, estimating an evolving population quantity \(g(w_t)\) from noisy observations \(G(w_t, \xi_t)\) requires combining fresh information with historical memory. A classical example is momentum~\cite{liu2020improved} that filters noise via \(v_t = (1-\beta) G(w_t,\xi_t) + \beta v_{t-1}\). If we define the estimation error by \(e_t \coloneqq v_t - g(w_t)\), then the corresponding error recursion by \(e_t = \beta e_{t-1} + (1-\beta)(G(w_t,\xi_t) - g(w_t)) - \beta (g(w_t) - g(w_{t-1}))\). The final term captures the deterministic lag induced by reusing information from the previous iterate.
Variance-reduced estimators, such as STORM and SPIDER, avoid this deterministic bias by incorporating stochastic differences. This suggests a common design principle for recursive estimators: once past information is carried over from iteration \(t-1\), it should be corrected before being reused at the new iterate \(w_t\). We make this principle explicit through an explicit \emph{correction term } \(\mathcal{T}_t\), which leads to a unified recursion governed by three design choices: memory retention \(\beta_t\), the reset probability \(p_t\), and the correction mechanism \(\mathcal{T}_t\). 

\noindent \textbf{The Unified Recursion.  }
Let \(\{\mathcal{F}_t\}_{t\ge0}\) be the filtration generated by oracle noise and reset events (see Appendix~\ref{apdx:prelim}). We assume parameter predictability and sub-Gaussian oracle noise stated below.

\begin{assumption}[Parameter Predictability]\label{ass:predictable_params}
The parameters \(\eta_{t-1}\), \(p_t \in [0, 1]\), and \(\beta_t \in (0, 1]\) are \(\mathcal{F}_{t-1}\)-adapted. The reset indicator \(b_t\sim \mathrm{Bernoulli}(p_t)\) is sampled at time \(t\), ensuring \(\mathbb{E}[b_t|\mathcal{F}_{t-1}] = p_t\).
\end{assumption}

\begin{assumption}[Sub-Gaussianity]\label{ass:subgaussian}
The oracle is unbiased, \(\mathbb{E}[G(w_t, \xi) \mid \mathcal{F}_{t-1} ] = g(w_t)\). Furthermore, for any fresh sample \(\xi\) drawn at time \(t\) (whether a single sample \(\xi_t\) or a component \(\xi_t^{(i)}\) of a batch of \(B\) independent samples \(\xi_t^B\equiv (\xi_t^{(i)})_{i=1}^B\)), the noise is conditionally sub-Gaussian with variance proxy \(\sigma^2>0\), satisfying \(\mathbb{E}[ \exp( \|G(w_t, \xi) - g(w_t)\|_\calG^2/\sigma^2 ) \mid \mathcal{F}_{t-1} ] \le 2\).
\end{assumption}

\vspace{-1mm}
At iteration \(t\), given the reset indicator \(b_t \sim \mathrm{Bernoulli}(p_t)\), we define the unified estimator as:
\begin{equation}
\label{eq:unified_est_general}
v_t \coloneqq
\begin{cases}
G(w_t,\xi_t^B), & b_t = 1,\\
G(w_t,\xi_t) + \beta_t \Delta_t + \mathcal{T}_t(w_t, w_{t-1}, \xi_t), & b_t = 0,
\end{cases}
\tag{\(\ast\)}
\end{equation}
where \(\xi_t^B\) represents a fresh batch of samples, \(\xi_t\) is a single stochastic sample, and \(\Delta_t \coloneqq v_{t-1} - G(w_{t-1},\xi_t)\), \(G(w_t, \xi_t^B) \coloneqq \frac{1}{B}\sum_{i=1}^B G(w_t, \xi_t^{(i)})\) 
is the average of evaluated function values \(G(w_t, \xi_t^{(i)}), i=1,\dots,B\). Viewed through this lens, \(\mathcal{T}_t\) is the mechanism that corrects stale information from the previous iterate. Different choices of \(\mathcal{T}_t\) recover classical estimators and naturally suggest new ones.
\subsection{Error Dynamics and Effective Innovation}
\label{subsec:error_dynamics}
We next derive the common error recursion induced by \eqref{eq:unified_est_general}. When a reset occurs, i.e.,  \(b_t=1\), the estimator is reinitialized using a fresh stochastic estimate. The more interesting case is the recursive steps \(b_t=0\), where past estimation error is propagated and modified by the correction mechanism.

Expanding \(e_t = v_t - g(w_t)\) and substituting \(v_{t-1}=e_{t-1}+g(w_{t-1})\) yields the core error recursion \(e_t = \beta_t e_{t-1} + \mathcal{I}_t(\mathcal{T})\). The estimation error thus consists of a contracted past error and a stochastic source term, which we define as the \emph{Effective Innovation}.
\begin{definition}[Effective Innovation]
\label{def:effective_innovation}
    The effective innovation \(\mathcal{I}_t(\mathcal{T})\) at iteration \(t\) is defined as
    \begin{equation} \label{eq:innovation_def}
        \begin{aligned}
        \mathcal{I}_t(\mathcal{T}) \coloneqq (1-\beta_t) \underbrace{(G(w_t, \xi_t) - g(w_t))}_{\text{Direct Noise }}  
        + \beta_t \underbrace{\Delta(w_t, w_{t-1}, \xi_t)}_{\text{Difference Noise}} + \mathcal{T}_t(w_t, w_{t-1}, \xi_t),
        \end{aligned}
    \end{equation}
    where \(\Delta(w_t, w_{t-1}, \xi_t) \coloneqq (G(w_t, \xi_t) - G(w_{t-1}, \xi_t)) - (g(w_t) - g(w_{t-1}))\) is the centered difference.
\end{definition}

The decomposition in \eqref{eq:innovation_def} makes the bias-variance structure of the unified recursion. The conditional expectation of \(\mathcal{I}_t(\mathcal{T})\) determines the predictable bias, while its variability drives the variance.

\textbf{Specific choices of the correction term \(\mathcal{T}_t\).  } 
By specifying \(\mathcal{T}_t\), we recover existing variance-reduced methods and suggest new higher-order variants, summarized in Table~\ref{tab:estimator_configs}.

\noindent \textbullet\ \textbf{Family 1: Zeroth-Order Correction (Standard Recursion and Resetting).}
    Setting
    \(\mathcal{T}_t \coloneqq \beta_t\bigl(G(w_{t-1}, \xi_t) - G(w_t, \xi_t)\bigr)\) in~\eqref{eq:unified_est_general}
    recovers the momentum-style update
    \(v_t = (1-\beta_t)G(w_t,\xi_t) + \beta_t v_{t-1}.\)
    This attenuates direct oracle noise through averaging, but introduces a deterministic bias induced by iterate movement, since information formed at \(w_{t-1}\) is reused at \(w_t\).

\noindent \textbullet\ \textbf{Family 2: First-Order Correction (Differential Recursion and Resetting).}
    By setting \(\mathcal{T}_t \coloneqq 0\), the recursion in~\eqref{eq:unified_est_general} is driven entirely by stochastic differences. This eliminates the deterministic drift inherent to momentum updates, leaving an estimation error governed solely by centered difference noise that scales with \(\|w_t-w_{t-1}\|\). This formulation recovers STORM (\(\beta_t<1\)), PAGE (\(p_t>0,\beta_t=1\)), and SPIDER (periodic resets).

\noindent \textbullet\ \textbf{Family 3: Second-Order Correction (Second-Order Recursion and Resetting).} Letting \(\mathcal{T}_t\) encode a Taylor correction, \( \mathcal{T}_t \coloneqq \beta_t\bigl(G(w_{t-1},\xi_t)-[G(w_t,\xi_t)+\nabla G(w_t,\xi_t)(w_{t-1}-w_t)]\bigr),\) incorporates local curvature information in~\eqref{eq:unified_est_general}. This reduces the leading-order deterministic bias to a second-order residual, while the remaining stochastic error is governed by the Hessian-vector product.

\begin{table}[t]
\vspace{-3mm} %
\centering
\caption{Special cases of the unified estimator. Here, \(\Delta G:=\beta_t(G(w_{t-1},\xi_t) - G_t(w_t, \xi_t))\), \(2^{\text{nd}}\coloneqq  \beta_t\bigl(G(w_{t-1},\xi_t)-[G(w_t,\xi_t)+\nabla G(w_t,\xi_t)(w_{t-1}-w_t)]\bigr)\), and \(\mathbb{I}_{\mathsf{Sched}}\) refers to a deterministic schedule (e.g., \(t \mod E = 0\)). Bolded entries are \textbf{new} estimators. \textbf{SO} denotes Second-Order.}
\label{tab:estimator_configs}
\setlength{\tabcolsep}{4pt}
\renewcommand{\arraystretch}{1.15}
\begin{tabular}{r@{\;}c@{\quad}c||c|c|c}
\toprule
& \multicolumn{2}{c||}{\textbf{Case}} & 
\textbf{Standard} & \textbf{Differential} & \textbf{Second-Order} \\
& \(\beta_t\) & \(p_t\) & 
\(\mathcal{T}_t = \Delta G\) & 
\(\mathcal{T}_t = 0\) & 
\(\mathcal{T}_t = 2^{\text{nd}}\) \\
\midrule
1. & \(\beta\)                & \(0\) & Momentum             & STORM       & SO Momentum  \\
2. & \(1\)             & \(p\)    & \algoname{Probabilistic Momentum}  & PAGE/ Loopless SVRG        & \algoname{SO PAGE}   \\
3. & \(1\) & \(\mathbb{I}_{\mathsf{Sched}}\) & \algoname{Periodic\ Momentum}   & SPIDER/SVRG & \algoname{SO SPIDER} \\
\bottomrule
\end{tabular}
\vspace{-4mm} %
\end{table}

\section{Unified Estimation Error Bound}\label{sec:main_results}
We now develop a high-probability bound for the unified estimator introduced in Section~\ref{sec:estimators}. The error recursion derived in Section~\ref{subsec:error_dynamics} shows that the estimation error is driven by the accumulation of effective innovations over epochs delimited by reset times. Our goal in this section is to control this accumulation uniformly with high probability. 
We first impose a regularity condition on the effective innovation, separating bias from the centered stochasticity and enabling martingale concentration.

\begin{assumption}[Conditional Regularity of Innovation] \label{ass:innovation_regularity}
    We assume there exist non-negative, predictable processes \(\{\mathcal{B}_t\}_{t \ge 1}\) and \(\{\Sigma_t\}_{t \ge 1}\) adapted to \(\mathcal{F}_{t-1}\) such that, almost surely

    \noindent \textit{i)}\  The conditional bias is bounded by \(\mathcal{B}_t\), i.e., \(\|\mathbb{E}[\mathcal{I}_t(\mathcal{T}) \mid \mathcal{F}_{t-1}]\|_\calG \le \mathcal{B}_t.\)

    \noindent \textit{ii)} \ The centered innovation \(\mathcal{Z}_t \coloneqq \mathcal{I}_t(\mathcal{T}) - \mathbb{E}[\mathcal{I}_t(\mathcal{T}) \mid \mathcal{F}_{t-1}]\) is sub-Gaussian conditioned on \(\mathcal{F}_{t-1}\) with a random, predictable variance proxy \(\Sigma_t^2\). That is, \(\mathbb{E}\left[ \exp\left( \frac{\|\mathcal{Z}_t\|_\calG^2}{\Sigma_t^2} \right) \middle| \mathcal{F}_{t-1} \right] \le 2.\) We also assume that on the \(\mathcal{F}_{t-1}\)-measurable event \(\{\Sigma_t = 0\}\), \(\mathcal{Z}_t = 0\) almost surely.
\end{assumption}

Assumption~\ref{ass:innovation_regularity} controls the effective innovation by bounding its conditional bias \(\mathcal{B}_t\) and its centered stochastic component via the variance proxy \(\Sigma_t^2\). We now state the main theorem for estimation error.

\begin{myframe}
\begin{main theorem}[Unified Estimation Error Bound]\label{thm:unified}
    Consider the unified estimator \eqref{eq:unified_est_general} governed by predictable sequences \(\{\beta_t\}_{t\ge 1}, \{p_t\}_{t\ge 1}\) and \(\{\eta_{t-1}\}_{t\geq1}\) satisfying Assumption~\ref{ass:predictable_params}. 
    Let \(\tau_m\) denote the reset times, defined recursively by \(\tau_0 = 0,\) \(\tau_m = \inf\{t>\tau_{m-1}: b_t=1\}\) for \(1\leq m \leq T,\), 
    and let \(N_t\) denote the counting process for the number of resets up to time \(t\) such that \(N_t\coloneqq \sum_{s=1}^t b_s\), 
    and define the cumulative momentum factors \(\Lambda_{t,j} \coloneqq \prod_{i=j+1}^t \beta_i,\) \(A_t \coloneqq \prod_{i=1}^t \beta_i^{-1}\), with the convention \(\Lambda_{t,t}=1\). 
    Under Assumptions~\ref{ass:subgaussian},~\ref{ass:innovation_regularity}, for any \(\delta\in(0,1)\) and any deterministic variance budget \(V_t\ge 0\), with probability at least \(1-\delta\), the condition \(\mathfrak V_t^2 \coloneqq \sum_{j=\tau_{N_t}+1}^t A_j^2\Sigma_j^2 \le V_t\) implies that the estimation error for any \(t\in[T], T\geq 2\) satisfies
    \begin{equation}
        \label{eq:unified_main_bound}
        \|e_t\|_\calG \leq \Lambda_{t,\tau_{N_t}} C\left(\frac{\delta}{2T}, \kappa\right) \sigma/\sqrt{B} 
        + \sum_{j=\tau_{N_t}+1}^t \Lambda_{t,j} \mathcal{B}_j + \sqrt{2} C\left(\frac{\delta}{2T},\kappa\right)A_t^{-1}\sqrt{V_t}.
    \end{equation}
\end{main theorem}
\end{myframe}

The bound in Theorem~\ref{thm:unified} decomposes the estimation error into three terms. The first term captures the effect of the most recent reset or batch initialization and decays according to the cumulative momentum since that reset. The second term accumulates the conditional bias of the effective innovations through the quantities \(\mathcal{B}_j\). The third term is a high-probability concentration term controlled by the predictable variance accumulated over the current reset interval. This decomposition makes explicit how reset probability, momentum strength, and the choice of correction term jointly determine the quality of the estimator. The complete details of the proof are presented in Appendix~\ref{apdx:proof_main}.
The technical difficulty in proving this result requires establishing concentration bounds for sums of martingale difference sequences over random reset times. 
To handle this, we combine a masking argument for stopping-time intervals (see Lemma~\ref{lemma:masked_mds}) with the vector-valued Freedman inequality (see Theorem~\ref{thm:freedman}), which is the main concentration tool underlying the proof. 

\section{Instantiations of the Unified Estimation Error Bound}\label{sec:instantiations}
Theorem~\ref{thm:unified} reduces the analysis of any estimator in our unified framework to controlling two quantities associated with the \textit{Effective Innovation}: the conditional bias bound \(\mathcal{B}_t\) and the predictable variance proxy \(\Sigma_t^2\) in Assumption~\ref{ass:innovation_regularity}. In this section, we instantiate these quantities for the three families. The resulting bounds reveal how \(\mathcal{T}_t\) dictates the bias-variance tradeoff.
Our standing assumptions from previous sections remain in play throughout. In particular, we use Assumption~\ref{ass:update_bound} to control the displacement of consecutive iterates, Assumption~\ref{ass:predictable_params} for the predictability of the control sequences, and Assumption~\ref{ass:subgaussian} for the sub-Gaussianity of the oracle.
The only additional ingredients needed in this section are family-specific regularity conditions that allow us to identify \(\mathcal{B}_t\) and \(\Sigma_t^2\) in Assumption~\ref{ass:innovation_regularity} concretely. 

Beyond Assumptions~\ref{ass:update_bound}, \ref{ass:predictable_params}, and \ref{ass:subgaussian}, we require the additional assumptions associated with the three estimator families. These enable us to satisfy Assumption~\ref{ass:innovation_regularity} and apply Theorem~\ref{thm:unified} to bound the estimation error \(\|e_t\|_\calG\). All proofs are available in Appendix~\ref{apdx:est_instant}.

\subsection{Family 1: Zeroth-Order Correction}
\label{subsec:family1}
We begin with the zeroth-order family, which recovers the momentum-style update, where the correction term is 
\(\mathcal{T}_t(w_t,w_{t-1},\xi_t) = \beta_t\bigl(G(w_{t-1},\xi_t)-G(w_t,\xi_t)\bigr)\), and its effective innovation is \(\calI_t(\calT)=(1-\beta_t)(G(w_t,\xi_t)-g(w_t))-\beta_t(g(w_t)-g(w_{t-1}))\). 

\begin{assumption}[Lipschitz Continuity] 
\label{ass:lipschitz}
    The mean function \(g\) is \(L\)-Lipschitz continuous, satisfying \(\|g(w) - g(w')\|_\calG \le L \|w - w'\|\) for all \(w, w' \in \mathcal{W}\).
\end{assumption}
Assumption~\ref{ass:lipschitz} controls the drift caused by reusing past information at \(w_{t-1}\) while the target has moved to \(g(w_t)\). 
Under Assumption~\ref{ass:lipschitz}, we show that the conditional regularity of the innovation (Assumption~\ref{ass:innovation_regularity}) is satisfied. This holds because the conditional bias \(\mathcal{B}_t\) is bounded as follows:
\[
\|\E[\calI_t(\calT)\mid \calF_{t-1}]\|_\calG = \beta_t \| \E[g(w_t) - g(w_{t-1})\mid \calF_{t-1}]\|_\calG\leq \beta_t\eta_{t-1} LG =: \calB_t,
\]
and the predictable variance proxy for \(\calZ_t = \calI_t(\calT)-\E[\calI_t(\calT)\mid \calF_{t-1}]=(1-\beta_t)(G(w_t, \xi_t)-g(w_t))\) evaluates to \(\Sigma_t^2 = (1-\beta_t)^2\sigma^2\) (see Proposition~\ref{prop:family1} and its proof in Appendix~\ref{apdx:est_instant}).

\subsection{Family 2: First-Order Correction}
We next consider the first-order differential family, corresponding to the choice \(\mathcal T_t(w_t,w_{t-1},\xi_t) = 0\) and its effective innovation is \(\calI_t(\calT)=(1-\beta_t)(G(w_t,\xi_t)-g(w_t))+\beta_t \Delta(w_t, w_{t-1}, \xi_t)\). This family includes STORM, PAGE, and SPIDER-type estimators as special cases. The key feature of this recursion is that the deterministic drift from the moving target is removed and replaced by the centered difference noise.

\begin{assumption}[Sub-Gaussian Lipschitzness of the Centered Differences] 
\label{ass:subg_lipschitz}
    For the analysis of variance-reduced estimators that rely on estimate differences, such as STORM, we require a stronger tail condition on the continuity of the noise. Let \(\Delta(w, w', \xi) \coloneqq (G(w, \xi) - G(w', \xi)) - (g(w) - g(w'))\) denote the centered estimate difference noise. We assume the stochastic oracle is \(\ell\)-sub-Gaussian Lipschitz, meaning that for all \(w, w' \in \mathcal{W}\), \( \mathbb{E}\left[ \exp\left( \tfrac{\|\Delta(w, w', \xi)\|_\calG^2}{\ell^2 \|w - w'\|^2} \right) \right] \le 2.\)
\end{assumption}
Assumption~\ref{ass:subg_lipschitz} is tailored to first-order differential recursions. It ensures that stochastic differences concentrate with a variance proxy proportional to the squared distance between iterates (see two concrete examples for Assumption~\ref{ass:subg_lipschitz} in Subsection~\ref{subsec:example}, Appendix~\ref{apdx:est_instant}). 
Under Assumption~\ref{ass:subg_lipschitz}, the innovation satisfies the conditional regularity condition in Assumption~\ref{ass:innovation_regularity}. Specifically, this holds with a conditional bias of \(\calB_t = \|\E[\calI_t(\calT) \mid\calF_{t-1}]\|_\calG = 0\) and a predictable variance proxy of \(\Sigma_t^2 = 2 (1-\beta_t)^2\sigma^2 + 2 \beta_t^2 \ell^2 \eta_{t-1}^2 G^2\) (see Proposition~\ref{prop:family2} and its proof in Appendix~\ref{apdx:est_instant}).

Specific choices of \((p_t,\beta_t)\) smoothly recover standard algorithms: STORM when \(p_t=0\), \(\beta_t<1\), loopless SVRG/PAGE when \(p_t \in (0,1)\), \(\beta_t=1\), and SPIDER/SVRG with periodic resets and \(\beta_t=1\). In the fully differential regime (\(\beta_t=1\)), the direct-noise term vanishes, leaving only the difference-noise variance proxy \(\Sigma_t^2 = 2\ell^2\eta_{t-1}^2G^2.\)

\begin{table*}[!t]
\centering
\caption{
Family-wise instantiations of Theorem~\ref{thm:unified}. Each row corresponds to a representative choice of \((\beta_t,p_t)\) within a family, assuming a constant step size \(\eta_t = \eta\). For deterministic periodic resets (Case~3), the initialization term is union-bounded over the \(T/E\) scheduled epochs to yield \(C(\delta/(2T/E),\kappa)\), though \(C(\delta/(2T),\kappa)\) from Theorem~\ref{thm:unified} remains valid up to logarithmic factors.
}
\label{tab:error_bounds}
\renewcommand{\arraystretch}{2.0}
\resizebox{\textwidth}{!}{%
\begin{tabular}{l || c | c || l}
\hline\hline
\textbf{Algorithm} & \(\beta_t\) & \(p_t\) & \textbf{Upper Bound on \(\|e_t\|_\calG,\;\forall t\in[T]\) 
} \\
\hline\hline

\rowcolor{orange!10} \multicolumn{4}{l}{\textit{Family 1: Standard Recursion and Resetting (Zeroth-Order)} \quad \textbf{Assumptions:} 
\ref{ass:update_bound},\ref{ass:predictable_params},\ref{ass:subgaussian},\ref{ass:lipschitz}} \\
Momentum       & \(\beta\) & \(0\)                              & \(\beta^t C\left(\frac{\delta}{2T},\kappa\right)\sigma/\sqrt{B} + \frac{GL\eta}{1-\beta} + C\left(\frac{\delta}{2T},\kappa\right)\sigma\sqrt{1-\beta} \) \\
Probabilistic Momentum & \(1\)     & \(p\)                              & \(C\left(\frac{\delta}{4T},\kappa\right)\sigma/\sqrt{B} + \frac{GL\eta}{p}\log\frac{4T}{\delta} \) \\
Periodic Momentum      & \(1\)     & \(\mathbb{I}_{\{t \bmod E = 0\}}\) & \(C\left(\frac{\delta}{2T/E},\kappa\right)\sigma/\sqrt{B} + GL\eta E \) \\
\hline

\rowcolor{orange!10} \multicolumn{4}{l}{\textit{Family 2: Differential Recursion and Resetting (First-Order)} \quad \textbf{Assumptions:} 
\ref{ass:update_bound},\ref{ass:predictable_params},\ref{ass:subgaussian},\ref{ass:subg_lipschitz}} \\
STORM                & \(\beta\) & \(0\)                              & \(\beta^t C\left(\frac{\delta}{2T},\kappa\right)\sigma/\sqrt{B} + C\left(\frac{\delta}{2T},\kappa\right) \sqrt{\frac{2(1-\beta)^2\sigma^2+2\beta^2\ell^2\eta^2G^2}{1-\beta^2}} \) \\
PAGE / Loopless SVRG & \(1\)     & \(p\)                              & \(C\left(\frac{\delta}{4T},\kappa\right)\sigma/\sqrt{B} + C\left(\frac{\delta}{4T},\kappa\right)\ell\eta G \sqrt{\frac{2}{p}\log\frac{4T}{\delta}} \) \\
SPIDER / SVRG        & \(1\)     & \(\mathbb{I}_{\{t \bmod E = 0\}}\) & \(C\left(\frac{\delta}{2T/E},\kappa\right)\sigma/\sqrt{B} + C\left(\frac{\delta}{2T},\kappa\right)\ell\eta G\sqrt{2E} \) \\
\hline

\rowcolor{orange!10} \multicolumn{4}{l}{\textit{Family 3: Second-Order Recursion and Resetting (Second-Order)} \quad \textbf{Assumptions:} 
\ref{ass:update_bound},\ref{ass:predictable_params},\ref{ass:subgaussian},\ref{ass:subg_hessian},\ref{ass:smoothness}} \\
Second-Order Momentum  & \(\beta\) & \(0\)                              & \(\beta^t C\left(\frac{\delta}{2T},\kappa\right)\sigma/\sqrt{B} + \frac{\alpha\beta\eta^2G^2}{2(1-\beta)} + C\!\left(\frac{\delta}{2T},\kappa\right) \sqrt{\frac{2(1-\beta)^2\sigma^2+2\beta^2\gamma^2\eta^2G^2}{1-\beta^2}} \) \\
Second-Order PAGE   & \(1\)     & \(p\)                              & \(C\left(\frac{\delta}{4T},\kappa\right)\sigma/\sqrt{B} + \frac{\alpha\eta^2G^2}{2p}\log\frac{4T}{\delta} + C\left(\frac{\delta}{4T},\kappa\right)\gamma\eta G \sqrt{\frac{2}{p}\log\frac{4T}{\delta}} \) \\
Second-Order SPIDER & \(1\)     & \(\mathbb{I}_{\{t \bmod E = 0\}}\) & \(C\left(\frac{\delta}{2T/E},\kappa\right)\sigma/\sqrt{B} + \frac{\alpha\eta^2G^2E}{2} + C \left(\frac{\delta}{2T},\kappa\right)\gamma\eta G\sqrt{2E}\) \\
\hline\hline

\end{tabular}%
}
\vspace{-3mm} %
\end{table*}

\subsection{Family 3: Second-order Correction}
\vspace{-2mm}
We finally consider the second-order family, where the correction term incorporates local curvature information \(\mathcal T_t(w_t,w_{t-1},\xi_t) \coloneqq \beta_t\left( G(w_{t-1},\xi_t) - \left[G(w_t,\xi_t)+\nabla G (w_t,\xi_t)(w_{t-1}-w_t)\right] \right)\), and its effective innovation is \(\calI_t(\calT)=(1-\beta_t)(G(w_t, \xi_t)-g(w_t))+\beta_t(g(w_{t-1})-g(w_t)-\nabla G(w_t, \xi_t)(w_{t-1}-w_t))\). This correction is designed to remove the leading-order drift term and leave only a second-order Taylor residual, therefore we introduce \(L(\mathcal{X},\calG)\) as the set of all bounded linear operators from \(\mathcal{X}\) to \(\calG\), and for any \(T \in L(\mathcal{X},\calG)\) and its norm is defined by \(\|T\|_{L(\mathcal{X},\calG)} := \sup_{\|x\| \leq 1, x\in \mathcal{X}} \|Tx\|_\calG\).

\begin{assumption}[Sub-Gaussian Hessian Noise] 
\label{ass:subg_hessian}
    Let \(\nabla G(w, \xi)\) %
    denote the stochastic operator returned by the oracle such that %
    \(\mathbb{E}_\xi[\nabla G(w, \xi)]= \nabla g(w)\). We assume the noise in Hessian-vector products is sub-Gaussian. Specifically, 
    \(\mathbb{E}\left[\exp\left(\|\nabla G(w,\xi)-\nabla g(w)\|_{L(\mathcal{X},\calG)}^2/\gamma^2\right)\right]\leq 2\) for any \(w \in \mathcal{W}\).
\end{assumption}

\begin{assumption}[Smoothness] \label{ass:smoothness}
    \(g\) is \(\alpha\)-smooth. 
    That is, 
    \(\|\nabla g(w) - \nabla g(w')\|_{L(\mathcal{X}, \calG)} \le \alpha \|w - w'\|\), where \(\nabla g\) is the Jacobian of \(g\), or the Hessian of the objective if \(g = \nabla f\), bounded in spectral norm.
\end{assumption}

Assumptions~\ref{ass:subg_hessian} and \ref{ass:smoothness} allow us to control the stochastic noise in the Hessian-vector products and bound the second-order Taylor remainder, such that \( \|g(w_{t-1})-g(w_t) - \nabla g(w_t)(w_{t-1}-w_t)\|_\calG \leq \frac{\alpha}{2} \|w_t-w_{t-1}\|^2 \). Under these assumptions, the conditional regularity of the innovation (Assumption~\ref{ass:innovation_regularity}) holds with a conditional bias bounded by \(\calB_t\) since
\[
\|\E[\calI_t(\calT)\mid \calF_{t-1}]\|_\calG = \beta_t \|g(w_{t-1})-g(w_t) - \nabla g(w_t)(w_{t-1}-w_t)\|_\calG \leq \frac{\alpha}{2} \beta_t \eta_{t-1}^2 G^2 =: \calB_t,
\]
and the predictable variance proxy for \(\calZ_t = \calI_t(\calT) - \E[\calI_t(\calT)\mid\calF_{t-1}]=(1-\beta_t)(G(w_t, \xi_t)-g(w_t))-\beta_t (\nabla G(w_t, \xi_t)-\nabla g(w_t))(w_{t-1} - w_t)\) is given by \(\Sigma_t^2 = 2(1-\beta_t)^2\sigma^2+2\beta_t^2\gamma^2\eta_{t-1}^2G^2\) (see Proposition~\ref{prop:family3} and its proof in Appendix~\ref{apdx:est_instant}).

\section{Optimization Guarantees from Unified Estimation}\label{sec:app}
\vspace{-2mm}
We now show how the unified estimation bound translates into optimization guarantees in two representative settings. The key point is that Theorem~\ref{thm:unified} reduces the analysis to obtaining a uniform high-probability bound on the estimator error. Once such a bound is available, the remaining argument becomes application-specific. In unconstrained stochastic optimization, this yields first-order stationarity guarantees for mirror descent. In expectation-constrained stochastic optimization, it yields feasibility and optimality guarantees for the optimization algorithm. 
\vspace{-2mm}
\subsection{Application 1: Stochastic Optimization with Mirror Descent}\label{sec:app1}
\vspace{-2mm}
We consider the stochastic optimization problem~\eqref{eq:application1} of minimizing \(f(w) = \mathbb{E}_{\xi}[F(w, \xi)]\) over a closed convex set \(\mathcal{W}\), where we assume \(f\) is bounded below such that \(\inf_{w\in\mathcal{W}} f(w) > -\infty\). We define the true objective gradient as \(g(w) := \nabla f(w) = \nabla \mathbb{E}_{\xi}[F(w, \xi)]\) and its unbiased stochastic estimator as \(G(w, \xi)\). To relate the unified estimator \(v_t\) of the objective gradient \(g(w_t)\) to the update vector \(U_t\), we select \((\calG, \|\cdot\|_{\calG}):=(\calX^*, \|\cdot\|_*)\) and choose \(U_t=v_t/\|v_t\|_*\) as the normalized direction of \(v_t\). We then perform a mirror descent update \(w_{t+1}=w_t - \eta_t P(w_t, U_t, \eta_t)\) using a constant step size \(\eta_t=\eta>0\); since the update vector \(\|U_t\|_*=1\), Assumption~\ref{ass:update_bound} is satisfied with \(G=1\).

\textbf{Stationarity Criterion.} For any choice of the correction term \(\mathcal{T}_t\), we measure convergence using the criterion \(\frac{1}{T}\sum_{t=1}^T \|\nabla f(w_t)\|_\ast \|P_t\|^2 \leq \varepsilon\) for mirror descent with a normalized update vector, where the \textit{proximal stationarity witness} is defined as \(P_t \coloneqq P\!\left(w_t,\nabla f(w_t)/\|\nabla f(w_t)\|_\ast,\frac{\eta}{2}\right)\). For an unconstrained Hilbert space (e.g., \(\mathcal{X} = \mathbb{R}^d\) with \(\mathcal{W} = \mathcal{X}\)) under the \(\ell_2\)-norm, selecting the mirror map \(\Phi(w) = \frac{1}{2}\|w\|_2^2\) yields exactly \(P(w, u, \eta) = u\). This implies \(P_t = \nabla f(w_t)/\|\nabla f(w_t)\|_2\) and \(\|P_t\| = 1\), meaning our results with above stationary criterion recover the classic theoretical guarantees for the \(\varepsilon\)-stationarity of \(\frac{1}{T}\sum_{t=1}^T \|\nabla f(w_t)\|_2\) \cite{liu2020improved,cutkosky2019momentum,cutkosky2020momentum,cutkosky2021high,li2021page,fang2018spider}. For comparison, STORM-type algorithm that uses the gradient estimator directly as the update vector~\cite{xu2023momentum} exhibits a comparable oracle complexity on the order of \(\varepsilon^{-3}\), while relying on \(\|P(w,\nabla f(w), \eta)\|\) as its criterion for first-order stationarity.

\textbf{Oracle / Iteration Complexities.} For brevity, the results of oracle / iteration complexities for all special cases of these three families are summarized in  Table~\ref{tab:complexity_summary} with different selections of \(\beta_t, p_t\) and \(\eta_t=\eta\) as listed in Table~\ref{tab:parameter_configurations}, Appendix~\ref{apdx:application1}, where we define the initial suboptimality gap as \(\Delta_f \coloneqq f(w_1)-\inf_{w\in\mathcal{W}} f(w)\), and introduce the following shorthand constants:
\begin{equation}\label{eq:define_simga}
\sigma_L \coloneqq (\Delta_f L)^{\frac{1}{2}}, \quad \sigma_\ell \coloneqq (\Delta_f \ell)^{\frac{1}{2}}, \quad \sigma_\gamma \coloneqq (\Delta_f \gamma)^{\frac{1}{2}}, \quad \sigma_\alpha \coloneqq (\Delta_f^2 \alpha)^{\frac{1}{3}}.
\end{equation}
The parameters \(L, \ell, \gamma,\) and \(\alpha\) refer to constants in Assumptions~\ref{ass:lipschitz}, \ref{ass:subg_lipschitz}, \ref{ass:subg_hessian}, and \ref{ass:smoothness}, respectively.

\begin{theorem}[Complexities, Informal]\label{thm:informal_complexities}
Under the parameter selections in Table~\ref{tab:parameter_configurations}, the oracle and iteration complexities \(N, T\) to achieve \(\varepsilon\)-stationarity with high probability are:

\algofam{Family 1:} 
\algoname{SGD-M:} \(\tilde{\mathcal{O}}\big(\frac{\sigma_L^2\sigma^2}{\varepsilon^4}\vee \frac{\sigma^3}{\varepsilon^3}\big)\); \\
\algoname{Probabilistic Momentum:} \(\tilde{\mathcal{O}}\big(\frac{\sigma_L^2\sigma^2}{\varepsilon^4}\big)\); 
\algoname{Periodic Momentum:} \(\tilde{\mathcal{O}}\big(\frac{\sigma_L^2\sigma^2}{\varepsilon^4}\big)\).

\algofam{Family 2:} 
\algoname{STORM:} \(\tilde{\mathcal{O}}\big(\frac{\sigma^3}{\varepsilon^3} \vee \frac{\sigma_\ell^2\sigma}{\varepsilon^3} \vee \frac{\sigma_L^2}{\varepsilon^2}\big)\); \\
\algoname{PAGE:} \(\tilde{\mathcal{O}}\big(\frac{\sigma_\ell^2\sigma}{\varepsilon^3} \vee \frac{\sigma_L^2}{\varepsilon^2}\big)\); 
\algoname{SPIDER:} \(\tilde{\mathcal{O}}\big(\frac{\sigma_\ell^2\sigma}{\varepsilon^3} \vee \frac{\sigma_L^2}{\varepsilon^2}\big)\).

\algofam{Family 3:} 
\algoname{SO Momentum:} \(\tilde{\mathcal{O}}\big(\frac{\sigma^3}{\varepsilon^3} \vee \frac{\sigma_\gamma^2\sigma}{\varepsilon^3} \vee \frac{\sigma_\alpha^{3/2}\sigma}{\varepsilon^{5/2}} \vee \frac{\sigma_L^2}{\varepsilon^2}\big)\); \\
\algoname{SO PAGE:} \(\tilde{\mathcal{O}}\big(\frac{\sigma_\gamma^2\sigma}{\varepsilon^3} \vee \frac{\sigma_\alpha^{3/2}\sigma}{\varepsilon^{5/2}} \vee \frac{\sigma_L^2}{\varepsilon^2}\big)\); 
\algoname{SO SPIDER:} \(\tilde{\mathcal{O}}\big(\frac{\sigma_\gamma^2\sigma}{\varepsilon^3} \vee \frac{\sigma_\alpha^{3/2}\sigma}{\varepsilon^{5/2}} \vee \frac{\sigma_L^2}{\varepsilon^2}\big)\).
\end{theorem}

We use \(\tilde{\mathcal{O}}\) to hide the dependency on \(\log\frac{1}{\varepsilon}\)  and \(\log\frac{1}{\delta}\) in the informal Theorem~\ref{thm:informal_complexities} above. It is important to note that our comparisons to existing literature remain valid and fair despite differences in underlying sub-Gaussian assumptions versus bounded expected variance. For Family 1, our complexities for SGD-M match the \(\varepsilon^{-4}\) rate of existing work that also employs a normalized update vector (see Theorem 1 in~\cite{cutkosky2020momentum}). It allows us to achieve a sharper subdominant \(\sigma^3/\varepsilon^3\) term compared to the \((\sigma^4/\sigma_L^2)/\varepsilon^2 + (\sigma_L^2 \sigma^2)/\varepsilon^4\) rate established in~\cite{cutkosky2020momentum}, since the latter is strictly lower-bounded by our rate via \(2\sqrt{xy} \leq x+y\). While previous works within Family 2 and Family 3 rely on unnormalized update vectors, our approach incorporates normalization; nevertheless, direct comparisons to their respective baselines of complexities remain mathematically fair. Specifically, the complexity rates for STORM align with the dominant \(\sigma^3/\varepsilon^3\) term established in existing work (Theorem 1 of \cite{cutkosky2019momentum}), and the \((\sigma_\ell^2 \sigma)/\varepsilon^3\) oracle complexity matches previous findings for SPIDER (Theorem 1 in~\cite{fang2018spider}). For Family~3, which incorporates second-order correction, the oracle complexities of our new second-order variants align with the \((\sigma_\gamma^2\sigma)/\varepsilon^3+(\sigma_\alpha^{3/2}\sigma)/\varepsilon^{5/2}+\sigma_L^2/\varepsilon^2\) rate in prior work on variance reduction via Hessian-vector products (Theorem 2, page 8 of~\cite{arjevani2020second}). For details and configurations of our theoretical results, please refer to Tables~\ref{tab:complexity_summary},~\ref{tab:parameter_configurations}, as well as the formal Theorem~\ref{thm:mirror_stationarity} in Appendix~\ref{apdx:application1}.
\begin{table*}[t]
\centering
\caption{
Dominant \(\varepsilon\)-dependence of the oracle complexity to achieve \(\varepsilon\)-stationarity (Application~1) or \(\varepsilon\)-optimality and feasibility tolerance (Application~2), with logarithmic factors in \(\varepsilon^{-1}\) and \(\delta^{-1}\) suppressed. Application~1 refers to stochastic optimization with mirror descent~\eqref{eq:application1}; Application~2 refers to stochastic optimization with an expectation constraint~\eqref{eq:application2}. Full expressions with explicit constants appear in Appendices~\ref{apdx:application1} and~\ref{apdx:constrained_sgm}.
}
\label{tab:unified_complexity}
\renewcommand{\arraystretch}{1.5} %
\setlength{\tabcolsep}{6pt}
\small
\begin{tabular}{l || c | c | c || c | c | c}
\hline\hline
& \multicolumn{3}{c||}{\cellcolor{orange!15}\textbf{Application 1}~\eqref{eq:application1}} & \multicolumn{3}{c}{\cellcolor{blue!10}\textbf{Application 2}~\eqref{eq:application2}} \\
& \cellcolor{orange!15}Family 1 & \cellcolor{orange!15}Family 2 & \cellcolor{orange!15}Family 3 & \cellcolor{blue!10}Family 1 & \cellcolor{blue!10}Family 2 & \cellcolor{blue!10}Family 3 \\
\hline
Case 1: \(\beta{<}1,\; p{=}0\)                            & \multirow{3}{*}{\(\varepsilon^{-4}\)} & \multirow{3}{*}{\(\varepsilon^{-3}\)} & \multirow{3}{*}{\(\varepsilon^{-3}\)} & \multirow{3}{*}{\(\varepsilon^{-4}\)} & \multirow{3}{*}{\(\varepsilon^{-3}\)} & \multirow{3}{*}{\(\varepsilon^{-3}\)} \\
Case 2: \(\beta{=}1,\; p{>}0\)                            & & & & & & \\
Case 3: \(\beta{=}1\), \(p=\mathbb{I}_{\{t \bmod E = 0\}}\) & & & & & & \\
\hline\hline
\end{tabular}
\end{table*}
\subsection{Application 2: Expectation-Constrained Stochastic Optimization}\label{sec:app2}
\label{subsec:constrained_app}
We now consider the expectation-constrained stochastic optimization problem~\eqref{eq:application2}, in which feasibility must be determined from noisy observations of the constraint function. Unlike Application~1, where the estimator tracks a gradient used for the update direction, here the estimator tracks the scalar constraint value used for feasibility detection. This setting fits directly into the estimation framework developed in the previous sections by setting
\(g(w) := h(w),\) \(G(w,\xi) := H(w,\xi),\)
and hence \((\calG,\|\cdot\|_\calG):=(\mathbb R,|\cdot|)\). The unified estimator~\eqref{eq:unified_est_general} then produces an estimate of \(h(w_t)\), whose error is controlled by Theorem~\ref{thm:unified}. This perspective yields a general reduction from high-probability estimation guarantees to optimization guarantees for constrained stochastic problems. Throughout this subsection, we use a constant step size \(\eta_t\equiv\eta>0\).

\begin{assumption}[Constrained Setup]
\label{ass:constrained_setup}
The functions \(f, h\) are convex on \(\mathcal{W}\). The domain \(\mathcal{W}\) is bounded such that \(\sup_{u,v\in\mathcal{W}}\|u-v\| \le D\) and \(\sup_{u,v\in\mathcal{W}} D_\Phi(u,v) \le \frac{R^2}{2}\). Furthermore, the stochastic subgradients are conditionally unbiased given the filtration \(\mathcal{F}_{t-1}\) and uniformly bounded:
\[
    \mathbb{E}[F'(w_t,\zeta_t)\mid \mathcal{F}_{t-1}] \in \partial f(w_t),\quad \mathbb{E}[H'(w_t,\zeta_t)\mid \mathcal{F}_{t-1}] \in \partial h(w_t),
\]
with \(\|F'(w,\zeta)\|_* \le G\) and \(\|H'(w,\zeta)\|_* \le G\) almost surely for all \(w\in\mathcal{W}\). Moreover, conditioned on the history and the switching decision, the stochastic subgradient used in the update remains unbiased, since \(\zeta_t\) is independent of \(\calF_{t-1}, \xi_t, \xi_t^B\) and \(b_t\). There exists an optimal solution \(w^\star \in \arg\min_{w\in\mathcal W:\, h(w)\le 0} f(w)\) to~\eqref{eq:application2}.
\end{assumption}

\vspace{-1mm}
\noindent\textbf{Switching Gradient Method (SGM).}
To solve~\eqref{eq:application2}, we use the primal-only \emph{Switching Gradient Method} (SGM)~\cite{polyak1967general,bayandina2018mirror,lan2020algorithms}. At iteration \(t\), the estimator \(v_t\) is used to decide whether to update using the objective subgradient or the constraint subgradient. Given a switching threshold \(\varepsilon > 0\), define
\begin{equation}
U_t \coloneqq \mathbb I_{\{v_t \le \varepsilon\}}\,F'(w_t,\zeta_t)
+ \bigl(1-\mathbb I_{\{v_t \le \varepsilon\}}\bigr)\,H'(w_t,\zeta_t),
\end{equation}
and perform mirror-descent \(w_{t+1} = \arg\min_{w\in\mathcal W} \left\{ \eta \langle U_t,w\rangle + D_\Phi(w,w_t) \right\}.\)
Thus, whenever the estimated constraint value is sufficiently small, the method prioritizes descent on the objective; otherwise, it prioritizes reducing constraint violation. 
\begin{theorem}[Constrained SGM via Unified Estimation]
\label{thm:constrained_main}
Suppose Assumption~\ref{ass:constrained_setup} holds and the unified estimator~\eqref{eq:unified_est_general} satisfies the uniform tail bound \(\mathbb{P}\big(\sup_{t \in [T]} |e_t| > \mathcal{E}\big) \leq \delta/2\) for some deterministic \(\mathcal{E} \ge 0\). If we set the switching threshold to
\begin{equation*}
    \varepsilon \coloneqq \frac{R^2}{2\eta T} + \frac{\eta G^2}{2} + \frac{2DG}{\sqrt{T}}\sqrt{2\log\frac{4}{\delta}} + \mathcal{E},
\end{equation*}
then, with probability at least \(1-\delta\), the selected set \(\mathcal{A}:=\{t\in[T] : v_t \leq \varepsilon\}\) is non-empty, and the uniform average \(\bar{w} = \frac{1}{|\mathcal{A}|}\sum_{t \in \mathcal{A}} w_t\) satisfies both \(f(\bar{w}) - f(w^\star) \le \varepsilon\) and \(h(\bar{w}) \le \varepsilon + \mathcal{E}\).
\end{theorem}

Theorem~\ref{thm:constrained_main} provides a reusable reduction: once a family-specific estimator yields a high-probability bound \(\mathcal{E}\) (see instantiations in Table~\ref{tab:error_bounds}), the optimization guarantee follows by combining this envelope with the SGM error terms from Table~\ref{tab:error_bounds}. Optimizing the parameters case by case yields the oracle complexities in Table~\ref{tab:unified_complexity}; full statements are given in Corollaries~\ref{cor:constrained_case1_all_families_complexity}--\ref{cor:constrained_case3_all_families_complexity} in Appendix~\ref{apdx:constrained_sgm}. The rate is dictated by the estimator envelope \(\mathcal{E}\). Family~1 retains a first-order estimation bias and matches the baseline \(\mathcal{O}(\varepsilon^{-4}\log^4(1/\delta))\) rate of~\cite{lan2020algorithms}, whereas Families~2 and~3 remove this and improve the dominant rate to \(\tilde{\mathcal O}(\varepsilon^{-3})\) with milder confidence dependence. In particular, for Families~2 and~3, the \(\varepsilon^{-3}\) oracle-complexity terms have at most \(\log^{3/2}(1/(\varepsilon\delta))\) logarithmic confidence dependence in the fully recursive and loopless-reset settings (Cases~1 and 2), whereas the periodic-reset setting (Case~3) exhibits an even tighter logarithmic dependence.
Although Family~1 does not improve the dominant rate, its fully recursive variant \((\beta_t=\beta<1,p_t=0)\) uses batch size \(1\) throughout, avoiding repeated large-batch constraint evaluations.

For primal-only switching methods with a convex expectation constraint and high-probability guarantees of simultaneous optimality and feasibility, \citep{lan2020algorithms} obtain \(\widetilde{\mathcal{O}}(\varepsilon^{-4})\) total oracle complexity using \(\widetilde{\mathcal{O}}(\varepsilon^{-2})\) samples per iteration for constraint estimation. Under comparable light-tail assumptions, our Family 1 matches this rate while using \(\mathcal{O}(1)\) constraint samples per iteration. Under the additional sub-Gaussian centered-difference assumption required by Family 2, the total oracle complexity improves to \(\widetilde{\mathcal{O}}(\varepsilon^{-3})\). Family 3 achieves the same rate when the population constraint is smooth and the stochastic derivative noise is sub-Gaussian.
Accordingly, we qualify our claim as the first, to our knowledge, high-probability \(\widetilde{\mathcal{O}}(\varepsilon^{-3})\) stochastic-oracle guarantee within this primal-only switching and convex expectation-constraint setting. This comparison does not include primal-dual, augmented-Lagrangian, or compositional methods that employ different oracle models, assumptions, or solution criteria.

\section{Conclusion}\label{sec:conc}
In this paper, we introduced a unified framework for the design and analysis of recursive stochastic estimators. By decomposing estimator dynamics into memory retention, reset probability, and a dedicated correction term for iterate movement, we provided a lens through which various variance-reduced estimators can be analyzed with high probability.
A main contribution of our work is the development of a unified high-probability estimation bound for general smooth normed spaces. By moving beyond expectation-based and Euclidean-centric analyses, this bound provides robust estimation guarantees for generalized first-order methods, including mirror descent. We demonstrated the power of this approach by recovering high-probability guarantees for unconstrained stochastic optimization and establishing the first variance-reduced oracle complexities for expectation-constrained problems, improving the dominant rate from \(\tilde{\mathcal{O}}(\varepsilon^{-4})\) to \(\tilde{\mathcal{O}}(\varepsilon^{-3})\).

\section*{Acknowledgements}

This work was supported in part by the National Science Foundation (NSF) under Grants DMS-2502560 and CNS-2313109, and in part by an NSF CAREER Award under Grant CCF-2337808.

\newpage
\bibliographystyle{plain}
\bibliography{ref_variance_reduced}

\newpage
\appendix
\part{}
\parttoc

\section*{Appendices}
We organize the Appendices as follows:

\begin{itemize}
    \item In Appendix~\ref{apdx:prelim}, we provide the notations, symbols, and preliminaries in Section~\ref{sec:prelim}.
    \item In Appendix~\ref{apdx:proof_main}, we provide proofs for the results of the unified estimation error bound in Section~\ref{sec:main_results}.
    \item In Appendix~\ref{apdx:est_instant}, we present proofs for instantiations of the unified estimation error bound in Section~\ref{sec:instantiations}.
    \item In Appendix~\ref{apdx:lemma}, we prepare lemmas used in the proofs for the two applications: \eqref{eq:application1} and \eqref{eq:application2}.
    \item In Appendix~\ref{apdx:application1}, we detail the first application~\eqref{eq:application1} regarding stochastic optimization with mirror descent in Subsection~\ref{sec:app1}.
    \item In Appendix~\ref{apdx:constrained_sgm}, we detail the second application~\eqref{eq:application2} regarding expectation-constrained stochastic optimization in Subsection~\ref{sec:app2}.
    \item In Appendix~\ref{apdx:related}, we discuss related work.
\end{itemize}

\newpage
\section{Notations, Symbols and Preliminaries}
\label{apdx:prelim}
\renewcommand{\arraystretch}{1.15}
\begin{longtable}{@{}p{0.22\linewidth}p{0.74\linewidth}@{}}
\label{tab:notation}
\\
\hline
\textbf{Symbol} & \textbf{Meaning} \\
\hline
\endfirsthead

\hline
\textbf{Symbol} & \textbf{Meaning} \\
\hline
\endhead

\hline
\endfoot

\hline
\endlastfoot

\(\mathcal{X},\|\cdot\|\) &  Reflexive Banach space with primal norm. \\
\(\mathcal{X}^*,\|\cdot\|_\ast\) & Dual space and dual norm of \(\mathcal{X}\). \\
\(\Phi, \mathcal{D}\) & Mirror domain for the mirror map \(\Phi:\mathcal{D}\to \mathbb{R}\), with \(\mathcal{D} \subseteq \mathcal{X}\). \\
\(\mathcal{W}\) & Closed convex feasible iterate set with \(\mathcal{W}\subseteq \text{int}(\mathcal{D})\). \\
\(D_\Phi(x,y)\) & Bregman divergence induced by \(\Phi\). \\
\(P(w,u,\eta)\) & Proximal gradient mapping: \(P(w,u,\eta)\coloneqq (w-w^+)/\eta\) where \(w^+\) minimizes \(\langle\eta u,w'\rangle+D_\Phi(w',w)\) over \(w'\in\mathcal{W}\) (Definition~\ref{def:proximal_gradient_mapping}). \\
\(w_t\) & Optimization iterate at time \(t\), with \(w_t\in\mathcal{W}\). \\
\(U_t\) & Update vector used by the (generalized) update rule at time \(t\). \\
\(\eta_t\) & Learning rate / step size at time \(t\). \\
\(\calG, \|\cdot\|_\calG\) & \(\kappa\)-smooth space equipped with \(\|\cdot\|_\calG\), and \(G:\mathcal{X}\times \Omega \to \calG\)\\
\(\kappa\) & Smoothness parameter (\(\geq 1\)) for the geometry regularity of \((\calG,\|\cdot\|_\calG)\). \\
\(G(w,\xi), g(w)\) & Stochastic oracle with randomness \(\xi\), and \(g(w):=\E_{\xi}[G(w, \xi)]\). \\
\(\xi_t\) / \(\xi_t^B:=(\xi_t^{(i)})_{i=1}^B\) & Fresh oracle randomness at time \(t\) / ``batch'' oracle randomness at resets. \\
\(v_t, e_t\) & Unified estimator tracking \(g(w_t)\), estimation error: \(e_t\coloneqq v_t-g(w_t)\). \\
\(L(\mathcal{X}, \calG), \|\cdot\|_{L(\mathcal{X}, \calG)}\) & \(L(\mathcal{X},\calG)\) is the set of all bounded linear operators from \(\mathcal{X}\) to \(\calG\), and for any \(T \in L(\mathcal{X},\calG)\), its norm is given by \(\|T\|_{L(\mathcal{X},\calG)} := \sup_{\|x\| \leq 1,x\in\mathcal{X}} \|Tx\|_{\calG}\)\\
\(\mathcal{F}_t\) & Filtration generated by oracle noise and reset events up to time \(t\). \\
\(b_t\) & Reset indicator: \(b_t\sim\mathrm{Bernoulli}(p_t)\). \\
\(p_t\) & Reset probability (predictable, taking values in \([0,1]\)). \\
\(\beta_t\) & Momentum parameter (in \((0,1]\)), controlling exponential forgetting. \\
\(\mathcal{T}_t\) & Correction term used by the unified estimator. \\
\(\mathcal{I}_t(\mathcal{T})\) & Effective innovation at time \(t\) induced by correction term \(\mathcal{T}_t\). \\
\(\Delta_t\) & Difference term in the unified recursion: \(\Delta_t\coloneqq v_{t-1}-G(w_{t-1},\xi_t)\). \\
\(\Delta(w_t,w_{t-1},\xi_t)\) & Centered estimate-difference noise: \(\Delta(w_t,w_{t-1},\xi_t)\coloneqq (G(w_t,\xi_t)-G(w_{t-1},\xi_t))-(g(w_t)-g(w_{t-1}))\). \\
\(\tau_m\) & \(m\)-th reset time, with \(\tau_0=0\) and \(\tau_m=\inf\{t>\tau_{m-1}:b_t=1\}\). \\
\(N_t\) & Counting process for the number of resets up to time \(t\): \(N_t\coloneqq \sum_{s=1}^t b_s\) with 
the exact equivalence \(\{\tau_m \leq t\} = \{N_t \geq m \}\). \\
\(m(t)\) & Epoch index such that \(t\in(\tau_{m(t)-1},\tau_{m(t)}]\), equivalently \(m(t) := 1+ N_{t-1} = 1 + \sum_{s=1}^{t-1} b_s\). \\
\(\Lambda_{t,j}\) & Cumulative momentum multiplier: \(\Lambda_{t,j}\coloneqq\prod_{i=j+1}^t\beta_i\), \(\Lambda_{t,t}=1\). \\
\(A_t\) & Auxiliary multiplier: \(A_t\coloneqq\prod_{i=1}^t\beta_i^{-1}\). \\
\(\mathcal{B}_t\) & Conditional bias magnitude of the effective innovation (Assumption~\ref{ass:innovation_regularity}). \\
\(\Sigma_t^2\) & Conditional variance proxy for the centered effective innovation (Assumption~\ref{ass:innovation_regularity}). \\
\(\mathfrak{V}_t^2\) & Accumulated predictable variance proxy over the current epoch. \\
\(V_t\) & Deterministic variance budget such that \(\mathfrak{V}_t^2\le V_t\). \\
\(\sigma^2\) & Oracle-noise variance proxy parameter (Assumption~\ref{ass:innovation_regularity}). \\
\(B\) & Batch/reset parameter used in the reset/initialization concentration term. \\
\(T\) & Deterministic time horizon (w.l.o.g. with \(b_T=1\)). \\
\(E\) & Epoch length for periodic reset schedules (e.g., \(p_t=\mathbb{I}_{\{t\ \mathrm{mod}\ E=0\}}\)). \\
\multicolumn{2}{c}{\textit{Application 1: Stochastic Optimization with Mirror Descent}} \\
\(P_t\) & Proximal stationarity witness at time \(t\): \(P_t \coloneqq P\!\left(w_t,\frac{\nabla f(w_t)}{\|\nabla f(w_t)\|_\ast},\frac{\eta}{2}\right)\) with \(f(w)=\mathbb{E}_{\xi}[F(w,\xi)], g(w) = \nabla f(w) = \mathbb{E}_{\xi}[G(w, \xi)]\). \\
\(\Delta_f\) & Initial objective gap: \(\Delta_f\coloneqq f(w_1)-\inf_{w\in\mathcal{W}} f(w) <+\infty\). \\
\(\sigma_L\) & Shorthand scale \(\sigma_L \coloneqq (\Delta_f\,L)^{1/2}\), where \(L\) is from Assumption~\ref{ass:lipschitz}. \\
\(\sigma_\ell\) & Shorthand scale \(\sigma_\ell \coloneqq (\Delta_f\,\ell)^{1/2}\), where \(\ell\) is from Assumption~\ref{ass:subg_lipschitz}. \\
\(\sigma_\gamma\) & Shorthand scale \(\sigma_\gamma \coloneqq (\Delta_f\,\gamma)^{1/2}\), where \(\gamma\) is from Assumption~\ref{ass:subg_hessian}. \\
\(\sigma_\alpha\) & Shorthand scale \(\sigma_\alpha \coloneqq (\Delta_f^2\,\alpha)^{1/3}\), where \(\alpha\) is from Assumption~\ref{ass:smoothness}. \\
\multicolumn{2}{c}{\textit{Application 2: Expectation-Constrained Stochastic Optimization}} \\
\(F(w,\xi),\,H(w,\xi)\) & Sample Objective and Constraint functions in the constrained optimization application. \\
\(f(w),\,h(w)\) & Population objective and population constraint:
\(f(w)\coloneqq \mathbb{E}[F(w,\xi)]\), \(h(w)\coloneqq \mathbb{E}[H(w,\xi)]\). \\
\(\zeta_t\) & Fresh stochastic randomness for the stochastic sub-gradients at iteration \(t\), which is independent of \(\calF_{t-1}, \xi_t, \xi_t^B\) and \(b_t\). \\
\(F'(w_t, \zeta_t),\, H'(w_t, \zeta_t)\) & Stochastic sub-gradients at iteration \(t\), unbiased gradient estimators.\\
\(w^\star\) & Any optimal feasible solution of the constrained problem. \\
\(\varepsilon\) & Switching threshold used to decide whether the update prioritizes the objective or the constraint. \\
\(\mathcal{A}\) & Active-phase feasible index set:
\(\mathcal{A}\coloneqq \{t\in[T]: v_t\le \varepsilon\}\). \\
\(\mathcal{B}\) & Active-phase non-feasible index set:
\(\mathcal{B}\coloneqq [T]\setminus \mathcal{A}\). \\
\(\bar w\) & Active-phase selected average:
\(\bar w \coloneqq \frac{1}{|\mathcal A|}\sum_{t\in\mathcal A} w_t\). \\
\(\mathcal{E}\) & Deterministic tail-envelope bound satisfying
\(|e_t|\le \mathcal{E}\) uniformly for all \(t\in[T]\) on a high-probability event. \\
\(\mathsf{OptError}(\eta,T,\delta)\) & Optimization Error in Theorem~\ref{thm:constrained_generic}:
\(\tfrac{R^2}{2\eta T}+\tfrac{\eta G^2}{2}+\tfrac{2DG}{\sqrt{T}}\sqrt{2\log\tfrac{4}{\delta}}\). \\
\hline
\end{longtable}

\noindent\textbf{Banach space (primal norm, dual norm).}
The norm \(\|\cdot\|\) on the reflexive Banach space \(\mathcal{X}\) is the primal norm used throughout.
Let \(\mathcal{X}^*\) be the dual space, which is also a Banach space. For \(u\in\mathcal{X}^*\), the dual norm is defined by
\(\|u\|_\ast \coloneqq \sup_{\|x\|\le 1} |u(x)|\). In particular, with the definition of the canonical
duality pairing \(\langle u, x\rangle := u(x)\) for \(u\in\mathcal{X}^*\) and \(x\in\mathcal{X}\), we have 
\[
\|u\|_\ast = \sup_{\|x\|\le 1, x\in \mathcal{X}} |\langle u, x\rangle| = \sup_{\substack{x \in \mathcal{X}\\ x\neq 0}} \frac{|\langle u, x\rangle|}{\|x\|},
\] and hence \(\|u\|_* \|x\| \geq \langle u, x\rangle\).

\begin{definition}[\(\kappa\)-Smoothness for Geometric Regularity](see \cite[Def.~2.1, p.~2]{juditsky2008large})\label{def:kappa_smoothness}
Let \(\kappa\geq 1\) and \((\calG,\|\cdot\|_\calG)\) is a  normed space, then \((\calG,\|\cdot\|_\calG)\) is \(\kappa\)-smooth if \(p:\calG\to \mathbb{R}, p(z)\equiv \|z\|_\calG^2, \forall z\in \calG\) is  continuously differentiable, and for all \(z,\Delta\in \calG\),
\[
p(z+\Delta)\le p(z)+\langle \nabla p(z),\Delta\rangle+\kappa p(\Delta).
\]
\end{definition}

\noindent In particular, for Hilbert spaces we have \(\kappa=1\).
For Lebesgue spaces \(\calG=L_q\):
\begin{itemize}[leftmargin=*]
    \item If \(q\ge 2\), then \(\kappa=q-1\).
    \item If \(q\in(1,2)\), then the dual exponent is \(q'=\frac{q}{q-1}\in(2,\infty)\), and one typically applies \(\kappa\)-smoothness to the dual space \(\calG^*=(L_q)^*=L_{q'}\), giving \(\kappa=q'-1=\frac{1}{q-1}\).
    \item If \(q=1\), then \(\calG^*=L_\infty\) is not \(\kappa\)-smooth for any finite \(\kappa\) in this framework.
\end{itemize}

\begin{remark}[Geometric Regularity]
For a \(\kappa\)-smooth normed space \((\calG, \|\cdot\|)\), the \(\kappa\)-regularity requires the existence of a proxy function \(V(\cdot)\) with a \(\kappa\)-Lipschitz gradient \cite{juditsky2008large}. For the Lebesgue space \(L_q\) with \(q \ge 2\), we have \(\kappa = q-1\). Furthermore, for Hilbert spaces (e.g., \(d\)-dimensional Euclidean space), we have \(\kappa = 1\).
To extend our analysis to the inherently non-smooth \(\ell_\infty\) geometry,
\footnote{This is desired as this is the natural space the gradients belong to in the most interesting application of mirror descent methods, i.e., solving problems on the probability simplex, on which the entropy is strongly convex with respect to the \(\ell_1\) norm.}
we utilize a smooth \(\ell_q\) approximation. Given a primal space \((\mathbb{R}^d, \|\cdot\|_p)\) with \(p = 1 + \varepsilon\) and its conjugate \(q = 1 + 1/\varepsilon\), the regularity constant scales as \(\kappa = q-1 = 1/\varepsilon\). Notably let \(\varepsilon = \frac{1}{\ln(\max\{e, d\})}\) which yields \(\kappa = \ln(\max\{e, d\})\), providing a dimension-insensitive deviation radius of order \(\mathcal{O}(\sqrt{\ln d})\), while enforcing \(\kappa = 1\) when \(d\leq 2\). 
\end{remark}

\newpage
\noindent\textbf{Mirror map and Bregman divergence.}

\begin{definition}[Mirror Map]
The mirror map \(\Phi: \mathcal{D} \to \mathbb{R} \cup \{+\infty\}\) with \(\mathcal{D} \subseteq \mathcal{X}\) is proper, weakly lower semi-continuous (w.l.s.c), and \(1\)-strongly convex on a nonempty closed convex set \(\mathcal{W}\subseteq \text{int}(\mathcal{D})\) with respect to \(\|\cdot\|\). Furthermore, \(\Phi\) is G\^{a}teaux differentiable on \(\mathcal{W}\).
\end{definition}
Its Bregman divergence \(D_\Phi(x,y)\) which is defined below satisfies the inequality for all \(x,y\in\mathcal{W}\):
\[
D_\Phi(x,y)\coloneqq \Phi(x)-\Phi(y)-\langle \nabla \Phi(y), x-y\rangle \geq \frac{1}{2}\|x-y\|^2.
\]
Then we have the following inequality for all \(x,y\in\mathcal{W}\):
\[
D_\Phi(x,y) + D_\Phi(y, x) =
\langle \nabla \Phi(x) - \nabla \Phi(y), x-y\rangle \geq \|x-y\|^2.
\]
\begin{definition}[Proximal Gradient Mapping]
    For \(w \in \mathcal{W}\), an input vector \(u\in\mathcal{X}^*\), and a step size \(\eta > 0\), the \textbf{proximal gradient mapping} is 
    \(P(w, u, \eta) = (w - w^+)/\eta\), 
    where the updated point \(w^+\) is the unique minimizer of the proximal subproblem:
    \(w^+= \arg\min_{w'\in\mathcal{W}} \langle \eta u, w'\rangle + D_\Phi(w', w)\)
    and \(D_\Phi(w', w)\) denotes the Bregman divergence induced by \(\Phi\).
\end{definition}

\begin{remark}[Existence and Uniqueness of Proximal Gradient Mapping]\label{remark:existence_uniqueness_proximal_mapping}
We guarantee the existence and uniqueness of the proximal update \(P(w, u, \eta)\) by considering the minimization problem \(\min_{w'\in \mathcal{W}} F(w')\). Specifically, defining the proximal objective functional as \(F(w') := \langle\eta u, w'\rangle + D_\Phi(w', w)\), we proceed as follows:

\begin{enumerate}
\def\labelenumi{\arabic{enumi}.}
\item
  \emph{Coercivity and Boundedness:} The strong convexity of the mirror map \(\Phi\) renders \(F(w')\) coercive, ensuring that \(F\) is bounded from below and a finite infimum \(\inf_{w'\in\mathcal{W}} F(w') > -\infty\) exists. Therefore, a minimizing sequence \(\{w'_n\}\subset \mathcal{W}\) exists such that \(\lim_{n\to \infty} F(w'_n) = \inf_{w'\in\mathcal{W}} F(w')\). Consequently, the coercivity of \(F\) indicates that the sequence \(\{w'_n\} \subset \mathcal{W}\) must be bounded.
\item
  \emph{Weak Subsequence (Eberlein-\v{S}mulian):} By the Eberlein-\v{S}mulian theorem, this bounded sequence within the reflexive Banach space \(\mathcal{X}\) is guaranteed to contain a weakly convergent subsequence \(w_{n_k}' \rightharpoonup w^+\).
\item
  \emph{Feasibility (Mazur's Lemma):} Because the feasible set \(\mathcal{W}\) is strongly closed and convex, Mazur's Lemma dictates it is also weakly closed. This ensures the weak limit point remains feasible, meaning \(w^+ \in \mathcal{W}\).
\item
  \emph{Existence:} Since the objective \(F\) is weakly lower semi-continuous (w.l.s.c.), it follows that \(F(w^+) \le \liminf_{k\to\infty} F(w'_{n_k}) = \lim_{n\to \infty} F(w'_n) = \inf_{w'\in\mathcal{W}} F(w')\). This proves \(w^+\) achieves the infimum and is a global minimizer.
\item
  \emph{Uniqueness:} Finally, the strict convexity of \(F\) guarantees that this global minimizer \(w^+\) is unique.
\end{enumerate}

Therefore, the proximal update \(P(w, u, \eta) := (w-w^+)/\eta\) exists and is unique for any \(w \in \mathcal{W} \subseteq \mathcal{X}, u \in \mathcal{X}^\ast\) and \(\eta >0\). 
\end{remark}

\begin{remark}[First-Order Sufficient/Necessary Optimality Condition]\label{remark:first_order_optimality_condition}
Because the proximal objective \(\mathcal{L}(w'; w, u, \eta) = \langle \eta u, w' \rangle + D_\Phi(w', w)\) is strictly convex and G\^{a}teaux differentiable, and the domain \(\mathcal{W}\) is convex, the variational inequality below is both a necessary and sufficient condition for global optimality. The necessity follows from the non-negativity of the directional derivative along feasible segments, while the sufficiency follows directly from the first-order characterization of convexity, which guarantees \(\mathcal{L}(w'; w, u, \eta) \ge \mathcal{L}(w^+; w, u, \eta) + \langle \nabla \mathcal{L}(w^+; w, u, \eta), w' - w^+ \rangle \ge \mathcal{L}(w^+; w, u, \eta)\) for all \(w' \in \mathcal{W}\).
\[
\langle \nabla \mathcal{L}(w^+; w, u, \eta),\, w' - w^+ \rangle =
\left\langle \eta u + \big(\nabla \Phi(w^+) - \nabla \Phi(w)\big),\, w' - w^+ \right\rangle \ge 0.
\]
By applying the Bregman Three-Point Identity, we have:
\[
    \langle \nabla \Phi(w^+)-\nabla \Phi(w), w'-w^+\rangle 
   = D_\Phi(w', w)-D_\Phi(w', w^+)-D_\Phi(w^+, w)
\]
Substituting the above into the first-order optimality condition, we obtain the equivalent condition:
\[
\eta \langle u, w'-w^+\rangle + D_\Phi(w', w)-D_\Phi(w', w^+)-D_\Phi(w^+, w) \ge 0.
\]
\end{remark}

For Mirror Descent, the update rule with some update vector \(U_t\) and step size \(\eta_t\) is given by
\[
w_{t+1}=\underset{w\in\mathcal{W}}{\arg\min}\left\{\langle U_t,w\rangle+\frac{1}{\eta_t}D_\Phi(w,w_t)\right\}
= w_t - \eta_t P(w_t, U_t, \eta_t)
\]
\begin{lemma}[Properties of Proximal Gradient Mapping, see Lemmas 1-3 in \cite{ghadimi2016mini} for first two properties]\label{lemma:properties_proximal_gradient_mapping}
For any \(w\in\mathcal{W}\), input vector \(u\in\mathcal{X}^*\), and step size \(\eta > 0\), the proximal gradient mapping \(P(w, u, \eta)\) has the following properties:
\begin{itemize}[leftmargin=*]
    \item (Inner Product Bounds) For any \(u_1, u_2\in \mathcal{X}^*\),
    \begin{eqnarray*}
    \|P(w, u_1, \eta) - P(w, u_2, \eta)\|^2
    &\leq& \langle u_1 - u_2, P(w, u_1, \eta) - P(w, u_2, \eta)\rangle\\
    &\leq& \|u_1 - u_2\|_\ast \|P(w, u_1, \eta) - P(w, u_2, \eta)\|.
    \end{eqnarray*}
    In particular, \(P(w, 0, \eta)=0\) from the definition of proximal gradient mapping, therefore
    \begin{eqnarray*}
    \|P(w, u, \eta)\|^2 \leq \langle u, P(w, u, \eta)\rangle \leq \|u\|_\ast \|P(w, u, \eta)\|.
    \end{eqnarray*}
    \item (Non-expansiveness) For any \(u_1, u_2\in \mathcal{X}^*\),
    \[
    \|P(w, u_1, \eta) - P(w, u_2, \eta)\| \leq \|u_1 - u_2\|_\ast.
    \]
    In particular, \(P(w, 0, \eta)=0\) from the definition of proximal gradient mapping, therefore
    \[
    \|P(w, u, \eta)\| =
    \|P(w, u, \eta)\| \leq \|u\|_\ast.
    \]
    \item (Scaling Property) For any scalar \(\alpha > 0\),
    \[
    P(w, \alpha u, \eta) = \alpha P(w, u, \alpha\eta).
    \]
    \item (Monotonicity) For any \(\eta'>\eta>0\),
    \[
    \langle u, \eta' P(w, u, \eta')\rangle \geq \langle u, \eta P(w, u, \eta)\rangle + \frac{\|\eta' P(w, u, \eta')-\eta P(w, u, \eta)\|^2}{\eta'-\eta}.
    \]
    Therefore, \(\langle u, \eta' P(w, u, \eta') \rangle \geq \langle u, \eta P(w, u, \eta) \rangle\) and:
    \[
    \|\eta' P(w, u, \eta')-\eta P(w, u, \eta)\| \leq (\eta'-\eta) \|u\|_\ast.
    \]
\end{itemize}
\end{lemma}

\begin{proof}

\textbf{Proof of the inner product bounds:}

We establish the lower bound of \(\<u_1-u_2, P(w, u_1, \eta) - P(w, u_2, \eta)\>\) by 
applying the first-order optimality condition discussed in Remark on Definition~\ref{def:proximal_gradient_mapping}
and using the fact \(D_\Phi(x, y) + D_\Phi(y, x) \geq \|x-y\|^2\) for the 1-strongly convex mirror map \(\Phi\).
\[
\eta \langle u, w'-w^+ \rangle + D_\Phi(w', w)-D_\Phi(w', w^+)-D_\Phi(w^+, w) \ge 0.
\]
By introducing \(w^+_1 := \arg\min_{w'\in\mathcal{W}} \langle\eta u_1, w'\rangle + D_\Phi(w', w)\), 
\(w^+_2 := \arg\min_{w'\in\mathcal{W}} \langle\eta u_2, w'\rangle + D_\Phi(w', w)\), therefore
\(w^+_1 = w- \eta P(w, u_1, \eta), w^+_2 = w- \eta P(w, u_2, \eta)\), 
and letting \(u\gets u_1, w^+\gets w^+_1, w'\gets w^+_2\) and \(u\gets u_2,w^+\gets w^+_2, w'\gets w^+_1\) 
in the first-order optimality condition above:
\[
\eta^2 \langle u_1, P(w, u_1, \eta) - P(w, u_2, \eta)\rangle 
\geq -D_\Phi(w^+_2, w)+D_\Phi(w^+_2, w^+_1)+D_\Phi(w^+_1, w)
\]
\[
\eta^2 \langle u_2, P(w, u_2, \eta) - P(w, u_1, \eta)\rangle 
\geq -D_\Phi(w^+_1, w)+D_\Phi(w^+_1, w^+_2)+D_\Phi(w^+_2, w)
\]
Adding the above two inequalities, and applying \(D_\Phi(x, y) + D_\Phi(y, x) \geq \|x-y\|^2\)
for all \(x, y\in\mathcal{W}\) and 1-strongly convex mirror map \(\Phi\), we obtain:
\[
\eta^2 \langle u_1-u_2, P(w, u_1, \eta) - P(w, u_2, \eta)\rangle 
\geq \|w^+_1-w^+_2\|^2 
= \eta^2 \|P(w, u_1, \eta) - P(w, u_2, \eta)\|^2.
\]
Therefore, since \(\eta > 0\), we establish the lower bound of \(\<u_1-u_2, P(w, u_1, \eta) - P(w, u_2, \eta)\>\):
\[
\<u_1-u_2, P(w, u_1, \eta) - P(w, u_2, \eta)\>\geq \|P(w, u_1, \eta) - P(w, u_2, \eta)\|^2.
\]
From the definition of \(\|\cdot\|_*\), we have \(\langle u, x\rangle\leq \|u\|_*\|x\|\) for all \(x\in\mathcal{X},u\in\mathcal{X}^*\). Therefore, we establish the upper bound of \(\<u_1-u_2, P(w, u_1, \eta) - P(w, u_2, \eta)\>\):
\[
\<u_1-u_2, P(w, u_1, \eta) - P(w, u_2, \eta)\>\leq \|u_1-u_2\|_\ast \|P(w, u_1, \eta) - P(w, u_2, \eta)\|.
\]

\textbf{Proof of the non-expansiveness:}

Combining the lower and upper bounds, and noting that \(x^2 \leq xy\) for some \(x,y\geq 0\) implies \(x \leq y\),
we obtain the desired non-expansiveness of the proximal gradient mapping:
\[
\|P(w, u_1, \eta) - P(w, u_2, \eta)\| \leq \|u_1 - u_2\|_\ast.
\]

\textbf{Proof of the scaling property:}

The scaling property of the proximal gradient mapping follows from its definition, 
and noting that \(\eta \cdot \alpha u = \alpha \eta \cdot u\).

\textbf{Proof of the monotonicity:}

For any \(\eta'>\eta\geq0\), \(u\in \mathcal{X}^\ast, w\in\mathcal{W}\subseteq \mathcal{X}\), 
we introduce \(w^+ := \arg\min_{w'\in\mathcal{W}} \langle\eta u, w'\rangle + D_\Phi(w', w)\),
\(w'^{+} := \arg\min_{w'\in\mathcal{W}} \langle\eta' u, w'\rangle + D_\Phi(w', w)\),
therefore \(w^+ = w-\eta P(w, u, \eta), w^{'+} = w-\eta' P(w, u, \eta')\),
and letting \(w' \gets w'^{+}\) and \(\eta \gets \eta', w^+ \gets w'^{+}, w'\gets w^+\) in the first-order optimality condition:
\[
\eta \langle u, \eta P(w, u, \eta)- \eta' P(w, u, \eta') \rangle \geq -D_\Phi(w'^{+}, w)+D_\Phi(w'^{+}, w^+)+D_\Phi(w^+, w)
\]
\[
\eta' \langle u, \eta' P(w, u, \eta')- \eta P(w, u, \eta) \rangle \geq -D_\Phi(w^+, w)+D_\Phi(w^+, w'^{+})+D_\Phi(w'^{+}, w)
\]
Adding the above two inequalities and dividing by \(\eta'-\eta >0\), then 
applying \(D_\Phi(x, y) + D_\Phi(y, x) \geq \|x-y\|^2\) for all \(x, y\in\mathcal{W}\) and 1-strongly convex mirror map \(\Phi\), we obtain:
\[
\langle u, \eta' P(w, u, \eta')- \eta P(w, u, \eta) \rangle \geq 
\frac{D_{\Phi}(w'^{+}, w^+)+D_{\Phi}(w^+, w'^{+})}{\eta'-\eta} \geq %
 \frac{\|\eta' P(w, u, \eta')-\eta P(w, u, \eta)\|^2}{\eta'-\eta}
\]
By \(\langle u, \eta' P(w, u, \eta')-\eta P(w, u, \eta)\rangle \leq \|u\|_\ast \|\eta' P(w, u, \eta')-\eta P(w, u, \eta)\|\) 
from the definition of the dual norm \(\|\cdot\|_\ast\), we have:
\[
\|\eta' P(w, u, \eta')-\eta P(w, u, \eta)\|\leq (\eta'-\eta) \|u\|_\ast.
\]
\end{proof}

\noindent\textbf{Unified estimator \(v_t\) and estimation error \(e_t\).}
For a target \(g(w)\) and a stochastic oracle \(G(w, \xi)\) with \[\mathbb{E}_{\xi}[G(w, \xi)] = g(w)\]
which is an unbiased estimator of the target function \(g(w)\).
We introduce a unified estimator \(v_t\) to track the target function value \(g(w_t)\) at the \(t\)-th iteration,
by using \(G(w, \xi)\) and the unbiased estimator of gradient \(\nabla G(w, \xi)\) such that \(\E[\nabla G(w, \xi)]=\nabla g(w)\)  to reduce the variance of estimation error.

At each iteration \(t\), let the reset indicator be \(b_t\sim \mathrm{Bernoulli}(p_t)\), 
where \(p_t\in[0,1]\) is the reset probability 
(possibly adaptive and predictable; see Assumption~\ref{ass:predictable_params}). 
When \(b_t=1\), the estimator uses an independent ``batch'' oracle randomness \(\xi_t^B := (\xi_t^{(i)})_{i=1}^B\) from the oracle,
then reset the estimator \(v_t\) to the average of evaluated function values \(G(w_t, \xi_t^{(i)}), i=1,\dots,B\), 
with \(B\) samples \(\xi_t^B:=(\xi_t^{(i)})_{i=1}^B\).
When \(b_t=0\), the estimator uses a single independent sample \(\xi_t\) drawn from the oracle, 
to evaluate function value \(G(w, \xi)\) or gradient value \(\nabla G(w, \xi)\) at \(w_t, w_{t-1}\) with the same sample \(\xi_t\),
and then update the estimator \(v_t\) by using the evaluated function/gradient value and the previous estimator \(v_{t-1}\) 
to track the target function value \(g(w_t)\).

The unified estimator is specified by momentum parameters \(\{\beta_t\}_{t\ge 1}\) with \(\beta_t\in(0,1]\) and by 
a correction term \(\Delta_t\) and 
a correction term \(\mathcal{T}_t\) as follows: 
\begin{equation}\label{eq:unified_est_general_def}
v_t :=
\begin{cases}
G(w_t,\xi_t^B), & b_t=1,\\
G(w_t,\xi_t)+\beta_t \Delta_t + \mathcal{T}_t, & b_t=0,
\end{cases}
\end{equation}
where \(\Delta_t\coloneqq v_{t-1}-G(w_{t-1},\xi_t)\),
\(\mathcal{T}_t\coloneqq \mathcal{T}_t(w_t,w_{t-1},\xi_t)\) 
depends on the evaluated function value \(G(w, \xi)\) and estimated gradient value \(\nabla G(w, \xi)\) at \(w_t, w_{t-1}\) with the same sample \(\xi_t\),
and \(G(w_t, \xi_t^B) \coloneqq \frac{1}{B}\sum_{i=1}^B G(w_t, \xi_t^{(i)})\) 
is the average of evaluated function values \(G(w_t, \xi_t^{(i)}), i=1,\dots,B\). 
At the first iteration \(t=1\), we set \(\Delta_1 = 0, \mathcal{T}_1 = 0\) by definition.

The tracking error \(e_t\) between the unified estimator \(v_t\) and the target function value \(g(w_t)\) is given by
\[
e_t\coloneqq v_t-g(w_t).
\]

\noindent\textbf{Filtration, resets/epochs, and cumulative momentum multipliers.}
For high-probability analysis, let \(\{\mathcal{F}_t\}_{t\ge 0}\) be the filtration generated by all oracle noise and reset events up to time \(t\).
\[
\mathcal{F}_0 = \Omega,\quad
\mathcal{F}_t = 
    \begin{cases} 
        \sigma(\mathcal{F}_{t-1}, b_t, \xi_t\big|_{b_t=0}, \xi_t^B\big|_{b_t=1}) & \text{for Application 1~\eqref{eq:application1}}, \\
        \sigma(\mathcal{F}_{t-1}, b_t, \xi_t\big|_{b_t=0}, \xi_t^B\big|_{b_t=1}, \zeta_t) & \text{for Application 2~\eqref{eq:application2}}.
    \end{cases}.
\]
For the Application 2, \(\zeta_t\) is independent of \(\mathcal{F}_{t-1}, b_t, \xi_t^B\mid_{b_t=1}, \xi_t\mid_{b_t=0}\).

Let the random sequence of epoch boundaries \(\{\tau_m\}_{m\ge 0}\) for \(m\)-th reset time be defined by
\[
\tau_0=0,\qquad \tau_m=\inf\{t>\tau_{m-1}: b_t=1\}\quad (1\le m\le T),
\]
where \(T\) is a deterministic horizon. For each \(t\in\{1,\dots,T\}\), let \(m(t)\) be the unique index such that \(t\in(\tau_{m(t)-1},\tau_{m(t)}]\).

For the momentum parameters \(\{\beta_t\}_{t\ge 1}\) with \(\beta_t\in(0,1]\), 
we define the cumulative momentum multipliers for \(j=0,\dots,t-1\) as follows:
\[
\Lambda_{t,j}\coloneqq \prod_{i=j+1}^{t}\beta_i,\qquad \Lambda_{t,t}=1.
\]
We also introduce \(A_t\coloneqq \prod_{i=1}^{t}\beta_i^{-1}\) and its inverse \(A_t^{-1}=\prod_{i=1}^{t}\beta_i\) to simplify the notations.

\textbf{Stopping time \(\tau_m\)}

To verify that the recursively defined reset sequence \(\tau_m\) is a valid \textbf{stopping time}, we must show that \(\{ \tau_m \le t \} \in \mathcal{F}_t\) for all \(t\). We can prove this by introducing a counting process \(N_t := \sum_{s=1}^t b_s\), which tracks the total number of resets that have occurred up to time \(t\). Because the filtration is increasing (\(\mathcal{F}_s \subseteq \mathcal{F}_t\) for \(s \le t\)) and each \(b_s\) is \(\mathcal{F}_s\)-measurable, the sum \(N_t \equiv \sum_{s=1}^t b_s\) is \(\mathcal{F}_t\)-measurable.

By our definition of \(\tau_m\), the \(m\)-th reset occurs at or before time \(t\) \textbf{if and only if} there have been at least \(m\) resets by time \(t\). This yields the exact equivalence: \[
\{ \tau_m \le t \} = \{ N_t \ge m \}
\]

Because \(N_t\) is \(\mathcal{F}_t\)-measurable, the event \(\{ N_t \ge m \}\) belongs to the \(\sigma\)-algebra \(\mathcal{F}_t\). Therefore, \(\{ \tau_m \le t \} \in \mathcal{F}_t\), confirming that \(\tau_m\) is a valid stopping time.

\textbf{Epoch index \(m(t)\)}

Furthermore, this counting process allows us to explicitly define the epoch index \(m(t)\) for any time \(t \ge 1\) purely in terms of the sequence \(b_t\). Because \(N_{t-1}\) represents the exact number of resets that have occurred strictly prior to time \(t\), we define: \[
m(t) := N_{t-1} + 1 = 1 + \sum_{s=1}^{t-1} b_s
\] Under this definition, the most recent reset prior to \(t\) occurred at \(\tau_{m(t)-1}\), meaning \(\tau_{m(t)-1} < t\). The current epoch's reset will occur at \(\tau_{m(t)}\), meaning \(t \le \tau_{m(t)}\). This explicitly satisfies the necessary interval condition \(t \in (\tau_{m(t)-1}, \tau_{m(t)}]\). Because \(m(t)\) depends only on the reset variables \(b_s\) up to time \(t-1\), it is \(\mathcal{F}_{t-1}\)-measurable, rendering \(m(t)\) a predictable process.

\noindent\textbf{Oracle noise variance proxy.}
We also use the oracle noise variance proxy \(\sigma^2>0\) and a (positive) reset/batch parameter \(B\) in \(\xi_t^B\) and \(\xi_t\), see also Assumption~\ref{ass:innovation_regularity}.
We also introduce the conditional bias magnitude \(\mathcal{B}_t\) and variance proxy \(\Sigma_t^2\)
for the effective innovation \(\mathcal{I}_t(\mathcal{T})\) (see Definition~\ref{def:effective_innovation}) as shown in Assumption~\ref{ass:innovation_regularity} to simplify the notations:
\[
\mathcal{B}_t \geq \big\|\mathbb{E}[\mathcal{I}_t(\mathcal{T})\mid \mathcal{F}_{t-1}]\big\|_\calG,\quad
\mathbb{E}\left[\exp\left(\frac{\|\mathcal{I}_t(\mathcal{T})-\mathbb{E}[\mathcal{I}_t(\mathcal{T})\mid \mathcal{F}_{t-1}]\|_\calG^2}{\Sigma_t^2}\right)\middle|\mathcal{F}_{t-1}\right] \leq 2.
\]

For the accumulated predictable variance proxy on the current epoch, we introduce the notation:
\[
\mathfrak{V}_t^2 \coloneqq \sum_{j=\tau_{m(t)-1}+1}^{t} A_j^2 \Sigma_j^2,
\]
which is bounded by a deterministic variance budget \(V_t\ge 0\) 
such that \(\mathfrak{V}_t^2 \le V_t\) as shown in Theorem~\ref{thm:freedman} and Main Theorem~\ref{thm:unified}.
\newpage
\section{Proofs for Results of Unified Estimation Error Bound}\label{apdx:proof_main}
\subsection{Lemma of Masked Martingale Difference Sequence}\label{subsec:masked}
\begin{lemma}[Masked Martingale Difference Sequence] \label{lemma:masked_mds}
    Let \((\Omega, \mathcal{F}, \mathbb{P})\) be a probability space equipped with a filtration \((\mathcal{F}_n)_{n \ge 1}\). Let \((X_n)_{n \ge 1}\) be a martingale difference sequence adapted to \((\mathcal{F}_n)\), satisfying \(\mathbb{E}[|X_n|] < \infty\) and \(\mathbb{E}[X_n \mid \mathcal{F}_{n-1}] = 0\) almost surely. 
    Let \(S\) and \(T\) be two stopping times with respect to \((\mathcal{F}_n)\) such that \(S \le T\) almost surely. Define the masked sequence \((Y_n)_{n \ge 1}\) as
    \begin{equation}
        Y_n \coloneqq X_n \cdot \mathbb{I}_{\{S < n \le T\}}.
    \end{equation}
    Then, \((Y_n)_{n \ge 1}\) is a martingale difference sequence with respect to \((\mathcal{F}_n)\). Consequently, for any \(N \ge 1\), the sum over the random interval corresponds to the sum of the masked sequence
    \begin{equation}
        \sum_{n=S+1}^{T} X_n = \sum_{n=1}^{\infty} Y_n.
    \end{equation}
\end{lemma}
\begin{proof}
    To establish that \((Y_n)_{n \ge 1}\) constitutes a martingale difference sequence, we must demonstrate that \(Y_n\) is integrable, adapted to \(\mathcal{F}_n\), and satisfies the zero conditional mean property. We begin by examining the indicator process \(H_n \coloneqq \mathbb{I}_{\{S < n \le T\}}\). The event \(\{S < n \le T\}\) can be decomposed as the intersection \(\{S \le n-1\} \cap \{T \le n-1\}^c\). By the definition of a stopping time, the event \(\{S \le n-1\}\) is \(\mathcal{F}_{n-1}\)-measurable. Similarly, \(\{T \le n-1\} \in \mathcal{F}_{n-1}\), which implies its complement is also in \(\mathcal{F}_{n-1}\). Since \(\mathcal{F}_{n-1}\) is a \(\sigma\)-algebra, the intersection of these events is \(\mathcal{F}_{n-1}\)-measurable.
    This establishes that the process \((H_n)_{n \ge 1}\) is predictable. Since \(\mathcal{F}_{n-1} \subseteq \mathcal{F}_n\), the variable \(H_n\) is also \(\mathcal{F}_n\)-measurable. As \(X_n\) is adapted to \(\mathcal{F}_n\), the product \(Y_n = X_n H_n\) is \(\mathcal{F}_n\)-measurable and integrable, satisfying \(\mathbb{E}[|Y_n|] \le \mathbb{E}[|X_n|] < \infty\).
    We proceed to compute the conditional expectation of \(Y_n\) given the filtration \(\mathcal{F}_{n-1}\)
    \begin{equation}
        \mathbb{E}[Y_n \mid \mathcal{F}_{n-1}] = \mathbb{E}[X_n \cdot H_n \mid \mathcal{F}_{n-1}].
    \end{equation}
    Because \(H_n\) is \(\mathcal{F}_{n-1}\)-measurable, it acts as a constant with respect to the conditional expectation and can be factored out. This yields
    \begin{equation}
        \mathbb{E}[Y_n \mid \mathcal{F}_{n-1}] = H_n \cdot \mathbb{E}[X_n \mid \mathcal{F}_{n-1}].
    \end{equation}
    By hypothesis, \((X_n)\) is a martingale difference sequence, so \(\mathbb{E}[X_n \mid \mathcal{F}_{n-1}] = 0\) almost surely. Therefore, \(\mathbb{E}[Y_n \mid \mathcal{F}_{n-1}] = 0\). This confirms that \((Y_n)\) is a martingale difference sequence. The equality of the summations follows directly from the definition of the indicator function, which takes the value \(1\)  when the index \(n\) falls between the stopping times \(S\) and \(T\), and \(0\) otherwise.
\end{proof}

\newpage
\subsection{Theorem of Vector-Valued Freedman Inequality}\label{subsec:freed}
\begin{theorem}[Vector-Valued Freedman Inequality]\label{thm:freedman}
Let \((\calG, \|\cdot\|_\calG)\) be a \(\kappa\)-smooth normed space. Let \((\Omega, \mathcal{F}, \{\mathcal{F}_t\}_{t \ge 0}, \mathbb{P})\) be a filtered probability space, and let \(\{Y_t\}_{t \ge 1}\) be an \(\mathcal{G}\)-valued martingale difference sequence adapted to the filtration \(\{\mathcal{F}_t\}_{t \ge 1}\). Assume there exists a sequence of non-negative predictable random variables \(\{\Sigma_t\}_{t \ge 1}\), where each \(\Sigma_t\) is \(\mathcal{F}_{t-1}\)-measurable. Assume that on the \(\mathcal{F}_{t-1}\)-measurable event \(\{\Sigma_t = 0\}\), \(Y_t = 0\) almost surely.\footnote{We adopt the analytic convention that \(0/0 = 0\); such a requirement is essential for our ensuing assumptions on sub-Gaussian notions of Lipschitzness where the variance proxy includes terms such as \(\|w-w'\|\) in its denominator.} Assume that the conditional light-tail bound evaluates to
\begin{equation}\label{eq:subg-def}
    \mathbb{E}\left[ \exp\left( \frac{\|Y_t\|_\calG^2}{\Sigma_t^2} \right) \middle| \mathcal{F}_{t-1} \right] \le 2,
\end{equation}
almost surely for all \(t \ge 1\). Define the accumulated predictable variance proxy as \(\mathfrak{V}_t^2 = \sum_{j=1}^t \Sigma_j^2\), and let \(M_t = \sum_{j=1}^t Y_j\) denote the martingale. For any deterministic variance budget \(V > 0\) and any deviation parameter \(\gamma \ge 0\), the following high-probability bound holds for all \(t \ge 1\):
\begin{equation}
    \mathbb{P}\left( \|M_t\|_\calG \ge \left(\sqrt{\kappa} + \gamma\right)\sqrt{V} \quad \text{and} \quad \mathfrak{V}_t^2 \le V \right) \le \exp\left(-\frac{\gamma^2}{3}\right).
\end{equation}
Hence, for any deterministic variance budget \(V > 0\) and any target confidence level \(\delta \in (0, 1)\), the following concentration bound holds for all \(t \ge 1\):
\begin{equation}
    \mathbb{P}\left( \|M_t\|_\calG \le C(\delta,\kappa)\sqrt{V} \right) \ge 1 - \delta - \mathbb{P}\left(\mathfrak{V}_t^2 > V\right),
\end{equation}
where \(C(\delta,\kappa):=\sqrt{\kappa} + \sqrt{3 \ln\frac{1}{\delta}}\).
\end{theorem}
\noindent\textbf{Geometric Implications. } Let us first remark. Note that the \((2, D)\)-smoothness of Pinelis \cite{pinelis1994optimum} and the \(\kappa\)-regularity of Juditsky and Nemirovski \cite{juditsky2008large} are connected via \(\kappa = D^2\) (see the definition of \(\kappa\)-smoothness in Appendix~\ref{apdx:prelim}).  Specifically, for a normed space \((\calG, \|\cdot\|_\calG)\), \cite{juditsky2008large} requires the existence of a proxy function \(V(\cdot)\) with a \(\kappa\)-Lipschitz gradient, whereas \cite{pinelis1994optimum} defines smoothness through the inequality \(\|x+y\|_\calG^2 + \|x-y\|_\calG^2 \le 2\|x\|_\calG^2 + 2D^2\|y\|_\calG^2\). For the Lebesgue space \(L_q\) with \(q \ge 2\), both frameworks converge on the optimal regularity parameter \(\kappa = q-1\). Furthermore, for Hilbert spaces (e.g., \(d\)-dimensional Euclidean space), we have \(\kappa = D = 1\). Further discussion is deferred to the Appendix~\ref{apdx:prelim}.

\begin{proof}
Following standard concentration techniques for martingales with predictable variance sequences, e.g. \cite{pinelis1994optimum}, we define the stopping time \(\tau\) with respect to the filtration \(\{\mathcal{F}_t\}_{t \ge 0}\) as
\begin{equation}\label{eq:stopping-time}
    \tau = \inf \left\{ k \ge 0 : \sum_{j=1}^{k+1} \Sigma_j^2 > V \right\},
\end{equation}
where \(\inf \emptyset = \infty\) (indicating the ideal case). The equivalence \(\{\tau \le t-1\} = \left\{ \sum_{j=1}^t \Sigma_j^2 > V \right\}\) guarantees the event is \(\mathcal{F}_{t-1}\)-measurable, ensuring \(\tau\) is a valid stopping time. We subsequently analyze the stopped martingale sequence \(M_{t \wedge \tau}\), where \(t \wedge \tau\) is a shorthand for \(\min(t, \tau)\).

We now follow the approach of \cite{juditsky2008large} to construct a scalar supermartingale utilizing a proxy function mapping derived from the dual geometry of \(\calG\). Let \(\|\cdot\|_{\calG^*}\) denote the dual norm. We define the base proxy function \(V: \calG \to \mathbb{R}\) via the Legendre transform
\begin{equation}\label{eq:proxy-function}
    V(z) = \sup_{\|\eta\|_{\calG^*} \le 1} \left[ \langle \eta, z \rangle - \frac{\kappa}{2}\|\eta\|_{\calG^*}^2 \right].
\end{equation}
Because the space \(\calG\) is \(\kappa\)-smooth, the dual penalty \(v(\eta) = \frac{\kappa}{2}\|\eta\|_{\calG^*}^2\) is globally \(1\)-strongly convex (Page 22 of~\cite{juditsky2008large}). Standard convex analysis (see our detailed standard analysis for Equations~\ref{eq:proxy-function},\ref{eq:proxy-smoothness},\ref{eq:primal-geometric-offset},\ref{eq:primal-geometric-scaled} in Remark~\ref{remark:convex_analysis} and Page 22 of~\cite{juditsky2008large}) establishes that \(V(\cdot)\) is continuously differentiable, its gradient satisfies \(\|V'(z)\|_{\calG^*} \le 1\), and \(V'\) is globally \(1\)-Lipschitz continuous, ensuring
\begin{equation}\label{eq:proxy-smoothness}
    V(z + \Delta) \le V(z) + \langle V'(z), \Delta \rangle + \frac{1}{2}\|\Delta\|_\calG^2.
\end{equation}
Further, the primal geometric offset is bounded by
\begin{equation}\label{eq:primal-geometric-offset}
    \|z\|_\calG = \sup_{\|\eta\|_{\calG^*} \le 1} \langle \eta, z \rangle \le \frac{\kappa}{2} + V(z).
\end{equation}
Following the proof of Theorem 4.1 in \cite{juditsky2008large}, for any fixed scalar \(\beta > 0\) (to be optimized later), the scaled proxy function \(V_\beta(z) = \beta V(z/\beta)\) preserves the gradient bound and modifies the spatial constraints to
\begin{equation}\label{eq:primal-geometric-scaled}
    \|z\|_\calG \le \frac{\kappa \beta}{2} + V_\beta(z) \quad \text{and} \quad V_\beta(z + \Delta) \le V_\beta(z) + \langle V_\beta'(z), \Delta \rangle + \frac{1}{2\beta}\|\Delta\|_\calG^2.
\end{equation}

By defining the scalar increments \(\psi_i = V_\beta(M_i) - V_\beta(M_{i-1})\), the condition \(\|V_\beta'\|_{\calG^*} \le 1\) (i.e. \(V_\beta\) having a Lipschitz constant of 1) restricts the sequence to \(|\psi_i| \le \|Y_i\|_\calG\). The Lipschitz continuity of \(V_\beta'\) with constant \((1/\beta)\) and \(M_i = M_{i-1} + Y_i\) yields
\begin{equation}
\psi_i = V_\beta(M_i) - V_\beta(M_{i-1}) \le \langle V_\beta'(M_{i-1}), Y_i \rangle + \frac{1}{2\beta} \|Y_i\|_\calG^2,
\end{equation}
such that after taking the expectation and noting \(Y_i\) is a martingale difference
\begin{equation}
    \mathbb{E}[\psi_i \mid \mathcal{F}_{i-1}] \le \frac{1}{2\beta} \mathbb{E}[\|Y_i\|_\calG^2 \mid \mathcal{F}_{i-1}].
\end{equation}
Applying the inequality \(1 + x \le \exp(x)\) at \(x = \|Y_i\|_\calG^2/\Sigma_i^2\) provides \(\mathbb{E}[\|Y_i\|_\calG^2 \mid \mathcal{F}_{i-1}] \le \Sigma_i^2\). Consequently, the predictable drift is  bounded by
\begin{equation}\label{eq:drift-bound}
   \mathbb{E}[\psi_i \mid \mathcal{F}_{i-1}] \le \mu_i \coloneqq \frac{1}{2\beta} \Sigma_i^2.
\end{equation}

To bound the exponential moments, the condition \(|\psi_i| \le \|Y_i\|_\calG\) and  Eq. \eqref{eq:subg-def} enforce \(\mathbb{E}[\exp(\psi_i^2 / \Sigma_i^2) \mid \mathcal{F}_{i-1}] \le 2\). Using the inequality \(\exp(s) \le s + \exp(9s^2/16)\) which holds for any \(s \in \mathbb{R}\) and Eq. \eqref{eq:drift-bound}, we evaluate the moment generating function for any \(\lambda \ge 0\)
\begin{equation}
    \mathbb{E}[\exp(\lambda\psi_i) \mid \mathcal{F}_{i-1}] \le \lambda\mu_i + \mathbb{E}\left[\exp\left(\frac{9\lambda^2\psi_i^2}{16}\right) \middle| \mathcal{F}_{i-1}\right].
\end{equation}
As standard in deriving Bernstein-type inequalities or uniform sub-Gaussian bounds, we need to handle two cases to separately deal with convex and concave cases. We first consider the latter, which corresponds to \(0 \le \lambda \le \frac{4}{3\Sigma_i}\). Defining \(\theta = \frac{9\lambda^2\Sigma_i^2}{16} \in [0, 1]\) and applying Jensen's inequality to the  concave function \(z \mapsto z^\theta\) yields \(\mathbb{E}[\exp(\psi_i^2 / \Sigma_i^2)^\theta \mid \mathcal{F}_{i-1}] \le \exp(\theta)\). Because \(\lambda\mu_i\) and \(9\lambda^2\Sigma_i^2/16\) are non-negative, \(x + \exp(y) \le \exp(x+y)\) establishes the majorant \(\exp(\lambda\mu_i + \frac{3}{4}\lambda^2\Sigma_i^2)\). 

For \(\lambda > \frac{4}{3\Sigma_i}\) (the convex case), applying Young's inequality \(\lambda x \le \frac{3}{8}\lambda^2\Sigma_i^2 + \frac{2}{3}\frac{x^2}{\Sigma_i^2}\) at \(x = \psi_i\) gives
\begin{equation}
    \mathbb{E}[\exp(\lambda\psi_i) \mid \mathcal{F}_{i-1}] \le \exp\left(\frac{3}{8}\lambda^2\Sigma_i^2\right) \mathbb{E}\left[\exp\left(\frac{2}{3}\frac{\psi_i^2}{\Sigma_i^2}\right) \middle| \mathcal{F}_{i-1}\right].
\end{equation}
Now we can apply Jensen's inequality for the exponent \(2/3 \in [0, 1]\) to bound the inner expectation by \(\exp(2/3)\). The domain constraint necessitates \(\frac{3}{8}\lambda^2\Sigma_i^2 > \frac{2}{3}\), ensuring \(\exp(2/3) \le \exp(\frac{3}{8}\lambda^2\Sigma_i^2)\). Thus, the uniform bound 
\begin{equation}\label{eq:uniform-bound}
    \mathbb{E}[\exp(\lambda\psi_i) \mid \mathcal{F}_{i-1}] \le \exp(\lambda\mu_i + \frac{3}{4}\lambda^2\Sigma_i^2)
\end{equation}
holds globally for all \(\lambda\). Furthermore, integrating the stopping time, we define the masked variables 
\begin{equation}\label{eq:masked-variables}
    \tilde{\psi}_i = \psi_i \mathbb{I}_{\{i \le \tau\}}, \quad \tilde{\mu}_i = \mu_i \mathbb{I}_{\{i \le \tau\}}, \quad \tilde{\Sigma}_i^2 = \Sigma_i^2 \mathbb{I}_{\{i \le \tau\}}.
\end{equation}
Since \(\{i \le \tau\}\) is \(\mathcal{F}_{i-1}\)-measurable, the conditional expectation bound in Eq. \eqref{eq:uniform-bound} is preserved , i.e.\footnote{To see this, note that on event \(\{t \le \tau\}\) the indicator evaluates to \(I_t = 1\) and \eqref{eq:uniform-bound-masked} reduces to \eqref{eq:uniform-bound} while on event \(\{t > \tau\}\) both sides equate to 1 as the indicator evaluates to \(I_t = 0\).}
\begin{equation}\label{eq:uniform-bound-masked}
    \mathbb{E}[\exp(\lambda\tilde{\psi}_t) \mid \mathcal{F}_{t-1}] \le \exp(\lambda\tilde{\mu}_t + \frac{3}{4}\lambda^2\tilde{\Sigma}_t^2).
\end{equation}
To construct the unconditional bound, we define the variance-penalized process
\begin{equation}
    Z_t = \exp\left( \lambda \sum_{i=1}^t \tilde{\psi}_i - \lambda \sum_{i=1}^t \tilde{\mu}_i - \frac{3}{4}\lambda^2 \sum_{i=1}^t \tilde{\Sigma}_i^2 \right).
\end{equation}
To verify that \(Z_t\) is a supermartingale adapted to \(\mathcal{F}_t\), we isolate the terminal increment following an approach similar to Lemma 20.2 in \cite{lattimore2020bandit}, i.e.
\begin{equation}
    Z_t = Z_{t-1} \exp\left( \lambda \tilde{\psi}_t - \lambda \tilde{\mu}_t - \frac{3}{4}\lambda^2 \tilde{\Sigma}_t^2 \right).
\end{equation}
Because \(\tau\) is a predictable stopping time, \(\tilde{\mu}_t\) and \(\tilde{\Sigma}_t^2\) are  \(\mathcal{F}_{t-1}\)-measurable. Factoring these out of the conditional expectation yields
\begin{equation}
    \mathbb{E}[Z_t \mid \mathcal{F}_{t-1}] = Z_{t-1} \exp\left( - \lambda \tilde{\mu}_t - \frac{3}{4}\lambda^2 \tilde{\Sigma}_t^2 \right) \mathbb{E}\left[\exp(\lambda \tilde{\psi}_t) \middle| \mathcal{F}_{t-1}\right].
\end{equation}
Substituting the bound from Eq. \eqref{eq:uniform-bound-masked} cancels the variance penalty terms, establishing \(\mathbb{E}[Z_t \mid \mathcal{F}_{t-1}] \le Z_{t-1}\). Thus, by tower expectation, \(\mathbb{E}[\mathbb{E}[Z_k \mid \mathcal{F}_{k-1}]] \le \mathbb{E}[Z_{k-1}]\), and doing so iteratively yields \(\mathbb{E}[Z_t] \le \mathbb{E}[Z_0] = 1\).

By definitions in Eq. \eqref{eq:stopping-time} and Eq. \eqref{eq:masked-variables}, the cumulative constraints \(\sum_{i=1}^t \tilde{\Sigma}_i^2 \le V\) and \(\sum_{i=1}^t \tilde{\mu}_i \le  \frac{1}{2\beta} V\) hold deterministically. On the event \(\sum_{i=1}^t \tilde{\psi}_i \ge \frac{1}{2\beta} V + r\), the supermartingale ensures \(Z_t \ge \exp(\lambda r - \frac{3}{4}\lambda^2V)\). Markov's inequality yields
\begin{equation}\label{eq:deviation}
    \mathbb{P}\left(\sum_{i=1}^t \tilde{\psi}_i \ge \frac{1}{2\beta} V + r\right) \le \exp\left(-\lambda r + \frac{3}{4}\lambda^2V\right).
\end{equation}
Minimizing the quadratic exponent at \(\lambda = \frac{2r}{3V}\) produces the optimal limit \(\exp(-\frac{r^2}{3V})\).

On the intersection event where \(\mathfrak{V}_t^2 \le V\), we have \(\tau \ge t\), and thus \(\sum_{i=1}^t \tilde{\psi}_i = \sum_{i=1}^t \psi_i\).

Thus, using the definition \(\psi_i = V_\beta(M_i) - V_\beta(M_{i-1})\) and a telescoping sum we have \(V_\beta(M_t) \equiv \sum_{i=1}^t \tilde{\psi}_i\). Utilizing the first part of~\eqref{eq:primal-geometric-scaled}, the norm satisfies
\begin{equation}
    \|M_t\|_\calG \le \frac{\kappa \beta}{2} + \sum_{i=1}^t \tilde{\psi}_i \le \frac{\kappa \beta}{2} + \frac{1}{2\beta} V + r.
\end{equation}
Optimizing with respect to \(\beta\) via differentiation yields the unique minimum \(\beta = \sqrt{\frac{V}{\kappa}}\). Substituting this parameter establishes
\begin{equation}
    \frac{\kappa}{2} \sqrt{\frac{V}{\kappa}} + \frac{1}{2} \sqrt{\kappa V} = \sqrt{\kappa V}.
\end{equation}
Parameterizing the deviation in Eq. \eqref{eq:deviation} as \(r = \gamma \sqrt{V}\), the deviation simplifies to \((\sqrt{\kappa} + \gamma)\sqrt{V}\), and the corresponding tail probability evaluates to \(\exp(-\gamma^2/3)\), thereby furnishing the proof.
\end{proof}

\begin{remark}[Detailed Standard Convex Analysis]\label{remark:convex_analysis}
Below, we provide the detailed, step-by-step derivation of Equations~\ref{eq:proxy-function},\ref{eq:proxy-smoothness},\ref{eq:primal-geometric-offset},\ref{eq:primal-geometric-scaled} for the proof of Theorem B.2 (Vector-Valued Freedman Inequality) in Appendix B.2. This demonstrates that the result follows purely from standard convex analysis without any additional regularity conditions. 

For a \(\kappa\)-smooth normed space \((\mathcal{G}, \|\cdot\|_{\mathcal{G}})\), its dual space \((\mathcal{G}^\ast, \|\cdot\|_{\mathcal{G}^\ast})\), and their canonical duality pairing \(\langle \cdot, \cdot \rangle: \mathcal{G}^\ast\times \mathcal{G} \to \mathbb{R}\), it is well known that the convex conjugate of the convex function \(\phi(z) := \frac{1}{2 \kappa}\|z\|_{\mathcal{G}}^2\) is \(v(\eta) := \phi^\ast(\eta) \equiv \sup_{z\in \mathcal{G}} \langle \eta, z \rangle - \phi(z) = \frac{\kappa}{2} \|\eta\|_{\mathcal{G}^\ast}^2\) no matter whether \(\mathcal{G}\) is finite-dimensional or infinite-dimensional, without requiring any extra regularity conditions. 

Since \(\mathcal{G}\) is \(\kappa\)-smooth (Def.~\ref{def:kappa_smoothness}), \(\phi(z)\) is continuously differentiable (\(\eta = \nabla \phi(z) \Leftrightarrow \eta \in \partial \phi(z) \equiv\{ \nabla \phi(z)\}\)), and for any \(z, \Delta \in \mathcal{G}\): \[
\phi(z+\Delta) \leq \phi(z) + \langle \nabla \phi(z), \Delta \rangle + \frac{1}{2} \|\Delta\|_{\mathcal{G}}^2
\] We now derive Eqs (15)-(18), showing \(v(\eta) \equiv \phi^\ast(\eta)\) is \emph{1-strongly convex} and the proxy \(V(z)\), defined as \textbf{(Eq.~\ref{eq:proxy-function})}: \[
V(z):= \sup_{\eta \in \mathcal{B}^\ast} \langle \eta, z \rangle - v(\eta),\quad 
\mathcal{B}^\ast := \{\eta\in \mathcal{G}^\ast : \|\eta\|_{\mathcal{G}^\ast} \leq 1 \}
\] is continuously differentiable (\(\partial V(z) = \{V'(z) \}\)) and \(\|V'(z)\|_{\mathcal{G}^\ast}\leq 1\), with dual unit ball \(\mathcal{B}^\ast\).

\textbf{1-Strong Convexity of \(v(\cdot)\)}

For base point \(z\in \mathcal{G}\), let \(\eta := \nabla \phi(z) \in \mathcal{G}^\ast\), and introduce concave \(\psi_\eta(\tilde{z}):= \langle \eta, \tilde{z} \rangle - \phi(\tilde{z})\) for test point \(\tilde{z}\in \mathcal{G}\). Its gradient \(\nabla \psi_\eta(\tilde{z}) = \eta - \nabla \phi(\tilde{z})\) is \(0\) at \(\tilde{z} = z\). Hence, the \emph{Fenchel-Young} equality holds: \[
\psi_\eta(\tilde{z}) \leq \psi_\eta(z) + \langle \psi_\eta(z), \tilde{z} - z\rangle = \psi_\eta(z)
\implies
v(\eta) \equiv \sup_{\tilde{z}\in \mathcal{G}} \psi_\eta(\tilde{z})
= \psi_\eta(z) \equiv \langle \eta, z\rangle - \phi(z)
\] For test point \(\tilde{\eta}\in \mathcal{G}^\ast\), \(v(\tilde{\eta})\equiv \sup_{\tilde{z} \in \mathcal{G}} \langle \tilde{\eta}, \tilde{z} \rangle - \phi(\tilde{z}) = \sup_{\Delta \in \mathcal{G}} \langle \tilde{\eta}, z+\Delta \rangle - \phi(z+\Delta)\) letting \(\Delta = \tilde{z} - z\). By \(\kappa\)-smoothness, \(\phi(z+\Delta) \leq \phi(z) + \langle \eta, \Delta \rangle + \frac{1}{2}\|\Delta\|_{\mathcal{G}}^2\). Using \(v(\eta) = \langle \eta, z\rangle - \phi(z)\) and \(\frac{1}{2}\|\tilde{\eta}-\eta\|_{\mathcal{G}^\ast}^2 = \sup_{\Delta \in \mathcal{G}} \langle \tilde{\eta}-\eta, \Delta \rangle - \frac{1}{2}\|\Delta\|_{\mathcal{G}}^2\): \[
v(\tilde{\eta}) \geq \sup_{\Delta \in \mathcal{G}} \langle \tilde{\eta}, z+\Delta \rangle -  \phi(z) - \langle \eta, \Delta \rangle - \frac{1}{2}\|\Delta\|_{\mathcal{G}}^2
= v(\eta) + \langle \tilde{\eta}-\eta, z\rangle + \frac{1}{2}\|\tilde{\eta}-\eta\|_{\mathcal{G}^\ast}^2
\] Thus \(v(\tilde{\eta}) \geq v(\eta) + \langle \tilde{\eta} - \eta, z \rangle\), implying \(z \in \partial v(\eta)\). Repeating this for any element in \(\partial v(\eta)\) yields the same inequality, proving \(v(\cdot)\) is 1-strongly convex.

\textbf{\(V(z)\) is Continuously Differentiable and \(\|V'(z)\|_{\mathcal{G}^\ast}\leq 1\)}

Since \(\mathcal{B}^\ast \equiv \{\eta\in \mathcal{G}^\ast : \|\eta\|_{\mathcal{G}^\ast} \leq 1 \}\) is compact under the weak-\(^\ast\) topology (Banach-Alaoglu theorem) and \(v(\eta)\) is lower semicontinuous, the function \(\langle \eta, z \rangle - v(\eta)\) is upper semicontinuous under the weak-\(^\ast\) topology, which guarantees the existence of \(\arg\max_{\eta \in \mathcal{B}^\ast} \{ \langle \eta, z \rangle - v(\eta) \}\) (Generalized Weierstrass theorem). Since \(v(\cdot)\) is \(1\)-strongly convex, \(\langle \eta, z \rangle - v(\eta)\) is strictly concave over convex \(\mathcal{B}^\ast\), ensuring a unique maximizer \(\eta_z\): \[
\eta_z := \arg\max_{\eta \in \mathcal{B}^\ast} \{ \langle \eta, z \rangle - v(\eta) \} \in \mathcal{B}^\ast
\] The \emph{Fenchel-Young} relation implies \(V(\tilde{z})\geq V(z) + \langle \eta_z, \tilde{z}-z\rangle\), namely: \[
V(z) = \sup_{\eta \in \mathcal{B}^\ast} \langle \eta, z \rangle - v(\eta) 
= \langle \eta_z, z\rangle - v(\eta_z) \implies \eta_z \in \partial V(z)
\]

\textbf{\(z\mapsto \eta_z\) is 1-Lipschitz}

\(\Psi_z(\eta) := \langle \eta, z \rangle - v(\eta)\) is \(1\)-strongly concave with unique maximizer \(\eta_z\) in convex \(\mathcal{B}^\ast\). For \(\alpha\in(0, 1)\) and test point \(\tilde{\eta} \in \mathcal{B}^\ast\): \[
\Psi_z(\eta_z) \geq \Psi_z(\alpha\eta_z+(1-\alpha)\tilde{\eta})
\geq \alpha \Psi_z(\eta_z) + (1-\alpha)\Psi_z(\tilde{\eta}) + \frac{\alpha(1-\alpha)}{2} \|\eta_z - \tilde{\eta} \|_{\mathcal{G}^\ast}^2
\] Rearranging, dividing by \(1-\alpha\), and taking \(\alpha \to 1\) yields: \[
v(\tilde{\eta}) - v(\eta_z) - \langle \tilde{\eta}-\eta_z, z\rangle \geq \frac{1}{2} \|\eta_z - \tilde{\eta}\|_{\mathcal{G}^\ast}^2
\] Let \(\eta_1 := \eta_{z_1}\) and \(\eta_2 := \eta_{z_2}\). Summing the two inequalities given by \((z \gets z_1, \eta_z \gets \eta_1, \tilde{\eta} \gets \eta_2)\) and \((z \gets z_2, \eta_z \gets \eta_2, \tilde{\eta} \gets \eta_1)\) yields: \[
\|\eta_1 - \eta_2\|_{\mathcal{G}^\ast}^2 \leq \langle \eta_1 - \eta_2, z_1 - z_2 \rangle \leq \|\eta_1 - \eta_2\|_{\mathcal{G}^\ast} \|z_1 - z_2 \|_{\mathcal{G}} \implies \|\eta_1 - \eta_2\|_{\mathcal{G}^\ast} \leq \|z_1 - z_2 \|_{\mathcal{G}}
\]

\textbf{\(\partial V(z) = \{\eta_z\}\) is a singleton}

For \(\eta \in \partial V(z)\) and \(\Delta \in \mathcal{G}\), since \(V(z) \geq \langle \eta_{z+\Delta}, z \rangle - v(\eta_{z+\Delta}) = V(z+\Delta) - \langle \eta_{z+\Delta}, \Delta \rangle\): \[
\langle \eta, \Delta \rangle \le V(z+\Delta) - V(z) \le \langle \eta_{z+\Delta}, \Delta \rangle
\] By the \(1\)-Lipschitz continuity of \(z\mapsto \eta_z\): \[
\langle \eta - \eta_z, \Delta \rangle \le \langle \eta_{z+\Delta} - \eta_z, \Delta \rangle \le \|\Delta\|_{\mathcal{G}}^2
\] Letting \(\Delta \gets -\Delta\) gives \(-\|\Delta\|_{\mathcal{G}}^2 \leq \langle \eta - \eta_z, \Delta \rangle\). For fixed direction \(d \in \mathcal{G}\) and \(t > 0\), substituting \(\Delta = td\): \[
t |\langle \eta - \eta_z, d \rangle| \le t^2 \|d\|_{\mathcal{G}}^2 \implies |\langle \eta - \eta_z, d \rangle| \le \lim_{t\to 0}t \|d\|_{\mathcal{G}}^2 =0
\] forcing \(\langle \eta - \eta_z, d \rangle = 0\) for all \(d\).

Thus \(\eta = \eta_z\), \(\partial V(z) = \{\eta_z\}\), and \(V'(z) = \eta_z\) exists. Since \(\eta_z \in \mathcal{B}^\ast\) and \(z\mapsto \eta_z\) is \(1\)-Lipschitz: \[
\|V'(z)\|_{\mathcal{G}^\ast} = \|\eta_z\|_{\mathcal{G}^\ast} \leq 1,\quad
\|V'(z_1) - V'(z_2)\|_{\mathcal{G}^\ast}\leq \|z_1 - z_2\|_{\mathcal{G}} 
\implies V(z) \text{ is continuously differentiable}
\]

\textbf{Lipschitz Gradient and Smoothness (Eq.~\ref{eq:proxy-smoothness})}

For \(z, \Delta \in \mathcal{G}\) and \(t\geq 0\), via \(1\)-Lipschitz continuity of \(V'(z) \equiv \eta_z\): \[
\langle V'(z + t\Delta) - V'(z), \Delta \rangle \leq t \|\Delta\|_{\mathcal{G}}^2
\] By the Fundamental Theorem of Calculus: \[
V(\xi + \Delta) = V(\xi) + \langle V'(\xi), \Delta \rangle + \int_0^1 \langle V'(\xi + t\Delta) - V'(\xi), \Delta \rangle dt
\le V(\xi) + \langle V'(\xi), \Delta \rangle + \frac{1}{2}\Vert{}\Delta\Vert{}_{\mathcal{G}}^2\]

\textbf{Primal Geometric Offset (Eq.~\ref{eq:primal-geometric-offset})}

Since \(\|z\|_{\mathcal{G}} \equiv \sup_{\eta \in \mathcal{B}^\ast} \langle \eta, z\rangle\) with \(\mathcal{B}^\ast \equiv \{\eta \in \mathcal{G}^\ast : \|\eta\|_{\mathcal{G}^\ast} \leq 1\}\). For \(v(\eta) \equiv \frac{\kappa}{2}\|\eta\|_{\mathcal{G}^\ast}^2\) and \(V(z) \equiv \sup_{\eta \in \mathcal{B}^\ast} \langle \eta, z\rangle - v(\eta)\), applying \(\sup(A + B) \le \sup A + \sup B\) yields: \[
\|z\|_{\mathcal{G}} \le \sup_{\eta \in \mathcal{B}^\ast} \{\langle \eta, z \rangle - v(\eta)\}  + \sup_{\eta \in \mathcal{B}^\ast}  v(\eta)
= V(z) + \frac{\kappa}{2}
\]

\textbf{Scaling via \(\beta\) (Eq.~\ref{eq:primal-geometric-scaled})}

For scale \(\beta > 0\), \(V_\beta(z) := \beta V(z/\beta)\). By the chain rule, \(V_\beta'(z) = V'(z/\beta)\). Since \(\|V'(\cdot)\|_{\mathcal{G}^\ast} \le 1\), \(\|V_\beta'(z)\|_{\mathcal{G}^\ast} \le 1\). Evaluating Eq.~\ref{eq:primal-geometric-offset} at \(z/\beta\) and multiplying by \(\beta\) gives the scaled offset via absolute homogeneity: \[
\|z\|_{\mathcal{G}} = \beta \left\| \frac{z}{\beta} \right\|_{\mathcal{G}} \le \beta \left( V\left(\frac{z}{\beta}\right) + \frac{\kappa}{2} \right) = V_\beta(z) + \frac{\kappa \beta}{2}
\] Evaluating Eq.~\ref{eq:proxy-smoothness} at base \(z/\beta\) with displacement \(\Delta/\beta\), and scaling by \(\beta\) gives: \[
V_\beta(z + \Delta) = \beta V\left(\frac{z}{\beta} + \frac{\Delta}{\beta}\right) 
\le \beta \left( V\left(\frac{z}{\beta}\right) + \left\langle V'\left(\frac{z}{\beta}\right), \frac{\Delta}{\beta} \right\rangle + \frac{1}{2}\left\|\frac{\Delta}{\beta}\right\|_{\mathcal{G}}^2 \right) 
= V_\beta(z) + \langle V_\beta'(z), \Delta \rangle + \frac{1}{2\beta} \|\Delta\|_{\mathcal{G}}^2
\]
\end{remark}

\begin{remark}[Optimization Implications]
\label{remark:opt}
The formulation of Theorem \ref{thm:freedman} exhibits a strict topological decoupling between the martingale geometry and the variance proxy constructions. The \(\kappa\)-smoothness condition is imposed exclusively on the space \((\calG, \|\cdot\|_{\calG})\) governing the martingale difference sequence \(Y_t\). The predictable variance proxy \(\Sigma_t\) operates purely as an \(\mathcal{F}_{t-1}\)-measurable scalar bound. Consequently, if \(\Sigma_t\) is parameterized by auxiliary algorithmic quantities, e.g., \(\|U\|_q\) or \(\|w-w'\|_r\) for arbitrary \(q, r \in [1, \infty]\), the probabilistic concentration mechanism remains completely agnostic to these auxiliary topologies. However, integrating this concentration into the ensuing application of optimization analysis imposes certain structural constraints on the permissible \(\ell_p\) norms, dictated entirely by the nature of the stochastic sequence:
\begin{itemize}[leftmargin=*]
    \item \textbf{Concentration of Gradients (Typical Application):} If the martingale models stochastic gradient residuals, the sequence inherently resides in the topological dual space, \(\calG = \mathcal{X}^*\). To invoke the theorem natively without dimension-dependent distortion, \(\mathcal{X}^*\) must be \(\kappa\)-smooth, which necessitates an \(\ell_p\) norm geometry where \(p \in [2,\infty)\). By strict Hölder conjugacy (\(1/p + 1/q = 1\)), the primal iterate space \(\mathcal{X}\) is formally restricted to an \(\ell_q\) geometry where \(q \in (1, 2]\).
    \item \textbf{Concentration of Function Values (SGM Application):} If the martingale models scalar stochastic noise in the objective evaluations, the sequence resides in \(\calG = \mathbb{R}\). The absolute value norm constitutes a trivially \(1\)-smooth Hilbert geometry (\(\kappa=1\)). The resulting concentration is unconditionally dimension-free. Under this regime, the primal iterate space \(\mathcal{W}\) may be equipped with any arbitrary norm without inducing geometric penalties, as the concentration bound does not interact with the dual pairing of the algorithmic step.
\end{itemize}
\end{remark}

\newpage
\subsection{Proof for Theorem \ref{thm:unified} of Unified Estimation Error Bound}\label{subsec:proof_main}

\begin{proof}[Proof of Theorem \ref{thm:unified}]
For any horizon \(T \ge 2\), we rigorously establish a uniform upper bound for the estimation error \(\|e_t\|_{\mathcal{G}}\) that holds simultaneously for all \(t \in [T]\) with probability at least \(1-\delta\).

\textbf{I. Resets, Stopping Time, and Epoch Index}

\begin{itemize}
\item
  \emph{Filtration:} Let \((\Omega, \mathcal{F}, \{\mathcal{F}_t\}_{t=0}^T, \mathbb{P})\) be a filtered space.
\item
  \emph{Resets:} Let \(b_t \in \{0, 1\}\) be an \(\mathcal{F}_t\)-measurable random variable representing the reset decision at time \(t\) with \(b_0 = 1\). Let \(N_t := \sum_{s=1}^t b_s\) denote the cumulative number of resets up to time \(t\).
\item
  \emph{Stopping Time:} Let \(\tau_m := \inf\{t > \tau_{m-1} : b_t = 1\}\) with \(\tau_0 = 0\). Since \(\{\tau_m \leq t\} \equiv \{N_t \geq m\} \in \mathcal{F}_t\), \(\tau_m\) is a valid stopping time with respect to the filtration.
\item
  \emph{Epoch Index:} For \(t \ge 1\), the predictable (\(\mathcal{F}_{t-1}\)-measurable) active epoch index is \(m(t) := N_{t-1} + 1\). The most recent prior reset is \(\tau_{m(t)-1}\), meaning \(b_j = 0\) for \(j \in (\tau_{m(t)-1}, t-1]\).
\end{itemize}

\textbf{II. Error Dynamics}

The estimation error \(e_t\) (with initial value \(e_0 = G(w_0, \xi_0^B) - g(w_0)\) and \(b_0=1\)) satisfies the recursive dynamic:

\[e_t = b_t E_{B,t} + (1-b_t)(\beta_t e_{t-1} + \mu_t + \mathcal{Z}_t)\]

where:

\begin{itemize}
\item
  \emph{Batch Error:} \(E_{B,t} := G(w_t, \xi_t^B) - g(w_t)\) satisfies the sub-Gaussianity condition of Assumption 3.2, where \(G(w_t, \xi_t^B) := \frac{1}{B}\sum_{i=1}^B G(w_t, \xi_t^{(i)})\) denotes the empirical average over a batch of independent samples \(\xi_t^B := (\xi_t^{(i)})_{i=1}^B\).
\item
  \emph{Decay:} Predictable factors \(\beta_t \in (0, 1]\) and scalars \(\Lambda_{t,j} := \prod_{k=j+1}^t \beta_k = A_t^{-1} A_j\), with \(A_t^{-1} \equiv \prod_{k=1}^t \beta_k\).
\item
  \emph{Innovation:} \(\mathcal{I}_t(\mathcal{T}) \equiv \mu_t + \mathcal{Z}_t\). The predictable bias satisfies \(\|\mu_t\|_{\mathcal{G}} \le \mathcal{B}_t\). The noise \(\mathcal{Z}_t\) is a martingale difference (\(\mathbb{E}[\mathcal{Z}_t \mid \mathcal{F}_{t-1}] = 0\)) with predictable variance proxy \(\Sigma_t^2\).
\end{itemize}

Unrolling \(e_{t-1}\) from \(\tau_{m(t)-1}\) and noting \(b_j = 0\) for \(j \in (\tau_{m(t)-1}, t-1]\) isolates the components:

\[e_t = \underbrace{b_t E_{B,t} + (1-b_t)\Lambda_{t, \tau_{m(t)-1}} E_{B, \tau_{m(t)-1}}}_{\text{Batch Error}} + (1-b_t) \underbrace{\sum_{j=\tau_{m(t)-1}+1}^{t} \Lambda_{t,j} \mu_j}_{\text{Predictable Bias}} + (1-b_t) \underbrace{A_t^{-1} \sum_{j=\tau_{m(t)-1}+1}^{t} A_j \mathcal{Z}_j}_{\text{Stochastic Noise}}\]

Analyzing \(b_t\) isolates the most recent reset at \(\tau_{N_t}\). Summing the disjoint states (\(b_t=1\) where \(\tau_{N_t}=t\); \(b_t=0\) where \(\tau_{N_t}=\tau_{m(t)-1}\)) and defining empty sums as zero (\(\sum_{j=s}^t (\cdot) \equiv 0\) if \(s > t\)) establishes the identity:

\[e_t = \underbrace{\Lambda_{t, \tau_{N_t}} E_{B, \tau_{N_t}}}_{\text{Batch Error}} + \underbrace{\sum_{j=\tau_{N_t}+1}^t \Lambda_{t,j} \mu_j}_{\text{Predictable Bias}} + \underbrace{A_t^{-1} \sum_{j=\tau_{N_t}+1}^t A_j \mathcal{Z}_j}_{\text{Stochastic Noise}}\]

\textbf{(Mapping to Localized Martingales).} For epoch \(m \in \{1, \dots, T\}\), define \(I_j^{(m)} := \mathbb{I}\{\tau_{m-1} < j \le \tau_m\} \equiv \mathbb{I}\{N_{j-1} = m-1\}\). Since \(I_j^{(m)}\) depends only on history up to \(j-1\), it is predictable. Thus, increments \(I_j^{(m)} A_j \mathcal{Z}_j\) form a martingale difference sequence (see the similar argument in Lemma~\ref{lemma:masked_mds}): 
\[
\mathbb{E}\left[I_j^{(m)} A_j \mathcal{Z}_j \mid \mathcal{F}_{j-1}\right] = I_j^{(m)} A_j \mathbb{E}[\mathcal{Z}_j \mid \mathcal{F}_{j-1}] = 0.
\] Hence, \(M_t^{(m)} := \sum_{j=1}^t I_j^{(m)} A_j \mathcal{Z}_j\) is a martingale satisfying \(\mathbb{E}[M_t^{(m)} \mid \mathcal{F}_{t-1}] = M_{t-1}^{(m)}\).

At step \(t\), the accumulated noise corresponds to epoch \(N_t+1 \equiv m(t+1)\). Since \(j \le \tau_{N_t+1}\) holds for \(j \le t\), the sum maps to:

\[\sum_{j=\tau_{N_t}+1}^t A_j \mathcal{Z}_j = \sum_{j=1}^t \mathbb{I}\{\tau_{N_t} < j \le \tau_{N_t+1}\} A_j \mathcal{Z}_j \equiv M_t^{(N_t+1)}\]

Applying the triangle inequality and \(\|\mu_j\|_{\mathcal{G}}\leq \mathcal{B}_j\) gives:

\[
\|e_t\|_{\mathcal{G}} \le \Lambda_{t, \tau_{N_t}} \left\| E_{B, \tau_{N_t}} \right\|_{\mathcal{G}} + \sum_{j=\tau_{N_t}+1}^t \Lambda_{t,j} \mathcal{B}_j + A_t^{-1} \left\| M_t^{(N_t+1)} \right\|_{\mathcal{G}}
\]

\textbf{III. High-Probability Bounds}

\textbf{(Batch Concentration).} Since \(\tau_{N_t} \le t \le T\), any realization of \(\tau_{N_t}\) maps to the deterministic sequence \(\{0, \ldots, T\}\) of size \(T+1\). Allocating local confidence \(\frac{\delta}{2T}\) per evaluation, applying a union bound yields the success event \(\mathcal{E}_B\): 
\begin{eqnarray*}
\mathbb{P}(\mathcal{E}_B^c) &=& \mathbb{P}\left( \exists t \in [T] : \| E_{B, \tau_{N_t}} \|_{\mathcal{G}} > C\left(\frac{\delta}{2T}, \kappa\right) \frac{\sigma}{\sqrt{B}} \right) \\
&\le& \mathbb{P}\left( \bigcup_{k=0}^T \left\{ \| E_{B, k} \|_{\mathcal{G}} > C\left(\frac{\delta}{2T}, \kappa\right) \frac{\sigma}{\sqrt{B}} \right\} \right) 
\le \sum_{k=0}^T \left(\frac{\delta}{2T}\right) = \frac{T+1}{2T}\delta
\end{eqnarray*}

\textbf{(Martingale Configuration).} At each step \(t \in [T]\), the cumulative resets satisfy \(N_t \le t\), meaning \(N_t + 1\) takes possible deterministic values \(m \in \{1, \dots, t+1\}\).

Because \(N_{j-1} \le j-1\), the indicator \(\mathbb{I}\{N_{j-1} = m-1\}\) within \(M_t^{(m)}:= \sum_{j=1}^t \mathbb{I}\{N_{j-1} = m-1\} A_j \mathcal{Z}_j\) requires \(m \le j \le t\) to be non-zero. Consequently, \(M_t^{(m)} \equiv 0\) almost surely for all \(m > t\). Therefore, any realization of the random pair \((N_t+1, t) = (m, t)\) is restricted by the deterministic, upper-triangular grid \(1 \le m \le t \le T\). The exact number of these deterministic pairs is \(\sum_{t=1}^T \sum_{m=1}^t 1 = \sum_{t=1}^T t = \frac{T(T+1)}{2}\).

Allocating local confidence \(\frac{\delta}{2T(T+1)}\) and substituting coordinate \((m, t)\) with envelope \(V_t \geq \sum_{j=\tau_{m-1}+1}^t A_j^2 \Sigma_j^2\) into Theorem~\ref{thm:freedman} of Vector-Valued Freedman Inequality guarantees:

\[\mathbb{P} \left( \|M_t^{(m)}\|_{\mathcal{G}} > C\left(\frac{\delta}{2T(T+1)}, \kappa\right) \sqrt{V_t} \right) \le \frac{\delta}{2T(T+1)}\]

Taking the union bound over this grid envelopes the realized noise, yielding event \(\mathcal{E}_M\):

\begin{eqnarray*}
\mathbb{P}(\mathcal{E}_M^c) &=& \mathbb{P}\left( \exists t \in [T] : \| M_t^{(N_t+1)} \|_{\mathcal{G}} > C\left(\frac{\delta}{2T(T+1)}, \kappa\right) \sqrt{V_t} \right) \\
&\le& \mathbb{P}\left( \bigcup_{t=1}^T \bigcup_{m=1}^t \left\{ \|M_t^{(m)}\|_{\mathcal{G}} > C\left(\frac{\delta}{2T(T+1)}, \kappa\right) \sqrt{V_t} \right\} \right) 
\le \sum_{t=1}^T \sum_{m=1}^t \frac{\delta}{2T(T+1)} = \frac{\delta}{4}
\end{eqnarray*}

\textbf{Global Bound.} The global success event \[
\mathcal{E}^\ast := \mathcal{E}_B \cap \mathcal{E}_M 
= \Bigg\{\forall t \in[T],\space
\left\Vert E_{B, \tau_{N_t}} \right\Vert_{\mathcal{G}} \leq C\left(\frac{\delta}{2T}, \kappa\right) \frac{\sigma}{\sqrt{B}}, \space
\left\Vert M_t^{(N_t+1)} \right\Vert_{\mathcal{G}} \leq C\left(\frac{\delta}{2T(T+1)}, \kappa\right) \sqrt{V_t}
\Bigg\}
\] holds with failure probability bounded by the direct sum for \(T \ge 2\):

\[\mathbb{P}((\mathcal{E}^\ast)^c) \le \frac{T+1}{2T}\delta + \frac{1}{4}\delta = \frac{3T+2}{4T}\delta \le \delta \implies \mathbb{P}(\mathcal{E}^\ast) \ge 1-\delta\]

Conditioned on \(\mathcal{E}^\ast\), this holds for all \(t \in [T]\):

\[\|e_t\|_{\mathcal{G}} \le \underbrace{\Lambda_{t, \tau_{N_t}} C\left(\frac{\delta}{2T}, \kappa\right) \frac{\sigma}{\sqrt{B}}}_{\text{Decayed Batch Error}} + \underbrace{\sum_{j=\tau_{N_t}+1}^t \Lambda_{t,j} \mathcal{B}_j}_{\text{Predictable Bias}} + \underbrace{A_t^{-1} C\left(\frac{\delta}{2T(T+1)}, \kappa\right) \sqrt{V_t}}_{\text{Martingale Noise}}\]

\textbf{Discussion on Bounds}

By the definition \(C(\delta, \kappa) := \sqrt{\kappa} + \sqrt{3 \ln(1/\delta)}\) and observing: \[
\ln \frac{2T(T+1)}{\delta} = \ln\frac{1}{\delta} + \ln(2T) + \ln(T+1) \leq 2 \left(\ln \frac{1}{\delta} + \ln(2T)\right) = 2 \ln\frac{2T}{\delta}
\] we can bound \(C\left(\frac{\delta}{2T(T+1)}, \kappa\right)\) by \(C\left(\frac{\delta}{2T}, \kappa\right)\): \[
C\left(\frac{\delta}{2T}, \kappa\right) \leq C\left(\frac{\delta}{2T(T+1)}, \kappa\right) \leq \sqrt{2} C\left(\frac{\delta}{2T}, \kappa\right)
\]

Thus, the simplified bound for all \(t \in [T]\) with \(T\geq 2\) is:

\[
\|e_t\|_{\mathcal{G}} \le \underbrace{\Lambda_{t, \tau_{N_t}} C\left(\frac{\delta}{2T}, \kappa\right) \frac{\sigma}{\sqrt{B}}}_{\text{Decayed Batch Error}} + \underbrace{\sum_{j=\tau_{N_t}+1}^t \Lambda_{t,j} \mathcal{B}_j}_{\text{Predictable Bias}} + \underbrace{\sqrt{2} A_t^{-1} C\left(\frac{\delta}{2T}, \kappa\right) \sqrt{V_t}}_{\text{Martingale Noise}}
\]
\end{proof}

\begin{remark}[Tightened Upper Bound for Case 1 with \(p_t =0\)]\label{remark:tightened_bound}
When \(p_t = 0\) for all \(t \in [T]\), the reset indicator \(b_t \sim \text{Bernoulli}(p_t)\) satisfies \(b_t = 0\) for all \(t \in [T]\) with a probability of 1. Conditioned on the event \(\{b_t = 0, \forall t \in [T]\}\), the cumulative number of resets \(N_t := \sum_{s=1}^t b_s\) (with \(N_0 = 0\)) and the epoch index \(m(t) := N_{t-1} + 1\) evaluate to: \[
N_t = 0, m(t) = 0 + 1 = 1,\quad \forall t\in [T] \implies \tau_{N_t} = \tau_{N_{t-1}} = \tau_{m(t)-1} = \tau_{0} = 0
\]

By examining the expression for \(e_t\) and substituting \(\tau_{N_t} = \tau_{m(t)-1} = 0\), we obtain: \[
e_t = 
\underbrace{\Lambda_{t, 0} E_{B, 0}}_{\text{Batch Error}} + \underbrace{\sum_{j=1}^t \Lambda_{t,j} \mu_j}_{\text{Predictable Bias}} + \underbrace{A_t^{-1} M_t^{(1)}}_{\text{Stochastic Noise}}
\]

Allocating a confidence of \(\frac{\delta}{2}\) to bound the batch concentration term \(\|E_{B, 0}\|_{\mathcal{G}}\), and a confidence of \(\frac{\delta}{2T}\) to bound each martingale term \(M_t^{(1)}\) for all \(t \in [T]\), the applications of Theorem B.2 (Vector-Valued Freedman Inequality) yield:

\[
\|E_{B, 0}\|_{\mathcal{G}} \leq C\left(\frac{\delta}{2}, \kappa\right) \frac{\sigma}{\sqrt{B}}, \quad
\|M_t^{(1)}\|_{\mathcal{G}} \leq C\left(\frac{\delta}{2T}, \kappa\right) \sqrt{V_t}, \forall t \in [T]
\] with probability at least \(1-\Big(\frac{\delta}{2} + T \times \frac{\delta}{2T} \Big) = 1- \delta\).

Applying these inequalities alongside the bound \(\|\mu_j\|_{\mathcal{G}} \leq \mathcal{B}_j\), we derive a tightened upper bound for the special case of \(p_t = 0\):

\[
\|e_t\|_{\mathcal{G}} 
\le \underbrace{\Lambda_{t, 0} C\left(\frac{\delta}{2}, \kappa\right) \frac{\sigma}{\sqrt{B}}}_{\text{Decayed Batch Error}} + \underbrace{\sum_{j=1}^t \Lambda_{t,j} \mathcal{B}_j}_{\text{Predictable Bias}} + \underbrace{A_t^{-1} C\left(\frac{\delta}{2T}, \kappa\right) \sqrt{V_t}}_{\text{Martingale Noise}}
\]

In this tightened upper bound, the coefficient for the decayed batch error is tightened from \(C\left(\frac{\delta}{2T}, \kappa\right)\) to the smaller \(C\left(\frac{\delta}{2}, \kappa\right)\), and the \(\sqrt{2}\) factor in the martingale noise term is dropped because we are bounding a single martingale sequence within one reset epoch rather than accounting for multiple reset epochs. While the tightened bound is more precise, the original general upper bound remains valid (albeit slightly looser than necessary) for this specific case.
\end{remark}

\newpage
\section{Proofs for Instantiations of the Unified Estimation Error Bound}
\label{apdx:est_instant}

\subsection{Propositions for Three Families with Zeroth / First / Second Order Corrections}
The strength of Theorem~\ref{thm:unified} lies in its ability to recover distinct convergence mechanisms via specific choices of the control sequences \(\{p_t, \beta_t, \eta_t, \mathcal{T}_t\}\). In this section, we consider three special cases corresponding to the three families of estimators we discussed previously.

\begin{proposition}[\textbf{Family 1} with Zeroth-Order Correction]
\label{prop:family1}
Suppose Assumptions~\ref{ass:update_bound},~\ref{ass:predictable_params},~\ref{ass:subgaussian}, %
and \ref{ass:lipschitz} hold. Then, for the zeroth-order family, \(\mathcal{B}_t = \beta_t \eta_{t-1} G L\) and \(\Sigma_t^2 = (1-\beta_t)^2\sigma^2.\) Consequently, Theorem~\ref{thm:unified} yields

\begin{equation}
\label{eq:family1_generic_bound}
\|e_t\|_\calG
\le \Lambda_{t,\tau_{N_t}} C\left(\frac{\delta}{2T},\kappa\right)
\sigma /\sqrt{B}
+
LG\sum_{j=\tau_{N_t}+1}^{t}\Lambda_{t,j}\beta_j\eta_{j-1}
+
\sqrt{2} C\!\left(\frac{\delta}{2T},\kappa\right)A_t^{-1}\sqrt{V_t},
\end{equation}
whenever \(\sum_{j=\tau_{N_t}+1}^{t}A_j^2(1-\beta_j)^2\sigma^2\le V_t\).
\end{proposition}

Proposition~\ref{prop:family1} shows that the zeroth-order family attenuates the oracle noise by the factor \((1-\beta_t)\), but this variance reduction comes at the cost of a tracking bias proportional to the step displacement. This is precisely the classical tradeoff in momentum-style estimation: larger momentum suppresses fresh noise, while simultaneously increasing the lag induced by the moving target.

\begin{proposition}[\textbf{Family 2} with First-Order Correction]
\label{prop:family2}
    Suppose Assumptions~\ref{ass:update_bound},~\ref{ass:predictable_params},~\ref{ass:subgaussian}, 
    and \ref{ass:subg_lipschitz} hold. Then, for the first-order family, \(\mathcal B_t = 0\) and \(\Sigma_t^2 = 2(1-\beta_t)^2\sigma^2 + 2\beta_t^2 \ell^2 \eta_{t-1}^2 G^2.\) Consequently, Theorem~\ref{thm:unified} yields
    \begin{equation}
        \label{eq:family2_generic_bound}
        \|e_t\|_\calG \le \Lambda_{t,\tau_{N_t}} C\left(\frac{\delta}{2T},\kappa\right)  \sigma /\sqrt{B} + \sqrt{2}  C\left(\frac{\delta}{2T},\kappa\right)A_t^{-1}\sqrt{V_t},
    \end{equation}
    whenever \(\sum_{j=\tau_{N_t}+1}^{t} A_j^2\Bigl(2(1 -\beta_j)^2\sigma^2+2\beta_j^2\ell^2\eta_{j-1}^2G^2\Bigr) \le V_t.\)
\end{proposition}

Proposition~\ref{prop:family2} isolates the essential advantage of first-order differential estimators: the conditional bias disappears entirely. The price is that the variance term now depends on the centered stochastic difference, whose scale is proportional to the distance \(\|w_t-w_{t-1}\|\). In other words, this family replaces tracking lag with local difference noise.

\begin{proposition}[\textbf{Family 3} with Second-Order Correction]
    \label{prop:family3}
    Suppose Assumptions~\ref{ass:update_bound},~\ref{ass:predictable_params},~\ref{ass:subgaussian},
    \ref{ass:subg_hessian}, and \ref{ass:smoothness} hold. Then, for the second-order family, \(\mathcal B_t = \frac{\alpha}{2}\beta_t \eta_{t-1}^2 G^2\) and \(\Sigma_t^2 = 2(1-\beta_t)^2\sigma^2 + 2\beta_t^2\gamma^2\eta_{t-1}^2G^2.\) Consequently, Theorem~\ref{thm:unified} yields
    \begin{equation}
        \label{eq:family3_generic_bound}
        \|e_t\|_\calG \le \Lambda_{t,\tau_{N_t}} C\left(\frac{\delta}{2T},\kappa\right)  \sigma /\sqrt{B} + \frac{\alpha G^2}{2}\sum_{j=\tau_{N_t}+1}^{t}\Lambda_{t,j}\beta_j\eta_{j-1}^2 +
        \sqrt{2}  C\left(\frac{\delta}{2T},\kappa\right)A_t^{-1}\sqrt{V_t},
    \end{equation}
    whenever \(\sum_{j=\tau_{N_t}+1}^{t} A_j^2\Bigl(2(1-\beta_j)^2\sigma^2+2\beta_j^2\gamma^2\eta_{j-1}^2G^2\Bigr) \le V_t.\)
\end{proposition}

Proposition~\ref{prop:family3} makes explicit the gain obtained from using curvature information. Relative to the zeroth-order family, the deterministic bias is reduced from first order in the displacement to second order. Relative to the first-order family, the remaining stochastic fluctuation is now tied to Hessian-vector product noise rather than to stochastic differences. Thus, the second-order correction interpolates between low drift and curvature-sensitive variance control.

Before providing the proofs for the propositions of three families, we first state two facts which will be used in the proofs.

\textbf{Fact 1.}~Under Assumption~\ref{ass:update_bound} for \(U_t\) in the mirror descent \(w_{t+1}=w_t -\eta_t P(w_t, U_t, \eta_t)\), then \(\|w_t-w_{t-1}\| \leq \eta_{t-1} G\) is bounded almost surely.

\textbf{Fact 2.}~For the random vectors \(X, Y\in \calG\) that admit variance proxies \(\Sigma_X^2, \Sigma_Y^2\) such that \(\E[\exp(\|X\|_\calG^2/\Sigma_X^2)]\leq 2\) and \(\E[\exp(\|Y\|_\calG^2/\Sigma_Y^2)]\leq 2\), then \(X+Y\) possesses a variance proxy \(2\Sigma_X^2+2\Sigma_Y^2\) such that \(\E[\exp(\|X+Y\|_\calG^2/(2\Sigma_X^2+2\Sigma_Y^2))]\leq 2\).

Fact 1 is established by noting \(\|w_t-w_{t-1}\| = \|\eta_{t-1} P(w_{t-1}, U_{t-1}, \eta_{t-1})\|\leq \eta_{t-1} \|U_{t-1}\|_*\) from  Lemma~\ref{lemma:properties_proximal_gradient_mapping} and \(\|U_{t-1}\|_*\leq G\) from Assumption~\ref{ass:update_bound}.

For fact 2, we establish a deterministic upper bound. For any two arbitrary, potentially dependent random vectors \(X\) and \(Y\) admitting variance proxies \(\Sigma_X^2\) and \(\Sigma_Y^2\), the triangle inequality and the geometric inequality \(\|X+Y\|_\calG^2 \leq (\|X\|_\calG+\|Y\|_\calG)^2 \le 2\|X\|_\calG^2 + 2\|Y\|_\calG^2\), paired with the convexity of the exponential function, dictates that
\begin{equation*}
    \mathbb{E}\left[\exp\left(\tfrac{\|X+Y\|_\calG^2}{2\Sigma_X^2 + 2\Sigma_Y^2}\right)\right] \le \tfrac{\Sigma_X^2}{\Sigma_X^2+\Sigma_Y^2} \mathbb{E}\left[\exp\left(\tfrac{\|X\|_\calG^2}{\Sigma_X^2}\right)\right] + \tfrac{\Sigma_Y^2}{\Sigma_X^2+\Sigma_Y^2} \mathbb{E}\left[\exp\left(\tfrac{\|Y\|_\calG^2}{\Sigma_Y^2}\right)\right] \le 2.
\end{equation*}
Consequently, the summation of two conditionally dependent sub-Gaussian vectors inherently possesses a variance proxy bounded by \(2\Sigma_X^2 + 2\Sigma_Y^2\). %

\begin{proof}[Proof of Proposition~\ref{prop:family1}]
Under the zeroth-order correction, the effective innovation becomes
\[
\mathcal I_t(\mathcal T) = (1-\beta_t)\bigl(G(w_t,\xi_t)-g(w_t)\bigr) - \beta_t\bigl(g(w_t)-g(w_{t-1})\bigr).\]
Taking conditional expectation, using Assumption~\ref{ass:lipschitz}, and noting fact 1 \(\|w_{t}-w_{t-1}\| \leq \eta_{t-1} G\) under Assumption~\ref{ass:update_bound} 
\begin{equation*}
    \begin{aligned}
    \left\|\mathbb E[\mathcal I_t(\mathcal T)\mid \mathcal F_{t-1}]\right\|_\calG 
    &= \beta_t \left\| \mathbb E[g(w_t)-g(w_{t-1})\mid \mathcal F_{t-1}] \right\|_\calG
    \\
    &\overset{\text{Jensen's}}{\le} \beta_t \mathbb E\!\left[\|g(w_t)-g(w_{t-1})\|_\calG\mid \mathcal F_{t-1}\right] 
    \\
    &\overset{\text{As.}~\ref{ass:lipschitz}}{\le} \beta_t L \mathbb E\!\left[\|w_t-w_{t-1}\|\mid \mathcal F_{t-1}\right]
    \le \beta_t \eta_{t-1} G L =: \mathcal{B}_t,
    \end{aligned}  
\end{equation*}
The centered innovation is
\[
\mathcal Z_t = \mathcal I_t(\mathcal T)-\mathbb E[\mathcal I_t(\mathcal T)\mid \mathcal F_{t-1}]
= (1-\beta_t)\bigl(G(w_t,\xi_t)-g(w_t)\bigr),
\]
Assumption~\ref{ass:innovation_regularity} implies \(\Sigma_t^2=(1-\beta_t)^2\sigma^2\). Substituting these quantities into Theorem~\ref{thm:unified} gives \eqref{eq:family1_generic_bound}.
\end{proof}
\begin{proof}[Proof of Proposition~\ref{prop:family2}]
    For \(\mathcal T_t = 0\), the effective innovation is 
    \[\mathcal I_t(\mathcal T) = (1-\beta_t)\bigl(G(w_t,\xi_t)-g(w_t)\bigr) + \beta_t \Delta(w_t,w_{t-1},\xi_t).\]

    Both terms have zero conditional mean given \(\mathcal F_{t-1}\), so \(\mathcal B_t = \left\|\mathbb E[\mathcal I_t(\mathcal T)\mid \mathcal F_{t-1}]\right\|_\calG=0.\) It remains to control the centered innovation, which in this case equals \(\mathcal I_t(\mathcal T)\) itself.

    To construct the variance proxy \(\Sigma_t^2\), we must account for the conditional dependence of the noise sources all evaluated at the same random realization \(\xi_t\). Because these components share the same underlying random seed conditionally on \(\mathcal{F}_{t-1}\), we cannot sum their variance proxies as if they were conditionally independent or orthogonal martingale sequences.

    By applying the above fact 1 and fact 2 and utilizing the Assumption~\ref{ass:update_bound}, the variance proxy for \(\mathcal{Z}_t\) is bounded by
    \begin{equation*}
        \Sigma_t^2 = 2(1-\beta_t)^2 \sigma^2 + 2\beta_t^2 \ell^2 \eta_{t-1}^2 G^2.
    \end{equation*}

    Substituting \(\mathcal B_t=0\) and this variance proxy into Theorem~\ref{thm:unified} gives \eqref{eq:family2_generic_bound}.
\end{proof}
\begin{proof}[Proof of Proposition~\ref{prop:family3}]
    Substituting the second-order correction into the effective innovation gives
    \[\mathcal I_t(\mathcal T) = (1-\beta_t)\bigl(G(w_t,\xi_t)-g(w_t)\bigr) + \beta_t\bigl(g(w_{t-1})-g(w_t)-\nabla G%
    (w_t,\xi_t)(w_{t-1}-w_t)\bigr).\]
    Taking conditional expectation and using \(\mathbb E[\nabla G(w_t,\xi_t)\mid \mathcal F_{t-1}]=\nabla g(w_t)\), we obtain 
    \[\left\|\mathbb E[\mathcal I_t(\mathcal T)\mid \mathcal F_{t-1}]\right\|_\calG  = \beta_t \|g(w_{t-1})-g(w_t)-\nabla g(w_t)(w_{t-1}-w_t)\|_\calG.\]
    Assumption~\ref{ass:smoothness} yields the second-order Taylor remainder bound, 
    and noting fact 1 under  Assumption~\ref{ass:update_bound}.
    \[\left\|\mathbb E[\mathcal I_t(\mathcal T)\mid \mathcal F_{t-1}]\right\|_\calG \le \frac{\alpha}{2}\beta_t \|w_t-w_{t-1}\|^2 \leq \frac{\alpha}{2}\beta_t\eta_{t-1}^2G^2 =: \mathcal{B}_t.\] 
    The centered innovation \(\mathcal{Z}_t = \mathcal{I}_t(\mathcal{T}) - \mathbb{E}[\mathcal{I}_t(\mathcal{T}) \mid \mathcal{F}_{t-1}]\) consequently distills into two purely stochastic components, representing the direct gradient evaluation noise and the error in the Hessian-vector product
    \begin{equation*}
        \mathcal{Z}_t = (1-\beta_t)(G(w_t, \xi_t) - g(w_t)) - \beta_t (\nabla G%
        (w_t, \xi_t) - \nabla g(w_t))(w_{t-1} - w_t).
    \end{equation*}
    By Assumption \ref{ass:subg_hessian} and fact 1 under Assumption~\ref{ass:update_bound}, the Hessian-vector product noise evaluated along the update direction \(w_{t-1} - w_t\) is sub-Gaussian with a variance proxy bounded by \(\gamma^2 \|w_t - w_{t-1}\|^2 \le \gamma^2 \eta_{t-1}^2 G^2\). Summing the proxies and using our previous argument of fact 2, we ascertain the uniform predictable variance proxy for \(\mathcal{Z}_t\) as
    \begin{equation*}
        \Sigma_t^2 = 2(1-\beta_t)^2 \sigma^2 + 2\beta_t^2 \gamma^2 \eta_{t-1}^2 G^2.
    \end{equation*}
    Substitution into Theorem~\ref{thm:unified} completes the proof.
\end{proof}

In the following subsections, we instantiate the unified estimation error bound for each of the three families under specific parameter regimes. We provide detailed derivations and highlight the implications of different choices of \(\{p_t, \beta_t, \eta_t\}\) on the error bounds.
For the case 1 of each family, we can drop the \(\sqrt{2}\) factor in the last term of the bound in Theorem~\ref{thm:unified}
by the finer analysis in Remark~\ref{remark:tightened_bound} on the tightened upper Bound for case 1 with \(p_t = 0\).

For the case 2 and case 3 of Family 1, because variance proxy remains zero when \(\beta_t =1\), then the last term related to the variance proxy of the upper bound in Theorem~\ref{thm:unified} can be dropped.
Also, for the case 2 and case 3 of Family 2 and Family 3, we can drop the \(\sqrt{2}\) factor in the last term of the bound in Theorem~\ref{thm:unified} as well,
since for these two cases, we always have \(\beta_t = 1\) and therefore the additional factor \(2\) introduced in bounding the variance proxy in Fact 2 is not needed. 

\subsection{Zeroth-Order Correction (Standard Recursion and Resetting)}
Recall that for the zeroth-order family, the correction map is defined as \(\mathcal{T}_t(w_t, w_{t-1}, \xi_t) \coloneqq \beta_t \big( G(w_{t-1}, \xi_t) - G(w_t, \xi_t) \big)\). 
\subsubsection{The case where \texorpdfstring{\(p_t=0\), \(\beta_t = \beta\)}{pt=0, bt=b}, and \texorpdfstring{\(\eta_t = \eta\)}{etat} deterministically}
\label{subsubsec:family1case1}

In this regime, the estimator is initialized with a large batch sample \(G(w_0, \xi_0^B)\) and the recursion proceeds deterministically as
\begin{equation}
    v_t = (1-\beta)G(w_t,\xi_t) + \beta v_{t-1}.
\end{equation}
The absence of stochastic resets implies the execution consists of a single continuous epoch where \(m=1\), \(\tau_0 = 0\), and \(\tau_1=T\). Applying the constant parameters, the accumulated variance proxy evaluates to
\begin{equation}
    \mathfrak{V}_t^2 = \sum_{j=1}^t A_j^2 \Sigma_j^2 = \sum_{j=1}^t \beta^{-2j} (1-\beta)^2 \sigma^2 = (1-\beta)^2 \sigma^2\frac{\beta^{-2t}-1}{1-\beta^2}.
\end{equation}
We set the deterministic variance budget \(V_t\)  to this exact bound. Utilizing the multiplier sequence \(\Lambda_{t,\tau_0} = A_t^{-1} = \beta^t\), the cumulative deterministic bias is bounded by
\begin{equation}
   \sum_{j=1}^t \Lambda_{t,j} \mathcal{B}_j  =  \sum_{j=1}^t LG \beta \eta \beta^{t-j} = LG \beta \eta \frac{1-\beta^t}{1-\beta} \leq \frac{LG \eta}{1-\beta}.
\end{equation}
Simultaneously, by letting \(V_t = \mathfrak{V}_t^2\), the scaled deviation radius simplifies to
\begin{equation}
    A_t^{-1} \sqrt{V_t} = \beta^t (1-\beta) \sigma \sqrt{\frac{\beta^{-2t}-1}{1-\beta^2}} = \sigma \sqrt{\frac{(1-\beta)(1-\beta^{2t})}{1+\beta}} \leq \sigma \sqrt{\frac{1-\beta}{1+\beta}} \leq \sigma\sqrt{1-\beta}.
\end{equation}
Therefore, applying Theorem \ref{thm:unified}, we establish that with probability at least \(1-\delta\), the estimation error satisfies
\begin{equation}
    \|e_t\|_\calG \leq \beta^t C\left(\frac{\delta}{2T},\kappa\right) \sigma /\sqrt{B} + \frac{GL \eta}{1-\beta} + C\left(\frac{\delta}{2T},\kappa\right)\sigma \sqrt{1-\beta}.
\end{equation}

\subsubsection{The case where \texorpdfstring{\(p_t=p\), \(\beta_t = 1\)}{pt=p, bt=1}, and \texorpdfstring{\(\eta_t = \eta\)}{etat} deterministically}\label{subsubsec:family1case2}
Under a probabilistic reset mechanism with full momentum, the estimator update takes the form
\begin{equation}
    v_t =
    \begin{cases}
    G(w_t, \xi_t^B), & \text{w.p. } p,\\
    v_{t-1}, & \text{w.p. } 1-p.
    \end{cases}
\end{equation}
The trajectory is partitioned into a random number of epochs, where the temporal length of each epoch is distributed geometrically with success probability \(p\). Setting \(\beta_t = 1\) enforces \(A_j=1\) and \(\Lambda_{t,j} = 1\) universally. Furthermore, the variance proxy evaluates to \(\mathfrak{V}_t^2 = 0\), as the direct stochastic noise is fully suppressed by the deterministic momentum weight. Assigning \(V_t = \mathfrak{V}_t^2= 0\), the cumulative bias inside the random epoch mapping to time \(t\) becomes the sole source of structural error:
\begin{equation}
     \sum_{j=\tau_{m-1}+1}^t \Lambda_{t,j} \mathcal{B}_j  =  \sum_{j=\tau_{m-1}+1}^t LG \eta =   LG \eta (t-\tau_{m-1}) \leq  LG \eta (\tau_m-\tau_{m-1}).
\end{equation}
Let \(E_m \coloneqq \tau_m-\tau_{m-1}\) for all \(m \in [T]\) denote the duration of the \(m\)-th epoch. The maximum epoch length across the horizon is controlled via a standard union bound over the geometric tails
\begin{equation}
    \mathbb{P}\left( \max_{m\in [T]} E_m \ge \mathfrak{E}_{max}\right) \leq \sum_{m=1}^T \mathbb{P}\left(  E_m \ge \mathfrak{E}_{max}\right) = T (1-p)^{\mathfrak{E}_{max}}.
\end{equation}
Enforcing the strict inequality \(\log(1/(1-p)) > p\), we constrain the failure probability to \(\delta/4\), which yields \(\mathfrak{E}_{max} \leq \frac{1}{p}\log \frac{4T}{\delta}\). Integrating this maximum accumulation span, we deduce that with probability at least \(1-\delta\), the estimation error respects the bound
\begin{equation}
    \|e_t\|_\calG \leq  C\left(\frac{\delta}{4T},\kappa\right) \sigma /\sqrt{B} + \frac{GL \eta}{p} \log \frac{4T}{\delta}.
\end{equation}

\subsubsection{The case where \texorpdfstring{\(p_t=\mathbb{I}_{\{t \mod E = 0\}}\), \(\beta_t = 1\)}{pt=Sched, bt=1}, and \texorpdfstring{\(\eta_t = \eta\)}{etat} deterministically}\label{subsubsec:family1case3}
Enforcing a deterministic periodic reset schedule restrains the maximum epoch length to \(E\) iterations. The estimator update evaluates as
\begin{equation}
    v_t =
    \begin{cases}
    G(w_t, \xi_t^B), & t \mod E = 0,\\
    v_{t-1}, & \text{otherwise}.
    \end{cases}
\end{equation}
Assuming the total horizon factors as \(T = E \times N\) for some positive integer \(N\), the algorithm generates  \(N = T/E\) deterministic epochs. Because the variance proxy \(\mathfrak{V}_t^2\) remains zero, by letting \(V_t = \mathfrak{V}_t^2=0\), the error is bounded entirely by the initialization noise and the maximum deterministic drift over the fixed interval \(E\). Because the total number of reset nodes is reduced from \(T\) to  \(T/E\), the union bound covering the reset initializations requires a refined confidence adjustment of \(\delta / (T/E) = \delta E / T\). Consequently, with probability at least \(1-\delta\), the estimation error satisfies
\begin{equation}
    \|e_t\|_\calG \leq C\left(\frac{\delta}{2T/E},\kappa\right) \sigma /\sqrt{B} + GL \eta E.
\end{equation}

\subsection{First-Order Correction (Differential Recursion and Resetting)}
Recall that for this family of estimators, the correction term is set to zero, yielding \(\mathcal{T}_t(w_t, w_{t-1}, \xi_t) \coloneqq 0\). The effective innovation, therefore, consists exclusively of the direct noise and the difference noise, expressed as
\begin{equation}
    \mathcal{I}_t(\mathcal{T}) = (1-\beta_t) (G(w_t, \xi_t) - g(w_t)) + \beta_t \Delta(w_t, w_{t-1}, \xi_t).
\end{equation}
\subsubsection{The case where \texorpdfstring{\(p_t=0\), \(\beta_t = \beta\)}{pt=0, bt=b}, and \texorpdfstring{\(\eta_t = \eta\)}{etat} deterministically}\label{subsubsec:family2case1}
In this regime, the estimator initializes with a large batch sample \(G(w_0, \xi_0^B)\) and processes deterministically without stochastic resets:
\begin{equation}
v_t =
G(w_t, \xi_t) + \beta (v_{t-1} - G(w_{t-1}, \xi_t)).
\end{equation}
Analogous to the zeroth-order family, the execution consists of a single continuous epoch where \(m=1\), \(\tau_0=0\), and \(\tau_1=T\). Applying the constant parameters to the derived variance proxy, the accumulated proxy scales according to the predefined multipliers \(A_j = \beta^{-j}\). The total variance proxy is evaluated as
\begin{equation}
    \mathfrak{V}_t^2 = \sum_{j=1}^t A_j^2 \Sigma_j^2 = \sum_{j=1}^t \beta^{-2j} \left( 2(1-\beta)^2 \sigma^2 + 2\beta^2 \ell^2 \eta^2 G^2 \right).
\end{equation}
Evaluating the geometric series, we define the deterministic variance budget \(V_t\) by letting \(V_t = \mathfrak{V}_t^2\)
\begin{equation}
    V_t = \left( 2(1-\beta)^2 \sigma^2 + 2\beta^2 \ell^2 \eta^2 G^2 \right) \frac{\beta^{-2t}-1}{1-\beta^2}.
\end{equation}
Because the systematic bias \(\mathcal{B}_j\) is zero, the cumulative bias component in the overarching theorem vanishes entirely, yielding \(\sum_{j=1}^t \Lambda_{t,j} \mathcal{B}_j = 0\). The concentration term governed by the martingale difference relies on the modulated variance budget, which simplifies to
\begin{equation}
    A_t^{-1} \sqrt{V_t} = \beta^t \sqrt{V_t} = \sqrt{\frac{1-\beta^{2t}}{1-\beta^2}} \sqrt{2(1-\beta)^2 \sigma^2 + 2\beta^2 \ell^2 \eta^2 G^2} \le \sqrt{\frac{2(1-\beta)^2 \sigma^2 + 2\beta^2 \ell^2 \eta^2 G^2}{1-\beta^2}}.
\end{equation}
Invoking Theorem \ref{thm:unified} we assert that with probability at least \(1-\delta\), the estimation error satisfies
\begin{equation}
    \|e_t\|_\calG \leq \beta^t C\left(\frac{\delta}{2T},\kappa\right) \sigma /\sqrt{B} + C\left(\frac{\delta}{2T},\kappa\right) \sqrt{\frac{2(1-\beta)^2 \sigma^2 + 2\beta^2 \ell^2 \eta^2 G^2}{1-\beta^2}}.
\end{equation}

\subsubsection{The case where \texorpdfstring{\(p_t=p\), \(\beta_t = 1\)}{pt=p, bt=1}, and \texorpdfstring{\(\eta_t = \eta\)}{etat} deterministically}\label{subsubsec:family2case2}
By assigning full weight to the differential update via \(\beta_t = 1\), the estimator replicates the behavior of the PAGE algorithm, resetting memory probabilistically with rate \(p\):
\begin{equation}
v_t :=
\begin{cases}
G(w_t, \xi_t^B), & w.p. \qquad p, \\
v_{t-1}+ G(w_t, \xi_t) - G(w_{t-1}, \xi_t), & w.p. \qquad 1-p,
\end{cases}
\end{equation}

The explicit choice of full momentum inherently nullifies the direct noise variance contribution, leaving the variance proxy dependent on the geometric continuity of the gradient space. Substituting \(\beta_t=1\) yields a constant variance proxy uniformly bounded by \(\Sigma_j^2 \le 2 \ell^2 \eta^2 G^2\). The corresponding multipliers resolve to unity (\(A_j=1\) and \(\Lambda_{t,j}=1\)), dictating that the variance accumulation grows linearly with the duration of the random epoch \(E_m \coloneqq \tau_m - \tau_{m-1}\). The accumulated variance proxy inside the epoch thus evaluates to
\begin{equation}
    \mathfrak{V}_t^2 = \sum_{j=\tau_{m-1}+1}^t 2 \ell^2 \eta^2 G^2 = 2\ell^2 \eta^2 G^2 (t - \tau_{m-1}) \leq 2\ell^2 \eta^2 G^2 E_m.
\end{equation}
As established previously, the maximal epoch length \(\mathfrak{E}_{max}\) governed by the underlying geometric distribution satisfies \(\mathfrak{E}_{max} \leq \frac{1}{p}\log \frac{4T}{\delta}\) with high probability. Setting the deterministic variance budget \(V_t\) to  \(2\ell^2 \eta^2 G^2 \frac{1}{p}\log \frac{4T}{\delta}\), and recalling the zero-bias condition \(\mathcal{B}_t = 0\), the structure of the error bound eliminates the linear growth associated with deterministic drift. For this configuration, with probability at least \(1-\delta\), the terminal estimation error is subsequently bounded by
\begin{equation}
    \|e_t\|_\calG \leq C\left(\frac{\delta}{4T},\kappa\right) \sigma /\sqrt{B} + C\left(\frac{\delta}{4T},\kappa\right) \ell \eta G \sqrt{\frac{2}{p} \log \frac{4T}{\delta}}.
\end{equation}

\subsubsection{The case where \texorpdfstring{\(p_t=\mathbb{I}_{\{t \mod E = 0\}}\), \(\beta_t = 1\)}{pt=Sched, bt=1}, and \texorpdfstring{\(\eta_t = \eta\)}{etat} deterministically}\label{subsubsec:family2case3}
Under a cyclic deterministic reset mechanism, the temporal length of each estimation epoch is constrained to  \(E\) iterations:
\begin{equation}
v_t :=
\begin{cases}
G(w_t, \xi_t^B), & t \mod E = 0, \\
v_{t-1}+ G(w_t, \xi_t) - G(w_{t-1}, \xi_t), & \text{otherwise}.
\end{cases}
\end{equation}
Assuming the horizon \(T = E \times N\) aligns perfectly with these periodic boundaries, the stochastic behavior is constrained similarly to the probabilistic reset case, yet enforces deterministic termination bounds for each tracking segment. Given \(\beta_t = 1\), the variance proxy bound simplifies  to \(\Sigma_j^2 \le 2\ell^2 \eta^2 G^2\). Because the length of the interval \(t - \tau_{m-1}\) is deterministically upper-bounded by \(E\), the variance budget across any arbitrary epoch deterministically accumulates to
\begin{equation}
    \mathfrak{V}_t^2 = \sum_{j=\tau_{m-1}+1}^t 2 \ell^2 \eta^2 G^2 \leq 2 \ell^2 \eta^2 G^2 E.
\end{equation}
Setting the strict variance budget \(V_t = 2 \ell^2 \eta^2 G^2 E\) and combining this with the intrinsic zero-bias property characterizing this estimator family, we embed these components into the unified framework. Consequently, incorporating the deterministic epoch structure and applying the appropriate uniform bounds over the respective intervals, with probability at least \(1-\delta\), the estimation error satisfies
\begin{equation}
    \|e_t\|_\calG \leq C\left(\frac{\delta}{2T/E},\kappa\right) \sigma /\sqrt{B} + C\left(\frac{\delta}{2T},\kappa\right) \ell \eta G \sqrt{2E}.
\end{equation}

\subsection{Second-Order Correction (Second-Order Recursion and Resetting)}
Recall the explicit formulation of the second-order correction mechanism designed to correct the drift using local curvature, given by \(\mathcal{T}_t(w_t, w_{t-1}, \xi_t) \coloneqq \beta_t \Big( G(w_{t-1}, \xi_t) - \big[ G(w_t, \xi_t) + \nabla G(w_t, \xi_t)(w_{t-1} - w_t) \big] \Big)\). Substituting this correction into the general effective innovation recursion, we obtain
\begin{equation}
    \mathcal{I}_t(\mathcal{T}) = (1-\beta_t)(G(w_t, \xi_t) - g(w_t)) + \beta_t \big( g(w_{t-1}) - g(w_t) - \nabla G%
    (w_t, \xi_t)(w_{t-1} - w_t) \big).
\end{equation}
\subsubsection{The case where \texorpdfstring{\(p_t=0\), \(\beta_t = \beta\)}{pt=0, bt=b}, and \texorpdfstring{\(\eta_t = \eta\)}{etat} deterministically}\label{subsubsec:family3case1}
Under a purely recursive update, absent random resets, the process evolves across a single continuous epoch:
\begin{equation}
v_t :=
(1-\beta_t)G(w_t, \xi_t) + \beta_t \Big(v_{t-1}-  \nabla G(w_t, \xi_t)(w_{t-1} - w_t) \Big).
\end{equation}

Imposing the fixed hyperparameters onto the variance framework, we establish the accumulated variance proxy utilizing the geometric sequence \(A_j = \beta^{-j}\). The proxy integral evaluates  to
\begin{equation}
    \mathfrak{V}_t^2 = \sum_{j=1}^t A_j^2 \Sigma_j^2 = \left( 2(1-\beta)^2 \sigma^2 + 2\beta^2 \gamma^2 \eta^2 G^2 \right) \frac{\beta^{-2t}-1}{1-\beta^2}.
\end{equation}
Setting the deterministic variance budget \(V_t=\mathfrak{V}_t^2\)  to this quantity, we concurrently evaluate the cumulative deterministic bias induced by the Taylor remainder. The bias accumulates through the discounted sum defined by \(\Lambda_{t,j} = \beta^{t-j}\), leading to
\begin{equation}
    \sum_{j=1}^t \Lambda_{t,j} \mathcal{B}_j \le \sum_{j=1}^t \beta^{t-j} \frac{\alpha}{2} \beta \eta^2 G^2 = \frac{\alpha}{2} \beta \eta^2 G^2 \frac{1-\beta^t}{1-\beta} \le \frac{\alpha \beta \eta^2 G^2}{2(1-\beta)}.
\end{equation}
Isolating the modulated variance term within the martingale concentration mechanism yields the upper bound \(A_t^{-1} \sqrt{V_t} \le \sqrt{\frac{2(1-\beta)^2 \sigma^2 + 2\beta^2 \gamma^2 \eta^2 G^2}{1-\beta^2}}\). Applying Theorem \ref{thm:unified} over this singular deterministic epoch, we conclude that with probability at least \(1-\delta\), the estimation error satisfies
\begin{equation}
    \|e_t\|_\calG \leq \beta^t C\left(\frac{\delta}{2T},\kappa\right) \sigma /\sqrt{B} + \frac{\alpha \beta \eta^2 G^2}{2(1-\beta)} + C\left(\frac{\delta}{2T},\kappa\right) \sqrt{\frac{2(1-\beta)^2 \sigma^2 + 2\beta^2 \gamma^2 \eta^2 G^2}{1-\beta^2}}.
\end{equation}

\subsubsection{The case where \texorpdfstring{\(p_t=p\), \(\beta_t = 1\)}{pt=p, bt=1}, and \texorpdfstring{\(\eta_t = \eta\)}{etat} deterministically}\label{subsubsec:family3case2}
Setting the momentum uniformly to 1 transitions the estimator into a differential tracker that leverages stochastic second-order corrections, accompanied by memory resets distributed geometrically with rate \(p\):
\begin{equation}
v_t :=
\begin{cases}
G(w_t, \xi_t^B), & w.p. \qquad p, \\
v_{t-1}-  \nabla G(w_t, \xi_t)(w_{t-1} - w_t) , & w.p. \qquad 1-p.
\end{cases}
\end{equation}
Because \(\beta_t = 1\), the direct noise term evaluates to zero,  coupling the variance proxy to the sub-Gaussian Hessian noise such that \(\Sigma_j^2 = 2\gamma^2 \eta^2 G^2\). Similarly, the per-step deterministic bias stabilizes at \(\mathcal{B}_j \le \frac{\alpha}{2} \eta^2 G^2\). Due to \(A_j = 1\) and \(\Lambda_{t,j} = 1\), the accumulation of both the bias and the variance proxy scales linearly with the duration of the current random epoch \(E_m \coloneqq \tau_m - \tau_{m-1}\). Within this epoch, the summations are bounded by
\begin{equation}
    \sum_{j=\tau_{m-1}+1}^t \mathcal{B}_j \le \frac{\alpha}{2} \eta^2 G^2 E_m \quad \text{and} \quad \mathfrak{V}_t^2 = \sum_{j=\tau_{m-1}+1}^t \Sigma_j^2 \le 2\gamma^2 \eta^2 G^2 E_m.
\end{equation}
Invoking the high-probability uniform bound on the maximum epoch length, we verify \(\mathfrak{E}_{max} \leq \frac{1}{p}\log \frac{4T}{\delta}\). By defining the deterministic variance budget as \(V_t = 2\gamma^2 \eta^2 G^2 \frac{1}{p}\log \frac{4T}{\delta}\), and bounding the total bias over the longest possible epoch, we assemble the complete error topology. Consequently, with probability at least \(1-\delta\), the terminal estimation error is bounded by
\begin{equation}
    \|e_t\|_\calG \leq C\left(\frac{\delta}{4T},\kappa\right) \sigma /\sqrt{B} + \frac{\alpha \eta^2 G^2}{2p} \log \frac{4T}{\delta} + C\left(\frac{\delta}{4T},\kappa\right) \gamma \eta G \sqrt{\frac{2}{p} \log \frac{4T}{\delta}}.
\end{equation}

\subsubsection{The case where \texorpdfstring{\(p_t=\mathbb{I}_{\{t \mod E = 0\}}\), \(\beta_t = 1\)}{pt=Sched, bt=1}, and \texorpdfstring{\(\eta_t = \eta\)}{etat} deterministically}\label{subsubsec:family3case3}
Imposing a  deterministic, periodic reset schedule ensures that memory truncation occurs  every \(E\) iterations:
\begin{equation}
v_t :=
\begin{cases}
G(w_t, \xi_t^B), & t \mod E = 0, \\
v_{t-1}-  \nabla G(w_t, \xi_t)(w_{t-1} - w_t), & \text{otherwise}.
\end{cases}
\end{equation}
Assuming the computational horizon is an integer multiple of the epoch length, \(T = E \times N\), this architecture completely arrests the unbounded accumulation of estimation lag while preserving the full corrective weight \(\beta_t=1\). The fundamental quantities parallel the probabilistic setting, maintaining \(\Sigma_j^2 = 2\gamma^2 \eta^2 G^2\) and \(\mathcal{B}_j \le \frac{\alpha}{2} \eta^2 G^2\). However, the interval lengths are now deterministically capped at \(t - \tau_{m-1} \le E\). The cumulative proxy and bias are correspondingly bounded across the fixed interval
\begin{equation}
    \sum_{j=\tau_{m-1}+1}^t \mathcal{B}_j \le \frac{\alpha}{2} \eta^2 G^2 E \quad \text{and} \quad \mathfrak{V}_t^2 \le 2\gamma^2 \eta^2 G^2 E.
\end{equation}
Assigning this upper bound to \(V_t\) and directly incorporating the maximal deterministic bias contribution, the bounds concentrate without the logarithmic geometric penalties. We establish that with probability at least \(1-\delta\), the resulting error magnitude honors the bound
\begin{equation}
    \|e_t\|_\calG \leq C\left(\frac{\delta}{2T/E},\kappa\right) \sigma /\sqrt{B} + \frac{\alpha \eta^2 G^2 E}{2} + C\left(\frac{\delta}{2T},\kappa\right) \gamma \eta G \sqrt{2E}.
\end{equation}

\newpage
\subsection{Examples for Assumption~\ref{ass:subg_lipschitz} of Sub-Gaussian Lipschitzness of the Centered Differences}\label{subsec:example}
The sub-Gaussian tail assumptions utilized in our analysis are automatically satisfied in standard Empirical Risk Minimization (ERM) configurations. Let the function (or the gradient) be defined over a finite dataset as \(g(w) = \frac{1}{n} \sum_{i=1}^n G_i(w)\), where the stochastic oracle samples a component index uniformly at random. If every individual component \(G_i\) is \(L_i\)-Lipschitz continuous, the discrete nature of the sample space guarantees that the centered difference noise is bounded almost surely by \(2 \max_{i} L_i \|w - w'\|\). Because any bounded random variable is sub-Gaussian, Assumption \ref{ass:subg_lipschitz} is globally satisfied with the variance proxy parameter \(\ell = 2 \max_{i} L_i\).  Nonetheless, we present concrete instantiations satisfying Assumption \ref{ass:subg_lipschitz} next.

We leverage the moment-generating function of a scaled chi-squared distribution: for a standard Gaussian random variable \(Z \sim \mathcal{N}(0,1)\), the expectation \(\mathbb{E}[\exp(\lambda Z^2)] = (1-2\lambda)^{-1/2}\) is valid for all \(\lambda < 1/2\). Enforcing the condition \((1-2\lambda)^{-1/2} \le 2\) necessitates \(\lambda \le \frac{3}{8}\).

Recall that Assumption \ref{ass:subg_lipschitz} requires the centered estimate difference noise \(\Delta(w, w', \xi) \coloneqq (G(w, \xi) - G(w', \xi)) - (g(w) - g(w'))\) to concentrate with a variance proxy bounded by \(\ell^2 \|w - w'\|^2\).

\paragraph{Example 1: Scalar Function Configuration.} Let the primal space \(\mathcal{X}:=\mathbb{R}^d\) be equipped with the Euclidean geometry \(\|\cdot\| := \|\cdot\|_2\). Consider a stochastic linear observation model where the oracle evaluates a scalar function \(G: \mathbb{R}^d \times \Omega \to \calG := \mathbb{R}\) with \(\|\cdot\|_\calG:=|\cdot|\) defined by
\begin{equation}
    G(w, \xi) = \langle \xi, w \rangle,
\end{equation}
where \(\xi \sim \mathcal{N}(0, \Sigma)\) is a centered Gaussian random vector. The population function is  zero, \(g(w) = \mathbb{E}[G(w, \xi)] = 0\). The centered difference evaluates  to the inner product
\begin{equation}
    \Delta(w, w', \xi) = \langle \xi, w - w' \rangle.
\end{equation}
For any fixed deterministic \(w\) and \(w'\), the random variable \(\Delta\) is a univariate Gaussian distributed as \(\mathcal{N}(0, (w-w')^\top \Sigma (w-w'))\). The variance is  bounded by the spectral norm, \(\nu^2 \coloneqq (w-w')^\top \Sigma (w-w') \le \|\Sigma\|_{\mathrm{op}} \|w-w'\|_2^2\). Letting \(Z = \Delta / \nu \sim \mathcal{N}(0,1)\), we verify the condition
\begin{equation}
    \mathbb{E}\left[ \exp\left( \frac{|\Delta(w, w', \xi)|^2}{\ell^2 \|w - w'\|_2^2} \right) \right] \le \mathbb{E}\left[ \exp\left( \frac{\|\Sigma\|_{\mathrm{op}}}{\ell^2} Z^2 \right) \right].
\end{equation}
By setting \(\lambda = \|\Sigma\|_{\mathrm{op}} / \ell^2\), the sub-Gaussian Lipschitz condition is  satisfied for any structural constant satisfying \(\ell \ge \sqrt{8 \|\Sigma\|_{\mathrm{op}} / 3}\).

\paragraph{Example 2: Vector Field Configuration (Gradient of a Potential).} Consider a stochastic quadratic objective \(F: \mathbb{R}^d \times \Omega \to \mathbb{R}\) with the primal space \((\mathcal{X},\|\cdot\|):= (\mathbb{R}^d, \|\cdot\|_2)\) defined by
\begin{equation}
    F(w, \xi) = \frac{1}{2} w^\top (A + \xi I) w,
\end{equation}
where \(A \in \mathbb{R}^{d \times d}\) is a deterministic symmetric matrix, and \(\xi \sim \mathcal{N}(0, \sigma^2)\) is a scalar Gaussian noise perturbing the curvature. The stochastic vector field is the gradient \(G(w, \xi) = \nabla_w F(w, \xi) = (A + \xi I)w\), implicitly setting \(\calG:=\mathbb{R}^d\) and \(\|\cdot\|_\calG := \|\cdot\|_2\), making the population gradient \(g(w) = Aw\). The centered estimate difference noise simplifies algebraically to
\begin{equation}
    \Delta(w, w', \xi) = \xi (w - w').
\end{equation}
Taking the squared dual norm (which coincides with the \(\ell_2\) norm), we obtain \(\|\Delta(w, w', \xi)\|_2^2 = \xi^2 \|w - w'\|_2^2\). The exponential moment evaluates as
\begin{equation}
    \mathbb{E}\left[ \exp\left( \frac{\|\Delta(w, w', \xi)\|_2^2}{\ell^2 \|w - w'\|_2^2} \right) \right] = \mathbb{E}\left[ \exp\left( \frac{\sigma^2}{\ell^2} \left(\frac{\xi}{\sigma}\right)^2 \right) \right].
\end{equation}
Recognizing \((\xi/\sigma) \sim \mathcal{N}(0,1)\), the condition holds globally provided \(\ell \ge \sigma \sqrt{8/3}\).

\newpage
\section{Lemmas in Proofs for Two Applications}\label{apdx:lemma}
\begin{lemma}[Bound for Proximal Gradient Mapping with Normalized Direction, see similar Lemma~2 in \cite{cutkosky2020momentum}]\label{lemma:bound_for_proximal_gradient_mapping_with_normalized_direction}
    Suppose \(w\in \mathcal{W}\subseteq \mathcal{X}\), \(u, \delta\in \mathcal{X}^*\), \(\eta > 0\), and 
    \(\|u\|_\ast \neq 0, \|u+\delta\|_\ast \neq 0\). 
    Then for the proximal gradient mapping \(P\) defined in Definition~\ref{def:proximal_gradient_mapping}, we have:
    \[
    \left\langle u, P\left( w, \frac{u+\delta}{\|u+\delta\|_\ast}, \eta\right) \right\rangle 
    \geq \frac{1}{4}\left\langle u, P\left( w, \frac{u}{\|u\|_\ast}, \frac{\eta}{2} \right) \right\rangle - 2 \|\delta\|_\ast 
    \]
\end{lemma}

\begin{proof}[Proof of Lemma~\ref{lemma:bound_for_proximal_gradient_mapping_with_normalized_direction}]
    Noting the scaling property of the proximal gradient mapping (Lemma~\ref{lemma:properties_proximal_gradient_mapping})
    such that \(P(w, \alpha u, \eta) = \alpha P(w, u, \alpha\eta), \forall \alpha >0\), we have:
    \[
    P\left( w, \frac{u+\delta}{\|u+\delta\|_\ast}, \eta\right) = \frac{P(w, u+\delta, \eta/\|u+\delta\|_\ast)}{\|u+\delta\|_\ast},\quad
    P\left( w, \frac{u}{\|u\|_\ast}, \frac{\eta}{2} \right) = \frac{P(w, u, \eta/(2\|u\|_\ast))}{\|u\|_\ast}
    \]
    Therefore, the desired inequality to be proved is equivalent to:
    \[
    -\frac{\left\langle u, P(w, u+\delta, \eta/\|u+\delta\|_\ast) \right\rangle}{\|u+\delta\|_\ast} 
    \leq -\frac{\left\langle u, P(w, u, \eta/(2\|u\|_\ast)) \right\rangle}{4\|u\|_\ast} + 2 \|\delta\|_\ast 
    \]
    Now we start to prove the above equivalent inequality, then the desired inequality will follow by noting the scaling property of the proximal gradient mapping.
    We consider the following two cases for the ease of analysis: 
    (1) \(\|\delta\|_\ast \leq \frac{1}{3}\|u\|_\ast\); (2) \(\|\delta\|_\ast > \frac{1}{3}\|u\|_\ast\).

    \noindent\textbf{Case (1):} \(\|\delta\|_\ast \leq \frac{1}{3}\|u\|_\ast\).

    Applying the properties of the proximal gradient mapping (Lemma~\ref{lemma:properties_proximal_gradient_mapping})
    such that \(\|P(w, u_1, \eta) - P(w, u_2, \eta)\| \leq \|u_1 - u_2\|_\ast\), 
    and \(\langle \bx, \by\rangle\leq \|\bx\|_\ast \|\by\|\) for any \(\bx \in \mathcal{X}^\ast, \by\in \mathcal{X}\), we have:
    \[
    -\<u, P(w, u+\delta, \eta/\|u+\delta\|_\ast)\> \leq
    -\<u, P(w, u, \eta/\|u+\delta\|_\ast)\> + \|u\|_\ast \|\delta\|_\ast 
    \]
    By using the condition \(\|\delta\|_\ast \leq \frac{1}{3}\|u\|_\ast\) and the triangle inequality, we have:
    \[
    \frac{2}{3} \|u\|_\ast \leq \|u\|_\ast -\|\delta\|_\ast \leq \|u+\delta\|_\ast \leq \|u\|_\ast + \|\delta\|_\ast \leq \frac{4}{3} \|u\|_\ast
    \]
    Noting the monotonicity of the proximal gradient mapping (Lemma~\ref{lemma:properties_proximal_gradient_mapping})
    \(\langle u, \eta' P(w, u, \eta') \rangle \geq \langle u, \eta P(w, u, \eta) \rangle \geq 0\) for any \(\eta'\geq\eta>0\), 
    and using the condition \(\|\delta\|_\ast \leq \frac{1}{3}\|u\|_\ast\):
    \[
    \<u, P(w, u, \eta/\|u+\delta\|_\ast)\> \geq
    \frac{\|u+\delta\|_\ast}{2\|u\|_\ast} \<u, P(w, u, \eta/(2\|u\|_\ast))\>
    \geq \frac{1}{3} \<u, P(w, u, \eta/(2\|u\|_\ast))\>
    \]
    By combining the above three inequalities, we have the desired inequality for Case (1): \(\|\delta\|_\ast \leq \frac{1}{3}\|u\|_\ast\).
    \[
    \left[-\frac{\langle u, P(w, u+\delta, \eta/\|u+\delta\|_\ast)\rangle}{\|u+\delta\|_\ast}\right]\1_{\|\delta\|_\ast \leq \frac{1}{3}\|u\|_\ast} 
    \leq \left[-\frac{\langle u, P(w, u, \eta/(2\|u\|_\ast))\rangle}{4\|u\|_\ast} + 2 \|\delta\|_\ast \right]\1_{\|\delta\|_\ast \leq \frac{1}{3}\|u\|_\ast} 
    \]

    \noindent\textbf{Case (2):} \(\|\delta\|_\ast > \frac{1}{3}\|u\|_\ast\).

    Applying the properties of the proximal gradient mapping (Lemma~\ref{lemma:properties_proximal_gradient_mapping})
    such that \(\langle u, P(w, u, \eta)\rangle \geq \|P(w, u, \eta)\|^2 \geq 0\), \(\|P(w, u, \eta)\|\leq \|u\|_\ast\), 
    and using \(\< \bx, \by\rangle\leq \|\bx\|_\ast \|\by\|\) for any \(\bx \in \mathcal{X}^\ast, \by\in \mathcal{X}\):
    \begin{eqnarray*}
        -\frac{\langle u, P(w, u+\delta, \eta/\|u+\delta\|_\ast)\rangle}{\|u+\delta\|_\ast} 
        &=& - \frac{\langle u+\delta, P(w, u+\delta, \frac{\eta}{\|u+\delta\|_\ast})\rangle}{\|u+\delta\|_\ast} 
        + \frac{\langle \delta, P(w, u+\delta, \frac{\eta}{\|u+\delta\|_\ast})\rangle}{\|u+\delta\|_\ast}\\
        &\leq& 0 + \frac{\|\delta\|_\ast \|P(w, u+\delta, \frac{\eta}{\|u+\delta\|_\ast})\|}{\|u+\delta\|_\ast}
        \leq \|\delta\|_\ast
    \end{eqnarray*}
    Given that condition \(\|\delta\|_\ast > \frac{1}{3}\|u\|_\ast\), 
    and using the properties of the proximal gradient mapping (Lemma~\ref{lemma:properties_proximal_gradient_mapping})
    such that \(\<u, P(w, u, \eta)\> \leq \|u\|_\ast \|P(w, u, \eta)\|, \|P(w, u, \eta)\|\leq \|u\|_\ast\), we have:
    \[
    \frac{\langle u, P(w, u, \eta/(2\|u\|_\ast))\rangle}{4\|u\|_\ast} \leq 
    \frac{1}{4}\|P(w, u, \eta/(2\|u\|_\ast))\| \leq \frac{1}{4}\|u\|_\ast \leq \frac{1}{3}\|u\|_\ast 
    < \|\delta\|_\ast
    \]
    Combining the above two inequalities, we have the desired inequality for Case (2): \(\|\delta\|_\ast > \frac{1}{3}\|u\|_\ast\).
    \[
    \left[-\frac{\langle u, P(w, u+\delta, \eta/\|u+\delta\|_\ast)\rangle}{\|u+\delta\|_\ast}\right]\1_{\|\delta\|_\ast > \frac{1}{3}\|u\|_\ast} 
    \leq \left[-\frac{\langle u, P(w, u, \eta/(2\|u\|_\ast))\rangle}{4\|u\|_\ast} + 2 \|\delta\|_\ast \right]\1_{\|\delta\|_\ast > \frac{1}{3}\|u\|_\ast} 
    \]
    Summing up the two inequalities for these \textbf{two cases} and noting that \(1=\1_{\|\delta\|_\ast \leq \frac{1}{3}\|u\|_\ast} + \1_{\|\delta\|_\ast > \frac{1}{3}\|u\|_\ast}\), 
    we complete the proof of the equivalent inequality, and the desired inequality follows. %
\end{proof}

\begin{lemma}[Optimal Decoupled Bounds with Selection of Step Size and Momentum Parameter]\label{lemma:opt_bound_stepsize_momentum}
    For constants \(C_1, \dots, C_6 \geq 0\), and \(C_2 C_4 C_5 C_6 \neq 0\), and horizon \(T \ge 1\), define the objective functions:
    \begin{align*}
        f_1(\eta, 1-\beta) &= \frac{C_1}{\eta T} + C_2 \eta + \frac{C_3}{(1-\beta) T} + C_4 \sqrt{1-\beta} + \frac{C_5 \eta}{1-\beta} \\
        f_2(\eta, 1-\beta) &= \frac{C_1}{\eta T} + C_2 \eta + \frac{C_3}{(1-\beta) T} + C_4 \sqrt{1-\beta} + \frac{C_5 \eta}{\sqrt{1-\beta}} \\
        f_3(\eta, 1-\beta) &= f_2(\eta, 1-\beta) + \frac{C_6 \eta^2}{1-\beta}
    \end{align*}

    Letting \(\eta = \min \left\{ \left( \frac{C_1}{C_2} \right)^{\frac{1}{2}} T^{-\frac{1}{2}}, \; \left( \frac{C_1^3}{C_4^2 C_5} \right)^{\frac{1}{4}} T^{-\frac{3}{4}} \right\}, 1-\beta = \max \left\{ \left( \frac{C_3}{C_4} \right)^{\frac{2}{3}} T^{-\frac{2}{3}}, \; \left( \frac{C_1 C_5}{C_4^2} \right)^{\frac{1}{2}} T^{-\frac{1}{2}} \right\}\)
    \begin{equation*}
        f_1(\eta, 1-\beta) \le 2\left( \frac{C_1 C_2}{T} \right)^{\frac{1}{2}} + 2\left( \frac{C_3 C_4^2}{T} \right)^{\frac{1}{3}} + 3\left( \frac{C_1 C_4^2 C_5}{T} \right)^{\frac{1}{4}}
        \, \textbf{(Family 1)}
    \end{equation*}
    when \(T \geq \max\left\{\frac{C_3}{C_4}, \frac{C_1 C_5}{C_4^2}\right\}\).

    Letting \(\eta = \min \left\{ \left( \frac{C_1}{C_2} \right)^{\frac{1}{2}} T^{-\frac{1}{2}}, \; \left( \frac{C_1^2}{C_4 C_5} \right)^{\frac{1}{3}} T^{-\frac{2}{3}} \right\}, 1-\beta = \max \left\{ \left( \frac{C_3}{C_4} \right)^{\frac{2}{3}} T^{-\frac{2}{3}}, \; \left( \frac{C_1 C_5}{C_4^2} \right)^{\frac{2}{3}} T^{-\frac{2}{3}} \right\}\)
    \begin{equation*}
        f_2(\eta, 1-\beta) \le 2\left( \frac{C_1 C_2}{T} \right)^{\frac{1}{2}} + 2\left( \frac{C_3 C_4^2}{T} \right)^{\frac{1}{3}} + 3\left( \frac{C_1 C_4 C_5}{T} \right)^{\frac{1}{3}}
        \, \textbf{(Family 2)}
    \end{equation*}
    when \(T \geq \max\left\{\frac{C_3}{C_4}, \frac{C_1 C_5}{C_4^2}\right\}\).

    Letting \(\eta = \min \left\{ \left( \frac{C_1}{C_2} \right)^{\frac{1}{2}} T^{-\frac{1}{2}}, \; \left( \frac{C_1^2}{C_4 C_5} \right)^{\frac{1}{3}} T^{-\frac{2}{3}}, \; \left( \frac{C_1^3}{C_4^2 C_6} \right)^{\frac{1}{5}} T^{-\frac{3}{5}} \right\}\),\\
    \(1-\beta = \max \left\{ \left( \frac{C_3}{C_4} \right)^{\frac{2}{3}} T^{-\frac{2}{3}}, \; \left( \frac{C_1 C_5}{C_4^2} \right)^{\frac{2}{3}} T^{-\frac{2}{3}}, \; \left( \frac{C_1^2 C_6}{C_4^3} \right)^{\frac{2}{5}} T^{-\frac{4}{5}} \right\}\)%
    \begin{equation*}
        f_3(\eta, 1-\beta) \le 2\left( \frac{C_1 C_2}{T} \right)^{\frac{1}{2}} + 2\left( \frac{C_3 C_4^2}{T} \right)^{\frac{1}{3}} + 3\left( \frac{C_1 C_4 C_5}{T} \right)^{\frac{1}{3}} + 3\left( \frac{C_1^2 C_4^2 C_6}{T^2} \right)^{\frac{1}{5}}
        \, \textbf{(Family 3)}
    \end{equation*}
    when \(T \geq \max\left\{\frac{C_3}{C_4}, \frac{C_1 C_5}{C_4^2}, \left(\frac{C_1^2 C_6}{C_4^3}\right)^{\frac{1}{2}} \right\}\).
\end{lemma}

\begin{proof}
    To simplify the analysis, we introduce \(y = 1-\beta\). We denote the fundamental structural limits present in the bounds as:
    \begin{equation*}
        L_1 = \left( \frac{C_1 C_2}{T} \right)^{\frac{1}{2}}, 
        L_2 = \left( \frac{C_3 C_4^2}{T} \right)^{\frac{1}{3}}, 
        L_3 = \left( \frac{C_1 C_4^2 C_5}{T} \right)^{\frac{1}{4}},
        L_4 = \left( \frac{C_1 C_4 C_5}{T} \right)^{\frac{1}{3}}, 
        L_5 = \left( \frac{C_1^2 C_4^2 C_6}{T^2} \right)^{\frac{1}{5}}
    \end{equation*}
    We define the unscaled parameter candidate blocks present in the Lemma selections as follows:
    \[
    \eta_A = \left( \frac{C_1}{C_2} \right)^{\frac{1}{2}} T^{-\frac{1}{2}},\, 
    \eta_B = \left( \frac{C_1^3}{C_4^2 C_5} \right)^{\frac{1}{4}} T^{-\frac{3}{4}},\,
    \eta_C = \left( \frac{C_1^2}{C_4 C_5} \right)^{\frac{1}{3}} T^{-\frac{2}{3}},\,
    \eta_D = \left( \frac{C_1^3}{C_4^2 C_6} \right)^{\frac{1}{5}} T^{-\frac{3}{5}}
    \]
    \[
    y_A = \left( \frac{C_3}{C_4} \right)^{\frac{2}{3}} T^{-\frac{2}{3}},\,
    y_B = \left( \frac{C_1 C_5}{C_4^2} \right)^{\frac{1}{2}} T^{-\frac{1}{2}},\,
    y_C = \left( \frac{C_1 C_5}{C_4^2} \right)^{\frac{2}{3}} T^{-\frac{2}{3}},\,
    y_D = \left( \frac{C_1^2 C_6}{C_4^3} \right)^{\frac{2}{5}} T^{-\frac{4}{5}}
    \]

    For \textbf{Family 1}, substituting \(\eta = \min\{\eta_A, \eta_B\}\) and \(y = \max\{y_A, y_B\}\) into the objective function
    and using \(\eta \leq \eta_A, \frac{1}{\eta}\leq \frac{1}{\eta_A}+\frac{1}{\eta_B}, \sqrt{y}\leq \sqrt{y_A}+\sqrt{y_B}, \frac{1}{y} \leq \frac{1}{y_A},\frac{\eta}{y}\leq \frac{\eta_B}{y_B}\) 
    yields the following inequality:
    \begin{align*}
        f_1(\eta, 1-\beta) &= \frac{C_1}{\eta T} + C_2 \eta + \frac{C_3}{y T} + C_4 \sqrt{y} + \frac{C_5 \eta}{y} \\
        &\le \left( \frac{C_1}{\eta_A T} + \frac{C_1}{\eta_B T} \right) + C_2 \eta_A + \frac{C_3}{y_A T} + C_4 \left( \sqrt{y_A} + \sqrt{y_B} \right) + \frac{C_5 \eta_B}{y_B} \\
        &= \underbrace{\left( \frac{C_1}{\eta_A T} + C_2 \eta_A \right)}_{2L_1} + \underbrace{\left( \frac{C_3}{y_A T} + C_4 \sqrt{y_A} \right)}_{2L_2} + \underbrace{\left( \frac{C_1}{\eta_B T} + C_4 \sqrt{y_B} + \frac{C_5 \eta_B}{y_B} \right)}_{3L_3} \\
        &= 2L_1 + 2L_2 + 3L_3.
    \end{align*}

    For \textbf{Family 2}, substituting \(\eta = \min\{\eta_A, \eta_C\}\) and \(y = \max\{y_A, y_C\}\) 
    and using \(\eta\leq \eta_A, \frac{1}{\eta}\leq \frac{1}{\eta_A} + \frac{1}{\eta_C}, \sqrt{y}\leq \sqrt{y_A} + \sqrt{y_C}, \frac{1}{y}\leq \frac{1}{y_A}, \frac{\eta}{\sqrt{y}}\leq \frac{\eta_C}{\sqrt{y_C}}\) 
    yields the following inequality:
    \begin{align*}
        f_2(\eta, 1-\beta) &= \frac{C_1}{\eta T} + C_2 \eta + \frac{C_3}{y T} + C_4 \sqrt{y} + \frac{C_5 \eta}{\sqrt{y}} \\
        &\le \left( \frac{C_1}{\eta_A T} + \frac{C_1}{\eta_C T} \right) + C_2 \eta_A + \frac{C_3}{y_A T} + C_4 \left( \sqrt{y_A} + \sqrt{y_C} \right) + \frac{C_5 \eta_C}{\sqrt{y_C}} \\
        &= \underbrace{\left( \frac{C_1}{\eta_A T} + C_2 \eta_A \right)}_{2L_1} + \underbrace{\left( \frac{C_3}{y_A T} + C_4 \sqrt{y_A} \right)}_{2L_2} + \underbrace{\left( \frac{C_1}{\eta_C T} + C_4 \sqrt{y_C} + \frac{C_5 \eta_C}{\sqrt{y_C}} \right)}_{3L_4} \\
        &= 2L_1 + 2L_2 + 3L_4.
    \end{align*}

    For \textbf{Family 3}, substituting \(\eta = \min\{\eta_A, \eta_C, \eta_D\}\) and \(y = \max\{y_A, y_C, y_D\}\) 
    and using \(\eta\leq \eta_A, \frac{1}{\eta}\leq \frac{1}{\eta_A} + \frac{1}{\eta_C} + \frac{1}{\eta_D}, \sqrt{y}\leq \sqrt{y_A} + \sqrt{y_C} + \sqrt{y_D}, \frac{1}{y}\leq \frac{1}{y_A}, \frac{\eta}{\sqrt{y}}\leq \frac{\eta_C}{\sqrt{y_C}}, \frac{\eta^2}{y}\leq \frac{\eta_D^2}{y_D}\) 
    extends the previous inequality by incorporating the quadratic term:
    \begin{align*}
        f_3(\eta, 1-\beta) &= f_2(\eta, 1-\beta) + \frac{C_6 \eta^2}{y} \\
        &\leq 2L_1 + 2L_2 + 3L_4 + \underbrace{\left( \frac{C_1}{\eta_D T} + C_4 \sqrt{y_D} + \frac{C_6 \eta_D^2}{y_D} \right)}_{3L_5} \\
        &= 2L_1 + 2L_2 + 3L_4 + 3L_5.
    \end{align*}
    This strictly recovers the bounds stated in the Lemma.
\end{proof}

\begin{lemma}[Optimal Bound with Selection of Step Size, Reset Probability and Batch Size]\label{lemma:opt_bound_stepsize_reset_batch}
    Let \(B, E, T\in \mathbb{Z}_+\), \(p=\frac{1}{E}\in (0, 1]\), \(\eta >0\). %
    For \(C_1,\ldots,C_6 \geq 0\) with \(C_1 C_2 C_4 C_5 C_6 \neq 0\) and \(C_3=0\), define the objective functions as follows:
    \begin{align*}
        f_1(\eta, p)_{C_3=0} &= \frac{C_1}{\eta T} + C_2 \eta + 0 + C_4 \sqrt{p} + \frac{C_5 \eta}{p}\\
        f_2(\eta, p)_{C_3=0} &= \frac{C_1}{\eta T} + C_2 \eta + 0 + C_4 \sqrt{p} + \frac{C_5 \eta}{\sqrt{p}}\\
        f_3(\eta, p)_{C_3=0} &= f_2(\eta, p) + \frac{C_6 \eta^2}{p}
    \end{align*}
    Letting \(\eta = \min\left\{\left( \frac{C_1}{C_2 T} \right)^{\frac{1}{2}}, \left( \frac{2 C_1^3}{C_4^2 C_5 T^3} \right)^{\frac{1}{4}}\right\}\) 
    and \(p=\frac{1}{E}\) with \(B = E = \left\lceil \left( \frac{C_4^2 T}{8 C_1 C_5} \right)^{\frac{1}{2}} \right\rceil\), 
    when \(T \geq \frac{8 C_1 C_5}{C_4^2}\)
    \[
    f_1(\eta, p)_{C_3=0} \leq 2 \left( \frac{C_1 C_2}{T} \right)^{\frac{1}{2}} + 4 \left( \frac{C_1 C_4^2 C_5}{2 T} \right)^{\frac{1}{4}}
    \, \textbf{(Family 1)}
    \]
    Letting \(\eta = \min\left\{\left( \frac{C_1}{C_2 T} \right)^{\frac{1}{2}}, \left( \frac{C_1^2}{\sqrt{2} C_4 C_5 T^2} \right)^{\frac{1}{3}}\right\}\) 
    and \(p=\frac{1}{E}\) with \(B = E = \left\lceil \left( \frac{C_4^4 T^2}{2 C_1^2 C_5^2} \right)^{\frac{1}{3}} \right\rceil\), 
    when \(T \geq \frac{\sqrt{2} C_1 C_5}{C_4^2}\)
    \[
    f_2(\eta, p)_{C_3=0} \leq 2 \left( \frac{C_1 C_2}{T} \right)^{\frac{1}{2}} + 3 \left( \frac{\sqrt{2} C_1 C_4 C_5}{T} \right)^{\frac{1}{3}}
    \, \textbf{(Family 2)}
    \]
    Letting \(\eta = \min\left\{\left( \frac{C_1}{C_2 T} \right)^{\frac{1}{2}}, \left( \frac{C_1^2}{\sqrt{2} C_4 C_5 T^2} \right)^{\frac{1}{3}}, \left( \frac{C_1^3}{4 C_4^2 C_6 T^3} \right)^{\frac{1}{5}}\right\}\)  
    and \(p=\frac{1}{E}\) with \(B = E = \left\lceil \min\left\{ \left( \frac{C_4^4 T^2}{2 C_1^2 C_5^2} \right)^{\frac{1}{3}}, \left( \frac{C_4^6 T^4}{16 C_1^4 C_6^2} \right)^{\frac{1}{5}} \right\} \right\rceil\), 
    when \(T \geq \max\left\{\frac{\sqrt{2} C_1 C_5}{C_4^2}, 2\left(\frac{C_1^2 C_6}{C_4^3}\right)^{\frac{1}{2}}\right\}\)
    \[
    f_3(\eta, p)_{C_3=0} \leq 2 \left( \frac{C_1 C_2}{T} \right)^{\frac{1}{2}} + 3 \left( \frac{\sqrt{2} C_1 C_4 C_5}{T} \right)^{\frac{1}{3}} + 5 \left( \frac{C_1^2 C_4^2 C_6}{8 T^2} \right)^{\frac{1}{5}}
    \, \textbf{(Family 3)}
    \]
\end{lemma}

\begin{proof}
    Before establishing the bounds, we define the following variables \(\eta_A, \eta_B, \eta_C, \eta_D\), the continuous batch size roots \(E_B, E_C, E_D\), and the AM-GM minimal values \(M_1, M_2, M_3, M_4\):
    \begin{align*}
        \eta_A &= \left( \frac{C_1}{C_2 T} \right)^{\frac{1}{2}} & & & M_1 &= \left( \frac{C_1 C_2}{T} \right)^{\frac{1}{2}} \\
        \eta_B &= \left( \frac{2 C_1^3}{C_4^2 C_5 T^3} \right)^{\frac{1}{4}} & E_B &= \left( \frac{C_4^2 T}{8 C_1 C_5} \right)^{\frac{1}{2}} & M_2 &= \left( \frac{C_1 C_4^2 C_5}{2 T} \right)^{\frac{1}{4}} \\
        \eta_C &= \left( \frac{C_1^2}{\sqrt{2} C_4 C_5 T^2} \right)^{\frac{1}{3}} & E_C &= \left( \frac{C_4^4 T^2}{2 C_1^2 C_5^2} \right)^{\frac{1}{3}} & M_3 &= \left( \frac{\sqrt{2} C_1 C_4 C_5}{T} \right)^{\frac{1}{3}} \\
        \eta_D &= \left( \frac{C_1^3}{4 C_4^2 C_6 T^3} \right)^{\frac{1}{5}} & E_D &= \left( \frac{C_4^6 T^4}{16 C_1^4 C_6^2} \right)^{\frac{1}{5}} & M_4 &= \left( \frac{C_1^2 C_4^2 C_6}{8 T^2} \right)^{\frac{1}{5}}
    \end{align*}

    For \textbf{Family 1}, substituting \(\eta = \min\{\eta_A, \eta_B\}, p = \frac{1}{E} = \frac{1}{\lceil E_B \rceil}\), 
    and noting that \(E_B \geq 1\) when \(T \geq \frac{8 C_1 C_5}{C_4^2}\) therefore \(E_B \leq \lceil E_B \rceil < E_B + 1\leq 2 E_B\),
    and using \(\eta \leq \eta_A, \frac{1}{\eta} \leq \frac{1}{\eta_A} + \frac{1}{\eta_B}, \sqrt{p} = \frac{1}{\sqrt{E}} \leq \frac{1}{\sqrt{E_B}}, \frac{\eta}{p} \leq \eta_B \cdot 2E_B\),
    we have:
    \begin{align*}
        f_1(\eta, p)_{C_3=0} &= \frac{C_1}{\eta T} + C_2 \eta + 0 + C_4 \sqrt{p} + \frac{C_5 \eta}{p} \\
        &\leq \left( \frac{C_1}{\eta_A T} + \frac{C_1}{\eta_B T} \right) + C_2 \eta_A + 0 
        + \frac{C_4}{\sqrt{E_B}} + 2C_5 \eta_B E_B \\
        &= \underbrace{\left( \frac{C_1}{\eta_A T} + C_2 \eta_A \right)}_{2M_1} 
        + \underbrace{\left( \frac{C_1}{\eta_B T} + \frac{C_4/2}{\sqrt{E_B}} + \frac{C_4/2}{\sqrt{E_B}} + 2C_5 \eta_B E_B \right)}_{4 M_2} \\
        &= 2M_1 + 4M_2.
    \end{align*}

    For \textbf{Family 2}, substituting \(\eta = \min\{\eta_A, \eta_C\}, p = \frac{1}{E} = \frac{1}{\lceil E_C \rceil}\), 
    and noting that \(E_C \geq 1\) when \(T \geq \frac{\sqrt{2} C_1 C_5}{C_4^2}\) therefore \(E_C \leq \lceil E_C \rceil < E_C + 1\leq 2 E_C\),
    and using \(\eta \leq \eta_A, \frac{1}{\eta} \leq \frac{1}{\eta_A} + \frac{1}{\eta_C}, \sqrt{p} = \frac{1}{\sqrt{E}} \leq \frac{1}{\sqrt{E_C}}, \frac{\eta}{\sqrt{p}} \leq \eta_C \cdot \sqrt{2E_C}\),
    we have:
    \begin{align*}
        f_2(\eta, p)_{C_3=0} &= \frac{C_1}{\eta T} + C_2 \eta + 0 + C_4 \sqrt{p} + \frac{C_5 \eta}{\sqrt{p}} \\
        &\leq \left( \frac{C_1}{\eta_A T} + \frac{C_1}{\eta_C T} \right) + C_2 \eta_A + 0 
        + \frac{C_4}{\sqrt{E_C}} + C_5 \eta_C \sqrt{2E_C} \\
        &= \underbrace{\left( \frac{C_1}{\eta_A T} + C_2 \eta_A \right)}_{2M_1} 
        + \underbrace{\left( \frac{C_1}{\eta_C T} + \frac{C_4}{\sqrt{E_C}} + C_5 \eta_C \sqrt{2E_C} \right)}_{3M_3} \\
        &= 2M_1 + 3M_3.
    \end{align*}

    For \textbf{Family 3}, substituting \(\eta = \min\{\eta_A, \eta_C, \eta_D\}, p = \frac{1}{E} = \frac{1}{\min\{\lceil E_C \rceil, \lceil E_D \rceil\}} = \frac{1}{\lceil \min\{E_C, E_D\} \rceil}\), 
    and noting that \(\min\{E_C, E_D\} \geq 1\) when \(T \geq \max\left\{\frac{\sqrt{2} C_1 C_5}{C_4^2}, 2\left(\frac{C_1^2 C_6}{C_4^3}\right)^{\frac{1}{2}}\right\}\) therefore \(E_C \leq \lceil E_C \rceil < E_C + 1\leq 2 E_C\), \(E_D \leq \lceil E_D \rceil < E_D + 1\leq 2 E_D\),
    and using \(\eta \leq \eta_A, \frac{1}{\eta} \leq \frac{1}{\eta_A} + \frac{1}{\eta_C} + \frac{1}{\eta_D}, \sqrt{p} = \frac{1}{\sqrt{E}} \leq \frac{1}{\sqrt{E_C}} + \frac{1}{\sqrt{E_D}}, \frac{\eta}{\sqrt{p}} \leq \eta_C \cdot \sqrt{2E_C}, \frac{\eta^2}{p} \leq \eta_D^2 \cdot 2E_D\),
    we have:
    \begin{align*}
        f_3(\eta, p)_{C_3=0} &= f_2(\eta, p)_{C_3=0} + \frac{C_6 \eta^2}{p} \\
        &\leq 2M_1 + 3M_3 + \underbrace{\left( \frac{C_1/2}{\eta_D T} + \frac{C_1/2}{\eta_D T} + \frac{C_4/2}{\sqrt{E_D}} + \frac{C_4/2}{\sqrt{E_D}} + 2C_6 \eta_D^2 E_D \right)}_{5M_4} \\
        &= 2M_1 + 3M_3 + 5M_4.
    \end{align*}
\end{proof}

\begin{lemma}[Iteration Complexity with Logarithmic Dependence]\label{lemma:complexity_log}
    For \(\varepsilon>0, \delta >0\), and any \(q>0\), choosing
    \[T \geq 
    \frac{\mathe^{q^2}}{\varepsilon} \log^q\left[\mathe \vee \frac{1}{\varepsilon \delta}\right]
    \implies \frac{\log^q(T / \delta)}{T} \leq \varepsilon \].
\end{lemma}

\begin{proof}
    From the premise, we multiply both sides by \(\varepsilon\), take the \(1/q\)-th power, 
    and isolate the maximum operator to obtain 
    \[ \log\left[\mathe \vee \frac{1}{\varepsilon \delta}\right] = \max\left(1, \log \frac{1}{\varepsilon \delta}\right) \leq \mathe^{-q} (T\varepsilon)^{1/q}\]

    The above inequality guarantee implies \(T\varepsilon \geq \mathe^{q^2}\). 
    Because the maximum of the function \(f(x) = \frac{\log x}{x^{1/q}}\) is bounded by \(1 - \mathe^{-q}\)
    for any \(x \geq \mathe^{q^2}\), we have:
    \[\log(T\varepsilon) \leq (1 - \mathe^{-q})(T\varepsilon)^{1/q}\]

    By leveraging \(y \leq \max(1, y)\), we can safely absorb any values of \(\varepsilon\delta\) and evaluate the logarithmic terms in a single, unbroken inequality chain:
    \begin{align*}
    \log\left(\frac{T}{\delta}\right) &= \log\left(\frac{1}{\varepsilon\delta}\right) + \log(T\varepsilon) 
    \leq \max\left(1, \log \frac{1}{\varepsilon\delta}\right) + \log(T\varepsilon) %
    \leq 
    (T\varepsilon)^{1/q}
    \end{align*}

    Since both sides are strictly positive, we raise the entire inequality to the power of \(q\):
    \(\log^q\left(\frac{T}{\delta}\right) \leq T\varepsilon\), dividing both sides by \(T\) completes the proof:
    \(\frac{\log^q(T/\delta)}{T} \leq \varepsilon\).
\end{proof}

\newpage
\section{Application 1: Stochastic Optimization with Mirror Descent}\label{apdx:application1}

\subsection{Theorem of First-Order Stationarity for Stochastic Optimization}
\begin{theorem}[First-Order Stationarity for Stochastic Optimization with Mirror Descent]\label{thm:mirror_stationarity}
Consider solving \eqref{eq:application1}, where \(\inf_{w\in\mathcal{W}} f(w)>-\infty\), via mirror descent \(w_{t+1}=w_t-\eta_t P(w_t, U_t, \eta_t)\) using the normalized gradient estimator \(U_t = v_t/\|v_t\|_*\). We establish the oracle and iteration complexities required to achieve \(\varepsilon\)-stationarity (\(\frac{1}{T} \sum_{t=1}^T \|\nabla f(w_t)\|_\ast \|P_t\|^2 \leq \varepsilon\)) with probability at least \(1-\delta\) under various configurations of the correction term, momentum \(\beta_t\), and reset probability \(p_t\). To highlight structural similarities, we extract a shared \textbf{base rate \(\rho\)} per family, use the \(\sigma_L, \sigma_\ell, \sigma_\gamma, \sigma_\alpha\) constants defined in \eqref{eq:define_simga}, and introduce \(\Theta_p(\cdot)\) to hide \(\log\frac{T}{\delta}\) and \(\kappa\) dependencies.

\noindent\textbf{Family 1 (zeroth-order correction).} 
Under Assumptions~\ref{ass:predictable_params},\ref{ass:subgaussian},\ref{ass:lipschitz}, using constant step size \(\eta_t = \eta = \frac{1}{L}\Theta_p(\sigma_L T^{-\frac{1}{2}} \wedge (\sigma_L^3/\sigma)^{\frac{1}{2}}T^{-\frac{3}{4}})\) and base rate \(\rho = \frac{\sigma_L}{\sigma}T^{-\frac{1}{2}}\):
\begin{itemize}
    \item \textbf{Momentum / SGD-M} (\(\beta_t=\beta, p_t=0\)): With momentum parameter \(1-\beta = \Theta_p(T^{-\frac{1}{3}} \vee \rho)\) and \(B=1\), oracle/iteration complexity is \(N,\,T = \mathcal{O}\left(\frac{\sigma_L^2\sigma^2}{\varepsilon^4}\log\frac{\sigma_L^2\sigma^2}{\varepsilon^4\delta} \vee \frac{\sigma^3}{\varepsilon^3}\log^{\frac{3}{2}}\frac{\sigma^3}{\varepsilon^3\delta}\right)\).
    \item \textbf{Probabilistic Momentum} (\(\beta_t=1, p_t=p\)): With reset probability \(p = \Theta_p(\rho)\) and \(B=1/p\), oracle/iteration complexity is \(N,\,T = \mathcal{O}\left(\frac{\sigma_L^2\sigma^2}{\varepsilon^4}\log^2\frac{\sigma_L^2\sigma^2}{\varepsilon^4\delta}\right)\).
    \item \textbf{Periodic Momentum} (\(\beta_t=1, p_t= \mathbb{I}_{\{t \mod E = 0\}}\)): With reset epoch \(1/E = \Theta_p(\rho)\) and \(B=E\), oracle/iteration complexity is \(N,\,T = \mathcal{O}\left(\frac{\sigma_L^2\sigma^2}{\varepsilon^4}\log\frac{\sigma_L^2\sigma^2}{\varepsilon^4\delta}\right)\).
\end{itemize}
\noindent\textbf{Family 2 (first-order correction).} 
Under Assumptions~\ref{ass:predictable_params},\ref{ass:subgaussian},\ref{ass:lipschitz},\ref{ass:subg_lipschitz},
using constant step size \(\eta_t = \eta = \Theta_p(\frac{1}{L}\sigma_L T^{-\frac{1}{2}} \wedge \frac{1}{\ell}(\sigma_\ell^4/\sigma)^{\frac{1}{3}}T^{-\frac{2}{3}})\) and base rate \(\rho = \left(\frac{\sigma_\ell}{\sigma}\right)^{\frac{4}{3}}T^{-\frac{2}{3}}\):
\begin{itemize}
    \item \textbf{STORM} (\(\beta_t=\beta, p_t=0\)): With momentum parameter \(1-\beta = \Theta_p(T^{-\frac{2}{3}} \vee \rho)\) and \(B=1\), oracle/iteration complexity is \(N,\,T = \mathcal{O}\left(\frac{\sigma^3}{\varepsilon^3}\log^{\frac{3}{2}}\frac{\sigma^3}{\varepsilon^3\delta} \vee \frac{\sigma_\ell^2\sigma}{\varepsilon^3}\log\frac{\sigma_\ell^2\sigma}{\varepsilon^3\delta} \vee \frac{\sigma_L^2}{\varepsilon^2}\right)\).
    \item \textbf{PAGE / Loopless SVRG} (\(\beta_t=1, p_t=p\)): With reset probability \(p = \Theta_p(\rho)\) and \(B=1/p\), oracle/iteration complexity is \(N,\,T = \mathcal{O}\left(\frac{\sigma_\ell^2\sigma}{\varepsilon^3}\log^{\frac{3}{2}}\frac{\sigma_\ell^2\sigma}{\varepsilon^3\delta} \vee \frac{\sigma_L^2}{\varepsilon^2}\right)\).
    \item \textbf{SPIDER / SVRG} (\(\beta_t=1, p_t= \mathbb{I}_{\{t \mod E = 0\}}\)): With reset epoch \(1/E = \Theta_p(\rho)\) and \(B=E\), oracle/iteration complexity is \(N,\,T = \mathcal{O}\left(\frac{\sigma_\ell^2\sigma}{\varepsilon^3}\log\frac{\sigma_\ell^2\sigma}{\varepsilon^3\delta} \vee \frac{\sigma_L^2}{\varepsilon^2}\right)\).
\end{itemize}
\noindent\textbf{Family 3 (second-order correction).} 
Under  Assumptions~\ref{ass:predictable_params},\ref{ass:subgaussian},\ref{ass:lipschitz},\ref{ass:subg_hessian},\ref{ass:smoothness}, 
using constant step size \(\eta_t = \eta = \Theta_p(\frac{1}{L}\sigma_L T^{-\frac{1}{2}} \wedge \frac{1}{\gamma}(\sigma_\gamma^4/\sigma)^{\frac{1}{3}}T^{-\frac{2}{3}} \wedge \sqrt{\frac{1}{\alpha}(\sigma_\alpha^9/\sigma^4)^{\frac{1}{5}}}T^{-\frac{3}{5}})\) and base rate \(\rho = \left(\frac{\sigma_\gamma}{\sigma}\right)^{\frac{4}{3}}T^{-\frac{2}{3}} \vee \left(\frac{\sigma_\alpha}{\sigma}\right)^{\frac{6}{5}}T^{-\frac{4}{5}}\):
\begin{itemize}
    \item \textbf{Second Order Momentum} (\(\beta_t=\beta, p_t=0\)): 

    With momentum parameter \(1-\beta = \Theta_p(T^{-\frac{2}{3}} \vee \rho)\) and \(B=1\), oracle/iteration complexity is \(N,\,T = \mathcal{O}\left(\frac{\sigma^3}{\varepsilon^3}\log^{\frac{3}{2}}\frac{\sigma^3}{\varepsilon^3\delta} \vee \frac{\sigma_\gamma^2\sigma}{\varepsilon^3}\log\frac{\sigma_\gamma^2\sigma}{\varepsilon^3\delta} \vee \sqrt{\frac{\sigma_\alpha^3\sigma^2}{\varepsilon^5}}\log^{\frac{1}{2}}\left(\frac{1}{\delta}\sqrt{\frac{\sigma_\alpha^3\sigma^2}{\varepsilon^5}}\right) \vee \frac{\sigma_L^2}{\varepsilon^2}\right)\).
    \item \textbf{Second Order PAGE} (\(\beta_t=1, p_t=p\)):

    With reset probability \(p = \Theta_p(\rho)\) and \(B=1/p\), oracle/iteration complexity is \(N,\,T = \mathcal{O}\left(\frac{\sigma_\gamma^2\sigma}{\varepsilon^3}\log^{\frac{3}{2}}\frac{\sigma_\gamma^2\sigma}{\varepsilon^3\delta} \vee \sqrt{\frac{\sigma_\alpha^3\sigma^2}{\varepsilon^5}}\log\left(\frac{1}{\delta}\sqrt{\frac{\sigma_\alpha^3\sigma^2}{\varepsilon^5}}\right) \vee \frac{\sigma_L^2}{\varepsilon^2}\right)\).
    \item \textbf{Second Order SPIDER} (\(\beta_t=1, p_t= \mathbb{I}_{\{t \mod E = 0\}}\)): 

    With reset epoch \(1/E = \Theta_p(\rho)\) and \(B=E\), oracle/iteration complexity is \(N,\,T = \mathcal{O}\left(\frac{\sigma_\gamma^2\sigma}{\varepsilon^3}\log\frac{\sigma_\gamma^2\sigma}{\varepsilon^3\delta} \vee \sqrt{\frac{\sigma_\alpha^3\sigma^2}{\varepsilon^5}}\log^{\frac{1}{2}}\left(\frac{1}{\delta}\sqrt{\frac{\sigma_\alpha^3\sigma^2}{\varepsilon^5}}\right) \vee \frac{\sigma_L^2}{\varepsilon^2}\right)\).
\end{itemize}
\end{theorem}

\begin{table*}[t]
\centering
\caption{Summary of Iteration (\(T\)) and Oracle (\(N\)) Complexities for Special Cases of the Unified Estimator to achieve \(\varepsilon\)-stationarity, i.e., \(\frac{1}{T}\sum_{t=1}^T \|\nabla f(w_t)\|_\ast \|P_t\|^2 \leq \varepsilon\). The \(\mathcal{O}(\cdot)\) notation is omitted from the complexity bounds, these \(\sigma_L, \sigma_\ell, \sigma_\gamma, \sigma_\alpha\) are defined in \eqref{eq:define_simga} and \(\sigma\) is from Assumption~\ref{ass:subgaussian}.}
\label{tab:complexity_summary}
\renewcommand{\arraystretch}{1.6}
\begin{tabular}{l || p{9.5cm}}
\hline\hline
\textbf{Algorithm} & \textbf{Complexity (\(N, T\))} \\
\hline\hline

\rowcolor{orange!10} \multicolumn{2}{l}{\textit{Family 1: Standard Recursion and Resetting (Zeroth-Order)} 
\quad \textbf{Asm:} \ref{ass:predictable_params},\ref{ass:subgaussian}, \ref{ass:lipschitz}} \\
Momentum (SGD-M) & \(\frac{\sigma_L^2\sigma^2}{\varepsilon^4}\log\frac{\sigma_L^2\sigma^2}{\varepsilon^4\delta} \vee \frac{\sigma^3}{\varepsilon^3}\log^{\frac{3}{2}}\frac{\sigma^3}{\varepsilon^3\delta}\) \\
Probabilistic Momentum & \(\frac{\sigma_L^2\sigma^2}{\varepsilon^4}\log^2\frac{\sigma_L^2\sigma^2}{\varepsilon^4\delta}\) \\
Periodic Momentum & \(\frac{\sigma_L^2\sigma^2}{\varepsilon^4}\log\frac{\sigma_L^2\sigma^2}{\varepsilon^4\delta}\) \\
\hline

\rowcolor{orange!10} \multicolumn{2}{l}{\textit{Family 2: Differential Recursion and Resetting (First-Order)} 
\quad \textbf{Asm:} \ref{ass:predictable_params},\ref{ass:subgaussian}, \ref{ass:lipschitz}, \ref{ass:subg_lipschitz}} \\
STORM & \(\frac{\sigma^3}{\varepsilon^3}\log^{\frac{3}{2}}\frac{\sigma^3}{\varepsilon^3\delta} \vee \frac{\sigma_\ell^2\sigma}{\varepsilon^3}\log\frac{\sigma_\ell^2\sigma}{\varepsilon^3\delta} \vee \frac{\sigma_L^2}{\varepsilon^2}\) \\
PAGE / Loopless SVRG & \(\frac{\sigma_\ell^2\sigma}{\varepsilon^3}\log^{\frac{3}{2}}\frac{\sigma_\ell^2\sigma}{\varepsilon^3\delta} \vee \frac{\sigma_L^2}{\varepsilon^2}\) \\
SPIDER / SVRG & \(\frac{\sigma_\ell^2\sigma}{\varepsilon^3}\log\frac{\sigma_\ell^2\sigma}{\varepsilon^3\delta} \vee \frac{\sigma_L^2}{\varepsilon^2}\) \\
\hline

\rowcolor{orange!10} \multicolumn{2}{l}{\textit{Family 3: Second-Order Recursion and Resetting (Second-Order)} 
\quad \textbf{Asm:} \ref{ass:predictable_params},\ref{ass:subgaussian}, \ref{ass:lipschitz}, \ref{ass:subg_hessian}, \ref{ass:smoothness}} \\
Second-Order Momentum & \(\frac{\sigma^3}{\varepsilon^3}\log^{\frac{3}{2}}\frac{\sigma^3}{\varepsilon^3\delta} \vee \frac{\sigma_\gamma^2\sigma}{\varepsilon^3}\log\frac{\sigma_\gamma^2\sigma}{\varepsilon^3\delta} \vee \sqrt{\frac{\sigma_\alpha^3\sigma^2}{\varepsilon^5}}\log^{\frac{1}{2}}\left(\frac{1}{\delta}\sqrt{\frac{\sigma_\alpha^3\sigma^2}{\varepsilon^5}}\right) \vee \frac{\sigma_L^2}{\varepsilon^2}\) \\
Second-Order PAGE & \(\frac{\sigma_\gamma^2\sigma}{\varepsilon^3}\log^{\frac{3}{2}}\frac{\sigma_\gamma^2\sigma}{\varepsilon^3\delta} \vee \sqrt{\frac{\sigma_\alpha^3\sigma^2}{\varepsilon^5}}\log\left(\frac{1}{\delta}\sqrt{\frac{\sigma_\alpha^3\sigma^2}{\varepsilon^5}}\right) \vee \frac{\sigma_L^2}{\varepsilon^2}\) \\
Second-Order SPIDER & \(\frac{\sigma_\gamma^2\sigma}{\varepsilon^3}\log\frac{\sigma_\gamma^2\sigma}{\varepsilon^3\delta} \vee \sqrt{\frac{\sigma_\alpha^3\sigma^2}{\varepsilon^5}}\log^{\frac{1}{2}}\left(\frac{1}{\delta}\sqrt{\frac{\sigma_\alpha^3\sigma^2}{\varepsilon^5}}\right) \vee \frac{\sigma_L^2}{\varepsilon^2}\) \\
\hline\hline

\end{tabular}
\end{table*}

\begin{table*}[!tbhp]
\centering
\caption{Configuration Details: Selection of Correction Term (\(\mathcal{T}_t\)), constant Step Size (\(\eta_t = \eta\)), and Momentum/Reset Parameters \((\beta_t, p_t)\). \(\Theta_p(\cdot)\) hides the dependencies of \(\log\frac{T}{\delta}\) and \(\kappa\), omitted for brevity and ease of comparison. To highlight the structural similarity across algorithms, we extract a shared \textbf{base rate \(\rho\)} for each family. The probabilistic reset \(p\) and periodic reset \(1/E\) exactly match \(\rho\), while the momentum parameter uses \(1-\beta = \text{base-decay} \vee \rho\).}
\label{tab:parameter_configurations}
\renewcommand{\arraystretch}{1.6}
\begin{tabular}{l || c | c | r @{\,\(=\)\,} l | c}
\hline\hline
\textbf{Algorithm} & \(\beta_t\) & \(p_t\) & \multicolumn{2}{c|}{\textbf{Parameter Selection}} & \textbf{Batch size} \(B\) \\
\hline\hline

\rowcolor{orange!10} \multicolumn{6}{l}{\textit{Family 1: Standard Recursion and Resetting (Zeroth-Order)} \quad \textbf{Asm:} \ref{ass:predictable_params},\ref{ass:subgaussian}, \ref{ass:lipschitz}} \\
\rowcolor{orange!10} \multicolumn{6}{l}{\quad \textbf{Correction Term} \(\mathcal{T}_t = \beta_t \bigl(G(w_{t-1},\xi_t)- G(w_t,\xi_t)\bigr)\)} \\
\rowcolor{orange!10} \multicolumn{6}{l}{\quad \textbf{Step Size} \(\eta = \frac{1}{L}\Theta_p\left( \sigma_L T^{-\frac{1}{2}} \wedge \left(\frac{\sigma_L^3}{\sigma}\right)^{\frac{1}{2}}T^{-\frac{3}{4}} \right)\) \quad \textbf{Base Rate \(\rho = \frac{\sigma_L}{\sigma}T^{-\frac{1}{2}}\)}} \\
Momentum (SGD-M)       & \(\beta\) & \(0\)                           & \(1-\beta\) & \(\Theta_p(T^{-\frac{1}{3}} \vee \rho)\) & \(1\) \\
Probabilistic Momentum & \(1\)     & \(p\)                           & \(p\)       & \(\Theta_p(\rho)\)                       & \(1/p\) \\
Periodic Momentum      & \(1\)     & \(\mathbb{I}\{t \pmod E = 0\}\) & \(1/E\)     & \(\Theta_p(\rho)\)                       & \(E\) \\
\hline

\rowcolor{orange!10} \multicolumn{6}{l}{\textit{Family 2: Differential Recursion and Resetting (First-Order)} \quad \textbf{Asm:} \ref{ass:predictable_params},\ref{ass:subgaussian}, \ref{ass:lipschitz}, \ref{ass:subg_lipschitz}} \\
\rowcolor{orange!10} \multicolumn{6}{l}{\quad \textbf{Correction Term} \(\mathcal{T}_t = 0\)} \\
\rowcolor{orange!10} \multicolumn{6}{l}{\quad \textbf{Step Size} \(\eta = \Theta_p\left(\frac{1}{L}\sigma_L T^{-\frac{1}{2}} \wedge \frac{1}{\ell}\left(\frac{\sigma_\ell^4}{\sigma}\right)^{\frac{1}{3}}T^{-\frac{2}{3}}\right)\) \quad \textbf{Base Rate \(\rho = \left(\frac{\sigma_\ell}{\sigma}\right)^{\frac{4}{3}}T^{-\frac{2}{3}}\)}} \\
STORM                & \(\beta\) & \(0\)                           & \(1-\beta\) & \(\Theta_p(T^{-\frac{2}{3}} \vee \rho)\) & \(1\) \\
PAGE / Loopless SVRG & \(1\)     & \(p\)                           & \(p\)       & \(\Theta_p(\rho)\)                       & \(1/p\) \\
SPIDER / SVRG        & \(1\)     & \(\mathbb{I}\{t \pmod E = 0\}\) & \(1/E\)     & \(\Theta_p(\rho)\)                       & \(E\) \\
\hline

\rowcolor{orange!10} \multicolumn{6}{l}{\textit{Family 3: Second-Order Recursion and Resetting (Second-Order)} \quad \textbf{Asm:} \ref{ass:predictable_params},\ref{ass:subgaussian}, \ref{ass:lipschitz}, \ref{ass:subg_hessian}, \ref{ass:smoothness}} \\
\rowcolor{orange!10} \multicolumn{6}{l}{\quad \textbf{Correction Term} \(\mathcal{T}_t = \beta_t \bigl(G(w_{t-1},\xi_t)-[G(w_t,\xi_t)+\nabla G(w_t,\xi_t)(w_{t-1}-w_t)]\bigr)\)} \\
\rowcolor{orange!10} \multicolumn{6}{l}{\quad \textbf{Step Size} \(\eta = \Theta_p\left(\frac{1}{L}\sigma_L T^{-\frac{1}{2}}\wedge \frac{1}{\gamma}\left(\frac{\sigma_\gamma^4}{\sigma}\right)^{\frac{1}{3}}T^{-\frac{2}{3}}\wedge \sqrt{\frac{1}{\alpha}\left(\frac{\sigma_\alpha^9}{\sigma^4}\right)^{\frac{1}{5}}}T^{-\frac{3}{5}}\right)\)} \\
\rowcolor{orange!10} \multicolumn{6}{l}{\quad \textbf{Base Rate \(\rho = \left(\frac{\sigma_\gamma}{\sigma}\right)^{\frac{4}{3}}T^{-\frac{2}{3}} \vee \left(\frac{\sigma_\alpha}{\sigma}\right)^{\frac{6}{5}}T^{-\frac{4}{5}}\)}} \\
Second-Order Momentum  & \(\beta\) & \(0\)                           & \(1-\beta\) & \(\Theta_p(T^{-\frac{2}{3}} \vee \rho)\) & \(1\) \\
Second-Order PAGE    & \(1\)     & \(p\)                           & \(p\)       & \(\Theta_p(\rho)\)                       & \(1/p\) \\
Second-Order SPIDER  & \(1\)     & \(\mathbb{I}\{t \pmod E = 0\}\) & \(1/E\)     & \(\Theta_p(\rho)\)                       & \(E\) \\
\hline\hline

\end{tabular}
\end{table*}

\newpage
\subsection{Proof for Theorem~\ref{thm:mirror_stationarity} of First-Order Stationarity for Stochastic Optimization}
\begin{proof}[Proof for Theorem~\ref{thm:mirror_stationarity}, Application 1: Stochastic Optimization with Mirror Descent]
    The update rule of the Mirror Descent framework is given by
    \(w_{t+1} = w_t - \eta_t P(w_t, U_t, \eta_t)\) using proximal gradient mapping \(P(w_t, U_t, \eta_t)\) with update vector \(U_t\) and step size \(\eta_t\).
    Moreover, in Application 1: Unconstrained Optimization, the update vector \(U_t = \frac{v_t}{\|v_t\|_\ast}\) is the normalized direction of the unified estimator \(v_t\) 
    of the objective gradient \(g(w_t) = \nabla f(w_t)\) at the \(t\)-th iteration.
    \[
    w_{t+1} = w_t - \eta_t P(w_t, U_t, \eta_t) 
    = w_t - \eta_t P\left(w_t, \frac{v_t}{\|v_t\|_\ast}, \eta_t\right) 
    \]
    By smoothness of \(f(w)\) (Assumption~\ref{ass:lipschitz} for \(g(w)=\nabla f(w)\))  and the above update rule, we have
    \[
    f(w_{t+1}) - f(w_t) \le - \eta_t \left\langle \nabla f(w_t), P\left(w_t, \frac{v_t}{\|v_t\|_\ast}, \eta_t\right) \right\rangle + \frac{L \eta_t^2}{2} 
    \left\| P\left(w_t, \frac{v_t}{\|v_t\|_\ast}, \eta_t\right) \right\|^2
    \]
    We upperbound the 1st term on the right-hand side by 
    noting the tracking error \(e_t \equiv v_t - g(w_t) = v_t - \nabla f(w_t)\) hence \(v_t = \nabla f(w_t) + e_t\),
    and invoking Lemma~\ref{lemma:bound_for_proximal_gradient_mapping_with_normalized_direction} 
    with \(u\gets \nabla f(w_t), \delta\gets e_t, \eta\gets \eta_t\).
    \[
    -\eta_t \left\langle \nabla f(w_t), P\left(w_t, \frac{v_t}{\|v_t\|_\ast}, \eta_t\right) \right\rangle 
    \leq -\frac{\eta_t}{4} \left\langle \nabla f(w_t), P\left(w_t, \frac{\nabla f(w_t)}{\|\nabla f(w_t)\|_\ast}, \frac{\eta_t}{2}\right) \right\rangle 
    + 2 \eta_t \|e_t\|_\ast
    \]
    We upperbound the 2nd term on the right-hand side by 
    using the properties of the proximal gradient mapping (Lemma~\ref{lemma:properties_proximal_gradient_mapping})
    such that \(\|P(w, u, \eta)\| \leq \|u\|_\ast\), we have:
    \[
    \frac{L \eta_t^2}{2} \left\| P\left(w_t, \frac{v_t}{\|v_t\|_\ast}, \eta_t\right) \right\|^2
    \leq \frac{L \eta_t^2}{2} \left\| \frac{v_t}{\|v_t\|_\ast} \right\|_\ast^2 = \frac{L \eta_t^2}{2} 
    \]
    Substituting the above upper bounds for 1st and 2nd terms on the right-hand side, multiplying by 4:
    \[
    \eta_t \left\langle \nabla f(w_t), P\left(w_t, \frac{\nabla f(w_t)}{\|\nabla f(w_t)\|_\ast}, \frac{\eta_t}{2}\right) \right\rangle 
    \leq 4\left[f(w_t) - f(w_{t+1})\right] + 2L \eta_t^2 + 8 \eta_t \|e_t\|_\ast
    \]

    By introducing the infimum \(\inf_{w\in\mathcal{W}} f(w)\) of the objective \(f(w)\) over \(\mathcal{W}\),
    then \(\inf_{w\in\mathcal{W}} f(w) \leq f(w_{T+1})\), and summing up the above inequality from \(t=1\) to \(T\):
    \[
    \sum_{t=1}^T \eta_t \left\langle \nabla f(w_t), P\left(w_t, \frac{\nabla f(w_t)}{\|\nabla f(w_t)\|_\ast}, \frac{\eta_t}{2}\right) \right\rangle 
    \leq 4\left[f(w_1) - \inf_{w\in\mathcal{W}}f(w)\right] + 2L \sum_{t=1}^T \eta_t^2 + 8 \sum_{t=1}^T \eta_t \|e_t\|_\ast
    \]
    Applying the inner product lower bound in properties of the proximal gradient mapping (Lemma~\ref{lemma:properties_proximal_gradient_mapping})
    such that \(\|P(w, u, \eta)\|^2 \leq \langle u, P(w, u, \eta)\rangle\), we have:
    \[
    \sum_{t=1}^T \eta_t \|\nabla f(w_t)\|_\ast \left\|P\left(w_t, \frac{\nabla f(w_t)}{\|\nabla f(w_t)\|_\ast}, \frac{\eta_t}{2}\right)\right\|^2
    \leq 4[f(w_1) - \inf_{w\in\mathcal{W}}f(w)] + 2L \sum_{t=1}^T \eta_t^2 + 8 \sum_{t=1}^T \eta_t \|e_t\|_\ast
    \]

    \noindent\textbf{Upper bound on stationarity for all special cases.}
    Since \(\eta_t=\eta\) are constants in all special cases, divide the inequality above by \(\eta T\):
    \[
    \frac{1}{T}\sum_{t=1}^T \|\nabla f(w_t)\|_\ast
    \left\|P\left(w_t,\frac{\nabla f(w_t)}{\|\nabla f(w_t)\|_\ast},\frac{\eta}{2}\right)\right\|^2
    \le
    \frac{4[f(w_1)-\inf_{w\in\mathcal{W}}f(w)]}{\eta T}+2L\eta+\frac{8}{T}\sum_{t=1}^T\|e_t\|_\ast.
    \]

    We define these notations for brevity, 
    \(P_t \coloneqq P\!\left(w_t,\frac{\nabla f(w_t)}{\|\nabla f(w_t)\|_\ast},\frac{\eta}{2}\right), \; \Delta_f\coloneqq f(w_1)-\inf_{w\in\mathcal{W}}f(w).\)
    Also, we have the trivial bound
    \(\frac{1}{T}\sum_{t=1}^T\|P_t\|^2\le 1,\)
    since \(\|P(w,u,\eta)\|\le\|u\|_\ast\) by Lemma~\ref{lemma:properties_proximal_gradient_mapping} and therefore
    \(\|P_t\|\leq\left\|\frac{\nabla f(w_t)}{\|\nabla f(w_t)\|_\ast}\right\|_\ast=1\).
    \[
    \frac{1}{T}\sum_{t=1}^T\|\nabla f(w_t)\|_\ast\|P_t\|^2
    \le \frac{4\Delta_f}{\eta T}+2L\eta+\frac{8}{T}\sum_{t=1}^T\|e_t\|_\ast.
    \]
    Substituting the upper bounds of \(\sum_{t=1}^T\|e_t\|_\ast\) for three families  
    with \(\|U_t\|_\ast =\left\|\frac{v_t}{\|v_t\|_\ast}\right\|_\ast=1\) (Assumption~\ref{ass:update_bound} is automatically satisfied with \(G=1\) and no need for extra assumption), 
    and introducing \(\sigma_L := (\Delta_f L)^{\frac{1}{2}}\) with \(L\) from Assumption~\ref{ass:lipschitz}, 
    \(\sigma_\ell := (\Delta_f \ell)^{\frac{1}{2}}\) with \(\ell\) from Assumption~\ref{ass:subg_lipschitz},
    \(\sigma_\gamma := (\Delta_f \gamma)^{\frac{1}{2}}\) with \(\gamma\) from Assumption~\ref{ass:subg_hessian},
    \(\sigma_\alpha := (\Delta_f^2 \alpha)^{\frac{1}{3}}\) with \(\alpha\) from Assumption~\ref{ass:smoothness}, 
    we obtain the following upperbounds of \(\frac{1}{T}\sum_{t=1}^T \|\nabla f(w_t)\|_\ast\!\left\|P_t\right\|^2\) for three cases in each of three families:

        \textbf{Family 1, Case 1 (Momentum (SGD-M))} \((p_t=0,\beta_t=\beta)\):
        By invoking the established upper bound of \(\|e_t\|_*\) for the estimator \(v_t\) with zeroth-order correction term \(\mathcal{T}_t\), the selection of \(p_t=0, \beta_t =\beta\) and \(\eta_t=\eta\) in the third-level section of Appendix~\ref{subsubsec:family1case1}, and noting that \(\sum_{t=1}^T\beta^t = \frac{\beta(1-\beta^{T})}{1-\beta}\leq \frac{1}{1-\beta}\) for \(\beta\in(0, 1)\), we show the following upper bound for \(\frac{1}{T}\sum_{t=1}^T\|e_t\|_*\):
        \[
        \frac{1}{T}\sum_{t=1}^T \|e_t\|_* \leq \frac{C\left(\frac{\delta}{2T},\kappa\right) \sigma/\sqrt{B}}{(1-\beta)T}
        + \frac{L\eta}{1-\beta} + C\left(\frac{\delta}{2T},\kappa\right) \sigma \sqrt{1-\beta}
        \]

        By selecting the batch size \(B = 1\) and 
        noting \(C\!\left(\frac{\delta}{2T},\kappa\right) \leq 2 \sqrt{2[\kappa \vee \log\frac{2T}{\delta}]}\), we have:
        \begin{eqnarray*}
        & &\frac{1}{T}\sum_{t=1}^T \|\nabla f(w_t)\|_*\|P_t\|^2\leq\frac{4\Delta_f}{\eta T}+2L\eta
        +8 \cdot \frac{1}{T}\sum_{t=1}^T \|e_t\|_*
        \\
        & \leq & \frac{4 \Delta_f}{\eta T} + 2L\cdot \eta + \frac{16 \sigma\sqrt{2[\kappa\vee\log\frac{2T}{\delta}]}}{(1-\beta)T}
        + 16\sigma \sqrt{2[\kappa \vee\log\frac{2T}{\delta}]}\cdot \sqrt{1-\beta} + 8L\cdot \frac{\eta}{1-\beta}
        \end{eqnarray*}
        By letting \(C_1 = 4 \Delta_f, C_2 = 2L, C_3 = 16 \sigma\sqrt{2[\kappa\vee\log\frac{2T}{\delta}]}, C_4 = 16 \sigma \sqrt{2[\kappa \vee \log\frac{2T}{\delta}]}, C_5 = 8 L\) 
        in Lemma~\ref{lemma:opt_bound_stepsize_momentum},
        selecting the step size and momentum parameter as
        \(\eta = \min\Big\{\left(\frac{C_1}{C_2 T}\right)^{\frac{1}{2}}, \allowbreak \left(\frac{C_1^3}{C_4^2 C_5 T^3}\right)^{\frac{1}{4}}\Big\} = \frac{1}{L}\min\Big\{\sqrt{2} \sigma_L T^{-\frac{1}{2}}, \allowbreak \left(\frac{\sigma_L^6}{64\sigma^2}\right)^{\frac{1}{4}} [\kappa \vee \log\frac{2T}{\delta}]^{-\frac{1}{4}} T^{-\frac{3}{4}} \Big\},\)
        \(1-\beta = \max\Big\{\left(\frac{C_3}{C_4 T}\right)^{\frac{2}{3}}, \allowbreak \left(\frac{C_1 C_5}{C_4^2 T}\right)^{\frac{1}{2}}\Big\} = \max\Big\{T^{-\frac{2}{3}}, \allowbreak \left(\frac{\sigma_L^2/\sigma^2}{16[\kappa \vee \log\frac{2T}{\delta}]}\right)^{\frac{1}{2}} T^{-\frac{1}{2}}\Big\}\)
        when \(T[\kappa \vee \log\frac{2T}{\delta}]\geq \frac{\sigma_L^2}{16 \sigma^2}\), we have:
        \begin{eqnarray*}
            & &\frac{1}{T}\sum_{t=1}^T\|\nabla f(w_t)\|_\ast \|P_t\|^2
            \leq 
            2\left(\frac{C_1 C_2}{T}\right)^{\frac{1}{2}} + 2\left(\frac{C_3 C_4^2}{T}\right)^{\frac{1}{3}}
            + 3\left(\frac{C_1 C_4^2 C_5}{T}\right)^{\frac{1}{4}}
            \\
            &\leq& 
            4\sqrt{2} \sigma_L T^{-\frac{1}{2}} +  
            32\sqrt{2}\sigma \left(\frac{[\kappa \vee \log\frac{2T}{\delta}]^{\frac{3}{2}}}{T}\right)^{\frac{1}{3}}
            + 24 \sqrt{\sigma_L \sigma}\left(\frac{4[\kappa \vee \log\frac{2T}{\delta}]}{T}\right)^{\frac{1}{4}}
        \end{eqnarray*}
        To analyze the iteration complexity, 
        the following inequalities of \(T\) are sufficient to show that \(\frac{1}{T}\sum_{t=1}^T\|\nabla f(w_t)\|_\ast \|P_t\|^2 \leq \varepsilon\) 
        for some \(\varepsilon > 0\) as 1st, 2nd and 3rd terms are upper bounded by \(\varepsilon/2, \varepsilon/4\) and \(\varepsilon/4\) respectively,
        noting \(\kappa\geq 1\) and \(T\geq \frac{\sigma_L^2/\sigma^2}{16\kappa}\) 
        implies \(T[\kappa \vee \log\frac{2T}{\delta}]\geq T\kappa \geq \frac{\sigma_L^2}{16 \sigma^2}\):
        \begin{align*}
            T \geq& \frac{\sigma_L^2/\sigma^2}{16\kappa} \vee \frac{128 \sigma_L^2}{\varepsilon^2} \vee \frac{\kappa \sigma^3}{(\varepsilon/128\sqrt{2})^3}
            \vee \frac{4\kappa\cdot \sigma_L^2\sigma^2}{(\varepsilon/96)^4},\\
            &\frac{\log^{\frac{3}{2}}\frac{2T}{\delta}}{T} \leq \left(\frac{\varepsilon}{128\sqrt{2}\sigma}\right)^3,\,
            \frac{\log\frac{2T}{\delta}}{T}\leq\frac{1}{4}\left( \frac{\varepsilon}{96\sqrt{\sigma_L \sigma}}\right)^4
        \end{align*}
        By substituting \(\delta \gets \delta/2\) and \((q,\varepsilon) \gets (\frac{3}{2},\left(\frac{\varepsilon}{128\sqrt{2}\sigma}\right)^3), (1,\frac{1}{4}\left( \frac{\varepsilon}{96\sqrt{\sigma_L \sigma}}\right)^4)\) in Lemma~\ref{lemma:complexity_log}, the following condition of \(T\)
        is sufficient to guarantee the above inequalities,
        and therefore \(\frac{1}{T}\sum_{t=1}^T\|\nabla f(w_t)\|_\ast \|P_t\|^2 \leq \varepsilon\).
        \begin{align*}
        T\geq& \frac{\sigma_L^2/\sigma^2}{16\kappa} \vee \frac{128 \sigma_L^2}{\varepsilon^2}
        \vee \frac{\sigma^3}{(\varepsilon/128\sqrt{2})^3}\left[\kappa\vee\mathe^{\frac{9}{4}}\log^{\frac{3}{2}}\left[\mathe\vee \frac{2\sigma^3}{(\varepsilon/128\sqrt{2})^3\delta} \right]\right]\\
        &\vee \frac{4\cdot \sigma_L^2\sigma^2}{(\varepsilon/96)^4}\left[
            \kappa\vee\mathe \log\left[\mathe\vee \frac{8\cdot \sigma_L^2\sigma^2}{(\varepsilon/96)^4\delta} \right]
        \right]
        \end{align*}
        In particular, for all \(\varepsilon \in \mathbb{R}_+\) is sufficiently small such 
        that \(\varepsilon \leq \varepsilon_0(\sigma_L, \sigma, \kappa, \delta)\),
        then the lower bound on \(T\) simplifies asymptotically to
        \[
            T \;\ge\; \max\Big\{\frac{4\mathe\cdot\sigma_L^2 \sigma^2}{(\varepsilon/96)^4}
            \log\frac{8\sigma_L^2 \sigma^2}{(\varepsilon/96)^4 \delta},\allowbreak \;
            \frac{\mathe^{\frac{9}{4}}\cdot\sigma^3}{(\varepsilon/128\sqrt{2})^3}\log^{\frac{3}{2}}\frac{2\sigma^3}{(\varepsilon/128\sqrt{2})^3\delta}\Big\}
        \]
        where the uniform upper bound \(\varepsilon_0(\sigma_L, \sigma, \kappa, \delta)\) of \(\varepsilon\) is obtained by noting \(\kappa\geq 1\) and solving the following inequalities for \(\varepsilon\)
        to ensure the dominated terms in the above asymptotic bound are larger than the other terms
        of the lower bound on \(T\).
        \[
        \left\{
        \begin{aligned}
         &\frac{2\sigma^3}{(\varepsilon/128\sqrt{2})^3\delta} \geq \mathe,\,
         \frac{8\sigma_L^2\sigma^2}{(\varepsilon/96)^4\delta} \geq \mathe\\
         &\mathe \log\frac{8\sigma_L^2\sigma^2}{(\varepsilon/96)^4\delta} \geq \kappa,\,
         \mathe^{\frac{9}{4}}\log^{\frac{3}{2}}\frac{2\sigma^3}{(\varepsilon/128\sqrt{2})^3\delta} 
         \geq \kappa\\
         &\frac{4\mathe \cdot \sigma_L^2 \sigma^2}{(\varepsilon/96)^4} \geq \frac{\sigma_L^2/\sigma^2}{16\kappa},\,
         \frac{4\mathe\cdot \sigma_L^2 \sigma^2}{(\varepsilon/96)^4} \geq \frac{128\sigma_L^2}{\varepsilon^2}
        \end{aligned}
        \right.
        \]
        The explicit expression of \(\varepsilon_0(\sigma_L, \sigma, \kappa, \delta) \equiv \sqrt{\sigma_L \sigma} \cdot \varepsilon_1(\kappa, \delta) \vee \sigma \cdot \varepsilon_2(\kappa, \delta) \) is given by:
        \[
        \begin{aligned}
        \varepsilon_1(\kappa, \delta) &= 96 \left( \frac{8}{\delta \, \mathe^{1 \vee (\kappa/\mathe)}} \right)^{\!1/4}\\
        \varepsilon_2(\kappa, \delta) &= 1152\sqrt{2}\mathe^{1/2} \;\wedge\; 192\sqrt{2}(\mathe\kappa)^{\!1/4} \;\wedge\; 128\sqrt{2} \left( \frac{2}{\delta \, \mathe^{1 \vee (\kappa^{2/3}/\mathe^{3/2})}} \right)^{\!1/3}        \end{aligned}
        \]
        Therefore, the iteration complexity is 
        \(T = \mathcal{O}\left(\frac{\sigma_L^2\sigma^2}{\varepsilon^4}\log\frac{\sigma_L^2\sigma^2}{\varepsilon^4\delta} \vee \frac{\sigma^3}{\varepsilon^3}\log^{\frac{3}{2}}\frac{\sigma^3}{\varepsilon^3\delta}\right)\),
        and the oracle complexity is \(N = B\cdot T = \mathcal{O}\left(\frac{\sigma_L^2\sigma^2}{\varepsilon^4}\log\frac{\sigma_L^2\sigma^2}{\varepsilon^4\delta} \vee \frac{\sigma^3}{\varepsilon^3}\log^{\frac{3}{2}}\frac{\sigma^3}{\varepsilon^3\delta}\right)\)
        with the selection of batch size \(B=1\) 
        to guarantee \(\frac{1}{T}\sum_{t=1}^T\|\nabla f(w_t)\|_\ast \|P_t\|^2 \leq \varepsilon\) with probability at least \(1-\delta\)
        for sufficiently small \(\varepsilon\).

        \textbf{Family 1, Case 2 (Probabilistic Momentum)} \((p_t=p,\beta_t=1)\): 
        By invoking the established upper bound of \(\|e_t\|_*\) for the estimator \(v_t\) with zeroth-order correction term \(\mathcal{T}_t\), the selection of \(p_t=p, \beta_t =1\) and \(\eta_t=\eta\) in the third-level section of Appendix~\ref{subsubsec:family1case2}, we show the following upper bound of \(\frac{1}{T}\sum_{t=1}^T\|e_t\|_*\):
        \[
        \frac{1}{T}\sum_{t=1}^T\|e_t\|_* \leq C\left(\frac{\delta}{4T},\kappa\right) \sigma/\sqrt{B}+\frac{L\eta}{p}\log\frac{4T}{\delta}
        \]

        By selecting the batch size \(B = \frac{1}{p}\), 
        and noting \(C\left(\frac{\delta}{4T},\kappa\right)  \leq 2 \sqrt{2[\kappa \vee\log\frac{4T}{\delta}]}\), we have:
        \begin{eqnarray*}
        & & \frac{1}{T}\sum_{t=1}^T \|\nabla f(w_t)\|_* \|P_t\|^2 \leq\frac{4\Delta_f}{\eta T}+2L\eta
        +8 \cdot \frac{1}{T}\sum_{t=1}^T\|e_t\|_*
        \\
        & \leq & \frac{4\Delta_f}{\eta T}+2L\cdot \eta
        + 0 + 16 \sigma\sqrt{2[\kappa \vee\log\frac{4T}{\delta}]}\cdot \sqrt{p}
        + 8L\log\frac{4T}{\delta} \cdot\frac{\eta}{p}
        \end{eqnarray*}
        Letting \(C_1 = 4 \Delta_f, C_2 = 2L, C_3 = 0, C_4 = 16 \sigma\sqrt{2[\kappa \vee \log\frac{4T}{\delta}]}, C_5 = 8L\log\frac{4T}{\delta}\) in Lemma~\ref{lemma:opt_bound_stepsize_reset_batch},
        selecting step size and reset probability as
        \(\eta = \min\Big\{\left(\frac{C_1}{C_2}\right)^{\frac{1}{2}}T^{-\frac{1}{2}}, \allowbreak \left(\frac{2 C_1^3}{C_4^2 C_5}\right)^{\frac{1}{4}} T^{-\frac{3}{4}} \Big\} =\frac{1}{L}\min\Big\{\sqrt{2} \sigma_L T^{-\frac{1}{2}}, \allowbreak \frac{1}{2}\left(\frac{\sigma_L^6}{2\sigma^2}\right)^{\frac{1}{4}} [\kappa\vee\log\frac{4T}{\delta}]^{-\frac{1}{4}}[\log\frac{4T}{\delta}]^{-\frac{1}{4}} T^{-\frac{3}{4}} \Big\},\)
        \(B = \frac{1}{p} = \left\lceil \left( \frac{C_4^2 T}{8 C_1 C_5} \right)^{\frac{1}{2}} \right\rceil = \left\lceil \frac{\sqrt{2}\sigma}{\sigma_L}\left(\frac{\kappa \vee \log\frac{4T}{\delta}}{\log\frac{4T}{\delta}}\right)^{\frac{1}{2}} T^{\frac{1}{2}} \right\rceil\), then 
        when \(T\left(\frac{\kappa \vee \log\frac{4T}{\delta}}{\log\frac{4T}{\delta}}\right) \geq \frac{\sigma_L^2}{2\sigma^2}\), we have:
        \begin{eqnarray*}
            & &\frac{1}{T}\sum_{t=1}^T\|\nabla f(w_t)\|_\ast \|P_t\|^2
            \leq 2\left(\frac{C_1 C_2}{T}\right)^{\frac{1}{2}} 
            + 4\left(\frac{C_1 C_4^2 C_5}{2T}\right)^{\frac{1}{4}}
            \\
            &\leq& 4\sqrt{2} \sigma_LT^{-\frac{1}{2}} 
            + 32 \sqrt{\sigma_L \sigma}\left(\frac{2[\kappa\vee\log \frac{4T}{\delta}]\log \frac{4T}{\delta}}{T}\right)^{\frac{1}{4}}
        \end{eqnarray*}
        To analyze the iteration complexity,
        the following inequalities of \(T\) are sufficient to show \(\frac{1}{T}\sum_{t=1}^T\|\nabla f(w_t)\|_\ast \|P_t\|^2 \leq \varepsilon\)
        for some \(\varepsilon > 0\) as 1st, 2nd terms are upper bounded by \(\varepsilon/2\), and noting that \(\kappa\geq 1\) and \(T\geq \frac{\sigma_L^2}{2\sigma^2}\) implies \(T\left(\frac{\kappa \vee \log\frac{4T}{\delta}}{\log\frac{4T}{\delta}}\right) \geq T\geq \frac{\sigma_L^2}{2\sigma^2}\):
        \begin{align*}
            T \geq& \frac{\sigma_L^2}{2\sigma^2} \vee \frac{128 \sigma_L^2}{\varepsilon^2},\,\frac{\log\frac{4T}{\delta}}{T} \leq
            \frac{(\varepsilon/64)^4}{2\kappa\cdot\sigma_L^2 \sigma^2},\,
            \frac{\log^2\frac{4T}{\delta}}{T} \leq
            \frac{(\varepsilon/64)^4}{2\sigma_L^2 \sigma^2}
        \end{align*}
        By substituting \(\delta \gets \delta/4\) and \((q, \varepsilon) \gets (1, \frac{(\varepsilon/64)^4}{2\kappa\cdot\sigma_L^2 \sigma^2}),(2, \frac{(\varepsilon/64)^4}{2\sigma_L^2 \sigma^2})\) 
        in Lemma~\ref{lemma:opt_bound_stepsize_reset_batch}, the following condition of \(T\) is sufficient to show \(\frac{1}{T}\sum_{t=1}^T\|\nabla f(w_t)\|_\ast \|P_t\|^2 \leq \varepsilon\):
        \[
        T \geq \frac{\sigma_L^2}{2\sigma^2} \vee \frac{128 \sigma_L^2}{\varepsilon^2}
        \vee \frac{2\sigma_L^2 \sigma^2}{(\varepsilon/64)^4} 
        \left[\mathe^4\log^2\left[\mathe\vee\frac{8\sigma_L^2 \sigma^2}{(\varepsilon/64)^4\delta}\right]
        \vee \mathe \kappa \log\left[\mathe\vee \kappa\frac{8\sigma_L^2 \sigma^2}{(\varepsilon/64)^4\delta}\right]
        \right]
        \]
        In particular, for all \(\varepsilon\in \mathbb{R}_+\) is sufficiently small such that \(\varepsilon \leq \varepsilon_0(\sigma_L, \sigma, \kappa, \delta)\),
        then the lower bound on \(T\) simplifies asymptotically to
        \[
            T \geq \frac{\mathe^4\cdot2\sigma_L^2 \sigma^2}{(\varepsilon/64)^4} \log^2\frac{8\sigma_L^2 \sigma^2}{(\varepsilon/64)^4\delta},\,
        \]
        where the uniform upper bound \(\varepsilon_0(\sigma_L, \sigma, \kappa, \delta)\) 
        is obtained by noting \(\kappa \geq 1\) and solving the following inequalities for \(\varepsilon\)
        to ensure the dominated terms in the above asymptotic bound are larger than the other terms
        of the lower bound on \(T\).

        \[
            \left\{
            \begin{aligned}
                &\frac{8\sigma_L^2\sigma^2}{(\varepsilon/96)^4\delta} \geq \mathe,\,
                \frac{8\sigma_L^2\sigma^2}{(\varepsilon/96)^4\delta} \geq \kappa\\
                & \frac{\mathe^4}{2\mathe} \log \frac{8\sigma_L^2\sigma^2}{(\varepsilon/96)^4\delta} \geq \kappa\\
                &\frac{\mathe^4\cdot 2\sigma_L^2 \sigma^2}{(\varepsilon/64)^4} \geq \frac{\sigma_L^2}{2\sigma^2},\,
                \frac{\mathe^4\cdot 2\sigma_L^2 \sigma^2}{(\varepsilon/64)^4} \geq \frac{128\sigma_L^2}{\varepsilon^2}
            \end{aligned}
            \right.
        \]

        The explicit expression of \(\varepsilon_0(\sigma_L, \sigma, \kappa, \delta) \equiv \sqrt{\sigma_L\sigma}\cdot \varepsilon_1(\kappa, \delta)\vee \sigma \cdot \varepsilon_2(\kappa, \delta)\) is given by:
        \[
        \begin{aligned}
            \varepsilon_1(\kappa, \delta) &= 96 \left( \frac{8}{\delta \left( \mathe \vee \kappa \vee \mathe^{2\kappa/\mathe^3} \right)} \right)^{\!1/4}\\            
            \varepsilon_2(\kappa, \delta) &= 64\sqrt{2}\mathe \;\wedge\; 512\mathe^2        
        \end{aligned}
        \]
        Therefore, the iteration complexity is \(T=\mathcal{O}\left(\frac{\sigma_L^2\sigma^2}{\varepsilon^4}\log^2\frac{\sigma_L^2\sigma^2}{\varepsilon^4\delta}\right)\),
        and the expected value of oracle complexity is \(N=[B\cdot p + 1\cdot(1-p)]T=T(2-1/B) = \mathcal{O}\left(\frac{\sigma_L^2\sigma^2}{\varepsilon^4}\log^2\frac{\sigma_L^2\sigma^2}{\varepsilon^4\delta}\right)\)
        with the selection of batch size \(B=\frac{1}{p}\) to guarantee \(\frac{1}{T}\sum_{t=1}^T\|\nabla f(w_t)\|_\ast \|P_t\|^2 \leq \varepsilon\) with probability at least \(1-\delta\)
        for sufficiently small \(\varepsilon\).

        \textbf{Family 1, Case 3 (Periodic Momentum)} \((p_t=\mathbb{I}_{\{t\!\!\mod E=0\}},\beta_t=1)\):
        By invoking the established upper bound of \(\|e_t\|_*\) for the estimator \(v_t\) with zeroth-order correction term \(\mathcal{T}_t\), the selection of \(p_t=\mathbb{I}_{\{t \mod E= 0\}}, \beta_t =1\) and \(\eta_t=\eta\) in the third-level section of Appendix~\ref{subsubsec:family1case3}, we show the following upper bound of \(\frac{1}{T}\sum_{t=1}^T\|e_t\|_*\):
        \[
        \frac{1}{T}\sum_{t=1}^T \|e_t\|_* \leq C\left(\frac{\delta}{2T/E},\kappa\right) \sigma/\sqrt{B}+L\eta E
        \]
        By selecting the batch size \(B = E\), noting \(C\left(\frac{\delta}{2T/E},\kappa\right)\leq C\left(\frac{\delta}{2T},\kappa\right)  \leq 2 \sqrt{2[\kappa \vee\log\frac{2T}{\delta}]}\), we have:
        \begin{eqnarray*}
        & & \frac{1}{T}\sum_{t=1}^T \|\nabla f(w_t)\|_* \|P_t\|^2\leq \frac{4\Delta_f}{\eta T}+2L\eta
        +8 \cdot \frac{1}{T}\sum_{t=1}^T\|e_t\|_*
        \\
        & \leq & 
        \frac{4 \Delta_f}{\eta T} + 2 L \cdot \eta + 0
        + 16 \sigma \sqrt{2[\kappa\vee\log\frac{2T}{\delta}]} \cdot \frac{1}{\sqrt{E}}
        + 8L \cdot \eta E
        \end{eqnarray*}
        Therefore, by letting \(C_1 = 4 \Delta_f, C_2 = 2L, C_3 =0, C_4 = 16 \sigma \sqrt{2[\kappa\vee\log\frac{2T}{\delta}]}, C_5 = 8L\) in Lemma~\ref{lemma:opt_bound_stepsize_reset_batch},
        selecting the step size \(\eta =\min\Big\{\left(\frac{C_1}{C_2}\right)^{\frac{1}{2}}T^{-\frac{1}{2}}, \allowbreak \left(\frac{2 C_1^3}{C_4^2 C_5}\right)^{\frac{1}{4}} T^{-\frac{3}{4}}\Big\} = \frac{1}{L}\min\{\sqrt{2} \sigma_L T^{-\frac{1}{2}}, \frac{1}{2}\left(\frac{\sigma_L^6}{2\sigma^2}\right)^{\frac{1}{4}} [\kappa\vee\log\frac{2T}{\delta}]^{-\frac{1}{4}} T^{-\frac{3}{4}}\}\) 
        and batch size \(B = E = \left\lceil \left( \frac{C_4^2 T}{8 C_1 C_5} \right)^{\frac{1}{2}} \right\rceil = \lceil \frac{\sqrt{2}\sigma}{\sigma_L}[\kappa\vee\log\frac{2T}{\delta}]^{\frac{1}{2}} T^{\frac{1}{2}} \rceil\),
        then when \(T [\kappa\vee\log\frac{2T}{\delta}] \geq \frac{\sigma_L^2}{2\sigma^2}\), we have:
        \begin{eqnarray*}
            & &\frac{1}{T}\sum_{t=1}^T\|\nabla f(w_t)\|_\ast \|P_t\|^2
            \leq 2\left(\frac{C_1 C_2}{T}\right)^{\frac{1}{2}} 
            + 4\left(\frac{C_1 C_4^2 C_5}{2T}\right)^{\frac{1}{4}}\\
            &\leq& 4\sqrt{2} \sigma_LT^{-\frac{1}{2}} 
            + 32 \sqrt{\sigma_L \sigma}\left(\frac{2[\kappa\vee\log \frac{2T}{\delta}]}{T}\right)^{\frac{1}{4}}
        \end{eqnarray*}
        To analyze the iteration complexity,
        the following inequalities of \(T\) are sufficient to show \(\frac{1}{T}\sum_{t=1}^T\|\nabla f(w_t)\|_\ast \|P_t\|^2 \leq \varepsilon\)
        for some \(\varepsilon > 0\) as 1st, 2nd terms are upper bounded by \(\varepsilon/2\),
        and noting that \(\kappa\geq 1\) and \(T\geq \frac{\sigma_L^2}{2\kappa \cdot \sigma^2}\) implies that \(T[\kappa \vee\log\frac{2T}{\delta}] \geq T\cdot \kappa \geq \frac{\sigma_L^2}{2\sigma^2}\):
        \begin{align*}
            T \geq& \frac{\sigma_L^2}{2\kappa\cdot\sigma^2} \vee \frac{128 \sigma_L^2}{\varepsilon^2}\vee \frac{2\kappa \cdot \sigma_L^2\sigma^2}{(\varepsilon/64)^4},\,
            \frac{\log\frac{2T}{\delta}}{T} \leq \left(\frac{\varepsilon}{64\sqrt{\sigma_L\sigma}}\right)^4
            = \frac{(\varepsilon/64)^4}{2\sigma_L^2 \sigma^2}
        \end{align*}
        By substituting \(\delta \gets \delta/2\) and \((q, \varepsilon) \gets (1, \frac{(\varepsilon/64)^4}{2\sigma_L^2 \sigma^2})\) 
        in Lemma~\ref{lemma:opt_bound_stepsize_reset_batch}, the following condition of \(T\) is sufficient to show \(\frac{1}{T}\sum_{t=1}^T\|\nabla f(w_t)\|_\ast \|P_t\|^2 \leq \varepsilon\):
        \[
        T \geq \frac{\sigma_L^2}{2\kappa\cdot\sigma^2} \vee \frac{128 \sigma_L^2}{\varepsilon^2}
        \vee \frac{2\sigma_L^2 \sigma^2}{(\varepsilon/64)^4}\left[\kappa \vee \mathe \log\left[\mathe\vee\frac{4\sigma_L^2 \sigma^2}{(\varepsilon/64)^4\delta}\right]\right]
        \]
        In particular, for all \(\varepsilon\in \mathbb{R}_+\) is sufficiently small such that \(\varepsilon \leq \varepsilon_0(\sigma_L, \sigma, \kappa, \delta)\),
        then the lower bound on \(T\) simplifies asymptotically to
        \[
            T \geq \frac{\mathe \cdot 2\sigma_L^2 \sigma^2}{(\varepsilon/64)^4} \log\frac{4\sigma_L^2 \sigma^2}{(\varepsilon/64)^4\delta},
        \]
        where the uniform upper bound \(\varepsilon_0(\sigma_L, \sigma, \kappa, \delta)\) 
        is obtained by noting \(\kappa\geq 1\) and solving the following inequalities for \(\varepsilon\)
        to ensure the dominated terms in the above asymptotic bound are larger than the other terms
        of the lower bound on \(T\).

        \[
            \left\{
            \begin{aligned}
                &\frac{4\sigma_L^2\sigma^2}{(\varepsilon/96)^4\delta} \geq \mathe,\,
                \mathe \log\frac{4\sigma_L^2 \sigma^2}{(\varepsilon/64)^4\delta} \geq \kappa\\
                &\frac{\mathe \cdot 2\sigma_L^2 \sigma^2}{(\varepsilon/64)^4} \geq \frac{\sigma_L^2}{2\kappa\cdot\sigma^2},\,
                \frac{\mathe \cdot 2\sigma_L^2 \sigma^2}{(\varepsilon/64)^4} \geq \frac{128\sigma_L^2}{\varepsilon^2}
            \end{aligned}
            \right.
        \]

        The explicit expression of \(\varepsilon_0(\sigma_L, \sigma, \kappa, \delta) \equiv \sqrt{\sigma_L\sigma}\cdot \varepsilon_1(\kappa, \delta)\vee \sigma \cdot \varepsilon_2(\kappa, \delta)\) is given by:
        \[
        \begin{aligned}
            \varepsilon_1(\kappa, \delta) &= 96 \left( \frac{4}{\mathe\delta} \right)^{\!1/4} \;\wedge\; 64 \left( \frac{4}{\delta \mathe^{\kappa/\mathe}} \right)^{\!1/4}\\            
            \varepsilon_2(\kappa, \delta) &= 64\sqrt{2}(\mathe\kappa)^{\!1/4} \;\wedge\; 512\mathe^{1/2}        
        \end{aligned}
        \]
        Therefore, the iteration complexity is \(T=\mathcal{O}\left(\frac{\sigma_L^2\sigma^2}{\varepsilon^4}\log\frac{\sigma_L^2\sigma^2}{\varepsilon^4\delta}\right)\),
        and the oracle complexity is \(N=B\cdot T/E+1\cdot (T-T/E) = T(2-1/B) = \mathcal{O}\left(\frac{\sigma_L^2\sigma^2}{\varepsilon^4}\log\frac{\sigma_L^2\sigma^2}{\varepsilon^4\delta}\right)\)
        with the selection of batch size \(B=E\) to guarantee \(\frac{1}{T}\sum_{t=1}^T\|\nabla f(w_t)\|_\ast \|P_t\|^2 \leq \varepsilon\) with probability at least \(1-\delta\)
        for sufficiently small \(\varepsilon\).

        \textbf{Family 2, Case 1 (STORM)} \((p_t=0,\beta_t=\beta)\):
        By invoking the established upper bound of \(\|e_t\|_*\) for the estimator \(v_t\) with first-order correction term \(\mathcal{T}_t\), the selection of \(p_t=0, \beta_t =\beta\) and \(\eta_t=\eta\) in the third-level section of Appendix~\ref{subsubsec:family2case1}, and noting that \(\sum_{t=1}^T\beta^t = \frac{\beta(1-\beta^{T})}{1-\beta}\leq \frac{1}{1-\beta}, \frac{(1-\beta)^2}{1-\beta^2} \leq 1-\beta, \frac{\beta^2}{1-\beta^2}\leq \frac{1}{1-\beta}\) for \(\beta\in(0, 1)\) and the elementary inequality \(\sqrt{x+y} \leq \sqrt{x} + \sqrt{y}, \forall x, y\geq 0\), we show the following upper bound for \(\frac{1}{T}\sum_{t=1}^T\|e_t\|_*\):
        \[
        \frac{1}{T}\sum_{t=1}^T \|e_t\|_* \leq \frac{C\left(\frac{\delta}{2T},\kappa\right)\sigma/\sqrt{B}}{(1-\beta)T}+C\!\left(\frac{\delta}{2T},\kappa\right)\left[\sigma \sqrt{2(1-\beta)}+\ell\eta\sqrt{\frac{2}{1-\beta}}\right]
        \]

        By selecting the batch size \(B=1\) and 
        noting \(C\!\left(\frac{\delta}{2T},\kappa\right) \leq 2 \sqrt{2[\kappa \vee \log\frac{2T}{\delta}]}\), we have:
        \begin{eqnarray*}
        & & \frac{1}{T}\sum_{t=1}^T \|\nabla f(w_t)\|_*\|P_t\|^2 \leq
        \frac{4\Delta_f}{\eta T}+2L\eta
        +8\cdot \frac{1}{T}\sum_{t=1}^T \|e_t\|_*
        \\
        & \leq & \frac{4 \Delta_f}{\eta T} 
        + 2L \eta 
        + \frac{16 \sigma \sqrt{2[\kappa\vee\log\frac{2T}{\delta}]}}{(1-\beta)T}
        + 32\sigma \sqrt{\kappa \vee \log\frac{2T}{\delta}} \sqrt{1-\beta}
        + 32\ell \sqrt{\kappa \vee \log\frac{2T}{\delta}} \frac{\eta}{\sqrt{1-\beta}}
        \end{eqnarray*}

        Letting \(C_1 = 4 \Delta_f, C_2 = 2L, C_3 = 16 \sigma \sqrt{2[\kappa\vee\log\frac{2T}{\delta}]}, C_4 = 32\sigma \sqrt{\kappa \vee \log\frac{2T}{\delta}}, C_5 = 32\ell \sqrt{\kappa \vee \log\frac{2T}{\delta}}\) in Lemma~\ref{lemma:opt_bound_stepsize_momentum},
        selecting the step size and the momentum parameter as \\
        \(\eta = \min\Big\{\left(\frac{C_1}{C_2}\right)^{\frac{1}{2}}T^{-\frac{1}{2}}, \allowbreak \left(\frac{C_1^2}{C_4 C_5}\right)^{\frac{1}{3}}T^{-\frac{2}{3}}\Big\} = \min\Big\{ \frac{\sqrt{2}}{L}\sigma_L T^{-\frac{1}{2}}, \allowbreak \frac{1}{4\ell}\left(\frac{\sigma_\ell^4}{\sigma}\right)^{\frac{1}{3}}[\kappa \vee \log\frac{2T}{\delta}]^{-\frac{1}{3}} T^{-\frac{2}{3}} \Big\}\),
        \(1-\beta= \max\Big\{\left(\frac{C_3}{C_4 T}\right)^{\frac{2}{3}}, \allowbreak \left(\frac{C_1 C_5}{C_4^2 T}\right)^{\frac{2}{3}}\Big\} = \max\Big\{ 2^{-\frac{1}{3}} T^{-\frac{2}{3}}, \allowbreak \frac{1}{4}\left(\frac{(\sigma_\ell/\sigma)^4}{[\kappa \vee \log\frac{2T}{\delta}]}\right)^{\frac{1}{3}}T^{-\frac{2}{3}} \Big\}\)
        then when \(T [\kappa \vee \log\frac{2T}{\delta}]^{\frac{1}{2}} \geq \frac{\sigma_\ell^2}{8\sigma^2}\), we have:
        \begin{eqnarray*}
            & &\frac{1}{T}\sum_{t=1}^T\|\nabla f(w_t)\|_\ast \|P_t\|^2
            \leq 2\left(\frac{C_1 C_2}{T}\right)^{\frac{1}{2}} + 2\left(\frac{C_3 C_4^2}{T}\right)^{\frac{1}{3}}
            + 3\left(\frac{C_1 C_4 C_5}{T}\right)^{\frac{1}{3}}
            \\
            &\leq&
            4 \sqrt{2} \sigma_L T^{-\frac{1}{2}}
            + 32\sqrt{2}\sigma \left(\frac{2[\kappa \vee \log\frac{2T}{\delta}]^{\frac{3}{2}}}{T}\right)^{\frac{1}{3}}
            + 48 (\sigma_\ell^2 \sigma)^{\frac{1}{3}} \left(\frac{\kappa \vee \log\frac{2T}{\delta}}{T}\right)^{\frac{1}{3}}
        \end{eqnarray*}
        To analyze the iteration complexity,
        the following inequalities of \(T\) are sufficient to show that
        \(\frac{1}{T}\sum_{t=1}^T\|\nabla f(w_t)\|_\ast \|P_t\|^2 \leq \varepsilon\)
        for some \(\varepsilon>0\) 
        as 1st, 2nd, and 3rd terms are upper bounded by \(\varepsilon/2, \varepsilon/4\)
        and \(\varepsilon/4\) respectively, 
        and noting that \(\kappa\geq 1\) and \(T\geq \frac{\sigma_\ell^2/\sigma^2}{8\sqrt{\kappa}}\) 
        implies that \(T[\kappa \vee \log\frac{2T}{\delta}]^{\frac{1}{2}}\geq T\sqrt{\kappa} \geq \frac{\sigma_\ell^2}{8\sigma^2}\).
        \begin{align*}
            T \geq& \frac{\sigma_\ell^2/\sigma^2}{8\sqrt{\kappa}}
            \vee \frac{128 \sigma_L^2}{\varepsilon^2}\vee \frac{2 \kappa^{\frac{3}{2}}\cdot \sigma^3}{(\varepsilon/128\sqrt{2})^3} \vee  \frac{\kappa\cdot\sigma_\ell^2 \sigma}{(\varepsilon/192)^3},\,
            \frac{\log^{\frac{3}{2}}\frac{2T}{\delta}}{T} \leq \frac{1}{2}\left(\frac{\varepsilon}{128\sqrt{2}\sigma}\right)^3,\,
            \frac{\log\frac{2T}{\delta}}{T}\leq\frac{(\varepsilon/192)^3}{\sigma_\ell^2 \sigma}
        \end{align*}
        By substituting \(\delta \gets \delta/2\) and \((q,\varepsilon) \gets (\frac{3}{2},\frac{1}{2}\left(\frac{\varepsilon}{128\sqrt{2}\sigma}\right)^3), (1,\frac{(\varepsilon/192)^3}{\sigma_\ell^2 \sigma})\) in Lemma~\ref{lemma:complexity_log}, the following condition of \(T\)
        is sufficient to guarantee the above inequalities,
        and therefore \(\frac{1}{T}\sum_{t=1}^T\|\nabla f(w_t)\|_\ast \|P_t\|^2 \leq \varepsilon\).
        \begin{align*}
        T\geq& \frac{\sigma_\ell^2/\sigma^2}{8\sqrt{\kappa}} \vee \frac{128 \sigma_L^2}{\varepsilon^2}\\
        &  \vee \frac{2\sigma^3}{(\varepsilon/128\sqrt{2})^3}\left[\kappa^{\frac{3}{2}} \vee\mathe^{\frac{9}{4}}\log^{\frac{3}{2}}\left[\mathe\vee \frac{4\sigma^3}{(\varepsilon/128\sqrt{2})^3\delta} \right]\right]
        \vee \frac{\sigma_\ell^2 \sigma}{(\varepsilon/192)^3}\left[
            \kappa\vee\mathe \log\left[\mathe\vee \frac{2\sigma_\ell^2 \sigma}{(\varepsilon/192)^3\delta} \right]
        \right]
        \end{align*}
        In particular, for all \(\varepsilon \in \mathbb{R}_+\) is sufficiently small such 
        that \(\varepsilon \leq \varepsilon_0(\sigma_\ell, \sigma, \kappa, \delta)\),
        then the lower bound on \(T\) simplifies asymptotically to
        \[
            T \;\ge\; \max\Big\{
            \frac{2\mathe^{\frac{9}{4}}\cdot \sigma^3}{(\varepsilon/128\sqrt{2})^3}\log^{\frac{3}{2}}\frac{4\sigma^3}{(\varepsilon/128\sqrt{2})^3\delta},\allowbreak \;
            \frac{\mathe\cdot\sigma_\ell^2 \sigma}{(\varepsilon/192)^3} \log\frac{2\sigma_\ell^2 \sigma}{(\varepsilon/192)^3 \delta},\allowbreak \;
            \frac{2 \sigma_L^2}{(\varepsilon/8)^2}
            \Big\}.
        \]        
        where the uniform upper bound \(\varepsilon_0(\sigma_\ell, \sigma, \kappa, \delta)\) of \(\varepsilon\) is obtained by noting \(\kappa\geq 1\) and solving the following inequalities for \(\varepsilon\)
        to ensure the dominated terms in the above asymptotic bound are larger than the other terms
        of the lower bound on \(T\).
        \[
        \left\{
        \begin{aligned}
         &\frac{4\sigma^3}{(\varepsilon/128\sqrt{2})^3\delta} \geq \mathe,\,
         \frac{2\sigma_\ell^2\sigma}{(\varepsilon/192)^3\delta} \geq \mathe\\
         &\mathe \log\frac{2\sigma_\ell^2\sigma}{(\varepsilon/192)^3\delta} \geq \kappa,\,
         \mathe^{\frac{9}{4}}\log^{\frac{3}{2}}\frac{4\sigma^3}{(\varepsilon/128\sqrt{2})^3\delta} 
         \geq \kappa^{\frac{3}{2}}\\
         &\frac{\mathe \cdot \sigma_\ell^2 \sigma}{(\varepsilon/192)^3} \geq \frac{\sigma_\ell^2/\sigma^2}{8\sqrt{\kappa}} & &
        \end{aligned}
        \right.
        \]
        The explicit expression of \(\varepsilon_0(\sigma_\ell, \sigma, \kappa, \delta) \equiv (\sigma_\ell^2 \sigma)^{\frac{1}{3}} \cdot \varepsilon_1(\kappa, \delta) \vee \sigma \cdot \varepsilon_2(\kappa, \delta) \) is given by:
        \[
        \begin{aligned}
            \varepsilon_1(\kappa, \delta) &= 192 \left( \frac{2}{\delta \, \mathe^{1 \vee (\kappa/\mathe)}} \right)^{\!1/3}\\        
            \varepsilon_2(\kappa, \delta) &= 384(\mathe\sqrt{\kappa})^{\!1/3} \;\wedge\; 128\sqrt{2} \left( \frac{4}{\delta \, \mathe^{1 \vee (\kappa/\mathe^{3/2})}} \right)^{\!1/3}        
        \end{aligned}
        \]
        Therefore, the iteration complexity is \(T = \mathcal{O}\left( \frac{\sigma^3}{\varepsilon^3}\log^{\frac{3}{2}}\frac{\sigma^3}{\varepsilon^3\delta}\vee \frac{\sigma_\ell^2\sigma}{\varepsilon^3}\log\frac{\sigma_\ell^2\sigma}{\varepsilon^3\delta}\vee \frac{\sigma_L^2}{\varepsilon^2} \right)\),
        and the oracle complexity is \(N = B\cdot T = \mathcal{O}\left( \frac{\sigma^3}{\varepsilon^3}\log^{\frac{3}{2}}\frac{\sigma^3}{\varepsilon^3\delta}\vee \frac{\sigma_\ell^2\sigma}{\varepsilon^3}\log\frac{\sigma_\ell^2\sigma}{\varepsilon^3\delta}\vee \frac{\sigma_L^2}{\varepsilon^2} \right)\)
        with the batch size \(B=1\) 
        to guarantee \(\frac{1}{T}\sum_{t=1}^T\|\nabla f(w_t)\|_\ast \|P_t\|^2 \leq \varepsilon\) with probability at least \(1-\delta\)
        for sufficiently small \(\varepsilon\).

        \textbf{Family 2, Case 2 (PAGE / Loopless SVRG)} \((p_t=p,\beta_t=1)\):
        By invoking the established upper bound of \(\|e_t\|_*\) for the estimator \(v_t\) with first-order correction term \(\mathcal{T}_t\), the selection of \(p_t=p, \beta_t =1\) and \(\eta_t=\eta\) in the third-level section of Appendix~\ref{subsubsec:family2case2}, we show the following upper bound of \(\frac{1}{T}\sum_{t=1}^T\|e_t\|_*\):
        \[
        \frac{1}{T}\sum_{t=1}^T \|e_t\|_* \leq C\left(\frac{\delta}{4T},\kappa\right)\sigma/\sqrt{B}+C\!\left(\frac{\delta}{4T},\kappa\right)\ell\eta\sqrt{\frac{2}{p}\log\frac{4T}{\delta}}
        \]

        By selecting the batch size \(B = \frac{1}{p}\)
        and noting \(C\!\left(\frac{\delta}{4T},\kappa\right) \leq 2 \sqrt{2[\kappa \vee \log\frac{4T}{\delta}]}\), we have:
        \begin{eqnarray*}
        & & \frac{1}{T}\sum_{t=1}^T \|\nabla f(w_t)\|_* \|P_t\|^2\leq \frac{4\Delta_f}{\eta T}+2L\eta
        +8\cdot \frac{1}{T} \sum_{t=1}^T \|e_t\|_*
        \\
        & \leq & 
        \frac{4 \Delta_f}{\eta T} + 2L\cdot \eta + 0 + 16\sigma \sqrt{2[\kappa\vee\log\frac{4T}{\delta}]}\cdot \sqrt{p} + 32\ell\sqrt{\log\frac{4T}{\delta}[\kappa\vee \log\frac{4T}{\delta}]}\cdot \frac{\eta}{\sqrt{p}}
        \end{eqnarray*}

        Letting \(C_1 = 4 \Delta_f, C_2 = 2L, C_3 = 0, C_4 = 16\sigma \sqrt{2[\kappa\vee\log\frac{4T}{\delta}]}, C_5 = 32\ell\sqrt{\log\frac{4T}{\delta}[\kappa\vee \log\frac{4T}{\delta}]}\) in Lemma~\ref{lemma:opt_bound_stepsize_reset_batch},
        selecting the step size as
        \(\eta = \min\Big\{\left(\frac{C_1}{C_2T}\right)^{\frac{1}{2}}, \allowbreak \left(\frac{C_1^2}{\sqrt{2} C_4 C_5 T^2}\right)^{\frac{1}{3}}\Big\} = \min\Big\{\frac{\sqrt{2}}{L}\sigma_L T^{-\frac{1}{2}}, \allowbreak \frac{1}{4\ell}\left(\frac{\sigma_\ell^4}{\sigma}\right)^{\frac{1}{3}} (\log\frac{4T}{\delta}[\kappa \vee \log\frac{4T}{\delta}]^2)^{-\frac{1}{6}} T^{-\frac{2}{3}}\Big\}\),
        reset probability such that \(\frac{1}{p} = B = \left\lceil \left( \frac{C_4^4 T^2}{2 C_1^2 C_5^2} \right)^{\frac{1}{3}} \right\rceil = \left\lceil 2\left(\frac{\sigma}{\sigma_\ell}\right)^{\frac{4}{3}}\left(\frac{\kappa \vee \log\frac{4T}{\delta}}{\log\frac{4T}{\delta}}\right)^{\frac{1}{3}}T^{\frac{2}{3}} \right\rceil\)
        when \(T \left(\frac{\kappa \vee \log\frac{4T}{\delta}}{\log\frac{4T}{\delta}}\right)^{\frac{1}{2}} \geq \frac{\sigma_\ell^2}{2\sqrt{2}\sigma^2}\), we have:
        \begin{eqnarray*}
            & &\frac{1}{T}\sum_{t=1}^T\|\nabla f(w_t)\|_\ast \|P_t\|^2
            \leq
            2\left(\frac{C_1 C_2}{T}\right)^{\frac{1}{2}} + 3\left(\frac{\sqrt{2} C_1 C_4 C_5}{T}\right)^{\frac{1}{3}}\\
            & =& 4\sqrt{2} \sigma_L T^{-\frac{1}{2}} + 48(\sigma_\ell^2 \sigma)^{\frac{1}{3}} \left(\frac{[\log\frac{4T}{\delta}]^{\frac{1}{2}}\cdot [\kappa\vee\log\frac{4T}{\delta}]}{T}\right)^{\frac{1}{3}}
        \end{eqnarray*}

        To analyze the iteration complexity, the following inequalities of \(T\) are sufficient to show that \(\frac{1}{T}\sum_{t=1}^T\|\nabla f(w_t)\|_\ast \|P_t\|^2 \leq \varepsilon\) 
        for some \(\varepsilon > 0\) as 1st, 2nd terms all are upper bounded by \(\varepsilon/2\),
        and noting \(\kappa \geq 1\) and \(T\geq \frac{\sigma_\ell^2}{2\sqrt{2}\sigma^2}\) 
        implies that \(T \left(\frac{\kappa \vee \log\frac{4T}{\delta}}{\log\frac{4T}{\delta}}\right)^{\frac{1}{2}}\geq T \geq \frac{\sigma_\ell^2}{2\sqrt{2}\sigma^2}\):
        \[
        \begin{aligned}
            T\geq & \frac{\sigma_\ell^2}{2\sqrt{2}\sigma^2} \vee \frac{128 \sigma_L^2}{\varepsilon^2},\\
            & \frac{\log^{\frac{1}{2}}\frac{4T}{\delta}}{T} \leq 
            \frac{1}{\kappa}\cdot \frac{(\varepsilon/96)^3}{\sigma_\ell^2 \sigma},\,
              \frac{\log^{\frac{3}{2}}\frac{4T}{\delta}}{T} \leq \frac{(\varepsilon/96)^3}{\sigma_\ell^2 \sigma}
        \end{aligned}
        \]
        By substituting \(\delta \gets \delta/4\) and \((q, \varepsilon) \gets (\frac{1}{2}, \frac{1}{\kappa}\cdot \frac{(\varepsilon/96)^3}{\sigma_\ell^2 \sigma}), (\frac{3}{2}, \frac{(\varepsilon/96)^3}{\sigma_\ell^2 \sigma})\) in Lemma~\ref{lemma:complexity_log}, the following condition of \(T\)
        is sufficient to guarantee the above inequalities,
        and therefore \(\frac{1}{T}\sum_{t=1}^T\|\nabla f(w_t)\|_\ast \|P_t\|^2 \leq \varepsilon\).
        \[
        \begin{aligned}
            T\geq & \frac{\sigma_\ell^2}{2\sqrt{2}\sigma^2} \vee \frac{128 \sigma_L^2}{\varepsilon^2}
            \vee \frac{\sigma_\ell^2\sigma}{(\varepsilon/96)^3}\left[\mathe^{\frac{9}{4}}\log^{\frac{3}{2}}\left[\mathe\vee \frac{4\sigma_\ell^2\sigma}{(\varepsilon/96)^3\delta}\right]
            \vee \mathe^{\frac{1}{4}}\kappa \log^{\frac{1}{2}}\left[\mathe \vee \kappa \cdot \frac{4\sigma_\ell^2\sigma}{(\varepsilon/96)^3\delta}\right]\right]
        \end{aligned}
        \]
        In particular, for all \(\varepsilon \in \mathbb{R}_+\) is sufficiently small such \(\varepsilon \leq \varepsilon_0(\sigma_\ell, \sigma, \kappa, \delta)\),
        then the lower bound on \(T\) simplifies asymptotically to
        \[
        T \geq \max\Big\{\frac{\mathe^{\frac{9}{4}}\cdot\sigma_\ell^2\sigma}{(\varepsilon/96)^3}\log^{\frac{3}{2}}\frac{4\sigma_\ell^2\sigma}{(\varepsilon/96)^3\delta},\allowbreak \, 
        \allowbreak \frac{2\sigma_L^2}{(\varepsilon/8)^2}\Big\}
        \]
        where the uniform upper bound \(\varepsilon_0(\sigma_\ell, \sigma, \kappa, \delta)\) of \(\varepsilon\) is obtained by noting \(\kappa\geq 1\) and solving the following inequalities for \(\varepsilon\)
        to ensure the dominated terms in the above asymptotic bound are larger than the other terms
        of the lower bound on \(T\).

        \[
        \left\{
        \begin{aligned}
            & \frac{4\sigma_\ell^2\sigma}{(\varepsilon/96)^3\delta} \geq \mathe,\,\frac{4\sigma_\ell^2\sigma}{(\varepsilon/96)^3\delta} \geq \kappa\\
            &\frac{\mathe^{\frac{9}{4}}}{\sqrt{2}\mathe^{\frac{1}{4}}}\log\frac{4\sigma_\ell^2\sigma}{(\varepsilon/96)^3\delta} \geq \kappa\\
            &\frac{\mathe^{\frac{9}{4}}\cdot\sigma_\ell^2\sigma}{(\varepsilon/96)^3} \geq \frac{\sigma_\ell^2}{2\sqrt{2}\sigma^2}
        \end{aligned}
        \right.
        \]

        The explicit expression of \(\varepsilon_0(\sigma_\ell, \sigma, \kappa, \delta) \equiv (\sigma_\ell^2 \sigma)^{\frac{1}{3}} \cdot \varepsilon_1(\kappa, \delta) \vee \sigma \cdot \varepsilon_2(\kappa, \delta) \) is given by:
        \[
        \begin{aligned}
            \varepsilon_1(\kappa, \delta) &= 96 \left( \frac{4}{\delta \left( \mathe \vee \kappa \vee \mathe^{\sqrt{2}\kappa/\mathe^2} \right)} \right)^{\!1/3}\\            
            \varepsilon_2(\kappa, \delta) &= 96\sqrt{2}\mathe^{3/4}        
        \end{aligned}
        \]
        Therefore, the iteration complexity is 
        \(T = \mathcal{O}\left(\frac{\sigma_\ell^2\sigma}{\varepsilon^3}\log^{\frac{3}{2}}\frac{\sigma_\ell^2\sigma}{\varepsilon^3\delta} \vee \frac{\sigma_L^2}{\varepsilon^2}\right)\),
        and the expected value of oracle complexity is \(N = [B\cdot p + 1\cdot (1-p)]T=T(2-1/B) = \mathcal{O}\left(\frac{\sigma_\ell^2\sigma}{\varepsilon^3}\log^{\frac{3}{2}}\frac{\sigma_\ell^2\sigma}{\varepsilon^3\delta} \vee \frac{\sigma_L^2}{\varepsilon^2}\right)\)
        with the selection of batch size \(B=\frac{1}{p}\) 
        to guarantee \(\frac{1}{T}\sum_{t=1}^T\|\nabla f(w_t)\|_\ast \|P_t\|^2 \leq \varepsilon\) with probability at least \(1-\delta\)
        for sufficiently small \(\varepsilon\).

        \textbf{Family 2, Case 3 (SPIDER)} \((p_t=\mathbb{I}_{\{t\!\!\mod E=0\}},\beta_t=1)\):
        By invoking the established upper bound of \(\|e_t\|_*\) for the estimator \(v_t\) with first-order correction term \(\mathcal{T}_t\), the selection of \(p_t=\mathbb{I}_{\{t\mod E = 0\}}, \beta_t =1\) and \(\eta_t=\eta\) in the third-level section of Appendix~\ref{subsubsec:family2case3}, we show the following upper bound of \(\frac{1}{T}\sum_{t=1}^T\|e_t\|_*\):
        \[
        \frac{1}{T}\sum_{t=1}^T \|e_t\|_* \leq C\left(\frac{\delta}{2T/E},\kappa\right)\sigma/\sqrt{B}+C\!\left(\frac{\delta}{2T},\kappa\right)\ell\eta\sqrt{2E}
        \]

        By selecting the batch size \(B = E\),
        noting \(C\left(\frac{\delta}{2T/E},\kappa\right)\leq C\!\left(\frac{\delta}{2T},\kappa\right) \leq 2 \sqrt{2[\kappa \vee \log\frac{2T}{\delta}]}\), we have:
        \begin{eqnarray*}
        & & \frac{1}{T} \sum_{t=1}^T \|\nabla f(w_t)\|_* \|P_t\|^2
        \leq \frac{4\Delta_f}{\eta T}+2L\eta
        +8 \cdot \frac{1}{T} \sum_{t=1}^T \|e_t\|_* 
        \\
        &\leq& \frac{4 \Delta_f}{\eta T} + 2L\cdot \eta + 0 + 16 \sigma \sqrt{2[\kappa\vee\log\frac{2T}{\delta}]} \cdot \frac{1}{\sqrt{E}} + 32\ell \sqrt{\kappa \vee \log\frac{2T}{\delta}} \cdot \eta \sqrt{E}
        \end{eqnarray*}

        Letting \(C_1 = 4 \Delta_f, C_2 = 2L, C_3 = 0, C_4 = 16 \sigma \sqrt{2[\kappa\vee\log\frac{2T}{\delta}]}, C_5 = 32\ell \sqrt{\kappa \vee \log\frac{2T}{\delta}}\) in Lemma~\ref{lemma:opt_bound_stepsize_reset_batch},
        selecting the step size as
        \(\eta = \min\Big\{\left(\frac{C_1}{C_2T}\right)^{\frac{1}{2}}, \allowbreak \left(\frac{C_1^2}{\sqrt{2} C_4 C_5 T^2}\right)^{\frac{1}{3}}\Big\} = \min\Big\{\frac{\sqrt{2}}{L}\sigma_L T^{-\frac{1}{2}}, \allowbreak \frac{1}{4\ell}\left(\frac{\sigma_\ell^4}{\sigma}\right)^{\frac{1}{3}} [\kappa \vee \log\frac{2T}{\delta}]^{-\frac{1}{3}} T^{-\frac{2}{3}}\Big\}\),
        reset epoch as \(E= B = \left\lceil \left( \frac{C_4^4 T^2}{2 C_1^2 C_5^2} \right)^{\frac{1}{3}} \right\rceil = \left\lceil 2\left(\frac{\sigma}{\sigma_\ell}\right)^{\frac{4}{3}}\left[\kappa \vee \log\frac{2T}{\delta}\right]^{\frac{1}{3}}T^{\frac{2}{3}} \right\rceil\)
        when \(T \left[\kappa \vee \log\frac{2T}{\delta}\right]^{\frac{1}{2}} \geq \frac{\sigma_\ell^2}{2\sqrt{2}\sigma^2}\), we have:
        \begin{eqnarray*}
            & &\frac{1}{T}\sum_{t=1}^T\|\nabla f(w_t)\|_\ast \|P_t\|^2
            \leq  
            2\left(\frac{C_1 C_2}{T}\right)^{\frac{1}{2}} + 3\left(\frac{\sqrt{2} C_1 C_4 C_5}{T}\right)^{\frac{1}{3}} \\
            & =& 4\sqrt{2} \sigma_L T^{-\frac{1}{2}} + 48(\sigma_\ell^2 \sigma)^{\frac{1}{3}} \left(\frac{\kappa\vee\log\frac{2T}{\delta}}{T}\right)^{\frac{1}{3}}
        \end{eqnarray*}
        To analyze the iteration complexity, the following inequalities of \(T\) are sufficient to show that \(\frac{1}{T}\sum_{t=1}^T\|\nabla f(w_t)\|_\ast \|P_t\|^2 \leq \varepsilon\) 
        for some \(\varepsilon > 0\) as 1st, 2nd terms all are upper bounded by \(\varepsilon/2\),
        and noting that \(T\geq \frac{\sigma_\ell^2}{2\sqrt{2\kappa}\sigma^2}\) 
        implies that \(T \left[\kappa \vee \log\frac{2T}{\delta}\right]^{\frac{1}{2}} \geq T\kappa^{\frac{1}{2}} \geq \frac{\sigma_\ell^2}{2\sqrt{2}\sigma^2}\):
        \[
        \begin{aligned}
            T\geq & \frac{\sigma_\ell^2}{2\sqrt{2\kappa}\sigma^2} \vee \frac{128 \sigma_L^2}{\varepsilon^2} \vee \frac{\kappa\cdot \sigma_\ell^2\sigma}{(\varepsilon/96)^3},\\
            & 
            \frac{\log\frac{2T}{\delta}}{T} \leq \frac{(\varepsilon/96)^3}{\sigma_\ell^2 \sigma}
        \end{aligned}
        \]
        By substituting \(\delta \gets \delta/2\) and \((q, \varepsilon) \gets (1, \frac{(\varepsilon/96)^3}{\sigma_\ell^2 \sigma})\) in Lemma~\ref{lemma:complexity_log}, the following condition of \(T\)
        is sufficient to guarantee the above inequalities,
        and therefore \(\frac{1}{T}\sum_{t=1}^T\|\nabla f(w_t)\|_\ast \|P_t\|^2 \leq \varepsilon\).
        \[
        \begin{aligned}
            T\geq & \frac{\sigma_\ell^2}{2\sqrt{2\kappa}\sigma^2} \vee \frac{128 \sigma_L^2}{\varepsilon^2}\vee \frac{ \sigma_\ell^2\sigma}{(\varepsilon/96)^3}\left[\kappa \vee\mathe\log\left[\mathe\vee \frac{2\sigma_\ell^2\sigma}{(\varepsilon/96)^3\delta} \right]\right]
        \end{aligned}
        \]
        In particular, for all \(\varepsilon \in \mathbb{R}_+\) is sufficiently small such \(\varepsilon \leq \varepsilon_0(\sigma_\ell, \sigma, \kappa, \delta)\),
        then the lower bound on \(T\) simplifies asymptotically to
        \[
        T \geq \max\Big\{
        \frac{\mathe\cdot \sigma_\ell^2\sigma}{(\varepsilon/96)^3}\log\frac{2\sigma_\ell^2\sigma}{(\varepsilon/96)^3\delta},\allowbreak \,
        \allowbreak \frac{2\sigma_L^2}{(\varepsilon/8)^2}\Big\}
        \]
        where the uniform upper bound \(\varepsilon_0(\sigma_\ell, \sigma, \kappa, \delta)\) of \(\varepsilon\) is obtained by solving the following inequalities for \(\varepsilon\)
        to ensure the dominated terms in the above asymptotic bound are larger than the other terms 
        of the lower bound on \(T\).
        \[
        \left\{
        \begin{aligned}
            & \frac{2\sigma_\ell^2\sigma}{(\varepsilon/96)^3\delta} \geq \mathe,\\
            &\mathe\log\frac{2\sigma_\ell^2\sigma}{(\varepsilon/96)^3\delta} \geq \kappa,\\
            &\frac{\mathe\cdot\sigma_\ell^2\sigma}{(\varepsilon/96)^3} \geq \frac{1}{2\sqrt{2\kappa}}\frac{\sigma_\ell^2}{\sigma^2}
        \end{aligned}
        \right.
        \]
        The explicit expression of \(\varepsilon_0(\sigma_\ell, \sigma, \kappa, \delta) \equiv (\sigma_\ell^2 \sigma)^{\frac{1}{3}} \cdot \varepsilon_1(\kappa, \delta) \vee \sigma \cdot \varepsilon_2(\kappa, \delta) \) is given by:
        \[
        \begin{aligned}
            \varepsilon_1(\kappa, \delta) &= 96 \left( \frac{2}{\delta \, \mathe^{1 \vee (\kappa/\mathe)}} \right)^{\!1/3}\\            
            \varepsilon_2(\kappa, \delta) &= 96\sqrt{2}(\mathe\sqrt{\kappa})^{\!1/3}        
        \end{aligned}
        \]
        Therefore, the iteration complexity is 
        \(T = \mathcal{O}\left(\frac{\sigma_\ell^2\sigma}{\varepsilon^3}\log\frac{\sigma_\ell^2\sigma}{\varepsilon^3\delta} \vee \frac{\sigma_L^2}{\varepsilon^2}\right)\),
        and the oracle complexity is \(N = B\cdot T/E+1\cdot (T-T/E) = T(2-1/B) = \mathcal{O}\left(\frac{\sigma_\ell^2\sigma}{\varepsilon^3}\log\frac{\sigma_\ell^2\sigma}{\varepsilon^3\delta} \vee \frac{\sigma_L^2}{\varepsilon^2}\right)\)
        with the selection of batch size \(B=E\) 
        to guarantee \(\frac{1}{T}\sum_{t=1}^T\|\nabla f(w_t)\|_\ast \|P_t\|^2 \leq \varepsilon\) with probability at least \(1-\delta\)
        for sufficiently small \(\varepsilon\).

        \textbf{Family 3, Case 1 (Second-Order Momentum)} \((p_t=0,\beta_t=\beta)\):
        By invoking the established upper bound of \(\|e_t\|_*\) for the estimator \(v_t\) with second-order correction term \(\mathcal{T}_t\), the selection of \(p_t=0, \beta_t =\beta\) and \(\eta_t=\eta\) in the third-level section of Appendix~\ref{subsubsec:family3case1}, and noting that \(\sum_{t=1}^T\beta^t = \frac{\beta(1-\beta^{T})}{1-\beta}\leq \frac{1}{1-\beta},\beta\leq 1, \frac{(1-\beta)^2}{1-\beta^2}\leq 1-\beta,\frac{\beta^2}{1-\beta^2}\leq\frac{1}{1-\beta}\) for \(\beta\in(0, 1)\) and the elementary inequality \(\sqrt{x+y}\leq \sqrt{x} + \sqrt{y}, \forall x, y\geq 0\), we show the following upper bound for \(\frac{1}{T}\sum_{t=1}^T\|e_t\|_*\):
        \[
        \frac{1}{T}\sum_{t=1}^T \|e_t\|_* \leq \frac{C\!\left(\frac{\delta}{2T},\kappa\right)\sigma/\sqrt{B}}{(1-\beta)T}+\frac{\alpha\eta^2}{2(1-\beta)}+C\!\left(\frac{\delta}{2T},\kappa\right)\left[\sigma\sqrt{2(1-\beta)}+\gamma \eta \sqrt{\frac{2}{1-\beta}}\right]
        \]

        By selecting the batch size \(B = 1\) and 
        noting \(C\!\left(\frac{\delta}{2T},\kappa\right) \leq 2 \sqrt{2[\kappa \vee \log\frac{2T}{\delta}]}\), we have:
        \begin{eqnarray*}
        & &\frac{1}{T}\sum_{t=1}^T \|\nabla f(w_t)\|_*\|P_t\|^2\leq\frac{4\Delta_f}{\eta T}+2L\eta
        +8 \cdot \frac{1}{T}\sum_{t=1}^T \|e_t\|_*
        \\
        & \leq & \frac{4 \Delta_f}{\eta T} 
        + 2L\cdot \eta 
        + \frac{16 \sigma \sqrt{2[\kappa \vee\log\frac{2T}{\delta}]}}{(1-\beta)T}
        + 32\sigma\sqrt{\kappa \vee \log\frac{2T}{\delta}}\cdot \sqrt{1-\beta}\\
        & &+ 32\gamma\sqrt{\kappa \vee \log\frac{2T}{\delta}}\cdot \frac{\eta}{\sqrt{1-\beta}}
        + 4\alpha \cdot \frac{\eta^2}{1-\beta}
        \end{eqnarray*}

        Subsequently, by letting \(C_1 = 4 \Delta_f, C_2 = 2L, C_3 = 16 \sigma \sqrt{2[\kappa\vee\log\frac{2T}{\delta}]}, C_4 = 32\sigma\sqrt{\kappa \vee \log\frac{2T}{\delta}}, C_5 = 32\gamma\sqrt{\kappa \vee \log\frac{2T}{\delta}}, C_6 = 4\alpha\) in Lemma~\ref{lemma:opt_bound_stepsize_momentum},
        and selecting the step size \(\eta\) such that
        \(\eta = \min\Big\{\left(\frac{C_1}{C_2}\right)^{\frac{1}{2}}T^{-\frac{1}{2}}, \allowbreak \left(\frac{C_1^2}{C_4 C_5}\right)^{\frac{1}{3}}T^{-\frac{2}{3}}, \allowbreak \left(\frac{C_1^3}{C_4^2 C_6}\right)^{\frac{1}{5}}T^{-\frac{3}{5}}\Big\} = \min\Big\{ \frac{\sqrt{2}}{L}\sigma_L T^{-\frac{1}{2}}, \allowbreak \frac{1}{4\gamma}\left(\frac{ \sigma_\gamma^4}{\sigma}\right)^{\frac{1}{3}}[\kappa \vee \log\frac{2T}{\delta}]^{-\frac{1}{3}} T^{-\frac{2}{3}}, \allowbreak \frac{1}{2}\sqrt{\frac{1}{\alpha}\left(\frac{\sigma_\alpha^9}{4\sigma^4}\right)^{\frac{1}{5}}}[\kappa \vee \log\frac{2T}{\delta}]^{-\frac{1}{5}}T^{-\frac{3}{5}} \Big\}\),
        the momentum parameter \(\beta\) as
        \(1-\beta = \max\Big\{\left(\frac{C_3}{C_4}\right)^{\frac{2}{3}}T^{-\frac{2}{3}}, \allowbreak \left(\frac{C_1 C_5}{C_4^2}\right)^{\frac{2}{3}}T^{-\frac{2}{3}}, \allowbreak \left(\frac{C_1^2 C_6}{C_4^3}\right)^{\frac{2}{5}}T^{-\frac{4}{5}}\Big\} =\max\Big\{ 2^{-\frac{1}{3}}T^{-\frac{2}{3}}, \allowbreak \frac{1}{4}\left(\frac{(\sigma_\gamma/\sigma)^4}{[\kappa \vee \log\frac{2T}{\delta}]}\right)^{\frac{1}{3}}T^{-\frac{2}{3}}, \allowbreak \frac{1}{8}\left(\frac{\sigma_\alpha}{\sqrt{2}\sigma}\right)^{\frac{6}{5}}[\kappa \vee \log\frac{2T}{\delta}]^{-\frac{3}{5}}T^{-\frac{4}{5}} \Big\}\),
        then when \(T[\kappa \vee \log\frac{2T}{\delta}]^{\frac{1}{2}} \geq \frac{\sigma_\gamma^2}{8 \sigma^2}\)
        and \(T[\kappa \vee \log\frac{2T}{\delta}]^{\frac{3}{4}} \geq 2^{-\frac{9}{2}}\left(\frac{\sigma_\alpha}{\sigma}\right)^{\frac{3}{2}}\), we have:
        \begin{eqnarray*}
            & &\frac{1}{T}\sum_{t=1}^T\|\nabla f(w_t)\|_\ast \|P_t\|^2
            \leq 2\left(\frac{C_1 C_2}{T}\right)^{\frac{1}{2}} + 2\left(\frac{C_3 C_4^2}{T}\right)^{\frac{1}{3}}
            + 3\left(\frac{C_1 C_4 C_5}{T}\right)^{\frac{1}{3}}
            + 3 \left(\frac{C_1^2 C_4^2 C_6}{T^2}\right)^{\frac{1}{5}}
            \\
            &\leq&
            4 \sqrt{2} \sigma_L T^{-\frac{1}{2}}
            + 32\sqrt{2}\sigma \left(\frac{2[\kappa \vee \log\frac{2T}{\delta}]^{\frac{3}{2}}}{T}\right)^{\frac{1}{3}}
            + 48 (\sigma_\gamma^2 \sigma)^{\frac{1}{3}}\left(\frac{\kappa \vee \log\frac{2T}{\delta}}{T}\right)^{\frac{1}{3}}\\
            & &
            + 24(\sigma_\alpha^3 \sigma^2)^{\frac{1}{5}}\left(\frac{2[\kappa \vee \log\frac{2T}{\delta}]}{T^2}\right)^{\frac{1}{5}}
        \end{eqnarray*}
        To analyze the iteration complexity, the following inequalities of \(T\) are sufficient to show that \(\frac{1}{T}\sum_{t=1}^T\|\nabla f(w_t)\|_\ast \|P_t\|^2 \leq \varepsilon\) 
        for some \(\varepsilon > 0\) as 1st, 2nd, 3rd and 4th terms all are upper bounded by \(\varepsilon/4\),
        and noting that \(\kappa\geq 1\) and \(T \geq \frac{\sigma_\gamma^2/\sigma^2}{8\sqrt{\kappa}}\vee 2^{-\frac{9}{2}}\frac{(\sigma_\alpha/\sigma)^\frac{3}{2}}{\kappa^{\frac{3}{4}}}\)
        implies that \(T[\kappa \vee \log\frac{2T}{\delta}]^{\frac{1}{2}} \geq T\kappa^{\frac{1}{2}} \geq \frac{\sigma_\gamma^2}{8 \sigma^2}\)
        and \(T[\kappa \vee \log\frac{2T}{\delta}]^{\frac{3}{4}} \geq T \kappa^{\frac{3}{4}} \geq 2^{-\frac{9}{2}}\left(\frac{\sigma_\alpha}{\sigma}\right)^{\frac{3}{2}}\):
        \begin{align*}
            T \geq & \frac{\sigma_\gamma^2/\sigma^2}{8\sqrt{\kappa}}\vee
            2^{-\frac{9}{2}}\frac{(\sigma_\alpha/\sigma)^\frac{3}{2}}{\kappa^{\frac{3}{4}}}\vee 
            \frac{512 \sigma_L^2}{\varepsilon^2}
            \vee \frac{2 \kappa^{\frac{3}{2}}\cdot \sigma^3}{(\varepsilon/128\sqrt{2})^3}
            \vee \frac{\kappa \cdot\sigma_\gamma^2 \sigma}{(\varepsilon/192)^3} \vee \sqrt{2\kappa \frac{\sigma_\alpha^3 \sigma^2}{(\varepsilon/96)^5}},\\
            & 
            \frac{\log^{\frac{3}{2}}\frac{2T}{\delta}}{T} \leq \frac{1}{2}\left(\frac{\varepsilon}{128\sqrt{2}\sigma}\right)^3,\,
            \frac{\log\frac{2T}{\delta}}{T} \leq \frac{(\varepsilon/192)^3}{\sigma_\gamma^2 \sigma},\,
            \frac{\log^{\frac{1}{2}}\frac{2T}{\delta}}{T} \leq \sqrt{\frac{(\varepsilon/96)^5}{2\sigma_\alpha^3 \sigma^2}},
        \end{align*}
        By substituting \(\delta \gets \delta/2\) and \((q,\varepsilon) \gets (\frac{3}{2},\frac{1}{2}\left(\frac{\varepsilon}{128\sqrt{2}\sigma}\right)^3), (1,\frac{(\varepsilon/192)^3}{\sigma_\gamma^2 \sigma}), (\frac{1}{2},\sqrt{\frac{(\varepsilon/96)^5}{2\sigma_\alpha^3 \sigma^2}})\) in Lemma~\ref{lemma:complexity_log}, the following condition of \(T\)
        is sufficient to guarantee the above inequalities,
        and therefore \(\frac{1}{T}\sum_{t=1}^T\|\nabla f(w_t)\|_\ast \|P_t\|^2 \leq \varepsilon\).
        \begin{align*}
        T\geq& \frac{\sigma_\gamma^2/\sigma^2}{8\sqrt{\kappa}}\vee
        2^{-\frac{9}{2}}\frac{(\sigma_\alpha/\sigma)^\frac{3}{2}}{\kappa^{\frac{3}{4}}}\vee \frac{512 \sigma_L^2}{\varepsilon^2}\\
        & 
        \vee \frac{2\sigma^3}{(\varepsilon/128\sqrt{2})^3}\left[\kappa^{\frac{3}{2}}\vee\mathe^{\frac{9}{4}}\log^{\frac{3}{2}}\left[\mathe\vee \frac{4\sigma^3}{(\varepsilon/128\sqrt{2})^3\delta} \right]\right]\\
        &\vee \frac{\sigma_\gamma^2\sigma}{(\varepsilon/192)^3}\left[
            \kappa\vee\mathe \log\left[\mathe\vee \frac{2\sigma_\gamma^2\sigma}{(\varepsilon/192)^3\delta} \right]
        \right]\\
        &\vee \sqrt{\frac{2\sigma_\alpha^3\sigma^2}{(\varepsilon/96)^5}}\left[
            \sqrt{\kappa}\vee\mathe^{\frac{1}{4}} \log^{\frac{1}{2}}\left[\mathe\vee \frac{2}{\delta}\sqrt{\frac{2\sigma_\alpha^3\sigma^2}{(\varepsilon/96)^5}} \right]
        \right]
        \end{align*}
        In particular, for all \(\varepsilon \in \mathbb{R}_+\) is sufficiently small such 
        that \(\varepsilon \leq \varepsilon_0(\sigma_\alpha, \sigma_\gamma, \sigma, \kappa, \delta)\),
        then the lower bound on \(T\) simplifies asymptotically to
        \[
        \begin{aligned}
            T \;\ge\;  \max\Bigg\{ &  
            \frac{2\mathe^{\frac{9}{4}}\cdot \sigma^3}{(\varepsilon/128\sqrt{2})^3}\log^{\frac{3}{2}}\frac{4\sigma^3}{(\varepsilon/128\sqrt{2})^3\delta},\,
            \frac{\mathe \cdot \sigma_\gamma^2\sigma}{(\varepsilon/192)^3}\log\frac{2\sigma_\gamma^2\sigma}{(\varepsilon/192)^3\delta},\,\\
            & \mathe^{\frac{1}{4}}\sqrt{\frac{2\sigma_\alpha^3\sigma^2}{(\varepsilon/96)^5}}\log^{\frac{1}{2}}\left(\frac{2}{\delta}\sqrt{\frac{2\sigma_\alpha^3\sigma^2}{(\varepsilon/96)^5}}\right),\,
            \frac{2\sigma_L^2}{(\varepsilon/16)^2}
        \Bigg\}.
        \end{aligned}
        \]
        where the uniform upper bound \(\varepsilon_0(\sigma_\alpha, \sigma_\gamma, \sigma, \kappa, \delta)\) 
        is obtained by noting \(\kappa\geq 1\) and solving the following inequalities for \(\varepsilon\)
        to ensure the dominated terms in the above asymptotic bound are larger than the other terms
        of the lower bound on \(T\).

        \[
            \left\{
            \begin{aligned}
                & \frac{4\sigma^3}{(\varepsilon/128\sqrt{2})^3\delta} \geq \mathe,\,
                \frac{2\sigma_\gamma^2\sigma}{(\varepsilon/192)^3\delta} \geq \mathe,\,
                \frac{2}{\delta}\sqrt{\frac{2\sigma_\alpha^3\sigma^2}{(\varepsilon/96)^5}} \geq \mathe,\\
                & \mathe \log\frac{2\sigma_\gamma^2\sigma}{(\varepsilon/192)^3\delta} \geq \kappa,\,
                \mathe^{\frac{1}{4}} \log^{\frac{1}{2}}\left(\frac{2}{\delta}\sqrt{\frac{2\sigma_\alpha^3\sigma^2}{(\varepsilon/96)^5}}\right) \geq \sqrt{\kappa},\\
                & \mathe^{\frac{9}{4}}\log^{\frac{3}{2}}\frac{4\sigma^3}{(\varepsilon/128\sqrt{2})^3\delta} \geq \kappa^{\frac{3}{2}},\\
                & \frac{\mathe \cdot \sigma_\gamma^2\sigma}{(\varepsilon/192)^3} \geq \frac{\sigma_\gamma^2/\sigma^2}{8\sqrt{\kappa}},\,
                \mathe^{\frac{1}{4}} \sqrt{\frac{2\sigma_\alpha^3\sigma^2}{(\varepsilon/96)^5}} \geq 2^{-\frac{9}{2}}\frac{(\sigma_\alpha/\sigma)^\frac{3}{2}}{\kappa^{\frac{3}{4}}},\\
            \end{aligned}
            \right.
        \]

        The explicit expression of 
        \[\varepsilon_0(\sigma_\alpha, \sigma_\gamma, \sigma, \kappa, \delta)
        \equiv (\sigma_\alpha^3 \sigma^2)^{\frac{1}{5}} \cdot \varepsilon_1(\kappa, \delta) 
        \vee (\sigma_\gamma^2 \sigma)^{\frac{1}{3}} \cdot \varepsilon_2(\kappa, \delta)
        \vee \sigma \cdot \varepsilon_3(\kappa, \delta)\] 
        is given by:
        \[
        \begin{aligned}
            \varepsilon_1(\kappa, \delta) &= 96 \left( \frac{8}{\delta^2 \, \mathe^{2 \vee (2\kappa/\sqrt{\mathe})}} \right)^{\!1/5}\\            
            \varepsilon_2(\kappa, \delta) &= 192 \left( \frac{2}{\delta \, \mathe^{1 \vee (\kappa/\mathe)}} \right)^{\!1/3}\\            
            \varepsilon_3(\kappa, \delta) &= 384(\mathe\kappa^{1/2})^{\!1/3} \;\wedge\; 384(\mathe^{1/2}\kappa^{3/2})^{\!1/5} \;\wedge\; 128\sqrt{2} \left( \frac{4}{\delta \, \mathe^{1 \vee (\kappa/\mathe^{3/2})}} \right)^{\!1/3}        
        \end{aligned}
        \]
        Therefore, the iteration complexity for \(T\) and oracle complexity for \(N=B\cdot T\) are given by:
        \[
        N=B\cdot T = T=\mathcal{O}\left(
            \frac{\sigma^3}{\varepsilon^3}\log^{\frac{3}{2}}\frac{\sigma^3}{\varepsilon^3\delta}
            \vee \frac{\sigma_\gamma^2 \sigma}{\varepsilon^3}\log\frac{\sigma_\gamma^2 \sigma}{\varepsilon^3\delta}
            \vee \sqrt{\frac{\sigma_\alpha^3\sigma^2}{\varepsilon^5}}\log^{\frac{1}{2}}\left(\frac{1}{\delta}\sqrt{\frac{\sigma_\alpha^3\sigma^2}{\varepsilon^5}}\right)
            \vee \frac{\sigma_L^2}{\varepsilon^2}
        \right)
        \]
        with the selection of batch size \(B=1\) to guarantee \(\frac{1}{T}\sum_{t=1}^T\|\nabla f(w_t)\|_\ast \|P_t\|^2 \leq \varepsilon\) with probability at least \(1-\delta\)
        for sufficiently small \(\varepsilon\).

        \textbf{Family 3, Case 2 (Second-Order PAGE)} \((p_t=p,\beta_t=1)\):
        By invoking the established upper bound of \(\|e_t\|_*\) for the estimator \(v_t\) with second-order correction term \(\mathcal{T}_t\), the selection of \(p_t=p, \beta_t =1\) and \(\eta_t=\eta\) in the third-level section of Appendix~\ref{subsubsec:family3case2}, we show the following upper bound of \(\frac{1}{T}\sum_{t=1}^T\|e_t\|_*\):
        \[
        \frac{1}{T}\sum_{t=1}^T \|e_t\|_* \leq C\left(\frac{\delta}{4T},\kappa\right)\sigma/\sqrt{B}+\frac{\alpha\eta^2}{2p}\log\frac{4T}{\delta}+C\!\left(\frac{\delta}{4T},\kappa\right)\gamma\eta\sqrt{\frac{2}{p}\log\frac{4T}{\delta}}
        \]

        By selecting the batch size \(B = \frac{1}{p}\),
        and noting \(C\!\left(\frac{\delta}{4T},\kappa\right) \leq 2 \sqrt{2[\kappa \vee \log\frac{4T}{\delta}]}\), we have:
        \begin{eqnarray*}
        & &\frac{1}{T}\sum_{t=1}^T\|\nabla f(w_t)\|_* \|P_t\|^2\leq\frac{4\Delta_f}{\eta T}+2L\eta
        +8 \cdot \frac{1}{T}\sum_{t=1}^T\|e_t\|_*
        \\
        & \leq &
        \frac{4 \Delta_f}{\eta T} + 2L\cdot \eta + 0 + 16 \sigma \sqrt{2[\kappa\vee\log\frac{4T}{\delta}]} \cdot \sqrt{p} + 32\gamma \sqrt{\log\frac{4T}{\delta}[\kappa\vee \log\frac{4T}{\delta}]} \cdot \frac{\eta}{\sqrt{p}}\\
        & & + 4\alpha \log\frac{4T}{\delta} \cdot \frac{\eta^2}{p}
        \end{eqnarray*}

        Therefore, by letting \(C_1 = 4 \Delta_f, C_2 = 2L, C_3 = 0, C_4 = 16 \sigma \sqrt{2[\kappa\vee\log\frac{4T}{\delta}]}, C_5 = 32\gamma \sqrt{\log\frac{4T}{\delta}[\kappa\vee \log\frac{4T}{\delta}]}, C_6 = 4\alpha \log\frac{4T}{\delta}\) in Lemma~\ref{lemma:opt_bound_stepsize_reset_batch},
        selecting the step size \(\eta\) and the reset probability \(p\) such that
        \(\eta = \min\Big\{\left(\frac{C_1}{C_2T}\right)^{\frac{1}{2}}, \allowbreak \left(\frac{C_1^2}{\sqrt{2} C_4 C_5 T^2}\right)^{\frac{1}{3}}, \allowbreak \left(\frac{C_1^3}{4C_4^2 C_6 T^3}\right)^{\frac{1}{5}}\Big\} = \min\Big\{\frac{\sqrt{2}}{L} \sigma_L T^{-\frac{1}{2}}, \allowbreak \frac{1}{4\gamma}\left(\frac{\sigma_\gamma^4}{\sigma}\right)^{\frac{1}{3}}(\log\frac{4T}{\delta}[\kappa\vee \log\frac{4T}{\delta}]^2)^{-\frac{1}{6}} T^{-\frac{2}{3}}, \allowbreak \frac{1}{2}\sqrt{\frac{1}{2\alpha}\left(\frac{2\sigma_\alpha^9}{\sigma^4}\right)^{\frac{1}{5}}} [\log\frac{4T}{\delta}]^{-\frac{1}{5}}[\kappa\vee\log\frac{4T}{\delta}]^{-\frac{1}{5}} T^{-\frac{3}{5}} \Big\}\),
        the reset probability \(p\) and batch size \(B\) as
        \(\frac{1}{p} = B = \left\lceil \min\Big\{ \left( \frac{C_4^4 T^2}{2 C_1^2 C_5^2} \right)^{\frac{1}{3}}, \allowbreak \left( \frac{C_4^6 T^4}{16 C_1^4 C_6^2} \right)^{\frac{1}{5}} \Big\} \right\rceil = \left\lceil \min\Big\{2\left(\frac{\sigma}{\sigma_\gamma}\right)^{\frac{4}{3}}\left(\frac{\kappa\vee \log\frac{4T}{\delta}}{\log\frac{4T}{\delta}}\right)^{\frac{1}{3}}T^{\frac{2}{3}}, \allowbreak 4\left( \frac{2\sigma^6}{\sigma_\alpha^6}\right)^{\frac{1}{5}} [\kappa\vee\log\frac{4T}{\delta}]^{\frac{1}{5}}\left(\frac{\kappa\vee\log\frac{4T}{\delta}}{\log\frac{4T}{\delta}}\right)^{\frac{2}{5}} T^{\frac{4}{5}}\Big\} \right\rceil\)\\
        when \(T\left(\frac{\kappa\vee\log\frac{4T}{\delta}}{\log\frac{4T}{\delta}}\right)^{\frac{1}{2}} \geq \frac{\sigma_\gamma^2}{2\sqrt{2}\sigma^2}\)
        and \(T [\kappa\vee\log\frac{4T}{\delta}]^{\frac{1}{4}} \left(\frac{\kappa\vee\log\frac{4T}{\delta}}{\log\frac{4T}{\delta}}\right)^{\frac{1}{2}}\geq 2^{-\frac{11}{4}} \left(\frac{\sigma_\alpha}{\sigma}\right)^{\frac{3}{2}}\), we have:
        \begin{eqnarray*}
        & &\frac{1}{T}\sum_{t=1}^T\|\nabla f(w_t)\|_\ast \|P_t\|^2
        \leq 2\left(\frac{C_1 C_2}{T}\right)^{\frac{1}{2}} + 3\left(\frac{\sqrt{2} C_1 C_4 C_5}{T}\right)^{\frac{1}{3}}
        + 5 \left(\frac{C_1^2 C_4^2 C_6}{8T^2}\right)^{\frac{1}{5}}\\
        & = & 4\sqrt{2} \sigma_L T^{-\frac{1}{2}} 
        + 48(\sigma_\gamma^2 \sigma)^{\frac{1}{3}} \left(\frac{[\log\frac{4T}{\delta}]^{\frac{1}{2}} [\kappa\vee\log\frac{4T}{\delta}]}{T}\right)^{\frac{1}{3}}
        + 20(\sigma_\alpha^3 \sigma^2)^{\frac{1}{5}} \left(\frac{2[\log\frac{4T}{\delta}]^{\frac{1}{2}} [\kappa\vee\log\frac{4T}{\delta}]^{\frac{1}{2}}}{T}\right)^{\frac{2}{5}}
        \end{eqnarray*}
        To analyze the iteration complexity, the following inequalities of \(T\) are sufficient to show that \(\frac{1}{T}\sum_{t=1}^T\|\nabla f(w_t)\|_\ast \|P_t\|^2 \leq \varepsilon\) 
        for some \(\varepsilon > 0\) as 1st, 2nd and 3rd terms are upper bounded by \(\varepsilon/2, \varepsilon/4\) and \(\varepsilon/4\) respectively,
        noting that \(T\geq \frac{\sigma_\gamma^2}{2\sqrt{2}\sigma^2} \vee \frac{2^{-\frac{11}{4}}}{\kappa^{\frac{1}{4}}}\left(\frac{\sigma_\alpha}{\sigma}\right)^{\frac{3}{2}}\) 
        implies \(T \left(\frac{\kappa \vee \log\frac{4T}{\delta}}{\log\frac{4T}{\delta}}\right)^{\frac{1}{2}} \geq T \geq \frac{\sigma_\gamma^2}{2\sqrt{2}\sigma^2}\)
        and \(T [\kappa\vee\log\frac{4T}{\delta}]^{\frac{1}{4}} \left(\frac{\kappa\vee\log\frac{4T}{\delta}}{\log\frac{4T}{\delta}}\right)^{\frac{1}{2}} \geq T\cdot \kappa^{\frac{1}{4}} \geq 2^{-\frac{11}{4}} \left(\frac{\sigma_\alpha}{\sigma}\right)^{\frac{3}{2}}\):
        \[
        \begin{aligned}
            T\geq & \frac{\sigma_\gamma^2}{2\sqrt{2}\sigma^2} \vee \frac{2^{-\frac{11}{4}}}{\kappa^{\frac{1}{4}}}\left(\frac{\sigma_\alpha}{\sigma}\right)^{\frac{3}{2}} \vee \frac{128 \sigma_L^2}{\varepsilon^2},\\
            & \frac{\log^{\frac{1}{2}}\frac{4T}{\delta}}{T} \leq \frac{(\varepsilon/192)^3}{\kappa \cdot\sigma_\gamma^2 \sigma},\,
            \frac{\log^{\frac{3}{2}}\frac{4T}{\delta}}{T} \leq \frac{(\varepsilon/192)^3}{\sigma_\gamma^2 \sigma},\,
            \frac{\log^{\frac{1}{2}}\frac{4T}{\delta}}{T} \leq \frac{1}{2} \sqrt{\frac{(\varepsilon/80)^5}{\kappa \sigma_\alpha^3\sigma^2}},\,
            \frac{\log\frac{4T}{\delta}}{T} \leq \frac{1}{2}\sqrt{\frac{(\varepsilon/80)^5}{\sigma_\alpha^3\sigma^2}}
        \end{aligned}
        \]
        By substituting \(\delta \gets \delta/4\) and \((q, \varepsilon) \gets (\frac{1}{2}, \frac{(\varepsilon/192)^3}{\kappa\cdot\sigma_\gamma^2 \sigma}), (\frac{3}{2}, \frac{(\varepsilon/192)^3}{\sigma_\gamma^2 \sigma}), (\frac{1}{2},\frac{1}{2}\sqrt{\frac{(\varepsilon/80)^5}{\kappa\cdot\sigma_\alpha^3\sigma^2}}), (1,\frac{1}{2}\sqrt{\frac{(\varepsilon/80)^5}{\sigma_\alpha^3\sigma^2}})\) in Lemma~\ref{lemma:complexity_log}, the following condition of \(T\)
        is sufficient to guarantee the above inequalities,
        and therefore \(\frac{1}{T}\sum_{t=1}^T\|\nabla f(w_t)\|_\ast \|P_t\|^2 \leq \varepsilon\).
        \[
        \begin{aligned}
            T\geq & \frac{\sigma_\gamma^2}{2\sqrt{2}\sigma^2} \vee \frac{2^{-\frac{11}{4}}}{\kappa^{\frac{1}{4}}}\left(\frac{\sigma_\alpha}{\sigma}\right)^{\frac{3}{2}} \vee \frac{128 \sigma_L^2}{\varepsilon^2}\\
            & \vee \frac{\sigma_\gamma^2\sigma}{(\varepsilon/192)^3}\left[\mathe^{\frac{9}{4}}\log^{\frac{3}{2}}\left[\mathe\vee\frac{4\sigma_\gamma^2\sigma}{(\varepsilon/192)^3\delta}\right]
            \vee \mathe^{\frac{1}{4}} \kappa\log^{\frac{1}{2}}\left[\mathe\vee\kappa \cdot\frac{4\sigma_\gamma^2\sigma}{(\varepsilon/192)^3\delta}\right] \right]\\
            &\vee 2\sqrt{\frac{\sigma_\alpha^3\sigma^2}{(\varepsilon/80)^5}}\left[\mathe\log\left[\mathe\vee\left(\frac{8}{\delta}\sqrt{\frac{\sigma_\alpha^3\sigma^2}{(\varepsilon/80)^5}}\right)\right]
            \vee \mathe^{\frac{1}{4}}\sqrt{\kappa}\log^{\frac{1}{2}}\left[\mathe\vee\sqrt{\kappa}\left(\frac{8}{\delta}\sqrt{\frac{2\sigma_\alpha^3\sigma^2}{(\varepsilon/80)^5}}\right)\right]\right]
        \end{aligned}
        \]
        In particular, for all \(\varepsilon \in \mathbb{R}_+\) is sufficiently small such 
        that \(\varepsilon \leq \varepsilon_0(\sigma_\alpha, \sigma_\gamma, \sigma, \kappa, \delta)\),
        then the lower bound on \(T\) simplifies asymptotically to
        \[
        T \geq \max\Big\{
        \frac{\mathe^{\frac{9}{4}}\sigma_\gamma^2\sigma}{(\varepsilon/192)^3}\log^{\frac{3}{2}}\frac{4\sigma_\gamma^2\sigma}{(\varepsilon/192)^3\delta},\allowbreak \,
        \allowbreak 2\mathe\sqrt{\frac{\sigma_\alpha^3\sigma^2}{(\varepsilon/80)^5}}\log\left(\frac{8}{\delta}\sqrt{\frac{\sigma_\alpha^3\sigma^2}{(\varepsilon/80)^5}}\right),\allowbreak \,
        \allowbreak \frac{2\sigma_L^2}{(\varepsilon/8)^2}\Big\}
        \]
        where the uniform upper bound \(\varepsilon_0(\sigma_\alpha, \sigma_\gamma, \sigma, \kappa, \delta)\) of \(\varepsilon\) is obtained by noting \(\kappa\geq 1\) and solving the following inequalities for \(\varepsilon\)
        to ensure the dominated terms in the above asymptotic bound are larger than the other terms
        of the lower bound on \(T\).

        \[
        \left\{
        \begin{aligned}
            &\frac{4\sigma_\gamma^2\sigma}{(\varepsilon/192)^3\delta} \geq \max(\mathe,\kappa),\,
            \frac{8}{\delta}\sqrt{\frac{\sigma_\alpha^3\sigma^2}{(\varepsilon/80)^5}} \geq \max(\mathe,\sqrt{\kappa})\\
            &\frac{\mathe^{\frac{9}{4}}}{\mathe^{\frac{1}{4}} \sqrt{2}} \log\frac{4\sigma_\gamma^2\sigma}{(\varepsilon/192)^3\delta} \geq 
            \kappa,\,
            \frac{\mathe}{\mathe^{\frac{1}{4}}\sqrt{2}} \log^{\frac{1}{2}}\left(\frac{8}{\delta}\sqrt{\frac{\sigma_\alpha^3\sigma^2}{(\varepsilon/80)^5}}\right) \geq \sqrt{\kappa}\\
            &\frac{\mathe^{\frac{9}{4}}\cdot \sigma_\gamma^2\sigma}{(\varepsilon/192)^3} \geq \frac{\sigma_\gamma^2}{2\sqrt{2}\sigma^2},\,
            2\mathe \sqrt{\frac{\sigma_\alpha^3\sigma^2}{(\varepsilon/80)^5}} \geq \frac{2^{-\frac{11}{4}}}{\kappa^{\frac{1}{4}}}\left(\frac{\sigma_\alpha}{\sigma}\right)^{\frac{3}{2}}
        \end{aligned}
        \right.
        \]

        The explicit expression of 
        \[\varepsilon_0(\sigma_\alpha, \sigma_\gamma, \sigma, \kappa, \delta) \equiv (\sigma_\alpha^3\sigma^2)^{\frac{1}{5}} \cdot \varepsilon_1(\kappa, \delta) \vee (\sigma_\gamma^2\sigma)^{\frac{1}{3}} \cdot \varepsilon_2(\kappa, \delta) \vee \sigma \cdot \varepsilon_3(\kappa, \delta)\] 
        is given by:
        \[
        \begin{aligned}
            \varepsilon_1(\kappa, \delta) &= 80 \left( \frac{64}{\delta^2 \left( \mathe^2 \vee \kappa \vee \mathe^{4\kappa/\mathe^{3/2}} \right)} \right)^{\!1/5}\\            
            \varepsilon_2(\kappa, \delta) &= 192 \left( \frac{4}{\delta \left( \mathe \vee \kappa \vee \mathe^{\sqrt{2}\kappa/\mathe^2} \right)} \right)^{\!1/3}\\            
            \varepsilon_3(\kappa, \delta) &= 192\sqrt{2}\mathe^{3/4} \;\wedge\; 160\sqrt{2}\mathe^{2/5}\kappa^{1/10}        
        \end{aligned}
        \]
        Therefore, the iteration complexity for \(T\) and the expected value of oracle complexity for 
        \(N = [B\cdot p + 1\cdot (1-p)]T=(2-1/B)T\) are given by:
        \[
        N,\,T = \mathcal{O}\left(\frac{\sigma_\gamma^2\sigma}{\varepsilon^3}\log^{\frac{3}{2}}\frac{\sigma_\gamma^2\sigma}{\varepsilon^3\delta}\vee 
        \sqrt{\frac{\sigma_\alpha^3\sigma^2}{\varepsilon^5}}\log\left(\frac{1}{\delta}\sqrt{\frac{\sigma_\alpha^3\sigma^2}{\varepsilon^5}}\right)\vee 
        \frac{\sigma_L^2}{\varepsilon^2}\right)
        \]
        with the selection of batch size \(B=\frac{1}{p}\) to guarantee \(\frac{1}{T}\sum_{t=1}^T\|\nabla f(w_t)\|_\ast \|P_t\|^2 \leq \varepsilon\) with probability at least \(1-\delta\)
        for sufficiently small \(\varepsilon\).

        \textbf{Family 3, Case 3 (Second-Order SPIDER)} \((p_t=\mathbb{I}_{\{t\!\!\mod E=0\}},\beta_t=1)\):
        By invoking the established upper bound of \(\|e_t\|_*\) for the estimator \(v_t\) with second-order correction term \(\mathcal{T}_t\), the selection of \(p_t=\mathbb{I}_{\{t\mod E = 0\}}, \beta_t =1\) and \(\eta_t=\eta\) in the third-level section of Appendix~\ref{subsubsec:family3case3}, we show the following upper bound of \(\frac{1}{T}\sum_{t=1}^T\|e_t\|_*\):
        \[
        \frac{1}{T}\sum_{t=1}^T\|e_t\|_* \leq C\!\left(\frac{\delta}{2T/E},\kappa\right)\sigma/\sqrt{B}+\frac{\alpha\eta^2E}{2}+C\!\left(\frac{\delta}{2T},\kappa\right)\gamma\eta\sqrt{2E}
        \]

        By selecting the batch size \(B = E\),
        noting \(C\left(\frac{\delta}{2T/E},\kappa\right)\leq C\!\left(\frac{\delta}{2T},\kappa\right) \leq 2 \sqrt{2[\kappa \vee \log\frac{2T}{\delta}]}\), we have:
        \begin{eqnarray*}
        & &\frac{1}{T}\sum_{t=1}^T\|\nabla f(w_t)\|_*\|P_t\|^2\leq\frac{4\Delta_f}{\eta T}+2L\eta
        +8 \cdot \frac{1}{T}\sum_{t=1}^T \|e_t\|_* 
        \\
        & \leq & 
        \frac{4 \Delta_f}{\eta T} + 2L\cdot \eta + 0 + 16 \sigma \sqrt{2[\kappa\vee\log\frac{2T}{\delta}]} \cdot \frac{1}{\sqrt{E}} 
        + 32\gamma \sqrt{\kappa\vee \log\frac{2T}{\delta}} \cdot \eta \sqrt{E}
        + 4\alpha \cdot \eta^2 E
        \end{eqnarray*}

        Set \(C_1 = 4 \Delta_f, C_2 = 2L, C_3 = 0, C_4 = 16 \sigma \sqrt{2[\kappa\vee\log\frac{2T}{\delta}]}, C_5 = 32\gamma \sqrt{\kappa\vee \log\frac{2T}{\delta}}, C_6 = 4\alpha\) in Lemma~\ref{lemma:opt_bound_stepsize_reset_batch},
        selecting \(\eta = \min\Big\{\left(\frac{C_1}{C_2T}\right)^{\frac{1}{2}}, \allowbreak \left(\frac{C_1^2}{\sqrt{2} C_4 C_5 T^2}\right)^{\frac{1}{3}}, \allowbreak \left(\frac{C_1^3}{4C_4^2 C_6 T^3}\right)^{\frac{1}{5}}\Big\}\)
        = \(\min\Big\{\frac{\sqrt{2}}{L} \sigma_L T^{-\frac{1}{2}}, \allowbreak \frac{1}{4\gamma}\left(\frac{\sigma_\gamma^4}{\sigma}\right)^{\frac{1}{3}}[\kappa\vee \log\frac{2T}{\delta}]^{-\frac{1}{3}} T^{-\frac{2}{3}}, \allowbreak \frac{1}{2}\sqrt{\frac{1}{2\alpha}\left(\frac{2\sigma_\alpha^9}{\sigma^4}\right)^{\frac{1}{5}}} [\kappa\vee\log\frac{2T}{\delta}]^{-\frac{1}{5}} T^{-\frac{3}{5}}\Big\}\)
        and the reset epoch \(E\) and batch size \(B\) as \(E = B = \left\lceil \min\Big\{\left(\frac{C_4^4 T^2}{2 C_1^2 C_5^2}\right)^{\frac{1}{3}}, \allowbreak \left(\frac{C_4^6 T^4}{16 C_1^4 C_6^2}\right)^{\frac{1}{5}}\Big\} \right\rceil\)
        = \(\left\lceil \min\Big\{2\left(\frac{\sigma}{\sigma_\gamma}\right)^{\frac{4}{3}}\left[\kappa\vee \log\frac{2T}{\delta}\right]^{\frac{1}{3}}T^{\frac{2}{3}}, \allowbreak 4\left( \frac{2\sigma^6}{\sigma_\alpha^6}\right)^{\frac{1}{5}} [\kappa\vee\log\frac{2T}{\delta}]^{\frac{3}{5}} T^{\frac{4}{5}}\Big\} \right\rceil\)
        when \(T\left[\kappa \vee \log\frac{2T}{\delta}\right]^{\frac{1}{2}} \geq \frac{\sigma_\gamma^2}{2\sqrt{2}\sigma^2}\)
        and \(T [\kappa\vee\log\frac{2T}{\delta}]^{\frac{3}{4}} \geq 2^{-\frac{11}{4}} \left(\frac{\sigma_\alpha}{\sigma}\right)^{\frac{3}{2}}\), we have:
        \begin{eqnarray*}
        & &\frac{1}{T}\sum_{t=1}^T\|\nabla f(w_t)\|_\ast \|P_t\|^2
        \leq 2\left(\frac{C_1 C_2}{T}\right)^{\frac{1}{2}} + 3\left(\frac{\sqrt{2} C_1 C_4 C_5}{T}\right)^{\frac{1}{3}}
        + 5 \left(\frac{C_1^2 C_4^2 C_6}{8T^2}\right)^{\frac{1}{5}}\\
        & = & 4\sqrt{2} \sigma_L T^{-\frac{1}{2}} 
        + 48(\sigma_\gamma^2 \sigma)^{\frac{1}{3}} \left(\frac{\kappa\vee\log\frac{2T}{\delta}}{T}\right)^{\frac{1}{3}}
        + 20(\sigma_\alpha^3 \sigma^2)^{\frac{1}{5}} \left(\frac{2[\kappa\vee\log\frac{2T}{\delta}]^{\frac{1}{2}}}{T}\right)^{\frac{2}{5}}
        \end{eqnarray*}
        To analyze the iteration complexity, the following inequalities of \(T\) are sufficient to show that \(\frac{1}{T}\sum_{t=1}^T\|\nabla f(w_t)\|_\ast \|P_t\|^2 \leq \varepsilon\) 
        for some \(\varepsilon > 0\) as 1st, 2nd and 3rd terms are upper bounded by \(\varepsilon/2, \varepsilon/4\) and \(\varepsilon/4\) respectively,
        noting that \(T\geq \frac{\sigma_\gamma^2}{2\sqrt{2\kappa}\sigma^2} \vee \frac{2^{-\frac{11}{4}}}{\kappa^{\frac{3}{4}}}\left(\frac{\sigma_\alpha}{\sigma}\right)^{\frac{3}{2}}\) 
        implies \(T\left[\kappa \vee \log\frac{2T}{\delta}\right]^{\frac{1}{2}} \geq T\kappa^{\frac{1}{2}} \geq \frac{\sigma_\gamma^2}{2\sqrt{2}\sigma^2}\)
        and \(T [\kappa\vee\log\frac{2T}{\delta}]^{\frac{3}{4}} \geq T \kappa^{\frac{3}{4}} \geq 2^{-\frac{11}{4}} \left(\frac{\sigma_\alpha}{\sigma}\right)^{\frac{3}{2}}\):
        \[
        \begin{aligned}
            T\geq & \frac{\sigma_\gamma^2}{2\sqrt{2\kappa}\sigma^2} \vee \frac{2^{-\frac{11}{4}}}{\kappa^{\frac{3}{4}}}\left(\frac{\sigma_\alpha}{\sigma}\right)^{\frac{3}{2}} \vee \frac{128 \sigma_L^2}{\varepsilon^2} \vee \frac{\kappa \cdot \sigma_\gamma^2\sigma}{(\varepsilon/192)^3} \vee 2\sqrt{\kappa}\sqrt{\frac{\sigma_\alpha^3 \sigma^2}{(\varepsilon/80)^5}},\\
            & 
            \frac{\log\frac{2T}{\delta}}{T} \leq \frac{(\varepsilon/192)^3}{\sigma_\gamma^2 \sigma},\,
            \frac{\log^{\frac{1}{2}}\frac{2T}{\delta}}{T} \leq \frac{1}{2}\sqrt{\frac{(\varepsilon/80)^5}{\sigma_\alpha^3\sigma^2}}
        \end{aligned}
        \]
        By substituting \(\delta \gets \delta/2\) and \((q, \varepsilon) \gets (1, \frac{(\varepsilon/192)^3}{\sigma_\gamma^2 \sigma}), (\frac{1}{2},\frac{1}{2}\sqrt{\frac{(\varepsilon/80)^5}{\sigma_\alpha^3\sigma^2}})\) in Lemma~\ref{lemma:complexity_log}, the following condition of \(T\)
        is sufficient to guarantee the above inequalities,
        and therefore \(\frac{1}{T}\sum_{t=1}^T\|\nabla f(w_t)\|_\ast \|P_t\|^2 \leq \varepsilon\).
        \[
        \begin{aligned}
            T\geq & \frac{\sigma_\gamma^2}{2\sqrt{2\kappa}\sigma^2} \vee \frac{2^{-\frac{11}{4}}}{\kappa^{\frac{3}{4}}}\left(\frac{\sigma_\alpha}{\sigma}\right)^{\frac{3}{2}} \vee \frac{128 \sigma_L^2}{\varepsilon^2}
            \vee \frac{\sigma_\gamma^2\sigma}{(\varepsilon/192)^3}\left[\kappa \vee \mathe\log\left[\mathe\vee \frac{2\sigma_\gamma^2\sigma}{(\varepsilon/192)^3\delta} \right]\right]\\
            &\vee 2\sqrt{\frac{\sigma_\alpha^3\sigma^2}{(\varepsilon/80)^5}}\left[\sqrt{\kappa} \vee \mathe^{\frac{1}{4}}\log^{\frac{1}{2}}\left[\mathe\vee\left(\frac{4}{\delta}\sqrt{\frac{\sigma_\alpha^3\sigma^2}{(\varepsilon/80)^5}}\right)\right]\right]
        \end{aligned}
        \]
        In particular, for all \(\varepsilon \in \mathbb{R}_+\) is sufficiently small such 
        that \(\varepsilon \leq \varepsilon_0(\sigma_\alpha, \sigma_\gamma, \sigma, \kappa, \delta)\),
        then the lower bound on \(T\) simplifies asymptotically to
        \[
        T \geq \max\Big\{
        \frac{\mathe\cdot\sigma_\gamma^2\sigma}{(\varepsilon/192)^3}\log\frac{2\sigma_\gamma^2\sigma}{(\varepsilon/192)^3\delta},\allowbreak \,
        \allowbreak 2\mathe^{\frac{1}{4}}\sqrt{\frac{\sigma_\alpha^3\sigma^2}{(\varepsilon/80)^5}}\log^{\frac{1}{2}}\left(\frac{4}{\delta}\sqrt{\frac{\sigma_\alpha^3\sigma^2}{(\varepsilon/80)^5}}\right),\allowbreak \,
        \allowbreak \frac{2\sigma_L^2}{(\varepsilon/8)^2}\Big\}
        \]
        where the uniform upper bound \(\varepsilon_0(\sigma_\alpha, \sigma_\gamma, \sigma, \kappa, \delta)\) of \(\varepsilon\) is obtained by noting \(\kappa\geq 1\) and solving the following inequalities for \(\varepsilon\)
        to ensure the dominated terms in the above asymptotic bound are larger than the other terms
        of the lower bound on \(T\).

        \[
        \left\{
        \begin{aligned}
            &\frac{2\sigma_\gamma^2\sigma}{(\varepsilon/192)^3\delta} \geq \mathe,\,
            \frac{4}{\delta}\sqrt{\frac{\sigma_\alpha^3\sigma^2}{(\varepsilon/80)^5}} \geq \mathe,\\
            &\mathe \log\frac{2\sigma_\gamma^2\sigma}{(\varepsilon/192)^3\delta} \geq 
            \kappa,
            \mathe^{\frac{1}{4}}\log^{\frac{1}{2}}\left(\frac{4}{\delta}\sqrt{\frac{\sigma_\alpha^3\sigma^2}{(\varepsilon/80)^5}}\right)\geq \sqrt{\kappa},\\
            &\frac{\mathe\cdot\sigma_\gamma^2\sigma}{(\varepsilon/192)^3} \geq \frac{1}{2\sqrt{2\kappa}}\frac{\sigma_\gamma^2}{\sigma^2},\,
            2\mathe^{\frac{1}{4}} \sqrt{\frac{\sigma_\alpha^3\sigma^2}{(\varepsilon/80)^5}} \geq \frac{2^{-\frac{11}{4}}}{\kappa^{\frac{3}{4}}}\left(\frac{\sigma_\alpha}{\sigma}\right)^{\frac{3}{2}}
        \end{aligned}
        \right.
        \]

        The explicit expression of 
        \[\varepsilon_0(\sigma_\alpha, \sigma_\gamma, \sigma, \kappa, \delta) \equiv (\sigma_\alpha^3\sigma^2)^{\frac{1}{5}} \cdot \varepsilon_1(\kappa, \delta) \vee (\sigma_\gamma^2\sigma)^{\frac{1}{3}} \cdot \varepsilon_2(\kappa, \delta) \vee \sigma \cdot \varepsilon_3(\kappa, \delta) \] 
        is given by:
        \[
        \begin{aligned}
            \varepsilon_1(\kappa, \delta) &= 80 \left( \frac{16}{\delta^2 \mathe^{2 \vee (2\kappa/\sqrt{\mathe})}} \right)^{\!1/5}\\            
            \varepsilon_2(\kappa, \delta) &= 192 \left( \frac{2}{\delta \mathe^{1 \vee (\kappa/\mathe)}} \right)^{\!1/3}\\            
            \varepsilon_3(\kappa, \delta) &= 192\sqrt{2}(\mathe\kappa^{1/2})^{\!1/3} \;\wedge\; 160\sqrt{2}(\mathe^{1/2}\kappa^{3/2})^{\!1/5}        
        \end{aligned}
        \]
        Therefore, the iteration complexity for \(T\) and the oracle complexity for 
        \(N = B\cdot T/E+1\cdot (T-T/E) = T(2-1/B)\) are given by:
        \[
        N,\,T = \mathcal{O}\left(\frac{\sigma_\gamma^2\sigma}{\varepsilon^3}\log\frac{\sigma_\gamma^2\sigma}{\varepsilon^3\delta}\vee 
        \sqrt{\frac{\sigma_\alpha^3\sigma^2}{\varepsilon^5}}\log^{\frac{1}{2}}\left(\frac{1}{\delta}\sqrt{\frac{\sigma_\alpha^3\sigma^2}{\varepsilon^5}}\right)\vee 
        \frac{\sigma_L^2}{\varepsilon^2}\right)
        \]
        with the selection of batch size \(B=E\) to guarantee \(\frac{1}{T}\sum_{t=1}^T\|\nabla f(w_t)\|_\ast \|P_t\|^2 \leq \varepsilon\) with probability at least \(1-\delta\)
        for sufficiently small \(\varepsilon\).
\end{proof}

\newpage
\section{Application 2: Expectation-Constrained Stochastic Optimization}
\label{apdx:constrained_sgm}
\paragraph{Optimization Problem}
We aim to solve the stochastic constrained optimization problem:
\[\min_{w \in \mathcal{W}} f(w) \coloneqq \mathbb{E}[F(w, \xi)] \quad \text{s.t.} \quad h(w) \coloneqq \mathbb{E}[H(w, \xi)] \le 0,\]
where \(F(w, \xi):\mathcal{W}\to \mathbb{R}\) and \(H(w, \xi):\mathcal{W} \to \mathbb{R}\) are convex, \(\mathcal{W}\) is a closed, convex subset of \(\mathcal{X}\).

\paragraph{The Unified Estimator and SGM Update Rule}
\begin{itemize}[leftmargin=*]
    \item \textbf{Update Direction:} \(U_t \coloneqq \mathbb{I}_{\{v_t \leq \varepsilon\}}  F^{'}(w_t, \zeta_t) + (1 - \mathbb{I}_{\{v_t \leq \varepsilon\}}) H^{'}(w_t, \zeta_t)\).

    \item \textbf{Mirror Descent Step:}
    \[w_{t+1} = \underset{w \in \mathcal{W}}{\arg\min} \left\{ \langle \eta_t U_t, w \rangle + D_\Phi(w, w_t) \right\} = w_t - \eta_t P(w_t, U_t, \eta_t)\]
    where \(\eta_t = \eta\) is the constant step size.
\end{itemize}

The active-phase feasible and non-feasible index sets
\[
    \mathcal A \coloneqq \{t\in[T]: v_t\le \varepsilon\}, \qquad
    \mathcal B \coloneqq [T]\setminus \mathcal A.
\]
\begin{theorem}[Constrained SGM via Unified Estimation, Theorem~\ref{thm:constrained_main}]
\label{thm:constrained_generic}
Suppose Assumption~\ref{ass:constrained_setup} hold. Suppose there exists a deterministic quantity
\(\mathcal{E}\ge 0\) such that
\begin{equation}
    \label{ass:tail_envelope}
    \mathbb P\!\left(
    |e_t|\le \mathcal{E}
    \quad \text{for all } t\in[T]
    \right)\ge 1-\frac{\delta}{2}.
\end{equation}
Define
\(\mathsf{OptError}(\eta,T,\delta) \coloneqq \frac{R^2}{2\eta T} + \frac{\eta G^2}{2} + \frac{2DG}{\sqrt{T}}\sqrt{2\log\frac{4}{\delta}}.\)
Set the switching threshold to
\(\varepsilon \coloneqq \mathsf{OptError}(\eta,T,\delta)+\mathcal{E}.\)
Then, with probability at least \(1-\delta\), the selected set \(\mathcal A\) is non-empty,
and the active-phase selected average
\(\bar w \coloneqq \frac{1}{|\mathcal A|}\sum_{t\in\mathcal A} w_t\)
satisfies
\[
f(\bar w)-f(w^\star)\le \varepsilon, \qquad
h(\bar w)\le \varepsilon+\mathcal{E},
\]
where \(w^\star\) is any optimal feasible solution.
\end{theorem}
In all family and parameter regimes considered below, the envelope \(\mathcal{E}\) in \eqref{ass:tail_envelope} is verified by specializing Theorem~\ref{thm:unified} and bounding the resulting estimator error uniformly over \(t\in[T]\).

\begin{proof}
We work on the intersection of the following two high-probability events:
\begin{itemize}
    \item the tail-envelope event from \eqref{ass:tail_envelope},
    \item the martingale concentration event established below.
\end{itemize}
By a union bound, this intersection of the two high-probability events has probability at least \(1-\delta\).

\noindent\textbf{Step 1: Mirror-descent inequality on the active phase.}
By the first-order optimality condition (see Remark~\ref{remark:first_order_optimality_condition} of First-Order Optimality Condition) for the mirror-descent update 
such that \(\langle \eta_t U_t + \nabla\Phi(w_{t+1})- \nabla\Phi(w_t),\, w^\star-w_{t+1}\rangle \ge 0\), 
and using the three-point identity for Bregman divergences, we obtain
\[
\eta_t \langle U_t, w_{t+1} - w^\star\rangle \leq [D_\Phi(w^\star,w_t)-D_\Phi(w^\star,w_{t+1})] - D_\Phi(w_{t+1},w_t).
\]
By the \(1\)-strong convexity of \(\Phi\), we obtain \(D_\Phi(w_{t+1},w_t) \geq \frac{1}{2}\|w_t-w_{t+1}\|^2 = \frac{\eta_t^2}{2}\|P(w_t, U_t, \eta_t)\|^2\),
also noting that \(\|U_t\|_\ast^2 + \|P(w_t, U_t, \eta_t)\|^2 \geq 2\|U_t\|_\ast \|P(w_t, U_t, \eta_t)\| \geq 2 \langle U_t, P(w_t, U_t, \eta_t)\rangle\), hence,
\[
\begin{aligned}
    \langle U_t,w_t-w^\star\rangle &= \langle U_t,w_{t+1}-w^\star\rangle + \langle U_t,w_t-w_{t+1}\rangle \\
    &\leq \frac{1}{\eta_t}[D_\Phi(w^\star,w_t)-D_\Phi(w^\star,w_{t+1})] - \frac{\eta_t}{2}\|P(w_t, U_t, \eta_t)\|^2 + \eta_t \langle U_t, P(w_t, U_t, \eta_t)\rangle\\
    &\leq \frac{1}{\eta_t}[D_\Phi(w^\star,w_t)-D_\Phi(w^\star,w_{t+1})] + \frac{\eta_t}{2}\|U_t\|_\ast^2 
\end{aligned}
\]
Summing over \(t\in[T]\) and using \(\eta_t = \eta\) yields
\[
\sum_{t\in[T]}\langle U_t,w_t-w^\star\rangle \le
\frac{D_\Phi(w^\star,w_{1})-D_\Phi(w^\star,w_{T+1})}{\eta} +
\frac{\eta}{2}\sum_{t\in[T]}\|U_t\|_\ast^2.
\]
Since \(\sup_{u,v\in\mathcal W}D_\Phi(u,v)\le R^2/2\) and \(\|U_t\|_\ast\le G\) almost surely, and \(D_\Phi(w^\star,w_{T+1}) \geq 0\),
\begin{equation}
\label{eq:active_md_rhs_clean}
\sum_{t\in[T]}\langle U_t,w_t-w^\star\rangle
\le
\frac{R^2}{2\eta}
+
\frac{\eta T G^2}{2}.
\end{equation}

\noindent\textbf{Step 2: Convexity decomposition and martingale concentration.}

Let \(\mathcal G_t := \sigma(\calF_{t-1}, \xi_t^B\mid_{b_t=1}, \xi_t\mid_{b_t=0} ,b_t)\) denote the sigma-field generated by the history \(\calF_{t-1}\) and the estimator randomness up to the switching decision (which also depends on \(\xi_t^B\mid_{b_t=1}, \xi_t\mid_{b_t=0} ,b_t\)) at iteration \(t\). By construction, the indicator \(I_t=\mathbb{I}_{\{v_t\le \varepsilon\}}\) is \(\mathcal G_t\)-measurable,
\(w_t - w^\star\) is \(\mathcal G_t\)-measurable, while \(\zeta_t\) is independent of \(\calG_t\).
Define
\[
g_t^f \coloneqq \mathbb E[ F^{'}(w_t,\zeta_t)\mid \mathcal G_t]\in \partial f(w_t), \qquad
g_t^h \coloneqq \mathbb E[ H^{'}(w_t,\zeta_t)\mid \mathcal G_t]\in \partial h(w_t).
\]
Also define
\[
M_t \coloneqq I_t\langle F^{'}(w_t,\zeta_t)-g_t^f,\; w_t-w^\star\rangle +
(1-I_t)\langle H^{'}(w_t,\zeta_t)-g_t^h,\; w_t-w^\star\rangle.
\]
By the conditional unbiasedness assumption,
\(
\mathbb E[M_t\mid \mathcal G_t]=0,
\)
so \(\{M_t\}_{t\in[T]}\) is a martingale difference sequence.
Moreover,
\[
|M_t| \le 2G\|w_t-w^\star\| \le 2DG.
\]
Hence, Azuma--Hoeffding inequality implies that with probability at least \(1-\delta/2\),
\[
-\sum_{t\in[T]} M_t \le
2DG\sqrt{2T\log\frac{4}{\delta}}.
\]
By convexity of \(f\) and \(h\), and \(w^\star\) is feasible, \(h(w^\star)\le 0\), hence
\[
\langle g_t^f,w_t-w^\star\rangle \ge f(w_t)-f(w^\star), \qquad
\langle g_t^h,w_t-w^\star\rangle \ge h(w_t)-h(w^\star) \ge h(w_t).
\]
Therefore,
\[
\sum_{t\in[T]}\langle U_t,w_t-w^\star\rangle \ge
\sum_{t\in\mathcal A}\bigl(f(w_t)-f(w^\star)\bigr) +
\sum_{t\in\mathcal B} h(w_t) + \sum_{t\in[T]} M_t.
\]
Combining this with \eqref{eq:active_md_rhs_clean}, dividing by \(T\), and using the Azuma bound gives
\begin{equation}
\label{eq:pre_substitution_clean}
\frac{1}{T}\sum_{t\in\mathcal A}\bigl(f(w_t)-f(w^\star)\bigr)
+
\frac{1}{T}\sum_{t\in\mathcal B} h(w_t)
\le
\frac{R^2}{2\eta T} + \frac{\eta G^2}{2} + \frac{2DG}{\sqrt{T}}\sqrt{2\log\frac{4}{\delta}}
=: \mathsf{OptError}(\eta,T,\delta).
\end{equation}

Now use \(h(w_t)=v_t-e_t\ge v_t-|e_t|\) on \(\mathcal B\):
\[
\frac{1}{T}\sum_{t\in\mathcal A}\bigl(f(w_t)-f(w^\star)\bigr) +
\frac{1}{T}\sum_{t\in\mathcal B} v_t \le
\mathsf{OptError}(\eta,T,\delta) +
\frac{1}{T}\sum_{t\in\mathcal B}|e_t|.
\]
On the tail-envelope event, \(|e_t|\le \mathcal{E}\) for all \(t\in[T]\), hence by letting \(\varepsilon = \mathsf{OptError}(\eta,T,\delta)+\mathcal{E}\),
\[
\frac{1}{T}\sum_{t\in\mathcal A}\bigl(f(w_t)-f(w^\star)\bigr) +
\frac{1}{T}\sum_{t\in\mathcal B} v_t \le
\mathsf{OptError}(\eta,T,\delta)+\mathcal{E} = \varepsilon.
\]
Thus
\begin{equation}
\label{eq:key_active_bound_clean}
\frac{1}{T}\sum_{t\in\mathcal A}\bigl(f(w_t)-f(w^\star)\bigr)
+
\frac{1}{T}\sum_{t\in\mathcal B} v_t
\le
\varepsilon.
\end{equation}

\noindent\textbf{Step 3: Non-emptiness of the selected set.}
Assume for contradiction that \(\mathcal A=\varnothing\). Then \(\mathcal B=[T]\), and by definition of \(\mathcal B\) we have \(v_t>\varepsilon\) for every \(t\in\mathcal B\). Therefore
\(\frac{1}{T}\sum_{t\in\mathcal B}v_t > \varepsilon,\)
which contradicts \eqref{eq:key_active_bound_clean}. Hence \(\mathcal A\neq\varnothing\).

\noindent\textbf{Step 4: Optimality of the selected average.}
If
\(\sum_{t\in\mathcal A}\bigl(f(w_t)-f(w^\star)\bigr)\le 0,\)
then convexity of \(f\) implies
\(f(\bar w)-f(w^\star)\le 0\le \varepsilon.\)
Otherwise, the sum over \(\mathcal A\) is positive. Since \(v_t>\varepsilon\) for all \(t\in\mathcal B\), \eqref{eq:key_active_bound_clean} implies
\[
\frac{1}{T}\sum_{t\in\mathcal A}\bigl(f(w_t)-f(w^\star)\bigr) <
\varepsilon-\frac{|\mathcal B|}{T}\varepsilon =
\frac{|\mathcal A|}{T}\varepsilon.
\]
Dividing by \(|\mathcal A|/T\) and using Jensen's inequality yields
\(f(\bar w)-f(w^\star) \le \frac{1}{|\mathcal A|}\sum_{t\in\mathcal A}\bigl(f(w_t)-f(w^\star)\bigr) < \varepsilon.\)
Hence
\(f(\bar w)-f(w^\star)\le \varepsilon.\)

\noindent\textbf{Step 5: Feasibility of the selected average.}
By the convexity of \(h\),
\[
h(\bar w) \le \frac{1}{|\mathcal A|}\sum_{t\in\mathcal A} h(w_t) =
\frac{1}{|\mathcal A|}\sum_{t\in\mathcal A}\bigl(v_t-e_t\bigr) \le
\frac{1}{|\mathcal A|}\sum_{t\in\mathcal A}\bigl(v_t+|e_t|\bigr).
\]
Since \(v_t\le \varepsilon\) on \(\mathcal A\) and \(|e_t|\le \mathcal{E}\) on the entire active phase,
\[
h(\bar w) \le \frac{1}{|\mathcal A|}\sum_{t\in\mathcal A}\bigl(\varepsilon+\mathcal{E}\bigr) =
\varepsilon+\mathcal{E}.
\]
This completes the proof.
\end{proof}

\begin{remark}[Constrained SGM on a Reusable Analysis Window]
We can extend the above analysis to a selected set such that for \(t_0 < T\) and \(T_a := T- t_0\)
\[
\mathcal{A} := \{ t\in [T]: t> t_0, v_t\leq \varepsilon\}
\]
By letting the threshold be
\[
\varepsilon:=\frac{R^2}{2 \eta T_a}+\frac{\eta G^2}{2}+\frac{2 D G}{\sqrt{T_a}} \sqrt{2 \log \frac{4}{\delta}}+\mathcal{E},
\]
and \(\Pr(\sup_{t\in[t_0+1:T]} |e_t|> \calE)\leq \delta/2\),
then, with probability at least \(1-\delta\), the selected set \(\mathcal{A}\) is non-empty, and the uniform average \(\bar{w}=\frac{1}{|\mathcal{A}|} \sum_{t \in \mathcal{A}} w_t\) satisfies both \(f(\bar{w})-f\left(w^*\right) \leq \varepsilon\) and \(h(\bar{w}) \leq \varepsilon+\mathcal{E}\).
\end{remark}

\begin{table}[!t]
\centering
\small                          %
\setlength{\tabcolsep}{3.5pt}  %
\renewcommand{\arraystretch}{1.35}
\caption{%
Cross-reference of results for Application~2
(Section~\ref{sec:app2} and Appendix~\ref{apdx:constrained_sgm}).
For each parameter regime (case), the table lists the tuned-bound theorem
and complexity corollary, and the dominant \(\varepsilon\)-dependence of the
oracle complexity for each family.
All entries use \(B{=}1\) for Case~1, \(B{=}1/p\) for Case~2,
and \(B{=}E\) for Case~3, yielding \(N{<}2T\).
}
\label{tab:app2_cross_reference}
\begin{tabular}{l || c | c || c | c | c}
\hline\hline
& \multicolumn{2}{c||}{\textbf{Reference}}
& \multicolumn{3}{c}{\textbf{Dominant \(\varepsilon\)-rate}} \\
\textbf{Case}
& \textbf{Bound}
& \textbf{Complexity}
& \textbf{Fam.\,1}
& \textbf{Fam.\,2}
& \textbf{Fam.\,3} \\
\hline\hline

\rowcolor{orange!10}
\multicolumn{6}{l}{\textit{Case~1:}\;
\(\beta_t{=}\beta{<}1,\; p_t{=}0,\; B{=}1\)} \\
Mom.\ / STORM / SO\,Mom.
& Thm.\,\ref{thm:constrained_case1_all_families_tuned}
& Cor.\,\ref{cor:constrained_case1_all_families_complexity}
& \(\varepsilon^{-4}\)
& \(\varepsilon^{-3}\)
& \(\varepsilon^{-3}\;(\varepsilon^{-\frac{5}{2}})\) \\
\hline

\rowcolor{orange!10}
\multicolumn{6}{l}{\textit{Case~2:}\;
\(\beta_t{=}1,\; p_t{=}p{>}0,\; B{=}1/p\)} \\
Prob.\,Mom.\ / PAGE / SO\,PAGE
& Thm.\,\ref{thm:constrained_case2_all_families_tuned}
& Cor.\,\ref{cor:constrained_case2_all_families_complexity}
& \(\varepsilon^{-4}\)
& \(\varepsilon^{-3}\)
& \(\varepsilon^{-3}\;(\varepsilon^{-\frac{5}{2}})\) \\
\hline

\rowcolor{orange!10}
\multicolumn{6}{l}{\textit{Case~3:}\;
\(\beta_t{=}1,\; p_t{=}\mathbb{I}_{\{t\bmod E=0\}},\; B{=}E\)} \\
Per.\,Mom.\ / SPIDER / SO\,SPIDER
& Thm.\,\ref{thm:constrained_case3_all_families_tuned}
& Cor.\,\ref{cor:constrained_case3_all_families_complexity}
& \(\varepsilon^{-4}\)
& \(\varepsilon^{-3}\)
& \(\varepsilon^{-3}\;(\varepsilon^{-\frac{5}{2}})\) \\
\hline\hline
\end{tabular}

\vspace{2pt}
{\footnotesize
The notation \(\varepsilon^{-3}\;(\varepsilon^{-5/2})\) for Family~3 indicates
a dominant \(\varepsilon^{-3}\) term (from the \(\gamma\)-dependent variance)
with a subdominant \(\varepsilon^{-5/2}\) contribution (from the
\(\alpha\)-dependent second-order bias).
The base SGM reduction used by all entries is
Theorem~\ref{thm:constrained_main}\,/\,Theorem~\ref{thm:constrained_generic}.
}
\end{table}

\subsection{Case 1 for Families~1--3 in the constrained setting}
\label{subsec:constrained_case1_all_families}
We instantiate Theorem~\ref{thm:constrained_generic} for the three Case~1 estimators, and note that \(\kappa =1\) for \((\mathbb{R}, |\cdot|)\).
Throughout this subsection, \(1+\sqrt{3} < 2 \sqrt{2},\, \log\frac{4T}{\delta} \geq \log 4 > 1 = \kappa\), then
\[
C\!\left(\frac{\delta}{4T},\kappa\right)
=\sqrt{\kappa}+\sqrt{3\log\frac{4T}{\delta}}\le (1+\sqrt{3})\sqrt{\kappa \vee \log\frac{4T}{\delta}} \leq 2\sqrt{2[\kappa \vee \log\frac{4T}{\delta}]} = 2\sqrt{2 \log\frac{4T}{\delta}}.
\]
\begin{theorem}[Constrained Case~1 Bounds for Families~1--3]
\label{thm:constrained_case1_all_families_tuned}
\allowdisplaybreaks
For each family below, set
\(\varepsilon \coloneqq \mathsf{OptError}(\eta,T_a,\delta)+\mathcal{E}\)
with  \(T_a = T- t_0, t_0 = \lfloor T/2\rfloor\), select the batch size \(B=1\), and let the selected set be
\(\mathcal A = \{t\in[T]:t> t_0, v_t\le \varepsilon\},\, \Gamma_T := 2 \log\frac{4T}{\delta}\).
\begin{enumerate}[leftmargin=*]
\item \textbf{Family~1.}
Choose step size
\(\eta = \min\Big\{\sqrt{2}\frac{R}{G} T^{-\frac{1}{2}}, \allowbreak \left(\frac{R^6}{32\sigma^2 L G\cdot \Gamma_T}\right)^{\frac{1}{4}} T^{-\frac{3}{4}} \Big\},\)
and momentum parameter
\(1-\beta = \max\Big\{T^{-\frac{2}{3}}, \allowbreak \left(\frac{R^2 L G}{8\sigma^2 \cdot \Gamma_T}\right)^{\frac{1}{2}} T^{-\frac{1}{2}}\Big\}.\)
Assume moreover that
\(T\Gamma_T \geq \frac{R^2LG}{8\sigma^2}.\)
The selected set
\(\mathcal A\)
is nonempty with probability at least \(1-\delta\), and the active-phase selected average
\(\bar w = \frac{1}{|\mathcal A|}\sum_{t\in\mathcal A} w_t\)
satisfies
\(f(\bar w)-f(w^\star)\le Q_T^{(1,1)},\) \(h(\bar w)\le Q_T^{(1,1)},\)
where
\[
Q_T^{(1,1)} = \sqrt{2} RG T^{-\frac{1}{2}} +  
            8\sigma \left(\frac{\Gamma_T^{\frac{3}{2}}}{T}\right)^{\frac{1}{3}}
            + 3 \left( 32 R^2 L G \sigma^2\right)^{\frac{1}{4}} \left(\frac{\Gamma_T}{T}\right)^{\frac{1}{4}} + 4DG \left( \frac{\log\frac{4}{\delta}}{T}\right)^{\frac{1}{2}}.
\]
\item \textbf{Family~2.}
Choose
\(\eta = \min\Big\{ \frac{\sqrt{2} R}{G} T^{-\frac{1}{2}}, \allowbreak \left(\frac{R^4}{32\ell G\sigma \Gamma_T}\right)^{\frac{1}{3}}T^{-\frac{2}{3}} \Big\},\)
\(1-\beta = \max\Big\{ 2^{-\frac{1}{3}} T^{-\frac{2}{3}}, \allowbreak \left(\frac{R^4 \ell^2 G^2/\sigma^4}{32 \Gamma_T}\right)^{\frac{1}{3}}T^{-\frac{2}{3}} \Big\}.\)
Assume moreover that
\(T\Gamma_T^{1/2} > \frac{R^2\ell G}{4\sqrt{2}\,\sigma^2}.\)
The selected set
\(\mathcal A\)
is nonempty with probability at least \(1-\delta\), and
\(f(\bar w)-f(w^\star)\le Q_T^{(2,1)},\) \(h(\bar w)\le Q_T^{(2,1)},\)
where
\[
Q_T^{(2,1)} = \sqrt{2} RG T^{-\frac{1}{2}} 
            + 8\sigma \left(\frac{2\Gamma_T^{\frac{3}{2}}}{T}\right)^{\frac{1}{3}}
            + 3 (32 R^2 \ell G\sigma)^{\frac{1}{3}} \left(\frac{ \Gamma_T}{T}\right)^{\frac{1}{3}}
            + 4DG \left( \frac{\log\frac{4}{\delta}}{T}\right)^{\frac{1}{2}}.
\]
\item \textbf{Family~3.}
Choose
\(\eta = \min\Big\{ \frac{\sqrt{2} R}{G} T^{-\frac{1}{2}}, \allowbreak \left(\frac{R^4}{32\gamma G\sigma \Gamma_T}\right)^{\frac{1}{3}}T^{-\frac{2}{3}}, \allowbreak \left(\frac{R^6}{32 \alpha G^2 \sigma^2\Gamma_T}\right)^{\frac{1}{5}}T^{-\frac{3}{5}} \Big\},\)
\(1-\beta = \max\Big\{ 2^{-\frac{1}{3}} T^{-\frac{2}{3}}, \allowbreak \left(\frac{R^4 \gamma^2 G^2/\sigma^4}{32 \Gamma_T}\right)^{\frac{1}{3}}T^{-\frac{2}{3}}, \allowbreak \frac{1}{8}\left(\frac{R^8 \alpha^2 G^4/\sigma^6}{\Gamma_T^3}\right)^{\frac{1}{5}}T^{-\frac{4}{5}} \Big\}.\)
Assume moreover that 
\(T\Gamma_T^{\frac{1}{2}} \geq \frac{R^2 \gamma G}{4\sqrt{2} \sigma^2}\)
        and \(T\Gamma_T^{\frac{3}{4}} \geq 2^{-\frac{15}{4}}\left(\frac{R^2 \sqrt{\alpha} G}{\sigma}\right)^{\frac{3}{2}}\).
The selected set
\(\mathcal A\)
is nonempty with probability at least \(1-\delta\), and
\(f(\bar w)-f(w^\star)\le Q_T^{(3,1)},\) \(h(\bar w)\le Q_T^{(3,1)},\)
where
\[
\begin{aligned} 
    Q_T^{(3,1)} &= \sqrt{2} RG T^{-\frac{1}{2}} 
            + 8\sigma \left(\frac{2\Gamma_T^{\frac{3}{2}}}{T}\right)^{\frac{1}{3}}
            + 3 (32 R^2 \gamma G\sigma)^{\frac{1}{3}} \left(\frac{ \Gamma_T}{T}\right)^{\frac{1}{3}}
            + 6( R^4\alpha G^2\sigma^2)^{\frac{1}{5}}\left(\frac{\Gamma_T}{T^2}\right)^{1/5}\\
            & + 4DG \left( \frac{\log\frac{4}{\delta}}{T}\right)^{\frac{1}{2}}. 
\end{aligned}
\]
\end{enumerate}
\end{theorem}

\begin{proof}
For the analysis of case 1 in all families, we let \(t_0 = \lfloor \frac{T}{2} \rfloor\), then \(T_a = T - t_0 \geq T/2\), by noting that \(x \exp(-x) \leq 1/\mathe < 1/2, \forall x>0\), then we have the following bound for \(t\in[t_0+1:T]\)
\[
\beta^t = (1-[1-\beta])^t 
\leq (1-[1-\beta])^{T/2} \leq \exp(-\frac{1-\beta}{2} T)
\leq \frac{1/2}{\frac{1-\beta}{2} T}
= \frac{1}{(1-\beta)T}
\]
The selected set is defined by \(\calA := \{t\in[T]: t > t_0, v_t \leq \varepsilon\}\).
Applying the established Theorem~\ref{thm:constrained_main} and using its remark.
\[
\max\{f(\bar{w})-f(w^\ast), h(\bar{w})\} \leq
\varepsilon + \calE = (\mathsf{OptError}(\eta,T_a,\delta)+\calE) + \calE
\leq \mathsf{OptError}(\eta,T/2,\delta)+ 2\calE
\]

\textbf{Family 1}
By substituting the high-probability upperbound \(\calE\) of \(|e_t|\) from the instantiations of Theorem~\ref{thm:unified} (replacing \(\delta\) with \(\delta/2\) and selecting \(B=1\) in the bound of Table~\ref{tab:error_bounds}) into the above inequality, 
using \(\beta^t \leq \frac{1}{(1-\beta) T}\) for \(t\in[t_0+1:T]\) and \(C(\delta/(4T), \kappa)\leq 2 \sqrt{2\log(4T/\delta)} =: 2\sqrt{ \Gamma_T}\), we have 
\[
\begin{aligned} &\mathsf{OptError}(\eta,T/2,\delta)+ 2\calE
\leq \left(\frac{R^2}{\eta T} + \frac{\eta G^2}{2} + 4DG \sqrt{\frac{\log\frac{4}{\delta}}{T}} \right)\\
&+ 2\left( \frac{C\left(\frac{\delta}{4T},\kappa\right)\sigma}{(1-\beta)T} +
\frac{LG\eta}{1-\beta} +
C\left(\frac{\delta}{4T},\kappa\right)\sigma\sqrt{1-\beta}\right)\\
&\leq 4 DG \sqrt{\frac{\log\frac{4}{\delta}}{T}}
+ \left[\frac{R^2}{\eta T} + \frac{G^2}{2}\cdot \eta + 4 \sqrt{\Gamma_T} \sigma \cdot \frac{1}{(1-\beta) T} + 4 \sqrt{\Gamma_T} \sigma \cdot \sqrt{1-\beta} + 2LG \cdot \frac{\eta}{1-\beta}\right] \end{aligned}
\]
By letting \(C_1 = R^2, C_2 =\frac{G^2}{2}, C_3 = 4 \sqrt{\Gamma_T} \sigma, C_4 = 4 \sqrt{\Gamma_T} \sigma, C_5 = 2LG\) in Lemma~\ref{lemma:opt_bound_stepsize_momentum}, selecting the step size and momentum parameter as
\(\eta = \min\Big\{\left(\frac{C_1}{C_2 T}\right)^{\frac{1}{2}}, \allowbreak \left(\frac{C_1^3}{C_4^2 C_5 T^3}\right)^{\frac{1}{4}}\Big\} = \min\Big\{\sqrt{2}\frac{R}{G} T^{-\frac{1}{2}}, \allowbreak \left(\frac{R^6}{32\sigma^2 L G\cdot \Gamma_T}\right)^{\frac{1}{4}} T^{-\frac{3}{4}} \Big\},\)
\(1-\beta = \max\Big\{\left(\frac{C_3}{C_4 T}\right)^{\frac{2}{3}}, \allowbreak \left(\frac{C_1 C_5}{C_4^2 T}\right)^{\frac{1}{2}}\Big\} = \max\Big\{T^{-\frac{2}{3}}, \allowbreak \left(\frac{R^2 L G}{8\sigma^2 \cdot \Gamma_T}\right)^{\frac{1}{2}} T^{-\frac{1}{2}}\Big\}\)
when \(T\Gamma_T\geq \frac{R^2 LG}{8 \sigma^2}\), we have:
\begin{eqnarray*}
    & &\max\{f(\bar{w})-f(w^\ast), h(\bar{w})\}\\
    & 
    \leq &
    2\left(\frac{C_1 C_2}{T}\right)^{\frac{1}{2}} + 2\left(\frac{C_3 C_4^2}{T}\right)^{\frac{1}{3}}
    + 3\left(\frac{C_1 C_4^2 C_5}{T}\right)^{\frac{1}{4}}
    + 4DG \left( \frac{\log\frac{4}{\delta}}{T}\right)^{\frac{1}{2}}
    \\
    &\leq& 
    \sqrt{2} RG T^{-\frac{1}{2}} +  
    8\sigma \left(\frac{\Gamma_T^{\frac{3}{2}}}{T}\right)^{\frac{1}{3}}
    + 3 \left( 32 R^2 L G \sigma^2\right)^{\frac{1}{4}} \left(\frac{\Gamma_T}{T}\right)^{\frac{1}{4}} + 4DG \left( \frac{\log\frac{4}{\delta}}{T}\right)^{\frac{1}{2}}
\end{eqnarray*}

\textbf{Family 2}
By substituting the high probability upperbound \(\calE\) of \(|e_t|\) from the instantiations of Theorem~\ref{thm:unified} (replacing \(\delta\) with \(\delta/2\) and selecting \(B=1\) in the bound of Table~\ref{tab:error_bounds}) into the above inequality, 
and using \(\beta^t \leq \frac{1}{(1-\beta) T}\) for \(t\in[t_0+1:T]\) and \(C(\delta/(4T), \kappa)\leq 2 \sqrt{2\log(4T/\delta)} =: 2\sqrt{ \Gamma_T}\), noting that \( \frac{(1-\beta)^2}{1-\beta^2} \leq 1-\beta, \frac{\beta^2}{1-\beta^2}\leq \frac{1}{1-\beta}\) for \(\beta\in(0, 1)\) and the elementary inequality \(\sqrt{x+y} \leq \sqrt{x} + \sqrt{y}, \forall x, y\geq 0\), we show the following upper bound for \(\calE\):
\[
\calE \leq \frac{C\left(\frac{\delta}{4T},\kappa\right)\sigma}{(1-\beta)T}+C\!\left(\frac{\delta}{4T},\kappa\right)\left[\sigma \sqrt{2(1-\beta)}+\ell G \eta\sqrt{\frac{2}{1-\beta}}\right],
\] then we have 
\[
\begin{aligned} &\mathsf{OptError}(\eta,T/2,\delta)+ 2\calE
\leq \left(\frac{R^2}{\eta T} + \frac{\eta G^2}{2} + 4DG \sqrt{\frac{\log\frac{4}{\delta}}{T}} \right)\\
&+ 2\left( \frac{C\left(\frac{\delta}{4T},\kappa\right)\sigma}{(1-\beta)T}+C\!\left(\frac{\delta}{4T},\kappa\right)\left[\sigma \sqrt{2(1-\beta)}+\ell G \eta\sqrt{\frac{2}{1-\beta}}\right]\right)\\
&\leq 4 DG \sqrt{\frac{\log\frac{4}{\delta}}{T}}
+ \left[\frac{R^2}{\eta T} + \frac{G^2}{2}\cdot \eta + 4 \sqrt{\Gamma_T} \sigma \cdot \frac{1}{(1-\beta) T} + 4 \sqrt{2\Gamma_T} \sigma \cdot \sqrt{1-\beta} + 4\ell G \sqrt{2\Gamma_T} \cdot \frac{\eta}{\sqrt{1-\beta}}\right] \end{aligned}
\]
By letting \(C_1 = R^2, C_2 =\frac{G^2}{2}, C_3 = 4 \sqrt{\Gamma_T} \sigma, C_4 = 4 \sqrt{2\Gamma_T} \sigma, C_5 = 4\ell G\sqrt{2\Gamma_T}\) in Lemma~\ref{lemma:opt_bound_stepsize_momentum}, selecting the step size and momentum parameter as
\(\eta = \min\Big\{\left(\frac{C_1}{C_2}\right)^{\frac{1}{2}}T^{-\frac{1}{2}}, \allowbreak \left(\frac{C_1^2}{C_4 C_5}\right)^{\frac{1}{3}}T^{-\frac{2}{3}}\Big\} = \min\Big\{ \frac{\sqrt{2} R}{G} T^{-\frac{1}{2}}, \allowbreak \left(\frac{R^4}{32\ell G\sigma \Gamma_T}\right)^{\frac{1}{3}}T^{-\frac{2}{3}} \Big\}\),
\(1-\beta= \max\Big\{\left(\frac{C_3}{C_4 T}\right)^{\frac{2}{3}}, \allowbreak \left(\frac{C_1 C_5}{C_4^2 T}\right)^{\frac{2}{3}}\Big\} = \max\Big\{ 2^{-\frac{1}{3}} T^{-\frac{2}{3}}, \allowbreak \left(\frac{R^4 \ell^2 G^2/\sigma^4}{32 \Gamma_T}\right)^{\frac{1}{3}}T^{-\frac{2}{3}} \Big\}\)
then when \(T \Gamma_T^{\frac{1}{2}} \geq \frac{R^2 \ell G}{4\sqrt{2}\sigma^2}\), we have:
\begin{eqnarray*}
    & &\max\{f(\bar{w})-f(w^\ast), h(\bar{w})\}\\
    &\leq &2\left(\frac{C_1 C_2}{T}\right)^{\frac{1}{2}} + 2\left(\frac{C_3 C_4^2}{T}\right)^{\frac{1}{3}}
    + 3\left(\frac{C_1 C_4 C_5}{T}\right)^{\frac{1}{3}}
    + 4DG \left( \frac{\log\frac{4}{\delta}}{T}\right)^{\frac{1}{2}}
    \\
    &\leq&
    \sqrt{2} RG T^{-\frac{1}{2}} 
    + 8\sigma \left(\frac{2\Gamma_T^{\frac{3}{2}}}{T}\right)^{\frac{1}{3}}
    + 3 (32 R^2 \ell G\sigma)^{\frac{1}{3}} \left(\frac{ \Gamma_T}{T}\right)^{\frac{1}{3}}
    + 4DG \left( \frac{\log\frac{4}{\delta}}{T}\right)^{\frac{1}{2}}
\end{eqnarray*}

\textbf{Family 3}
By substituting the high probability upperbound \(\calE\) of \(|e_t|\) from the instantiations of Theorem~\ref{thm:unified} (replacing \(\delta\) with \(\delta/2\) and selecting \(B=1\) in the bound of Table~\ref{tab:error_bounds}) into the above inequality, 
and using \(\beta^t \leq \frac{1}{(1-\beta) T}\) for \(t\in[t_0+1:T]\) and \(C(\delta/(4T), \kappa)\leq 2 \sqrt{2\log(4T/\delta)} =: 2\sqrt{ \Gamma_T}\), noting that \( \frac{(1-\beta)^2}{1-\beta^2} \leq 1-\beta, \frac{\beta^2}{1-\beta^2}\leq \frac{1}{1-\beta}\) for \(\beta\in(0, 1)\) and the elementary inequality \(\sqrt{x+y} \leq \sqrt{x} + \sqrt{y}, \forall x, y\geq 0\), we show the following upper bound for \(\calE\):
\[
\calE \leq  \frac{C\left(\frac{\delta}{4T},\kappa\right)\sigma}{(1-\beta) T} +
\frac{\alpha\beta\eta^2G^2}{2(1-\beta)} +
C\left(\frac{\delta}{4T},\kappa\right)\left[\sigma \sqrt{2(1-\beta)}+\gamma G \eta\sqrt{\frac{2}{1-\beta}}\right],
\] then we have 
\[
\begin{aligned} &\mathsf{OptError}(\eta,T/2,\delta)+ 2\calE
\leq \left(\frac{R^2}{\eta T} + \frac{\eta G^2}{2} + 4DG \sqrt{\frac{\log\frac{4}{\delta}}{T}} \right)\\
&+ 2\left( \frac{C\left(\frac{\delta}{4T},\kappa\right)\sigma}{(1-\beta) T} +
\frac{\alpha\beta\eta^2G^2}{2(1-\beta)} +
C\left(\frac{\delta}{4T},\kappa\right)\left[\sigma \sqrt{2(1-\beta)}+\gamma G \eta\sqrt{\frac{2}{1-\beta}}\right]\right)\\
&\leq \left[\frac{R^2}{\eta T} + \frac{G^2}{2}\cdot \eta + 4 \sqrt{\Gamma_T} \sigma \cdot \frac{1}{(1-\beta) T} + 4 \sqrt{2\Gamma_T} \sigma \cdot \sqrt{1-\beta} + 4\gamma G \sqrt{2\Gamma_T} \cdot \frac{\eta}{\sqrt{1-\beta}}
+ \alpha G^2 \cdot \frac{\eta^2}{1-\beta}\right] \\
&  + 4 DG \sqrt{\frac{\log\frac{4}{\delta}}{T}} \end{aligned}
\]
By letting \(C_1 = R^2, C_2 =\frac{G^2}{2}, C_3 = 4 \sqrt{\Gamma_T} \sigma, C_4 = 4 \sqrt{2\Gamma_T} \sigma, C_5 = 4\gamma G\sqrt{2\Gamma_T}, C_6 = \alpha G^2\) in Lemma~\ref{lemma:opt_bound_stepsize_momentum}, selecting the step size and momentum parameter such that \(\eta = \min\Big\{\left(\frac{C_1}{C_2}\right)^{\frac{1}{2}}T^{-\frac{1}{2}}, \allowbreak \left(\frac{C_1^2}{C_4 C_5}\right)^{\frac{1}{3}}T^{-\frac{2}{3}}, \allowbreak \left(\frac{C_1^3}{C_4^2 C_6}\right)^{\frac{1}{5}}T^{-\frac{3}{5}}\Big\} = \min\Big\{ \frac{\sqrt{2} R}{G} T^{-\frac{1}{2}}, \allowbreak \left(\frac{R^4}{32\gamma G\sigma \Gamma_T}\right)^{\frac{1}{3}}T^{-\frac{2}{3}}, \allowbreak \left(\frac{R^6}{32 \alpha G^2 \sigma^2\Gamma_T}\right)^{\frac{1}{5}}T^{-\frac{3}{5}} \Big\}\),
the momentum parameter \(\beta\) as
\(1-\beta = \max\Big\{\left(\frac{C_3}{C_4}\right)^{\frac{2}{3}}T^{-\frac{2}{3}}, \allowbreak \left(\frac{C_1 C_5}{C_4^2}\right)^{\frac{2}{3}}T^{-\frac{2}{3}}, \allowbreak \left(\frac{C_1^2 C_6}{C_4^3}\right)^{\frac{2}{5}}T^{-\frac{4}{5}}\Big\} =\max\Big\{ 2^{-\frac{1}{3}} T^{-\frac{2}{3}}, \allowbreak \left(\frac{R^4 \gamma^2 G^2/\sigma^4}{32 \Gamma_T}\right)^{\frac{1}{3}}T^{-\frac{2}{3}}, \allowbreak \frac{1}{8}\left(\frac{R^8 \alpha^2 G^4/\sigma^6}{\Gamma_T^3}\right)^{\frac{1}{5}}T^{-\frac{4}{5}} \Big\}\),
then when \(T\Gamma_T^{\frac{1}{2}} \geq \frac{R^2 \gamma G}{4\sqrt{2} \sigma^2}\)
and \(T\Gamma_T^{\frac{3}{4}} \geq 2^{-\frac{15}{4}}\left(\frac{R^2 \sqrt{\alpha} G}{\sigma}\right)^{\frac{3}{2}}\), we have:
\begin{eqnarray*}
    & &\max\{f(\bar{w})-f(w^\ast), h(\bar{w})\}\\
    &\leq &2\left(\frac{C_1 C_2}{T}\right)^{\frac{1}{2}} + 2\left(\frac{C_3 C_4^2}{T}\right)^{\frac{1}{3}}
    + 3\left(\frac{C_1 C_4 C_5}{T}\right)^{\frac{1}{3}}
    + 3 \left(\frac{C_1^2 C_4^2 C_6}{T^2}\right)^{\frac{1}{5}}
    + 4DG \left( \frac{\log\frac{4}{\delta}}{T}\right)^{\frac{1}{2}}
    \\
    &\leq&
    \sqrt{2} RG T^{-\frac{1}{2}} 
    + 8\sigma \left(\frac{2\Gamma_T^{\frac{3}{2}}}{T}\right)^{\frac{1}{3}}
    + 3 (32 R^2 \gamma G\sigma)^{\frac{1}{3}} \left(\frac{ \Gamma_T}{T}\right)^{\frac{1}{3}}
    + 6( R^4\alpha G^2\sigma^2)^{\frac{1}{5}}\left(\frac{\Gamma_T}{T^2}\right)^{1/5}\\
    & &+ 4DG \left( \frac{\log\frac{4}{\delta}}{T}\right)^{\frac{1}{2}}
\end{eqnarray*}
\end{proof}

For later use in the complexity bounds, define
\(\mathcal L_q(u,\delta;c) \coloneqq \frac{\mathe^{q^2}}{u}\log^q\!\left(\mathe \vee \frac{c}{u\delta}\right),\)
\(q>0,\ u>0,\ c>0\).
By Lemma~\ref{lemma:complexity_log},
\[
T\ge \mathcal L_q(u,\delta;c)
\quad\Longrightarrow\quad
\frac{\log^q(cT/\delta)}{T}\le u.
\]
\begin{corollary}[Constrained Case~1 Oracle Complexity for Families~1--3]
\label{cor:constrained_case1_all_families_complexity}
Under the assumptions of Theorem~\ref{thm:constrained_case1_all_families_tuned}, let \(\varepsilon>0\), the iteration / oracle complexities to guarantee \(f(\bar{w}) -f(w^\star) \leq \varepsilon, h(\bar{w}) \leq \varepsilon\) for Families 1-3 are listed as follows, since \(B=1\) is selected as the batch size, the oracle complexity has the same order of the iteration complexity.

\begin{enumerate}[leftmargin=*]
\item \textbf{Family~1.}
\[
N,\,T = \mathcal{O}\Bigg(
\frac{R^2LG\sigma^2}{\varepsilon^4}
\log\frac{R^2LG\sigma^2}{\varepsilon^4\delta}
\vee \frac{\sigma^3}{\varepsilon^3}
\log^{3/2}\frac{\sigma^3}{\varepsilon^3\delta}
\vee \frac{R^2G^2+D^2G^2\log(1/\delta)}{\varepsilon^2}
\Bigg).
\]

\item \textbf{Family~2.}
\[
N,\,T = \mathcal{O}\Bigg(
\frac{\sigma^3}{\varepsilon^3}
\log^{3/2}\frac{\sigma^3}{\varepsilon^3\delta}
\vee \frac{R^2\ell G\sigma}{\varepsilon^3}
\log\frac{R^2\ell G\sigma}{\varepsilon^3\delta}
\vee \frac{R^2G^2+D^2G^2\log(1/\delta)}{\varepsilon^2}
\Bigg).
\]

\item \textbf{Family~3.}
\[
\begin{aligned} 
N,\,T = \mathcal{O}\Bigg(
\frac{\sigma^3}{\varepsilon^3}
\log^{3/2}\frac{\sigma^3}{\varepsilon^3\delta}
\vee \frac{R^2\gamma G\sigma}{\varepsilon^3}
\log\frac{R^2\gamma G\sigma}{\varepsilon^3\delta}  
&\vee \sqrt{\frac{ R^4\alpha G^2\sigma^2}{\varepsilon^5}}
\log^{1/2}\left( \frac{1}{\delta}\sqrt{\frac{ R^4\alpha G^2\sigma^2}{\varepsilon^5}}
\right) \\
&\vee \frac{R^2G^2+D^2G^2\log(1/\delta)}{\varepsilon^2} 
\Bigg). 
\end{aligned}
\]
\end{enumerate}
\end{corollary}

\begin{proof}
\textbf{Family 1}. To guarantee
\[
\begin{aligned} &\max\{f(\bar{w}) -f(w^\star), h(\bar{w})\} \\
\leq& Q_T^{(1, 1)} \equiv \left(\sqrt{2} RG+4 DG \sqrt{\log\frac{4}{\delta}}\right)
T^{-1/2} + 8\sigma \left(\frac{\Gamma_T^{\frac{3}{2}}}{T}\right)^{\frac{1}{3}}
            + 3 \left( 32 R^2 L G \sigma^2\right)^{\frac{1}{4}} \left(\frac{\Gamma_T}{T}\right)^{\frac{1}{4}}\\
\leq& \varepsilon/2 + \varepsilon/4 + \varepsilon/ 4  = \varepsilon \end{aligned}
\]
with sufficient large \(T\), where \(\Gamma_T : = 2\log\frac{4T}{\delta}\) satisfies \(T \Gamma_T \geq T \geq \frac{R^2LG}{8\sigma^2}\).
By noting \(T \geq \calL_q(u, \delta; c)= \frac{\mathe^{q^2}}{u}\log^q(\mathe\vee \frac{c}{u \delta})\) is sufficient to show \(\frac{\log^q \frac{cT}{\delta}}{T}\leq u\) (see Lemma~\ref{lemma:complexity_log}), we can choose \(T\) as follow to ensure \(\max\{f(\bar{w}) -f(w^\star), h(\bar{w})\}\leq \varepsilon\).
\[
\begin{aligned} T \geq &\frac{R^2LG}{8\sigma^2}
\vee (\varepsilon/2)^{-2} \left(\sqrt{2} RG+4 DG \sqrt{\log\frac{4}{\delta}}\right)^2\\
&\vee\frac{(32\sqrt{2}\sigma)^3}{\varepsilon^3} 
\vee \mathe^{9/4} \log^{3/2}\left[\mathe \vee \frac{(32\sqrt{2} \sigma)^3}{\varepsilon^3}\cdot \frac{4}{\delta}\right]\\ 
&\vee \frac{64 R^2 LG \sigma^2}{(\varepsilon/12)^4} 
\mathe \log\left[\mathe \vee\frac{64 R^2 LG \sigma^2}{(\varepsilon/12)^4}\cdot \frac{4}{\delta}\right]
\end{aligned}
\]

\textbf{Family 2}. To guarantee
\[
\begin{aligned} &\max\{f(\bar{w}) -f(w^\star), h(\bar{w})\} \\
\leq& Q_T^{(2, 1)} \equiv \left(\sqrt{2} RG+4 DG \sqrt{\log\frac{4}{\delta}}\right)
T^{-1/2} + 8\sigma \left(\frac{2\Gamma_T^{\frac{3}{2}}}{T}\right)^{\frac{1}{3}}
            + 3 \left( 32 R^2 \ell G \sigma\right)^{\frac{1}{3}} \left(\frac{\Gamma_T}{T}\right)^{\frac{1}{3}}\\
\leq& \varepsilon/2 + \varepsilon/4 + \varepsilon/ 4  = \varepsilon \end{aligned}
\]
with sufficient large \(T\), where \(\Gamma_T : = 2\log\frac{4T}{\delta}\) satisfies \(T \Gamma_T\geq T \geq \frac{R^2\ell G}{4 \sqrt{2}\sigma^2}\).
By noting \(T \geq \calL_q(u, \delta; c)= \frac{\mathe^{q^2}}{u}\log^q(\mathe\vee \frac{c}{u \delta})\) is sufficient to show \(\frac{\log^q \frac{cT}{\delta}}{T}\leq u\) (see Lemma~\ref{lemma:complexity_log}), we can choose \(T\) as follow to ensure \(\max\{f(\bar{w}) -f(w^\star), h(\bar{w})\}\leq \varepsilon\).
\[
\begin{aligned} T \geq &\frac{R^2\ell G}{4\sqrt{2} \sigma^2}
\vee (\varepsilon/2)^{-2} \left(\sqrt{2} RG+4 DG \sqrt{\log\frac{4}{\delta}}\right)^2\\
& \vee\frac{2(32\sqrt{2}\sigma)^3}{\varepsilon^3} 
\vee \mathe^{9/4} \log^{3/2}\left[\mathe \vee \frac{2(32\sqrt{2} \sigma)^3}{\varepsilon^3}\cdot \frac{4}{\delta}\right]\\ 
&\vee \frac{64 R^2 \ell G \sigma}{(\varepsilon/12)^3} 
\mathe \log\left[\mathe \vee\frac{64 R^2 \ell G \sigma}{(\varepsilon/12)^3}\cdot \frac{4}{\delta}\right]
\end{aligned}
\]

\textbf{Family 3}. To guarantee
\[
\begin{aligned} &\max\{f(\bar{w}) -f(w^\star), h(\bar{w})\} \\
\leq& Q_T^{(3, 1)} \equiv \left(\sqrt{2} RG+4 DG \sqrt{\log\frac{4}{\delta}}\right)
T^{-1/2} + 8\sigma \left(\frac{2\Gamma_T^{\frac{3}{2}}}{T}\right)^{\frac{1}{3}}
            + 3 \left( 32 R^2 \gamma G \sigma\right)^{\frac{1}{3}} \left(\frac{\Gamma_T}{T}\right)^{\frac{1}{3}}\\
& + 6( R^4\alpha G^2\sigma^2)^{\frac{1}{5}}\left(\frac{\Gamma_T}{T^2}\right)^{1/5}\\
\leq& \varepsilon/4 + \varepsilon/4 +\varepsilon/4 + \varepsilon/ 4  = \varepsilon \end{aligned}
\]
with sufficient large \(T\), where \(\Gamma_T : = 2\log\frac{4T}{\delta}\) satisfies \(T \Gamma_T\geq T \geq \frac{R^2\gamma G}{4 \sqrt{2}\sigma^2}\) and \(T \Gamma_T^{3/4} \geq 2^{-15/4} \left(\frac{R^2 \sqrt{\alpha} G}{\sigma}\right)^{3/2}\).
By noting \(T \geq \calL_q(u, \delta; c)= \frac{\mathe^{q^2}}{u}\log^q(\mathe\vee \frac{c}{u \delta})\) is sufficient to show \(\frac{\log^q \frac{cT}{\delta}}{T}\leq u\) (see Lemma~\ref{lemma:complexity_log}), we can choose \(T\) as follow to ensure \(\max\{f(\bar{w}) -f(w^\star), h(\bar{w})\}\leq \varepsilon\).
\[
\begin{aligned} T \geq &\frac{R^2\gamma G}{4\sqrt{2} \sigma^2}
\vee 2^{-15/4} \left(\frac{R^2 \sqrt{\alpha} G}{\sigma}\right)^{3/2}
\vee (\varepsilon/4)^{-2} \left(\sqrt{2} RG+4 DG \sqrt{\log\frac{4}{\delta}}\right)^2\\
& \vee\frac{2(32\sqrt{2}\sigma)^3}{\varepsilon^3}\mathe^{9/4} \log^{3/2}\left[\mathe \vee \frac{2(32\sqrt{2} \sigma)^3}{\varepsilon^3}\cdot \frac{4}{\delta}\right]
\vee \frac{64 R^2 \gamma G \sigma}{(\varepsilon/12)^3}\mathe \log\left[\mathe \vee\frac{64 R^2 \gamma G \sigma}{(\varepsilon/12)^3}\cdot \frac{4}{\delta}\right]
\\ & \vee \sqrt{\frac{2R^4 \alpha G^2 \sigma^2}{(\varepsilon/24)^5}} \mathe^{1/4} \log^{1/2}\left[\mathe \vee \sqrt{\frac{2R^4 \alpha G^2 \sigma^2}{(\varepsilon/24)^5}} \cdot \frac{4}{\delta}\right]
\end{aligned}
\]
\end{proof}

\newpage
\subsection{Case 2 for Families~1--3 in the constrained setting}
\label{subsec:constrained_case2_all_families}

We now specialize Theorem~\ref{thm:constrained_generic} to the probabilistic-reset setting with
\(p=\frac{1}{B}\in(0,1]\). Since Application~2 estimates the scalar constraint value, we have
\((\calG,\|\cdot\|_{\calG})=(\mathbb R,|\cdot|)\), and hence \(\kappa=1\). Let
\(\Lambda_T \coloneqq \log\frac{8T}{\delta},\) \(C_T\coloneqq C\!\left(\frac{\delta}{8T},1\right),\) and \(\Gamma_T \coloneqq 2\Lambda_T.\)
Then, for \(T\ge 1\), we have \(\Lambda_T>1\), and
\(C_T=1+\sqrt{3\Lambda_T}\le 2\sqrt{\Gamma_T}\).

\begin{theorem}[Constrained Case~2 Bounds for Families~1--3]
\label{thm:constrained_case2_all_families_tuned}
Assume the hypotheses of Theorem~\ref{thm:constrained_generic}. For each family below, let
\(\mathcal A=\{t\in[T]:v_t\le \varepsilon\}\), 
\(\varepsilon=\mathsf{OptError}(\eta,T,\delta)+\mathcal E\), and
\(\bar w=\frac{1}{|\mathcal A|}\sum_{t\in\mathcal A}w_t\).

\begin{enumerate}[leftmargin=*]
\item \textbf{Family~1.}
Assume that, with probability at least \(1-\delta/2\),
\[
|e_t| \le C_T\frac{\sigma}{\sqrt B} +
\frac{GL\eta}{p}\Lambda_T, \qquad \forall t\in[T].
\]
Choose
\(\eta = \min\Big\{ \frac{R}{G\sqrt{T}}, \allowbreak \left(\frac{R^6}{256GL\sigma^2\Lambda_T^2T^3}\right)^{1/4} \Big\},\)
\(B = \left\lceil \left(\frac{4\sigma^2T}{R^2GL}\right)^{1/2} \right\rceil,\)
\(p=\frac{1}{B}.\)
Assume moreover that
\(T\ge \frac{R^2GL}{4\sigma^2}.\)
Then, with probability at least \(1-\delta\), 
\(f(\bar w)-f(w^\star)\le Q_T^{(1,2)},\) \(h(\bar w)\le Q_T^{(1,2)},\)
where
\[
Q_T^{(1,2)} = \frac{RG+2\sqrt{2}\,DG\sqrt{\log(4/\delta)}}{\sqrt{T}} +
4\left(\frac{16R^2GL\sigma^2\Lambda_T^2}{T}\right)^{1/4}.
\]

\item \textbf{Family~2.}
Assume that, with probability at least \(1-\delta/2\),
\[
|e_t| \le C_T\frac{\sigma}{\sqrt B} +
C_T\ell\eta G\sqrt{\frac{2\Lambda_T}{p}}, \qquad \forall t\in[T].
\]
Choose
\(\eta = \min\Big\{ \frac{R}{G\sqrt{T}}, \allowbreak \left(\frac{R^4}{256\sigma\ell G\Lambda_T^{3/2}T^2}\right)^{1/3} \Big\},\)
\(B = \left\lceil \left( \frac{32\sigma^4T^2}{R^4\ell^2G^2} \right)^{1/3} \right\rceil,\) \(p=\frac{1}{B}.\)
Assume moreover that
\(T\ge \frac{R^2\ell G}{4\sqrt{2}\,\sigma^2}.\)
Then, with probability at least \(1-\delta\),
\(f(\bar w)-f(w^\star)\le Q_T^{(2,2)},\) \(h(\bar w)\le Q_T^{(2,2)},\)
where
\[
Q_T^{(2,2)} = \frac{RG+2\sqrt{2}\,DG\sqrt{\log(4/\delta)}}{\sqrt{T}} +
3\left(\frac{32R^2\sigma\ell G\Lambda_T^{3/2}}{T}\right)^{1/3}.
\]

\item \textbf{Family~3.}
Assume that, with probability at least \(1-\delta/2\),
\[
|e_t| \le C_T\frac{\sigma}{\sqrt B} +
\frac{\alpha\eta^2G^2}{2p}\Lambda_T +
C_T\gamma\eta G\sqrt{\frac{2\Lambda_T}{p}}, \qquad \forall t\in[T].
\]
Choose
\(\eta = \min\Big\{ \frac{R}{G\sqrt{T}}, \allowbreak \left(\frac{R^4}{256\sigma\gamma G\Lambda_T^{3/2}T^2}\right)^{1/3}, \allowbreak \left(\frac{R^6}{1024\alpha G^2\sigma^2\Lambda_T^2T^3}\right)^{1/5} \Big\},\)
\(B = \left\lceil \min\Big\{ \left( \frac{32\sigma^4T^2}{R^4\gamma^2G^2} \right)^{1/3}, \allowbreak \left( \frac{32768\sigma^6\Lambda_TT^4}{\alpha^2R^8G^4} \right)^{1/5} \Big\} \right\rceil,\)
\(p=\frac{1}{B}.\)
Assume moreover that
\(T\ge \frac{R^2\gamma G}{4\sqrt{2}\sigma^2},\) \(2^{3/2}T^2\Lambda_T^{1/2}\ge \frac{\alpha R^4G^2}{64\sigma^3}.\)
Then, with probability at least \(1-\delta\),
\(f(\bar w)-f(w^\star)\le Q_T^{(3,2)},\) \(h(\bar w)\le Q_T^{(3,2)},\)
where
\[
Q_T^{(3,2)} = \frac{RG+2\sqrt{2}\,DG\sqrt{\log(4/\delta)}}{\sqrt{T}} +
3\left(\frac{32R^2\sigma\gamma G\Lambda_T^{3/2}}{T}\right)^{\tfrac{1}{3}} +
5\left(\frac{\alpha R^4\sigma^2G^2\Lambda_T^2}{T^2}\right)^{\tfrac{1}{5}}.
\]
\end{enumerate}
\end{theorem}

\newpage
\begin{proof}
Since each estimator bound holds uniformly for all \(t\in[T]\) (substituting the high-probability upperbound \(\calE\) of \(|e_t|\) from the instantiations of Theorem~\ref{thm:unified}, replacing \(\delta\) with \(\delta/2\) and selecting \(B=1/p\) in the bound of Table~\ref{tab:error_bounds}), \eqref{ass:tail_envelope} in Theorem~\ref{thm:constrained_generic} holds. %
Hence, with
\(\varepsilon = \mathsf{OptError}(\eta,T,\delta)+\mathcal E,\)
Theorem~\ref{thm:constrained_generic} yields
\[
f(\bar w)-f(w^\star)\le Q, \qquad
h(\bar w)\le Q, \qquad
Q \coloneqq \mathsf{OptError}(\eta,T,\delta)+2\mathcal E.
\]
In every family,
\[
\frac{2DG}{\sqrt{T}}\sqrt{2\log\frac{4}{\delta}} =
\frac{2\sqrt{2}\,DG\sqrt{\log(4/\delta)}}{\sqrt{T}}.
\]
Also,
\[
2C_T\frac{\sigma}{\sqrt B} \le
4\sigma\sqrt{\frac{\Gamma_T}{B}} =
4\sigma\sqrt{p\Gamma_T}
=4\sigma\sqrt{2p\Lambda_T}.
\]

For \textbf{Family~1},
\[
Q \le \frac{R^2}{2\eta T} +
\frac{\eta G^2}{2} + \frac{2\sqrt{2}\,DG\sqrt{\log(4/\delta)}}{\sqrt{T}} +
4\sigma\sqrt{2p\Lambda_T} +
\frac{2GL\Lambda_T}{p}\eta.
\]
Set
\(C_1=\frac{R^2}{2}, C_2=\frac{G^2}{2}, C_4=4\sigma\sqrt{2\Lambda_T}, C_5=2GL\Lambda_T.\)
Then
\[
Q \le \frac{C_1}{\eta T}+C_2\eta+C_4\sqrt{p}+\frac{C_5\eta}{p} +
\frac{2\sqrt{2}\,DG\sqrt{\log(4/\delta)}}{\sqrt{T}}.
\]
Lemma~\ref{lemma:opt_bound_stepsize_reset_batch} with \(C_3=0\) gives the stated choice of \((\eta,B,p)\) and yields
\[
\frac{C_1}{\eta T}+C_2\eta+C_4\sqrt{p}+\frac{C_5\eta}{p} \le
\frac{RG}{\sqrt{T}} +
4\left(\frac{16R^2GL\sigma^2\Lambda_T^2}{T}\right)^{1/4}.
\]

For \textbf{Family~2}, using \(C_T\le 2\sqrt{\Gamma_T}=2\sqrt{2\Lambda_T}\), we get
\[
Q \le \frac{R^2}{2\eta T} +
\frac{\eta G^2}{2} + \frac{2\sqrt{2}\,DG\sqrt{\log(4/\delta)}}{\sqrt{T}} +
4\sigma\sqrt{2p\Lambda_T} +
8\ell G\Lambda_T\frac{\eta}{\sqrt{p}}.
\]
Set
\(C_1=\frac{R^2}{2}, C_2=\frac{G^2}{2}, C_4=4\sigma\sqrt{2\Lambda_T}, C_5=8\ell G\Lambda_T.\)
Then
\[
Q \le \frac{C_1}{\eta T}+C_2\eta+C_4\sqrt{p}+\frac{C_5\eta}{\sqrt{p}} +
\frac{2\sqrt{2}\,DG\sqrt{\log(4/\delta)}}{\sqrt{T}}.
\]
Lemma~\ref{lemma:opt_bound_stepsize_reset_batch} yields
\[
\frac{C_1}{\eta T}+C_2\eta+C_4\sqrt{p}+\frac{C_5\eta}{\sqrt{p}} \le
\frac{RG}{\sqrt{T}} +
3\left(\frac{32R^2\sigma\ell G\Lambda_T^{3/2}}{T}\right)^{1/3}.
\]

For \textbf{Family~3}, the same estimate gives
\[
Q \le \frac{R^2}{2\eta T} +
\frac{\eta G^2}{2} + \frac{2\sqrt{2}\,DG\sqrt{\log(4/\delta)}}{\sqrt{T}} +
4\sigma\sqrt{2p\Lambda_T} +
8\gamma G\Lambda_T\frac{\eta}{\sqrt{p}} +
\alpha G^2\Lambda_T\frac{\eta^2}{p}.
\]
Set
\[
C_1=\frac{R^2}{2}, \quad
C_2=\frac{G^2}{2}, \quad
C_4=4\sigma\sqrt{2\Lambda_T}, \quad
C_5=8\gamma G\Lambda_T, \quad
C_6=\alpha G^2\Lambda_T.
\]
Then
\[
Q \le \frac{C_1}{\eta T}+C_2\eta+C_4\sqrt{p}+\frac{C_5\eta}{\sqrt{p}}+\frac{C_6\eta^2}{p} +
\frac{2\sqrt{2}\,DG\sqrt{\log(4/\delta)}}{\sqrt{T}}.
\]
Lemma~\ref{lemma:opt_bound_stepsize_reset_batch} yields
\[
\frac{C_1}{\eta T}+C_2\eta+C_4\sqrt{p}+\frac{C_5\eta}{\sqrt{p}}+\frac{C_6\eta^2}{p} \le
\frac{RG}{\sqrt{T}} +
3\left(\frac{32R^2\sigma\gamma G\Lambda_T^{3/2}}{T}\right)^{1/3} +
5\left(\frac{\alpha R^4\sigma^2G^2\Lambda_T^2}{T^2}\right)^{1/5}.
\]
Combining these gives the three stated bounds.
\end{proof}

\begin{corollary}[Constrained Case~2 Iteration Complexity for Families~1--3]
\label{cor:constrained_case2_all_families_complexity}
Under the assumptions of Theorem~\ref{thm:constrained_case2_all_families_tuned}, let \(\varepsilon>0\).

\begin{enumerate}[leftmargin=*]
\item \textbf{Family~1.}
Define
\(A_1 \coloneqq RG+2\sqrt{2}\,DG\sqrt{\log(4/\delta)},\) \(A_2 \coloneqq 4(16R^2GL\sigma^2)^{1/4},\) \(u_2 \coloneqq \left(\frac{\varepsilon}{2A_2}\right)^4.\)
If
\(T \ge \frac{4A_1^2}{\varepsilon^2} \vee \mathcal L_2\!\left(u_2,\delta;8\right) \vee \left\lceil\frac{R^2GL}{4\sigma^2}\right\rceil,\)
then
\(f(\bar w)-f(w^\star)\le \varepsilon,\) \(h(\bar w)\le \varepsilon.\)
Moreover,
\[
T = \mathcal{O}\left(
\frac{R^2GL\sigma^2}{\varepsilon^4}
\log^2\frac{R^2GL\sigma^2}{\varepsilon^4\delta}
\vee \frac{R^2G^2+D^2G^2\log(1/\delta)}{\varepsilon^2} \right).
\]

\item \textbf{Family~2.}
Define
\(A_1 \coloneqq RG+2\sqrt{2}\,DG\sqrt{\log(4/\delta)},\) \(A_2 \coloneqq 3(32R^2\sigma\ell G)^{1/3},\) \(u_2 \coloneqq \left(\frac{\varepsilon}{2A_2}\right)^3.\)
If
\(T \ge \frac{4A_1^2}{\varepsilon^2} \vee \mathcal L_{\frac32}\!\left(u_2,\delta;8\right) \vee \left\lceil\frac{R^2\ell G}{4\sqrt{2}\sigma^2}\right\rceil,\)
then
\(f(\bar w)-f(w^\star)\le \varepsilon,\) \(h(\bar w)\le \varepsilon.\)
Moreover,
\[
T = \mathcal{O}\left(
\frac{R^2\sigma\ell G}{\varepsilon^3}
\log^{3/2}\frac{R^2\sigma\ell G}{\varepsilon^3\delta}
\vee \frac{R^2G^2+D^2G^2\log(1/\delta)}{\varepsilon^2} \right).
\]

\item \textbf{Family~3.}
Define
\(A_1 \coloneqq RG+2\sqrt{2}\,DG\sqrt{\log(4/\delta)},\) \(A_2 \coloneqq 3(32R^2\sigma\gamma G)^{1/3}, A_3 \coloneqq 5\left(\alpha R^4\sigma^2G^2\right)^{1/5},\) \(u_2 \coloneqq \left(\frac{\varepsilon}{3A_2}\right)^3,\) \(u_3 \coloneqq \left(\frac{\varepsilon}{3A_3}\right)^5.\)
If
\(T \ge \frac{9A_1^2}{\varepsilon^2} \vee \mathcal L_{\frac32}\!\left(u_2,\delta;8\right) \vee \mathcal L_1\!\left(\sqrt{u_3},\delta;8\right) \vee \left\lceil\frac{R^2\gamma G}{4\sqrt{2}\sigma^2}\right\rceil \vee \left\lceil\left(\frac{\alpha R^4G^2}{64\sigma^3}\right)^{1/2}\right\rceil,\)
then
\(f(\bar w)-f(w^\star)\le \varepsilon,\) \(h(\bar w)\le \varepsilon.\)
Moreover,
\[
\begin{aligned} T = \mathcal{O}\Bigg( \frac{R^2\sigma\gamma G}{\varepsilon^3}
\log^{3/2}\frac{R^2\sigma\gamma G}{\varepsilon^3\delta}
&\vee \sqrt{\frac{\alpha R^4\sigma^2G^2}{\varepsilon^5}}
\log\!\left( \frac{1}{\delta}\sqrt{\frac{\alpha R^4\sigma^2G^2}{\varepsilon^5}} \right) \\
&\vee \frac{R^2G^2+D^2G^2\log(1/\delta)}{\varepsilon^2}\Bigg). \end{aligned}
\]
\end{enumerate}

In all three families, the expected oracle complexity is
\(N=[B\cdot p+1\cdot(1-p)]T=(2-1/B)T<2T,\)
hence is of the same order as \(T\).
\end{corollary}

\begin{proof}
As in the previous corollary, it is enough to make every term of the corresponding \(Q\)-bound at most \(\varepsilon/r\), where \(r=2\) for Families~1 and~2 and \(r=3\) for Family~3. Since
\(\Lambda_T=\log(8T/\delta)\), we use
\[
\frac{\Lambda_T^2}{T}\le u
\Leftarrow T\ge \mathcal L_2\!\left(u,\delta;8\right),
\]
\[
\frac{\Lambda_T^{3/2}}{T}\le u
\Leftarrow T\ge \mathcal L_{\frac32}\!\left(u,\delta;8\right),
\]
and
\[
\frac{\Lambda_T^2}{T^2}\le u
\Leftarrow T\ge \mathcal L_1\!\left(\sqrt{u},\delta;8\right).
\]
Applying these implications to the three \(Q\)-bounds, together with the side conditions in the theorem, gives the stated lower bounds on \(T\). The displayed asymptotic orders follow by collecting the largest \(\varepsilon\)-dependences. Finally,
\(N=(2-1/B)T<2T.\)
\end{proof}

\newpage
\subsection{Case 3 for Families~1--3 in the constrained setting}
\label{subsec:constrained_case3_all_families}

In Case~3, the reset period is deterministic. We write
\(E\in \mathbb Z_+,\) \(p_t = \mathbb I_{\{t\bmod E=0\}},\) \(B=E.\)
Since Application~2 estimates the scalar constraint value, we have
\((\calG,\|\cdot\|_{\calG})=(\mathbb R,|\cdot|)\), and hence \(\kappa=1\). Let
\(\Lambda_T \coloneqq \log\frac{8T}{\delta},\) \(C_T\coloneqq C\!\left(\frac{\delta}{8T},1\right),\) and \(\Gamma_T \coloneqq 2\Lambda_T.\)
Then, for \(T\ge 1\), we have \(\Lambda_T>1\), and
\(C_T=1+\sqrt{3\Lambda_T}\le 2\sqrt{\Gamma_T}\).

\begin{theorem}[Constrained Case~3 Tuned Bounds for Families~1--3]
\label{thm:constrained_case3_all_families_tuned}
Assume the hypotheses of Theorem~\ref{thm:constrained_generic}. For each family below, let
\(\varepsilon \coloneqq \mathsf{OptError}(\eta,T,\delta)+\mathcal E\),
\(\mathcal A=\{t\in[T]:v_t\le \varepsilon\}\), and
\(\bar w = \frac{1}{|\mathcal A|}\sum_{t\in\mathcal A}w_t\).

\begin{enumerate}[leftmargin=*]
\item \textbf{Family~1.}
Assume that, with probability at least \(1-\delta/2\),
\[
|e_t| \le C_T\frac{\sigma}{\sqrt E} + GL\eta E, \qquad \forall t\in[T].
\]
Choose
\(\eta = \min\Big\{ \frac{R}{G\sqrt{T}}, \allowbreak \left(\frac{R^6}{256GL\sigma^2\Lambda_TT^3}\right)^{1/4} \Big\},\) \(E=B = \left\lceil \left(\frac{4\sigma^2\Lambda_TT}{R^2GL}\right)^{1/2} \right\rceil.\)
Assume moreover that
\(T\Lambda_T\ge \frac{R^2GL}{4\sigma^2}.\)
Then, with probability at least \(1-\delta\),
\(f(\bar w)-f(w^\star)\le Q_T^{(1,3)},\) \(h(\bar w)\le Q_T^{(1,3)},\)
where
\[
Q_T^{(1,3)} = \frac{RG+2\sqrt{2}\,DG\sqrt{\log(4/\delta)}}{\sqrt{T}} +
4\left(\frac{16R^2GL\sigma^2\Lambda_T}{T}\right)^{1/4}.
\]

\item \textbf{Family~2.}
Assume that, with probability at least \(1-\delta/2\),
\[
|e_t| \le C_T\frac{\sigma}{\sqrt E} +
C_T\ell\eta G\sqrt{2E}, \qquad \forall t\in[T].
\]
Choose
\(\eta = \min\Big\{ \frac{R}{G\sqrt{T}}, \allowbreak \left(\frac{R^4}{256\sigma\ell G\Lambda_TT^2}\right)^{1/3} \Big\},\)
\(E=B = \left\lceil \left( \frac{32\sigma^4\Lambda_TT^2}{R^4\ell^2G^2} \right)^{1/3} \right\rceil.\)
Assume moreover that
\(T\sqrt{2\Lambda_T}\ge \frac{R^2\ell G}{4\sigma^2}.\)
Then, with probability at least \(1-\delta\),
\(f(\bar w)-f(w^\star)\le Q_T^{(2,3)},\) \(h(\bar w)\le Q_T^{(2,3)},\)
where
\[
Q_T^{(2,3)} = \frac{RG+2\sqrt{2}\,DG\sqrt{\log(4/\delta)}}{\sqrt{T}} +
3\left(\frac{32R^2\sigma\ell G\Lambda_T}{T}\right)^{1/3}.
\]

\item \textbf{Family~3.}
Assume that, with probability at least \(1-\delta/2\),
\[
|e_t| \le C_T\frac{\sigma}{\sqrt E} +
\frac{\alpha\eta^2G^2E}{2} +
C_T\gamma\eta G\sqrt{2E}, \qquad \forall t\in[T].
\]
Choose
\(\eta = \min\Big\{ \frac{R}{G\sqrt{T}}, \allowbreak \left(\frac{R^4}{256\sigma\gamma G\Lambda_TT^2}\right)^{1/3}, \allowbreak \left(\frac{R^6}{1024\alpha G^2\sigma^2\Lambda_TT^3}\right)^{1/5} \Big\},\)
\(E=B = \left\lceil \min\Big\{ \left( \frac{32\sigma^4\Lambda_TT^2}{R^4\gamma^2G^2} \right)^{1/3}, \allowbreak \left( \frac{32768\sigma^6\Lambda_T^3T^4}{\alpha^2R^8G^4} \right)^{1/5} \Big\} \right\rceil.\)
Assume moreover that
\(T\sqrt{2\Lambda_T}\ge \frac{R^2\gamma G}{4\sigma^2},\) \(2^{3/2}T^2\Lambda_T^{3/2}\ge \frac{\alpha R^4G^2}{64\sigma^3}.\)
Then, with probability at least \(1-\delta\),
\(f(\bar w)-f(w^\star)\le Q_T^{(3,3)},\) \(h(\bar w)\le Q_T^{(3,3)},\)
where
\[
Q_T^{(3,3)} = \frac{RG+2\sqrt{2}\,DG\sqrt{\log(4/\delta)}}{\sqrt{T}} +
3\left(\frac{32R^2\sigma\gamma G\Lambda_T}{T}\right)^{1/3} +
5\left(\frac{\alpha R^4\sigma^2G^2\Lambda_T}{T^2}\right)^{1/5}.
\]
\end{enumerate}
\end{theorem}

\newpage
\begin{proof}
As in Case~2, the uniform estimator bound (substituting the high-probability upperbound \(\calE\) of \(|e_t|\) from the instantiations of Theorem~\ref{thm:unified}, replacing \(\delta\) with \(\delta/2\) and selecting \(B=E\) in the bound of Table~\ref{tab:error_bounds}) implies that Theorem~\ref{thm:constrained_generic} applies %
and
\(Q \coloneqq \mathsf{OptError}(\eta,T,\delta)+2\mathcal E.\)
Since \(C_T\le 2\sqrt{\Gamma_T}\),
\(2C_T\sigma/\sqrt{E} \le 4\sigma\sqrt{\Gamma_T}/\sqrt{E}=4\sigma\sqrt{2\Lambda_T}/\sqrt{E}.\)

For \textbf{Family~1},
\[
Q \le \frac{R^2}{2\eta T} +
\frac{\eta G^2}{2} + \frac{2\sqrt{2}\,DG\sqrt{\log(4/\delta)}}{\sqrt{T}} +
4\sigma\frac{\sqrt{2\Lambda_T}}{\sqrt{E}} + 2GL\eta E.
\]
Set
\(C_1=\frac{R^2}{2},\) \(C_2=\frac{G^2}{2},\) \(C_4=4\sigma\sqrt{2\Lambda_T},\) \(C_5=2GL.\)
Introducing \(p=1/E\), this becomes
\[
Q \le \frac{C_1}{\eta T}+C_2\eta+C_4\sqrt{p}+\frac{C_5\eta}{p} +
\frac{2\sqrt{2}\,DG\sqrt{\log(4/\delta)}}{\sqrt{T}}.
\]
Lemma~\ref{lemma:opt_bound_stepsize_reset_batch} yields
\(\frac{C_1}{\eta T}+C_2\eta+C_4\sqrt{p}+\frac{C_5\eta}{p} \le \frac{RG}{\sqrt{T}} + 4\left(\frac{16R^2GL\sigma^2\Lambda_T}{T}\right)^{1/4}.\)

For \textbf{Family~2}, using \(C_T\le 2\sqrt{\Gamma_T}=2\sqrt{2\Lambda_T}\), we get
\[
Q \le \frac{R^2}{2\eta T} +
\frac{\eta G^2}{2} + \frac{2\sqrt{2}\,DG\sqrt{\log(4/\delta)}}{\sqrt{T}} +
4\sigma\frac{\sqrt{2\Lambda_T}}{\sqrt{E}} +
8\ell G\sqrt{\Lambda_T}\,\eta\sqrt{E}.
\]
Set
\(C_1=\frac{R^2}{2},\) \(C_2=\frac{G^2}{2},\) \(C_4=4\sigma\sqrt{2\Lambda_T},\) \(C_5=8\ell G\sqrt{\Lambda_T}.\)
After setting \(p=1/E\),
\[
Q \le \frac{C_1}{\eta T}+C_2\eta+C_4\sqrt{p}+\frac{C_5\eta}{\sqrt{p}} +
\frac{2\sqrt{2}\,DG\sqrt{\log(4/\delta)}}{\sqrt{T}}.
\]
Lemma~\ref{lemma:opt_bound_stepsize_reset_batch} yields
\(\frac{C_1}{\eta T}+C_2\eta+C_4\sqrt{p}+\frac{C_5\eta}{\sqrt{p}} \le \frac{RG}{\sqrt{T}} + 3\left(\frac{32R^2\sigma\ell G\Lambda_T}{T}\right)^{1/3}.\)

For \textbf{Family~3}, the same estimate gives
\[
Q \le \frac{R^2}{2\eta T} +
\frac{\eta G^2}{2} + \frac{2\sqrt{2}\,DG\sqrt{\log(4/\delta)}}{\sqrt{T}} +
4\sigma\frac{\sqrt{2\Lambda_T}}{\sqrt{E}} +
8\gamma G\sqrt{\Lambda_T}\,\eta\sqrt{E} + \alpha G^2\eta^2E.
\]
Set
\(C_1=\frac{R^2}{2},\) \(C_2=\frac{G^2}{2},\) \(C_4=4\sigma\sqrt{2\Lambda_T},\) \(C_5=8\gamma G\sqrt{\Lambda_T},\) \(C_6=\alpha G^2.\)
After setting \(p=1/E\),
\[
Q \le \frac{C_1}{\eta T}+C_2\eta+C_4\sqrt{p}+\frac{C_5\eta}{\sqrt{p}}+\frac{C_6\eta^2}{p} +
\frac{2\sqrt{2}\,DG\sqrt{\log(4/\delta)}}{\sqrt{T}}.
\]
Lemma~\ref{lemma:opt_bound_stepsize_reset_batch} yields
\[
\frac{C_1}{\eta T}+C_2\eta+C_4\sqrt{p}+\frac{C_5\eta}{\sqrt{p}}+\frac{C_6\eta^2}{p} \le
\frac{RG}{\sqrt{T}} +
3\left(\frac{32R^2\sigma\gamma G\Lambda_T}{T}\right)^{1/3} +
5\left(\frac{\alpha R^4\sigma^2G^2\Lambda_T}{T^2}\right)^{1/5}.
\]
Combining these gives the bounds.
\end{proof}

\newpage
\begin{corollary}[Constrained Case~3 Iteration Complexity for Families~1--3]
\label{cor:constrained_case3_all_families_complexity}
Under the assumptions of Theorem~\ref{thm:constrained_case3_all_families_tuned}, let \(\varepsilon>0\).

\begin{enumerate}[leftmargin=*]
\item \textbf{Family~1.}
Define
\(A_1 \coloneqq RG+2\sqrt{2}\,DG\sqrt{\log(4/\delta)},\) \(A_2 \coloneqq 4(16R^2GL\sigma^2)^{1/4},\) \(u_2 \coloneqq \left(\frac{\varepsilon}{2A_2}\right)^4.\)
If
\(T \ge \frac{4A_1^2}{\varepsilon^2} \vee \mathcal L_1\left(u_2,\delta;8\right) \vee \left\lceil\frac{R^2GL}{4\sigma^2}\right\rceil,\)
then
\(f(\bar w)-f(w^\star)\le \varepsilon,\) \(h(\bar w)\le \varepsilon.\)
Moreover,
\[
T = \mathcal{O}\left(
\frac{R^2GL\sigma^2}{\varepsilon^4}
\log\frac{R^2GL\sigma^2}{\varepsilon^4\delta} \vee
\frac{R^2G^2+D^2G^2\log(1/\delta)}{\varepsilon^2} \right).
\]

\item \textbf{Family~2.}
Define
\(A_1 \coloneqq RG+2\sqrt{2}\,DG\sqrt{\log(4/\delta)},\) \(A_2 \coloneqq 3(32R^2\sigma\ell G)^{1/3},\) \(u_2 \coloneqq \left(\frac{\varepsilon}{2A_2}\right)^3.\)
If
\(T \ge \frac{4A_1^2}{\varepsilon^2} \vee \mathcal L_1\left(u_2,\delta;8\right) \vee \left\lceil\frac{R^2\ell G}{4\sqrt{2}\sigma^2}\right\rceil,\)
then
\(f(\bar w)-f(w^\star)\le \varepsilon,\) \(h(\bar w)\le \varepsilon.\)
Moreover,
\[
T = \mathcal{O}\left(
\frac{R^2\sigma\ell G}{\varepsilon^3}
\log\frac{R^2\sigma\ell G}{\varepsilon^3\delta} \vee
\frac{R^2G^2+D^2G^2\log(1/\delta)}{\varepsilon^2} \right).
\]

\item \textbf{Family~3.}
Define
\(A_1 \coloneqq RG+2\sqrt{2}\,DG\sqrt{\log(4/\delta)},\) \(A_2 \coloneqq 3(32R^2\sigma\gamma G)^{1/3},\)
\(A_3 \coloneqq 5\left(\alpha R^4\sigma^2G^2\right)^{1/5},\) \(u_2 \coloneqq \left(\frac{\varepsilon}{3A_2}\right)^3,\) \(u_3 \coloneqq \left(\frac{\varepsilon}{3A_3}\right)^5.\)
If
\(T \ge \frac{9A_1^2}{\varepsilon^2} \vee \mathcal L_1\left(u_2,\delta;8\right) \vee \mathcal L_{1/2}\left(\sqrt{u_3},\delta;8\right) \vee \left\lceil\frac{R^2\gamma G}{4\sqrt{2}\sigma^2}\right\rceil \vee \left\lceil 2^{-3/4}\left(\frac{\alpha R^4G^2}{64\sigma^3}\right)^{1/2}\right\rceil,\)
then
\(f(\bar w)-f(w^\star)\le \varepsilon,\) \(h(\bar w)\le \varepsilon.\)
Moreover,
\[
\begin{aligned} T = \mathcal{O}\Bigg(
\frac{R^2\sigma\gamma G}{\varepsilon^3}
\log\frac{R^2\sigma\gamma G}{\varepsilon^3\delta}
&\vee \sqrt{\frac{\alpha R^4\sigma^2G^2}{\varepsilon^5}}
\log^{1/2}\left( \frac{1}{\delta} \sqrt{\frac{\alpha R^4\sigma^2G^2}{\varepsilon^5}}
\right)\\
&\vee \frac{R^2G^2+D^2G^2\log(1/\delta)}{\varepsilon^2}
\Bigg). \end{aligned}
\]
\end{enumerate}
In all three families, the deterministic oracle complexity satisfies
\(N =B\Bigl\lfloor \frac{T}{E}\Bigr\rfloor + \Bigl(T-\Bigl\lfloor \frac{T}{E}\Bigr\rfloor\Bigr) = T+(E-1)\Bigl\lfloor \frac{T}{E}\Bigr\rfloor < 2T.\)
\end{corollary}

\begin{proof}
As before, it is enough to make each term in the relevant bound \(Q_T^{(i,3)}\) at most \(\varepsilon/r\), with \(r=2\) for Families~1 and~2 and \(r=3\) for Family~3. Since
\(\Lambda_T=\log(8T/\delta),\)
we use
\[
\frac{\Lambda_T}{T}\le u
\Leftarrow T\ge\mathcal L_1\left(u,\delta;8\right),
\]
and
\[
\frac{\Lambda_T}{T^2}\le u
\Leftarrow T\ge\mathcal L_{1/2}\left(\sqrt{u},\delta;8\right).
\]
Applying these implications to \(Q_T^{(1,3)},Q_T^{(2,3)},Q_T^{(3,3)}\), together with the side conditions in the theorem, gives the stated lower bounds on \(T\). The asymptotic statements follow by collecting the dominant \(\varepsilon\)-terms. The oracle bound is the displayed identity together with \(\lfloor T/E\rfloor\le T/E\).
\end{proof}

\newpage

\newpage
\section{Related Work}\label{apdx:related}
\paragraph{Variance-reduced stochastic estimation.}The mathematical foundations of stochastic first-order methods have been extensively developed since the introduction of stochastic approximation, traditionally relying on uniformly bounded stochastic subgradients or globally bounded variance assumptions to establish convergence \cite{robbins1951stochastic,blum1954approximation,gladyshev1965stochastic,polyak1987introduction,nemirovski2009robust,moulines2011non,rakhlin2012making,shamir2013stochastic,bottou2018optimization,jain2018parallelizing,ruder2016overview,haji2021comparison}. In smooth nonconvex optimization, vanilla SGD attains an \(\mathcal{O}(\varepsilon^{-4})\) oracle complexity for finding an \(\varepsilon\)-stationary point~\cite{ghadimi2013stochastic}. Variance reduction improves this benchmark through two broad routes: finite-sum control-variate methods such as SAG, SAGA, and SVRG~\cite{roux2012stochastic,johnson2013accelerating,defazio2014saga,botev2017variance}, which use periodic reference-gradient computation, and recursive estimators such as SARAH, SPIDER, STORM, and PAGE~\cite{nguyen2017sarah,fang2018spider,cutkosky2019momentum,tran2019hybrid,liu2020optimal,li2021page,chen2021communication,das2022faster,hashemi2024unified}, which update an online estimator via stochastic differences. Under mean-squared Lipschitz gradients, recursive estimators achieve the optimal \(\mathcal{O}(\varepsilon^{-3})\) complexity, matching lower-bound theory~\cite{arjevani2022lower,drori2019complexity}. Second-order variants further improve the dependence on problem constants using Hessian-vector products while preserving the \(\mathcal{O}(\varepsilon^{-3})\) rate~\cite{arjevani2020second}, and higher-order methods under Lipschitz Hessians can reach rates of \(\mathcal{O}(\varepsilon^{-3.5})\) or better through cubic regularization or perturbed acceleration~\cite{xu2018first,allen2018neon2,tripuraneni2018stochastic,fang2019sharp}. More recent works have also pursued broader viewpoints, including filtering-based perspectives~\cite{yang2020stochastic}, manifold extensions such as SERENA~\cite{liu2025serena}, and unified adaptive theories that relax unbiasedness~\cite{shestakov2025unified}. On the practical side, MARS shows that recursive variance reduction can substantially accelerate large-model training when combined with preconditioned updates~\cite{yuan2025mars}. Our framework differs from these lines by making the \emph{correction term} explicit, thereby isolating the mechanism that aligns stale information from \(w_{t-1}\) with the current iterate \(w_t\) and yielding a unified analysis that also suggests new second-order variants.

\paragraph{High-probability guarantees.}
Most guarantees in stochastic optimization are stated only in expectation, which can mask rare but consequential failures. High-probability results are therefore especially important in nonconvex regimes where noise can substantially degrade the signal-to-noise ratio. For non-variance-reduced methods, such guarantees are available for convex SGD~\cite{harvey2019tight}, normalized SGD with momentum achieving \(\mathcal{O}(\varepsilon^{-4})\) complexity~\cite{cutkosky2020momentum}, and heavy-tailed stochastic optimization in arbitrary smooth norms~\cite{cutkosky2021high}. For recursive variance-reduced methods, however, high-probability analyses remain scarce and largely estimator-specific: existing results for PAGE~\cite{li2021page} and STORM-type methods~\cite{xu2023momentum} do not directly transfer across estimator families. A unified high-probability treatment for multiple recursive estimators has been missing. Related efforts under relaxed noise assumptions include bounds under bounded \(\alpha\)-th central moments~\cite{sadiev2023high} and various heavy-tailed analyses~\cite{gorbunov2024high,gorbunov2023high,nguyen2023improved,gurbuzbalaban2021heavy}. These works address the noise-model axis, whereas our focus is on the estimator-design axis. Our main technical tool is a novel vector-valued Freedman inequality for \(\kappa\)-smooth Banach spaces, building on~\cite{pinelis1994optimum,juditsky2008large}; crucially, it accommodates the random predictable variance proxies and stopping-time intervals induced by algorithmic reset events, enabling a single, unified high-probability guarantee across all estimator families.

\paragraph{Relaxed variance assumptions.}
To move beyond globally bounded variance, recent work has introduced structural relaxations such as expected smoothness and ABC-type conditions~\cite{gorbunov2020unified,khaled2022better,khaled2023unified,ilandarideva2023accelerated,grimmer2019convergence}, which control the variance by a constant plus terms depending on the gradient norm and suboptimality. The BG-0 condition is weaker still and captures realistic optimization trajectories without uniform noise bounds~\cite{alacaoglu2025towards,wang2016stochastic,cui2021analysis,asi2019stochastic,jacobsen2023unconstrained,telgarsky2022stochastic,domke2023provable,neu2024dealing}. Recent lower bounds show that finding an \(\varepsilon\)-stationary point under BG-0 requires at least \(\Omega(\varepsilon^{-6})\) and \(\Omega(\varepsilon^{-4})\) queries in smooth and mean-square smooth settings, respectively~\cite{fazla2026lower}. Our sub-Gaussian oracle model is stronger than these relaxations, and extending the present unified high-probability framework to the ABC or BG-0 regime remains an important open problem.

\paragraph{Mirror descent and non-Euclidean optimization.}
Mirror descent~\cite{nemirovsky1983wiley,beck2003mirror} generalizes projected gradient descent to non-Euclidean geometries through Bregman divergences and is particularly natural on domains such as the probability simplex. High-probability convergence for stochastic mirror descent without variance reduction has been established in convex settings under sub-Gaussian noise~\cite{liu2023high} and extended to heavy-tailed settings via clipping~\cite{nguyen2023improved}. In nonconvex problems, \cite{cutkosky2021high} obtain high-probability guarantees for normalized SGD with momentum in smooth norms. However, these analyses do not incorporate recursive variance-reduced estimators within the mirror-descent framework. By formulating our estimator-level analysis directly in a \(\kappa\)-smooth normed space, we decouple the concentration argument from the specific update rule and obtain guarantees that apply when the estimator lives in the dual space, and iterates are updated by mirror descent. The parameter \(\kappa\) captures the geometric regularity, including \(\kappa=1\) in Hilbert spaces, \(\kappa=q-1\) in \(\ell_q\) for \(q\ge 2\), and \(\kappa=\mathcal{O}(\log d)\) for smooth approximations of \(\ell_\infty\) (see Appendix~\ref{apdx:prelim} for detailed discussion).

\paragraph{Constrained stochastic optimization.}
Stochastic optimization with expectation constraints~\eqref{eq:application2} requires estimating both objective and feasibility quantities from noisy observations. Standard approaches include primal-dual, saddle-point, and augmented-Lagrangian methods~\cite{nemirovski2004prox,chambolle2011first,bertsekas2014constrained,hamedani2021primal,zhang2022solving,hounie2023resilient,boob2024optimal}, which can be sensitive to dual tuning and bounded-domain projections. A primal-only alternative is the switching gradient method (SGM), originating in~\cite{polyak1967general}, which alternates between objective decrease and constraint reduction. Its mirror-descent variants address non-smoothness and non-Euclidean geometry~\cite{titov2018mirror,stonyakin2019mirror,alkousa2020modification,stonyakin2019adaptive,stonyakin2019some,stonyakin2020mirror,titov2020analogues,bayandina2018mirror}. In the stochastic setting, \cite{lan2020algorithms} establish an \(\mathcal{O}(\varepsilon^{-4})\) oracle complexity for convex expectation-constrained problems using batch averaging for constraint estimation, and related weakly convex analyses retain the same order~\cite{huang2023oracle,liu2025single,ma2020quadratically,jia2025first,davis2019stochastic,drusvyatskiy2019efficiency,mai2020convergence}. For nonconvex problems, \cite{li2024stochastic} combine momentum-based variance reduction with an inexact augmented-Lagrangian framework and obtain an \(\mathcal{O}(\varepsilon^{-5})\) complexity in expectation; however, the method is a double-loop scheme, using PStorm as an inner subroutine. Our contribution is to show that replacing batch averaging in the SGM-style setting of~\cite{lan2020algorithms} with recursive variance-reduced estimators from our unified framework yields improved high-probability oracle complexity bounds of order \(\mathcal{O}(\varepsilon^{-3})\) in the settings studied here.

\paragraph{Complementary variance reduction and gradient estimation strategies.}
Beyond recursive estimators, variance can also be reduced through negative dependence. Antithetic variates are a classical example~\cite{HammersleyMorton1956,KahnMarshall1953,RubinsteinKroese2017,HammersleyMauldon1956,Whitt1976}. More recent work strengthens these guarantees through strongly isotonic assumptions and optimal-transport connections~\cite{hashemi2025strong,hashemi2025strongtheory} %
(cf. \cite{AmbrosioBrueSemola2021,PanaretosZemel2019,BoucheronLugosiMassart2013}). The RAMPAGE framework~\cite{hashemi2026rampage} extends this viewpoint to variational inequalities through geometric path integration that removes first-order bias. Separately, exploiting structural properties for estimation, such as low rank, data smoothness, and interpolation, leads to improved rates in canonical settings~\cite{lee2026oracle,jadbabaie2024gradient,cosson2023low,jadbabaie2023adaptive,cosson2023gradient}. These directions are complementary to our estimator design and suggest promising ways to enrich the correction term \(\mathcal{T}_t\) in future work.

\paragraph{Decentralized and federated optimization.}
Distributed and federated optimization must balance local computation, communication, and statistical heterogeneity~\cite{mcmahan2017communication,rieke2020future,peng2024depth,wang2024adaptive,banabilah2022federated,wen2023survey,chatterjee2023federated,jadbabaie2023federated}. Gradient-tracking approaches such as Global Update Tracking (GUT) mitigate local drift without additional communication~\cite{aketi2023global}, while switching-based and partially participating schemes address failures caused by heterogeneity~\cite{islamov2025safe,upadhyay2026fedsgm,chellapandi2024fednmut,guo2023disa,huang2023distributed,ran2023distributed,huang2025distributed,wang2023task,luo2026first}.  Incorporating switching-based mechanics with partial client
participation mitigates systemic failures induced by statistical heterogeneity, appropriately shifting
the algorithmic focus from maximum likelihood estimation~\cite{luo2024unveiling,luo2025structural,luo2026characterizing} to worst-case distributionally robust optimization. 
Extending this robust methodology to fully adversarial environments, online learning in MDPs has likewise benefited from optimistic estimators with sublinear regret~\cite{moon2024optimistic}. Since client drift compounds the temporal iterate movement that drives estimation error in our framework, extending the correction-term viewpoint to communication-constrained and decentralized settings may yield sharper communication-complexity guarantees.

\noindent\textbf{Future directions.}
Several directions remain open. Extending the unified high-probability theory to relaxed noise models such as BG-0~\cite{alacaoglu2025towards,fazla2026lower}, and combining it with adaptive or hyperparameter-free step-size rules~\cite{shestakov2025unified,duchi2011adaptive,kingma2014adam,defazio2023learning,loshchilov2017decoupled,mishchenko2023prodigy,carmon2022making}, are especially natural next steps. Another promising direction is to transport the estimator-level perspective developed here to decentralized and federated settings, where client drift and communication delays create additional sources of tracking error.

\end{document}